\documentclass[11pt]{article}

\usepackage{fullpage}
\usepackage{natbib}
\usepackage{algorithm}
\usepackage{algorithmic}

\usepackage{comment} 
\usepackage{lipsum} 
\usepackage{amsmath}
\usepackage{mathtools}
\usepackage[utf8]{inputenc}
\usepackage[T2A]{fontenc}
\usepackage{amssymb}  
\usepackage{algorithm}
\usepackage{algorithmic}
\usepackage{graphicx}
\usepackage[dvipsnames]{xcolor}
\usepackage{nicefrac}
\usepackage{bm}

\usepackage{threeparttable}
\usepackage{makecell}
\usepackage{multirow}
\usepackage{colortbl}
\definecolor{bgcolor}{rgb}{0.8,1,1}
\definecolor{bgcolor2}{rgb}{0.8,1,0.8}

\usepackage{algorithm}
\usepackage{algorithmic}

\usepackage{threeparttable}
\usepackage{makecell}

\usepackage{multirow}
\usepackage{colortbl}
\definecolor{bgcolor}{rgb}{0.8,1,1}
\definecolor{bgcolor2}{rgb}{0.8,1,0.8}
\usepackage{bm}
\usepackage{graphicx} 
\graphicspath{{../}}

\usepackage{xspace}
\newcommand{\algname}[1]{{\sf\green#1}\xspace}
\newcommand{\dataname}[1]{{\tt#1}\xspace}

\newcommand{\eqdef}{\; { := }\;}

\newcommand{\R}{\mathbb{R}}

\newcommand{\E}{\mathbb{E}}
\newcommand{\ExpBr}[1]{\mathbb{E}\left[#1\right]}

\newcommand{\norm}[1]{\left\|#1\right\|}
\def\<#1,#2>{\langle #1,#2\rangle}

\newcommand{\HS}{L_{*}}      
\newcommand{\HF}{L_{\rm F}}
\newcommand{\HM}{L_{\infty}}

\def\<{\left\langle}
\def\>{\right\rangle}
\def\ll{\left[}
\def\rr{\right]}
\def\({\left(}
\def\){\right)}

\usepackage{tcolorbox}
\usepackage{pifont}
\definecolor{mydarkgreen}{RGB}{39,130,67}
\definecolor{mydarkred}{RGB}{192,47,25}
\newcommand{\green}{\color{mydarkgreen}}

\newcommand{\cO}{\mathcal{O}}
\newcommand{\mA}{\mathbf{A}}

\newcommand{\mH}{\mathbf{H}}

\newcommand{\mQ}{\mathbf{Q}}
\newcommand{\mI}{\mathbf{I}}

\newcommand{\mM}{\mathbf{M}}
\newcommand{\mY}{\mathbf{Y}}

\newcommand{\mX}{\mathbf{X}}

\newcommand{\cC}{{\mathcal{C}}}

\newcommand{\cH}{{\mathcal{H}}}

\usepackage{amsmath,amsfonts,amssymb,amsthm,array}

\usepackage{mdframed} 
\usepackage{thmtools}
\usepackage{textcomp}

\declaretheorem[within=section]{definition}
\declaretheorem[sibling=definition]{theorem}

\declaretheorem[sibling=definition]{assumption}

\declaretheorem[sibling=definition]{lemma}
\declaretheorem[sibling=definition]{example}

\usepackage[colorinlistoftodos,bordercolor=orange,backgroundcolor=orange!20,linecolor=orange,textsize=scriptsize]{todonotes}

\usepackage{microtype}
\usepackage{subfigure}
\usepackage{booktabs} 
\usepackage{grffile}
\usepackage{hyperref}

\newcommand{\squeeze}{}

\usepackage{graphicx} 
\graphicspath{{}}

\usepackage[font=normalsize, labelfont=bf]{caption}

\usepackage{tcolorbox}
\usepackage{pifont}
\definecolor{mydarkgreen}{RGB}{39,130,67}
\definecolor{mydarkred}{RGB}{192,47,25}

\usepackage{amsmath,amsfonts,amssymb,amsthm,array}

\usepackage{mdframed} 
\usepackage{thmtools}
\usepackage{textcomp}

\usepackage[colorinlistoftodos,bordercolor=orange,backgroundcolor=orange!20,linecolor=orange,textsize=scriptsize]{todonotes}

\usepackage{microtype}
\usepackage{subfigure}
\usepackage{booktabs} 

\usepackage{grffile}

\usepackage{hyperref}

\usepackage[utf8]{inputenc} 
\usepackage[T1]{fontenc}    
\usepackage{hyperref}       
\usepackage{url}            
\usepackage{booktabs}       
\usepackage{amsfonts}       
\usepackage{nicefrac}       
\usepackage{microtype}      
\usepackage{xcolor}         

\title{Distributed Newton-Type Methods with \\ Communication Compression and Bernoulli Aggregation}

\author{Rustem Islamov\thanks{Institut Polytechnique de Paris (IP Paris), Palaiseau, France.}\and Xun Qian\thanks{JD Explore Academy, Beijing, China.} \and Slavom\'{i}r Hanzely\thanks{King Abdullah University of Science and Technology, Thuwal, Saudi Arabia.} \and Mher Safaryan\footnotemark[3]   \and Peter Richt\'{a}rik\footnotemark[3]}
\date{\today}

\begin{document}

	\maketitle

	\begin{abstract}
		Despite their high computation and communication costs, Newton-type methods remain an appealing option for distributed training due to their robustness against ill-conditioned convex problems. In this work, we study {\em communication compression} and {\em aggregation mechanisms} for curvature information in order to reduce these costs while preserving theoretically superior local convergence guarantees. We prove that the recently developed class of {\em three point compressors (3PC)} of \citet{richtarik3PC}  for gradient communication can be generalized to Hessian communication as well. This result opens up a wide variety of communication strategies, such as {\em contractive compression} and {\em lazy aggregation}, available to our disposal to compress prohibitively costly curvature information. Moreover, we discovered several new 3PC mechanisms, such as {\em adaptive thresholding} and {\em Bernoulli aggregation}, which require reduced communication and occasional Hessian computations. Furthermore, we extend and analyze our approach to bidirectional communication compression and partial device participation setups to cater to the practical considerations of applications in federated learning. For all our methods, we derive fast {\em condition-number-independent} local linear and/or superlinear convergence rates. Finally, with extensive numerical evaluations on convex optimization problems, we illustrate that our designed schemes achieve state-of-the-art communication complexity compared to several key baselines using second-order information.
	\end{abstract}

	\clearpage
	
	\tableofcontents
	
	\clearpage
	
	\section{Introduction}

	In this work we consider the distributed optimization problem given by the form of ERM:
	\begin{equation}\label{dist-opt-prob}
		\squeeze
		\min\limits_{x\in\R^d} \left\{ f(x) \eqdef \frac{1}{n}\sum\limits_{i=1}^n f_i(x) \right\},
	\end{equation}
	where $d$ is the (potentially large) number of parameters of the model $x\in\R^d$ we aim to train, $n$ is the (potentially large) total number of devices in the distributed system, $f_i(x)$ is the loss/risk associated with the data stored on machine $i\in[n] \eqdef \{1, 2, \dots, n\}$, and $f(x)$ is the empirical loss/risk.
	
	In order to jointly train a single machine learning model using all devices' local data, {\em collective efforts} are necessary from all compute nodes. Informally, each entity should invest some ``knowledge'' from its local ``wisdom'' to create the global ``wisdom''. The classical approach in distributed training to implement the collective efforts was to literally collect all the raw data devices acquired and then perform the training in one place with traditional methods. However, the mere access to the raw data hinders the clients' {\em data privacy} in federated learning applications \citep{FEDLEARN,FEDOPT,FL2017-AISTATS}. Besides, even if we ignore the privacy aspect, accumulating all devices' data into a single machine is often {\em infeasible} due to its increasingly large size \citep{bekkerman2011scaling}.
	
	Because of these considerations, there has been a serious stream of works studying distributed training with decentralized data. This paradigm of training brings its own advantages and limitations. Perhaps the major advantage is that each remote device's data can be processed simultaneously using {\em local computational resources}. Thus, from another perspective, we are scaling up the traditional single-device training to a distributed training of multiple parallel devices with decentralized data and local computation. However, the cost of scaling the training over multiple devices forces {\em intensive communication} between nodes, which is the {\em key bottleneck} in distributed systems.

	\subsection{Related work: from first-order to second-order distributed optimization}
	Currently, first-order optimization methods are the default options for large-scale distributed training due to their cheap per-iteration costs. Tremendous amount of work has been devoted to extend and analyze gradient-type algorithms to conform to various practical constraints such as 
	efficient communication through compression mechanisms \citep{qsgd,alistarh2018convergence,terngrad,tonko,Sahu2021threshold,Tyurin2022Dasha} and local methods \citep{Gorbunov2020localSGD,Stich-localSGD,SCAFFOLD,Nadiradze2021ADQL-SGD,Mishchenko2022ProxSkip},
	peer-to-peer communication through graphs \citep{Koloskova2019Decentralized,Koloskova2020Decentralized,Kovalev2021Decentralized},
	asynchronous communication \citep{Feyzmahdavian2021Asynchronous,Nadiradze2021Asynchronous},
	partial device participation \citep{Yang2021PP+nonIID},
	Byzantine or adversarial attacks \citep{Karimireddy2021Byzantine,Karimireddy2022Byzantine},
	faster convergence through acceleration \citep{allen2017katyusha,ADIANA,Qian2021ECLK} and variance reduction techniques \citep{Lee2017DSVRG,DIANA,DIANA-VR,Cen2020DSVRG,MARINA},
	data privacy and heterogeneity over the nodes \citep{FL-big,FL_survey_2020},

	Nevertheless, despite their wide applicability, all first-order methods (including accelerated ones) inevitably suffer from ill-conditioning of the problem. In the past few years, several algorithmic ideas and mechanisms to tackle the above-mentioned constraints have been adapted for second-order optimization. The goal in this direction is to enhance the convergence by increasing the resistance of gradient-type methods against ill-conditioning using the knowledge of curvature information. The basic motivation that the Hessian computation will be useful in optimization is the fast {\em condition-number-independent} (local) convergence rate of classic Newton's method \citep{Beck-book-nonlinear}, that is beyond the reach of {\em all} first-order methods.
	
	Because of the quadratic dependence of Hessian information ($d^2$ floats per each Hessian matrix) from the dimensionality of the problem, the primary challenge of taming second-order methods was efficient communication between the participating devices. To alleviate prohibitively costly Hessian communication, many works such as DiSCO \citep{DiSCO2015,Zhuang2015,Lin2014LargescaleLR,Newton-MR2019}, GIANT \citep{GIANT2018,DANE,Reddi:2016aide} and DINGO \citep{DINGO,CompressesDINGO2020} impart second-order information by condensing it into Hessian-vector products. Inspired from compressed first-order methods, an orthogonal line of work, including DAN-LA \citep{DAN-LA2020}, Quantized Newton \citep{Alimisis2021QNewton}, NewtonLearn \citep{Islamov2021NewtonLearn}, FedNL \citep{FedNL2021}, Basis Learn \citep{qian2021basis} and IOS \citep{IOSFabbro2022}, applies lossy compression strategies directly to Hessian matrices reducing the number of encoding bits. Other techniques that have been migrated from first-order optimization literature are local methods \citep{LocalNewton2021}, partial device participation \citep{FedNL2021,qian2021basis}, defenses against Byzantine attacks \citep{Ghosh2020ByzantineNewton,Ghosh2021ByzantineNewton}.

	\section{Motivation and Contributions}
	
	Handling and taking advantage of the second-order information in distributed setup is rather challenging. As opposed to gradient-type methods, Hessian matrices are both harder to compute and much more expensive to communicate. To avoid directly accessing costly Hessian matrices, methods like DiSCO \citep{DiSCO2015}, GIANT \citep{GIANT2018} and DINGO \citep{DINGO} exploit Hessian-vector products only, which are as cheap to compute as gradients \citep{HessianXvector1994}. However, these methods typically suffer from data heterogeneity, need strong assumptions on problem structure (e.g., generalized linear models) and/or do not provide fast local convergence rates.
	
	On the other hand, recent works \citep{FedNL2021,qian2021basis} have shown that, with the access of Hessian matrices, fast local rates can be guaranteed for solving general finite sums \eqref{dist-opt-prob} under compressed communication and arbitrary heterogeneous data. In view of these advantages, in this work we adhere to this approach and study communication mechanisms that can further lighten communication and reduce computation costs. Below, we summarize our key contributions.
	
	\subsection{Flexible communication strategies for Newton-type methods}
	We prove that the recently developed class of {\em three point compressors (3PC)} of \citet{richtarik3PC} for gradient communication can be generalized to Hessian communication as well. In particular, we propose a new method, which we call \algname{Newton-3PC} (Algorithm \ref{alg:N3PC}), extending \algname{FedNL} \citep{FedNL2021} algorithm for arbitrary 3PC mechanism. This result opens up a wide variety of communication strategies, such as {\em contractive compression} \citep{StichNIPS2018-memory,Alistarh-SparsGradMethods2018,Karimireddy2019EFsignSGD} and {\em lazy aggregation} \citep{Chen2018LAG,Sun2019LAG,Ghadikolaei2021LENA}, available to our disposal to compress prohibitively costly curvature information. Besides, \algname{Newton-3PC} (and its local convergence theory) recovers \algname{FedNL}\citep{FedNL2021} (when {\em contractive compressors} \eqref{def:CC} are used as 3PC) and \algname{BL} \citep{qian2021basis} (when {\em rotation compression} \eqref{def:rotation_comp} is used as 3PC) in special cases.

	\subsection{New compression and aggregation schemes}
	Moreover, we discovered several new 3PC mechanisms, which require reduced communication and occasional Hessian computations. In particular, to reduce communication costs, we design an {\em adaptive thresholding} (Example \ref{ex:AT}) that can be seamlessly combined with an already adaptive {\em lazy aggregation} (Example \ref{ex:CLAG}). In order to reduce computation costs, we propose {\em Bernoulli aggregation} (Example \ref{ex:CBAG}) mechanism which allows local workers to {\em skip} both computation and communication of local information (e.g., Hessian and gradient) with some predefined probability.

	\subsection{Extensions}
	Furthermore, we provide several extensions to our approach to cater to the practical considerations of applications in federated learning. In the main part of the paper, we consider only bidirectional communication compression (\algname{Newton-3PC-BC}) setup, where we additionally apply Bernoulli aggregation for gradients (worker to server direction) and another 3PC mechanism for the global model (server to worker direction). The extension for partial device participation (\algname{Newton-3PC-BC-PP}) setup and the discussion for globalization are deferred to the Appendix.

	\subsection{Fast local linear/superlinear rates}
	All our methods are analyzed under the assumption that the global objective is strongly convex and local Hessians are Lipschitz continuous. In this setting, we derive fast {\em condition-number-independent} local linear and/or superlinear convergence rates.

	\subsection{Extensive experiments and Numerical Study}
	Finally, with extensive numerical evaluations on convex optimization problems, we illustrate that our designed schemes achieve state-of-the-art communication complexity compared to several key baselines using second-order information.

	\section{Three Point Compressors for Matrices}\label{sec:3PC4M}
	
	To properly incorporate second-order information in distributed training, we need to design an efficient strategy to synchronize locally evaluated $d\times d$ Hessian matrices. Simply transferring $d^2$ entries of the matrix each time it gets computed would put significant burden on communication links of the system. Recently, \citet{richtarik3PC} proposed a new class of gradient communication mechanisms under the name {\em three point compressors (3PC)}, which unifies contractive compression and lazy aggregation mechanisms into one class. Here we extend the definition of 3PC for matrices under the Frobenius norm $\|\cdot\|_{\rm F}$ and later apply to matrices involving Hessians.

	\begin{definition}[3PC for Matrices]
		We say that a (possibly randomized) map
		\begin{equation}\label{def:3PC}
			\textstyle
			\cC_{\mH, \mY}(\mX): \underbrace{\,\R^{d\times d}\,}_{\mH\in} \times \underbrace{\,\R^{d\times d}\,}_{\mY\in} \times \underbrace{\,\R^{d\times d}\,}_{\mX\in} \to \R^{d\times d}
		\end{equation}
		is a three point compressor (3PC) if there exist constants $0 < A\leq 1$ and $B \geq 0$ such that
		\begin{equation}\label{def:3PC_comp}
			\ExpBr{\norm{\cC_{\mH,\mY}(\mX) - \mX}^2_{\rm F}} \leq (1-A)\norm{\mH-\mY}^2_{\rm F} + B\norm{\mX-\mY}^2_{\rm F}.
		\end{equation}
		holds for all matrices $\mH, \mY, \mX \in \R^{d\times d}$.
	\end{definition}
	
	The matrices $\mY$ and $\mH$ can be treated as parameters defining the compressor that would be chosen adaptively. Once they fixed, $\cC_{\mH,\mY}:\R^{d\times d} \to \R^{d\times d}$ is a map to compress a given matrix $\mX$. Let us discuss special cases with some examples.

	\begin{example}[{\bf Contractive compressors} \citep{Karimireddy2019EFsignSGD}] The (possibly randomized) map $\cC\colon\R^d\to\R^d$ is called contractive compressor with contraction parameter $\alpha\in(0,1]$, if the following holds for any matrix $\mX\in\R^{d\times d}$
		\begin{equation}\label{def:CC}
			\ExpBr{\|\cC(\mX) - \mX\|^2_{\rm F}} \le (1-\alpha)\|\mX\|^2_{\rm F}.
		\end{equation}
	\end{example}
	
	Notice that \eqref{def:CC} is a special case of \eqref{def:3PC} when $\mH=\bm{0},\, \mY=\mX$ and $A=\alpha,\, B=0$. Therefore, contractive compressors are already included in the 3PC class. Contractive compressors cover various well known compression schemes such as greedy sparsification, low-rank approximation and (with a suitable scaling factor) arbitrary unbiased compression operator \citep{biased2020}. There have been several recent works utilizing these compressors for compressing Hessian matrices \citep{DAN-LA2020,Alimisis2021QNewton,Islamov2021NewtonLearn,FedNL2021,qian2021basis,IOSFabbro2022}. Below, we introduce yet another contractive compressor based on thresholding idea which shows promising performance in our experiments.

	\begin{example}[{\bf Adaptive Thresholding [NEW]}]\label{ex:AT}
		Following \citet{Sahu2021threshold}, we design an {\em adaptive thresholding} operator with parameter $\lambda \in (0,1]$ defined as follows
		\begin{equation}\label{def:AT}
			\left[\cC(\mX)\right]_{jl} \eqdef
			\begin{cases}
				\mX_{jl} & \text{if } |\mX_{jl}| \geq \lambda\|\mX\|_\infty, \\
				0 & \text{otherwise},
			\end{cases}
			\quad \text{for all } j,l\in[d] \text{ and } \mX\in\R^{d\times d}.
		\end{equation}
	\end{example}
	In contrast to {\em hard thresholding} operator of \citet{Sahu2021threshold}, \eqref{def:AT} uses adaptive threshold $\lambda\|\mX\|_\infty$ instead of fixed threshold $\lambda$. With this choice, we ensures that at least the Top-$1$ is transferred. In terms of computation, thresholding approach is more efficient than Top-$K$ as only single pass over the values is already enough instead of partial sorting.
	
	\begin{lemma}\label{lem:3PC__AT}
		The adaptive thresholding \eqref{def:AT} is a contractive compressor with $\alpha = \max(1-(d\lambda)^2,\nicefrac{1}{d^2})$.
	\end{lemma}
	
	The next two examples are 3PC schemes which in addition to contractive compressors utilize aggregation mechanisms, which is an orthogonal approach to contractive compressors.
	
	\begin{example}[{\bf Compressed Lazy AGgregation (CLAG)} \citep{richtarik3PC}]\label{ex:CLAG}
		Let $\cC\colon\R^d\to\R^d$ be a contractive compressor with contraction parameter $\alpha\in(0,1]$ and $\zeta\ge0$ be a trigger for the aggregation. Then CLAG mechanism is defined as
		\begin{equation}\label{def:3PC__CLAG}
			\cC_{\mH,\mY}(\mX) = 
			\begin{cases}
				\mH + \cC(\mX-\mH) & \text{if } \|\mX-\mH\|^2_{\rm F} > \zeta\|\mX-\mY\|^2_{\rm F}, \\
				\mH & \text{otherwise}.
			\end{cases}
		\end{equation}
	\end{example}
	In the special case of identity compressor $\cC=\textrm{Id}$ (i.e., $\alpha=1$), CLAG reduces to lazy aggregation \citep{Chen2018LAG}. On the other extreme, if the trigger $\zeta=0$ is trivial, CLAG recovers recent variant of error feedback for contractive compressors, namely EF21 mechanism \citep{EF21}.
	
	\begin{lemma}[see Lemma 4.3 in \citep{richtarik3PC}]
		CLAG mechanism \eqref{def:3PC__CLAG} is a 3PC compressor with $A = 1-(1-\alpha)(1+s)$ and $B = \max\{(1-\alpha)(1+\nicefrac{1}{s}),\zeta\}$, for any $s\in(0,\nicefrac{\alpha}{(1-\alpha)})$.
	\end{lemma}
	
	From the first glance, the structure of CLAG in \eqref{def:3PC__CLAG} may not seem communication efficient as the the matrix $\mH$ (appearing in both cases) can potentially by dense. However, as we will see in the next section, $\cC_{\mH,\mY}$ is used to compress $\mX$ when there is no need to communicate $\mH$. Thus, with CLAG we either send compressed matrix $\cC(\mX-\mH)$ if the condition with trigger $\zeta$ activates or nothing.
	
	\begin{example}[{\bf Compressed Bernoulli AGgregation (CBAG) [NEW]}]\label{ex:CBAG}
		Let $\cC\colon\R^d\to\R^d$ be a contractive compressor with contraction parameter $\alpha\in(0,1]$ and $p\in(0,1]$ be the probability for the aggregation. We then define CBAG mechanism is defined as
		\begin{equation}\label{def:3PC__CBAG}
			\cC_{\mH,\mY}(\mX) =
			\begin{cases}
				\mH + \cC(\mX-\mH) & \text{with probability } p,\\
				\mH & \text{with probability } 1-p.
			\end{cases}
		\end{equation}
	\end{example}
	
	The advantage of CBAG \eqref{def:3PC__CBAG} over CLAG is that there is no condition to evaluate and check. This choice of probabilistic switching reduces computation costs as with probability $1-p$  it is useless to compute $\mX$. Note that CBAG has two independent sources of randomness: Bernoulli aggregation and possibly random operator $\cC$.
	
	\begin{lemma}\label{lem:3PC__CBAG}
		CBAG mechanism \eqref{def:3PC__CBAG} is a 3PC compressor with $A = (1-p\alpha)(1+s)$ and $B = (1-p\alpha)(1+\nicefrac{1}{s})$, for any $s\in(0,\nicefrac{p\alpha}{(1-p\alpha)})$.
	\end{lemma}
	
	For more examples of 3PC compressors see section C of \citep{richtarik3PC} and the Appendix.

	\section{\algname{Newton-3PC}: Newton's Method with 3PC Mechanism}\label{sec:N3PC}
	
	In this section we present our first Newton-type method, called \algname{Newton-3PC}, employing communication compression through 3PC compressors discussed in the previous section. The proposed method is an extension of \algname{FedNL} \citep{FedNL2021} from contractive compressors to arbitrary 3PC compressors. From this perspective, our \algname{Newton-3PC} (see Algorithm \ref{alg:N3PC}) is much more flexible, offering a wide variety of communication strategies beyond contractive compressors.

	\subsection{General technique for learning the Hessian}
	The central notion in \algname{FedNL} is the technique for learning {\em a priori unknown} Hessian $\nabla^2 f(x^*)$ at the (unique) solution $x^*$ in a communication efficient manner. This is achieved by maintaining and iteratively updating local Hessian estimates $\mH_i^k$ of $\nabla^2 f_i(x^*)$ for all devices $i\in[n]$ and the global Hessian estimate $\mH^k = \frac{1}{n}\sum_{i=1}^n \mH_i^k$ of $\nabla^2 f(x^*)$ for the central server.
	We adopt the same idea of Hessian learning and aim to update local estimates in such a way that $\mH_i^k\to\nabla^2 f_i(x^*)$ for all $i\in[n]$, and as a consequence, $\mH^k\to\nabla^2 f(x^*)$, throughout the training process. However, in contrast to \algname{FedNL}, we update local Hessian estimates via generic 3PC mechanism, namely $$\textstyle\mH_i^{k+1} = \cC_{\mH_i^k,\nabla^2f_i(x^k)}\left(\nabla^2f_i(x^{k+1})\right),$$
	which is a particular instantiation of 3PC compressor $\cC_{\mH,\mY}(\mX)$ using previous local Hessian $\mY = \nabla^2 f_i(x^k)$ and previous estimate $\mH = \mH_i^k$ to compress current local Hessian $\mX = \nabla^2 f_i(x^{k+1})$.
	
	\begin{algorithm}[H]
		\caption{\algname{Newton-{\color{blue} 3PC}} (Newton's method with {\color{blue}three point compressor}) }
		\label{alg:N3PC}
		\begin{algorithmic}[1]
			\STATE \textbf{Input:} $x^0\in\R^d,\, \mH_1^0, \dots, \mH_n^0 \in \R^{d\times d},\, \mH^0 \eqdef \frac{1}{n}\sum_{i=1}^n \mH_i^0,\, l^0 = \frac{1}{n}\sum_{i=1}^n \|\mH_i^0 - \nabla^2 f_i(x^0)\|_{\rm F}$.
			\STATE {\bf on} server 
			\STATE \quad \textit{Option 1:} $x^{k+1} = x^k - [\mH^{k}]_{\mu}^{-1} \nabla f(x^k)$
			\STATE \quad \textit{Option 2:} $x^{k+1} = x^k - [\mH^{k} + l^k\mI]^{-1} \nabla f(x^k)$
			\STATE \quad Broadcast $x^{k+1}$ to all nodes
			\FOR{each device $i = 1, \dots, n$ in parallel} 
			\STATE Get $x^{k+1}$ and compute local gradient $\nabla f_i(x^{k+1})$ and local Hessian $\nabla^2 f_i(x^{k+1})$
			\STATE Apply {\color{blue}3PC} and update local Hessian estimator to $\mH_i^{k+1} = {\color{blue}\cC_{\mH_i^k,\nabla^2f_i(x^k)}\left(\nabla^2f_i(x^{k+1})\right)}$
			\STATE Send $\nabla f_i(x^{k+1})$,\; $\mH^{k+1}_i$ and $l_i^{k+1} \eqdef \|\mH_i^{k+1} - \nabla^2 f_i(x^{k+1})\|_{\rm F}$ to the server
			\ENDFOR
			\STATE \textbf{on} server
			\STATE \quad Aggregate $ \nabla f(x^{k+1}) = \frac{1}{n}\sum_{i=1}^n \nabla f_i(x^{k+1}), \mH^{k+1} = \frac{1}{n}\sum_{i=1}^n\mH_i^{k+1}, l^{k+1} = \frac{1}{n}\sum_{i=1}^n l_i^{k+1}$
		\end{algorithmic}
	\end{algorithm}

	In the special case, when EF21 scheme $\cC_{\mH_i^k,\nabla^2f_i(x^k)}\left(\nabla^2f_i(x^{k+1})\right) = \mH_i^k + \cC(\nabla^2f_i(x^{k+1}) - \mH_i^k)$ is employed as a 3PC mechanism, we recover the Hessian learning technique of \algname{FedNL}. Our \algname{Newton-3PC} method also recovers recently proposed {\em Basis Learn} (\algname{BL}) \citep{qian2021basis} algorithm if we specialize the 3PC mechanism to {\em rotation compression} (see Appendix \ref{apx:rot-comp}).
	
	\subsection{Flexible Hessian communication and computation schemes}
	The key novelty \algname{Newton-3PC} brings is the flexibility of options to handle costly local Hessian matrices both in terms of computation and communication.
	
	Due to the adaptive nature of CLAG mechanism \eqref{def:3PC__CLAG}, \algname{Newton-CLAG} method {\em does not send any information} about the local Hessian $\nabla^2 f_i(x^{k+1})$ if it is sufficiently close to previous Hessian estimate $\mH_i^k$, namely $$\|\nabla^2 f_i(x^{k+1}) - \mH_i^k\|^2_{\rm F} \le \zeta\|\nabla^2 f_i(x^{k+1}) - \nabla^2 f_i(x^{k})\|^2_{\rm F}$$ with some positive trigger $\zeta>0$. In other words, the server {\em reuses} local Hessian estimate $\mH_i^k$ while there is no essential discrepancy between locally computed Hessian $\nabla^2 f_i(x^{k+1})$. Once a sufficient change is detected by the device, only the compressed difference $\cC(\nabla^2 f_i(x^{k+1}) - \mH_i^k)$ is communicated since the server knows $\mH_i^k$. By adjusting the trigger $\zeta$, we can control the frequency of Hessian communication in an adaptive manner. Together with adaptive thresholding operator \eqref{def:AT} as a contractive compressor, CLAG is a doubly adaptive communication strategy that makes \algname{Newton-CLAG} highly efficient in terms of communication complexity.
	
	Interestingly enough, we can design such 3PC compressors that can reduce computational costs too. To achieve this, we consider CBAG mechanism \eqref{def:3PC__CBAG} which replaces the adaptive switching condition of CLAG by probabilistic switching according to Bernoulli random variable. Due to the probabilistic nature of CBAG mechanism, \algname{Newton-CBAG} method requires devices to compute local Hessian $\nabla^2 f_i(x^{k+1})$ and communicate compressed difference $\cC(\nabla^2 f_i(x^{k+1}) - \mH_i^k)$ {\em only} with probability $p\in(0,1]$. Otherwise, the whole Hessian computation and communication is {\em skipped}.

	\subsection{Options for updating the global model}
	We adopt the same two update rules for the global model as was design in \algname{FedNL}. If the server knows the strong convexity parameter $\mu>0$ (see Assumption \ref{asm:main}), then the global Hessian estimate $\mH^k$ is projected onto the set $\left\{\mM\in\R^{d\times d} \colon \mM^\top = \mM,\; \mu\mI \preceq \mM \right\}$ to get the projected estimate $[\mH^k]_{\mu}$. 
	Alternatively, all devices additionally compute and send compression errors $l_i^k \eqdef \|\mH_i^k - \nabla^2 f_i(x^k)\|_{\rm F}$ (extra float from each device in terms of communication complexity) to the server, which then formulates the regularized estimate $\mH^k + l^k\mI$ by adding the average error $l^k = \frac{1}{n}\sum_{i=1}^n l_i^k$ to the global Hessian estimate $\mH^k$. 
	
	\subsection{Local convergence theory}
	To derive theoretical guarantees, we consider the standard assumption that the global objective is strongly convex and local Hessians are Lipschitz continuous.
	
	\begin{assumption}\label{asm:main}
		The average loss $f$ is $\mu$-strongly convex, and all local losses $f_i(x)$ have Lipschitz continuous Hessians. Let $\HS$, $\HF$ and $\HM$ be the Lipschitz constants with respect to three different matrix norms: spectral, Frobenius and infinity norms, respectively. Formally,  we require 
		\begin{eqnarray*}
		\|\nabla^2 f_i(x) - \nabla^2 f_i(y)\| &\leq  \HS \|x-y\|, \\
		\|\nabla^2 f_i(x) - \nabla^2 f_i(y)\|_{\rm F}  &\leq  \HF \|x-y\|, \\
		\max_{j,l}| (\nabla^2 f_i(x) - \nabla^2 f_i(y))_{jl}|  &\leq  \HM \|x-y\|
		\end{eqnarray*}
		to hold for all $i\in[n]$ and $x,y\in\R^d$.
	\end{assumption}
	
	Define constants $C$ and $D$ depending on which option is used for global model update, namely $C = 2, D = \HS^2$ if {\em Option 1} is used, and $C = 8, D = (\HS+2\HF)^2$ if {\em Option 2} is used. We prove three local rates for \algname{Newton-3PC}: for the squared distance to the solution $\|x^k-x^*\|^2$, and for the Lyapunov function
	\begin{equation*}
		\squeeze
		\Phi^k \eqdef {\cal H}^k + 6\(\frac{1}{A} + 3AB\)\HF^2 \|x^k-x^*\|^2, \quad\text{where}\quad {\cal H}^k \eqdef \frac{1}{n} \sum\limits_{i=1}^n \|\mH_i^k - \nabla^2 f_i(x^*)\|^2_{\rm F}.
	\end{equation*}
	
	We present our theoretical results for local convergence with two stages. For the first stage, we derive convergence rates using specific {\em locality conditions} for model/Hessian estimation error. In the second stage, we prove that these locality conditions are satisfied for different situations.
	
	\begin{theorem}\label{th:NLU}
		Let Assumption \ref{asm:main} hold. Assume $\|x^0-x^*\| \leq \frac{\mu}{\sqrt{2D}}$ and ${\cal H}^k \leq \frac{\mu^2}{4C}$ for all $k\geq 0$. Then, \algname{Newton-3PC} (Algorithm~\ref{alg:N3PC}) with any 3PC mechanism converges with the following rates:
		\begin{equation}\label{rate:local-linear-iter}
			\squeeze
			\|x^k - x^*\|^2 \leq  \frac{1}{2^k}   \|x^0-x^*\|^2, \qquad \E\ll\Phi^k\rr \leq  \left(  1 - \min\left\{  \frac{A}{2}, \frac{1}{3}  \right\}  \right)^k\Phi^0,
		\end{equation}
		\begin{equation}\label{rate:local-superlinear-iter}
			\squeeze
			\E\ll\frac{\|x^{k+1}-x^*\|^2}{\|x^k-x^*\|^2}\rr \leq  \left(  1 - \min\left\{  \frac{A}{2}, \frac{1}{3}  \right\}  \right)^k \left(  C + \frac{AD}{12(1+3AB)\HF^2}  \right) \frac{\Phi^0}{\mu^2}. 
		\end{equation}    
	\end{theorem}
	
	Clearly, these rates are independent of the condition number of the problem, and the choice of 3PC can control the parameter $A$. Notice that locality conditions here are upper bounds on the initial model error $\|x^0-x^*\|$ and the errors $\cH^k$ for all $k\ge0$. It turns out that the latter condition may not be guaranteed in general since it depends on the structure of the 3PC mechanism. Below, we show these locality conditions under some assumptions on 3PC, covering practically all compelling cases.
	
	\begin{lemma}[Deterministic 3PC]\label{lm:boundforbiased}
		Let the 3PC compressor in \algname{Newton-3PC} be deterministic. Assume the following initial conditions hold: $$\|x^0 - x^*\| \le e_1 \eqdef \min\left\{ \frac{A\mu}{\sqrt{8(1+3AB)}\HF}, \frac{\mu}{\sqrt{2D}}  \right\} \quad\text{and}\quad \|\mH_i^0 - \nabla^2 f_i(x^*)\|_{\rm F} \leq \frac{\mu}{2\sqrt{C}}.$$
		Then $\|x^k-x^*\| \leq e_1$ and $\|\mH_i^k - \nabla^2 f_i(x^*)\|_{\rm F}  \leq \frac{\mu}{2\sqrt{C}}$ for all $k\geq 0$. 
	\end{lemma}
	
	\begin{lemma}[CBAG]\label{lm:boundforcbag}
		Consider CBAG mechanism with only source of randomness from Bernoulli aggregation. Assume $$\|x^0 - x^*\| \le e_2 \eqdef \min\left\{  \frac{(1-\sqrt{1-\alpha})\mu}{4\sqrt{C}\HF}, \frac{\mu}{\sqrt{2D}}  \right\} \quad\text{and}\quad \|\mH_i^0 - \nabla^2 f_i(x^*)\|_{\rm F} \leq \frac{\mu}{2\sqrt{C}}.$$
		Then $\|x^k-x^*\| \leq e_2$ and $\|\mH_i^k - \nabla^2 f_i(x^*)\|_{\rm F}  \leq \frac{\mu}{2\sqrt{C}}$ for all $k\geq 0$. 
	\end{lemma}
	
	\section{Extension to Bidirectional Compression (\algname{Newton-3PC-BC})}\label{sec:N3PC-BC}
	
	In this section, we consider the setup where both directions of communication between devices and the central server are bottleneck. For this setup, we propose \algname{Newton-3PC-BC} (Algorithm \ref{alg:N3PC-BC}) which additionally applies Bernoulli aggregation for gradients (worker to server direction) and another 3PC mechanism for the global model (server to worker direction) employed the master.
	
	Overall, the method integrates three independent communication schemes: workers' 3PC (denoted by $\cC^W$) for local Hessian matrices $\nabla^2 f_i(z^{k+1})$, master's 3PC (denoted by $\cC^M$) for the global model $x^{k+1}$ and Bernoulli aggregation with probability $p\in(0,1]$ for local gradients $\nabla f_i(z^{k+1})$. Because of these three mechanisms, the method maintains three sequences of model parameters $\{x^k,w^k,z^k\}_{k\ge0}$. Notice that, Bernoulli aggregation for local gradients is a special case of CBAG (Example \ref{ex:CBAG}), which allows to skip the computation of local gradients with probability $(1-p)$. However, this reduction in gradient computation necessitates algorithmic modification in order to guarantee convergence. Specifically, we design gradient estimator $g^{k+1}$ to be the full gradient $\nabla f(z^{k+1})$ if devices compute local gradients (i.e., $\xi=1$). Otherwise, when gradient computation is skipped (i.e., $\xi=0$), we estimate the missing gradient using Hessian estimate $\mH^{k+1}$ and stale gradient $\nabla f(w^{k+1})$, namely we set $$g^{k+1} = [\mH^{k+1}]_{\mu}(z^{k+1}-w^{k+1}) + \nabla f(w^{k+1}).$$

	\begin{algorithm}[h]
		\caption{\algname{Newton-{\color{blue}3PC-BC}} (Newton's method with {\color{blue}3PC} and {\color{blue}Bidirectional Compression)}}
		\label{alg:N3PC-BC}
		\begin{algorithmic}[1]
			\STATE \textbf{Parameters:} {\color{blue} Workers' ($\cC^W$) and Master's ($\cC^M$) 3PC, gradient probability $p\in(0,1]$}
			\STATE \textbf{Input:} $x^0\;{\color{blue}=w^0=z^0}\in\R^d$;  $\mH_i^0 \in \R^{d\times d}$, and $\mH^0 \eqdef \tfrac{1}{n}\sum_{i=1}^n \mH_i^0$; {\color{blue}$\xi^0 = 1$}; $g^0 = \nabla f(z^0)$
			\STATE \textbf{on} server
			\STATE ~~ Update the global model to $x^{k+1} = z^k - [\mH^k]_{\mu}^{-1} g^k$
			\STATE ~~ Apply {\color{blue} Master's 3PC} and send model estimate $z^{k+1} = {\color{blue}\cC^M_{z^k, x^k} (x^{k+1})}$ to all devices $i\in[n]$
			\STATE ~~ {\color{blue}Sample $\xi^{k+1} \sim \text{Bernoulli}(p)$ and send to all devices $i\in[n]$}
			\STATE \textbf{for} each device $i = 1, \dots, n$ in parallel \textbf{do}
			\STATE ~~ Get $z^{k+1} = {\color{blue}\cC^M_{z^k, x^k} (x^{k+1})}$ and ${\color{blue}\xi^{k+1}}$ from the server
			\STATE ~~ \textbf{if} {\color{blue}$\xi^{k+1} = 1$}
			\STATE ~~~~ $ w^{k+1} = z^{k+1}$, compute local gradient $\nabla f_i(z^{k+1})$ and send to the server
			\STATE ~~ \textbf{if} {\color{blue}$\xi^{k+1} = 0$}
			\STATE ~~~~ $w^{k+1} = w^k$
			\STATE ~~ Apply {\color{blue}Worker's 3PC} and update local Hessian estimator to $\mH_i^{k+1} = {\color{blue}\cC^W_{\mH_i^k, \nabla^2 f_i(z^{k})} (\nabla^2 f_i(z^{k+1}))}$
			\STATE \textbf{end for}
			\STATE \textbf{on} server
			\STATE ~~ Aggregate $\nabla f(z^{k+1}) = \frac{1}{n}\sum_{i=1}^n \nabla f_i(z^{k+1}),\, \mH^{k+1} = \frac{1}{n} \sum_{i=1}^n \mH_i^k$
			\STATE ~~ \textbf{if} {\color{blue}$\xi^{k+1} = 1$}
			\STATE ~~~~ $w^{k+1} = z^{k+1},\; g^{k+1} = \nabla f(z^{k+1})$
			\STATE ~~ \textbf{if} {\color{blue}$\xi^{k+1} = 0$}
			\STATE ~~~~ $w^{k+1} = w^k,\; g^{k+1} = [\mH^{k+1}]_{\mu}(z^{k+1}-w^{k+1}) + \nabla f(w^{k+1})$ 
		\end{algorithmic}
	\end{algorithm}
	
	Similar to the previous result, we present convergence rates and guarantees for locality separately. Let $A_M (A_W),\, B_M (B_W)$ be parameters of the master's (workers') 3PC mechanisms. Define constants $$C_M \eqdef  \frac{4}{A_M} + 1 + \frac{5B_M}{2},\, C_W \eqdef  \frac{4}{A_W} + 1 + \frac{5B_W}{2}$$ and Lyapunov function $$\Phi_1^k \eqdef \|z^k-x^*\|^2 + C_M\|x^k-x^*\|^2 + \frac{A_M(1-p)}{4p} \|w^k-x^*\|^2.$$
	
	\begin{theorem}\label{th:3PCBL1}
		Let Assumption \ref{asm:main} holds. Assume $\|z^k-x^*\|^2 \leq \frac{A_M \mu^2}{24C_M \HS^2}$ and $\cH^k \leq \frac{A_M \mu^2}{96C_M}$ for all $k\geq 0$. Then, \algname{Newton-3PC-BC} (Algorithm~\ref{alg:N3PC-BC}) converges with the following linear rate:
		\begin{equation}\label{rate:local-linear-BC}
			\squeeze
			\mathbb{E}[\Phi_1^k] \leq \left(  1 - \min\left\{  \frac{A_{M}}{4}, \frac{3p}{8}  \right\}  \right)^k \Phi_1^0.
		\end{equation}
	\end{theorem}
	
	Note that the above linear rate for $\Phi_1^k$ does not depend on the conditioning of the problem and implies linear rates for all three sequences $\{x^k,w^k,z^k\}$. Next we prove locality conditions used in the theorem for two cases: for non-random 3PC schemes and for schemes that preserve certain convex combination condition. It can be seen easily that random sparsification fits into the second case.

	\begin{lemma}[Deterministic 3PC]\label{lm:nbor-N3PCBC-det}
		Let Assumption \ref{asm:main} holds. Let $\cC^M$ and $\cC^W$ be deterministic. Assume $$\|x^0 - x^*\|^2 \leq \frac{11A_M}{24C_M} e_3^2 \eqdef \frac{11A_M}{24C_M}\min\left\{  \frac{A_M \mu^2}{24C_M \HS^2}, \frac{A_WA_M \mu^2}{384C_WC_M\HF^2}  \right\} \quad\text{and}\quad \cH^0 \leq \frac{A_M \mu^2}{96C_M}.$$ Then $\|x^k-x^*\|^2 \leq \frac{11A_M}{24C_M} e_3^2$, $\|z^k - x^*\|^2 \leq e_3^2$ and $\cH^k \leq  \frac{A_M \mu^2}{96C_M}$ for all $k\geq 0$.
	\end{lemma}
	\begin{lemma}[Random sparsification]\label{lm:nbor-N3PCBC-conv}
		Let Assumption \ref{asm:main} holds. Assume $(z^k)_j$ is a convex combination of $\{(x^t)_j\}_{t=0}^k$, and $(\mH_i^k)_{jl}$ is a convex combination of $\{  (\nabla^2 f_i(z^k))_{jl}  \}_{t=0}^k$ for all $i\in [n]$, $j,l \in [d]$, and $k\geq 0$. If $$\|x^0-x^*\|^2 \leq e_4^2 \eqdef \min\left\{ \frac{\mu^2}{d^2 \HS^2}, \frac{A_M \mu^2}{24d C_M \HS^2}, \frac{A_M\mu^2}{96d^3 C_M \HM^2}, \frac{\mu^2}{4d^4 \HM^2} \right\},$$ then $\|z^k-x^*\|^2 \leq d e_4^2$ and $\cH^k \leq \min \{  \frac{A_M\mu^2}{96C_M}, \frac{\mu^2}{4d}  \}$ for all $k\geq 0$. 
	\end{lemma}

	\section{Experiments}\label{sec:exp-main}
	
	In this part, we study the empirical performance of \algname{Newton-3PC} comparing its performance against other second-order methods on logistic regression problems of the form
	\begin{equation}\label{prob:log-reg}
		\squeeze	\min\limits_{x\in\R^d}\left\{f(x)\eqdef \frac{1}{n}\sum\limits_{i=1}^n f_i(x) +\frac{\lambda}{2}\|x\|^2\right\}, \qquad f_i(x) = \frac{1}{m}\sum \limits_{j=1}^m\log\(1+\exp(-b_{ij}a_{ij}^\top x)\),
	\end{equation}
	where $\{a_{ij},b_{ij}\}_{j\in [m]}$ are data points belonging to $i$-th client. We use datasets from LibSVM library \citep{chang2011libsvm} such as \dataname{a1a}, \dataname{a9a}, \dataname{w2a}, \dataname{w8a}, and \dataname{phishing}. Each dataset was shuffled and split into $n$ equal parts. Detailed description of datasets and the splitting is given in the Appendix.
	
	\subsection{Choice of parameters}
	For \algname{DINGO} \citep{DINGO} we use the authors' choice of hyperparameters: $\theta=10^{-4}, \phi=10^{-6}, \rho=10^{-4}$. Backtracking line search selects the largest stepsize from $\{1,2^{-1},\dots,2^{-10}\}.$ The initialization of $\mH_i^0$ for \algname{FedNL} \citep{FedNL2021}, \algname{NL1} \citep{Islamov2021NewtonLearn} is chosen as $\nabla^2f_i(x^0)$ if it is not specified. Finally, for \algname{Fib-IOS} \citep{IOSFabbro2022} we set $d_k^i = 1$. Local Hessians are computed following the partial sums of Fibonacci number and the parameter $\rho=\lambda_{q_{j+1}}$. This is stated in the description of the method. The parameters of backtracking line search for \algname{Fib-IOS} are $\alpha=0.5$ and $\beta=0.9$.
	
	\begin{figure}[t]
		\begin{center}
			\begin{tabular}{cccc}
				\includegraphics[width=0.22\linewidth]{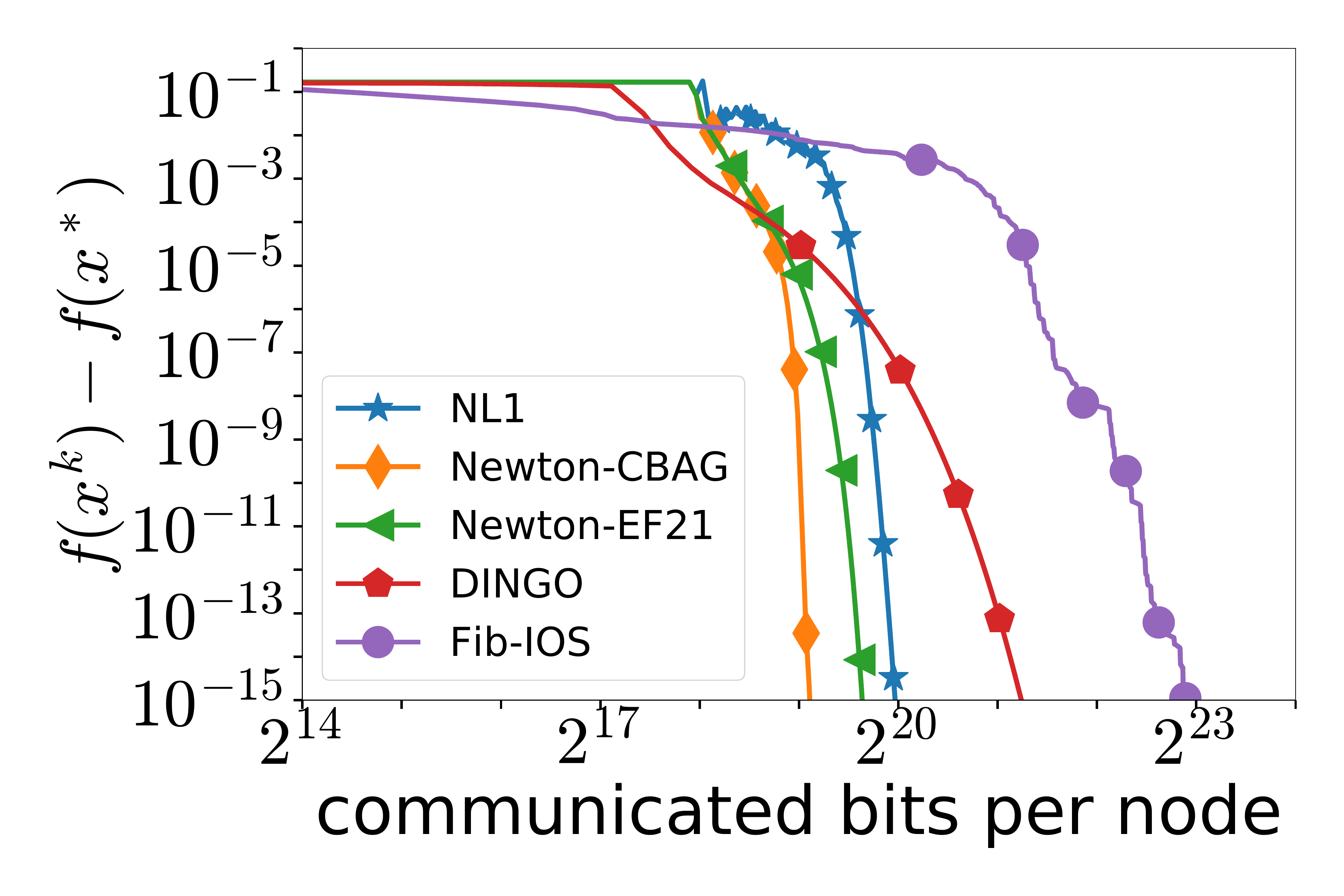} &
				\includegraphics[width=0.22\linewidth]{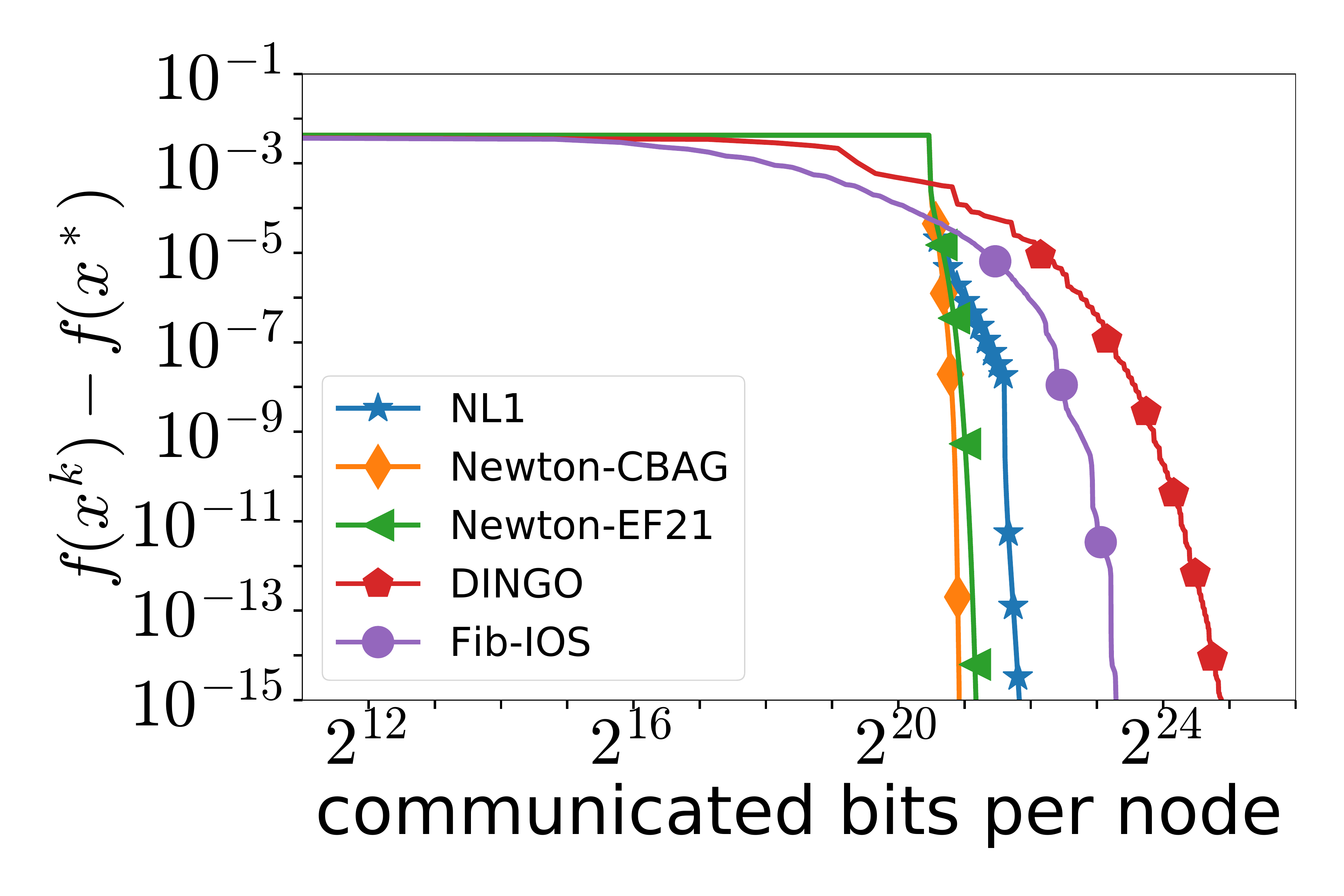} &
				\includegraphics[width=0.22\linewidth]{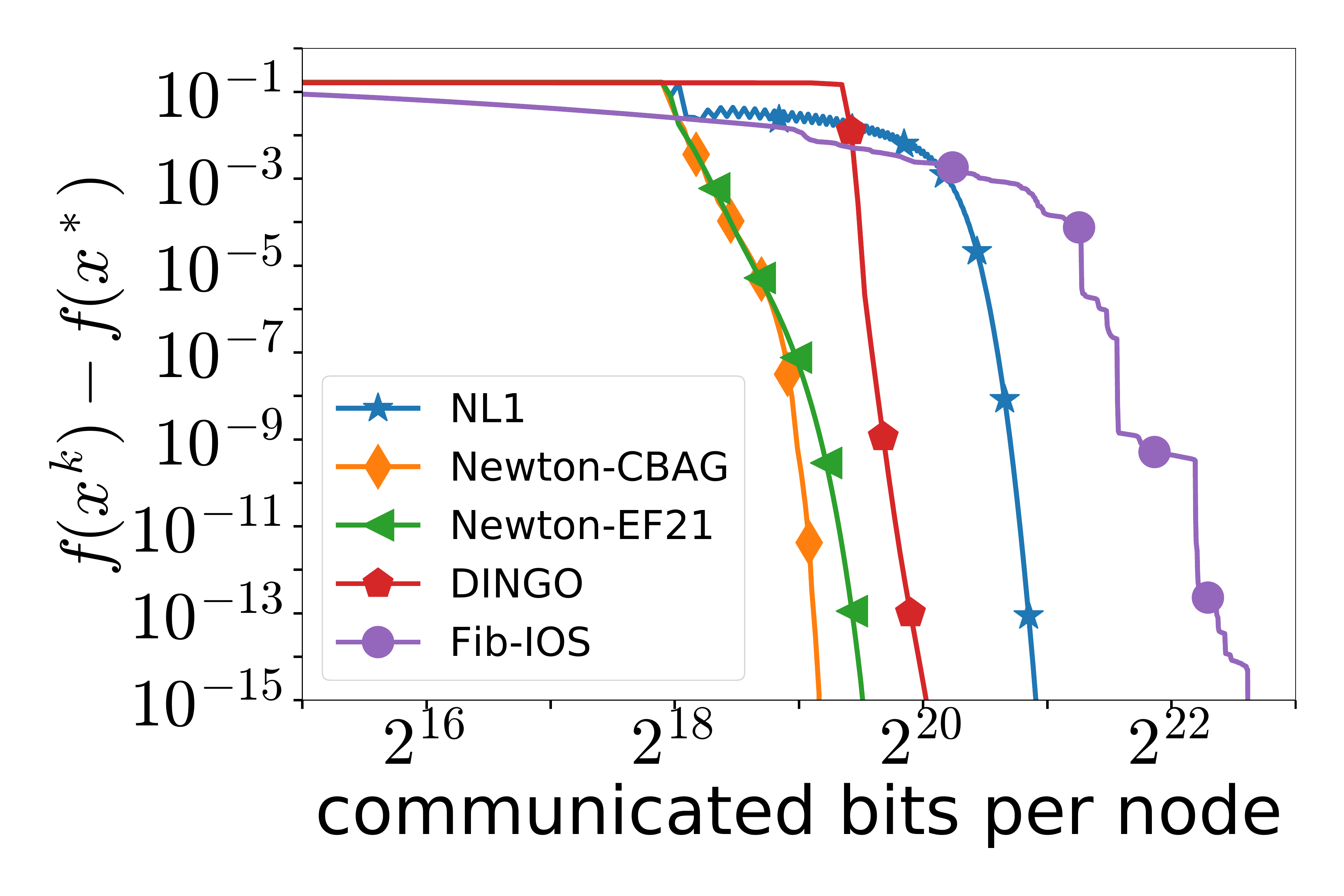} & 
				\includegraphics[width=0.22\linewidth]{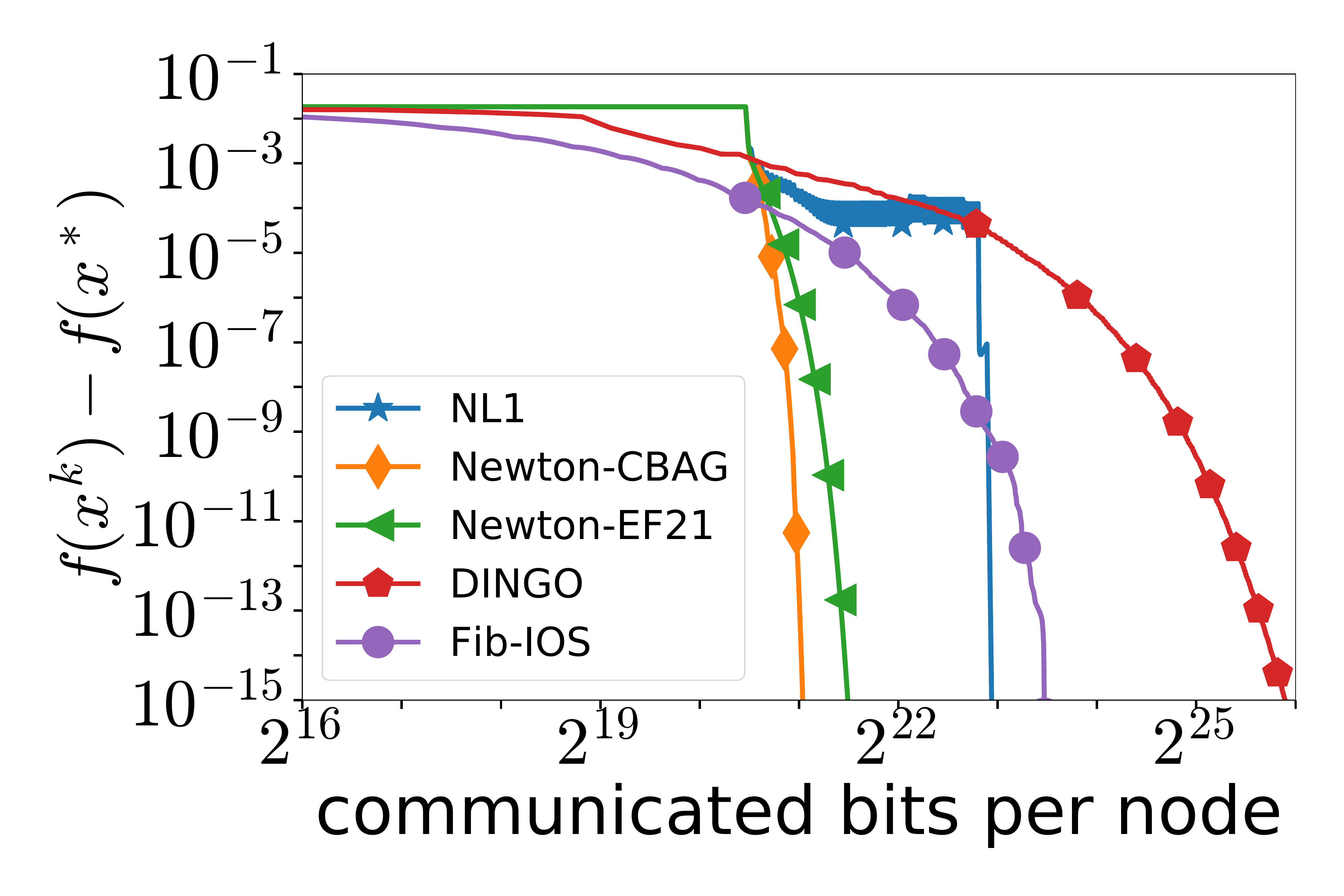}\\
				(a) \dataname{a1a}, {\scriptsize$ \lambda=10^{-3}$} &
				(b) \dataname{w2a}, {\scriptsize $\lambda=10^{-4}$} &
				(c) \dataname{a9a}, {\scriptsize$ \lambda=10^{-3}$} &
				(d) \dataname{w8a}, {\scriptsize$ \lambda=10^{-4}$} \\
				\includegraphics[width=0.22\linewidth]{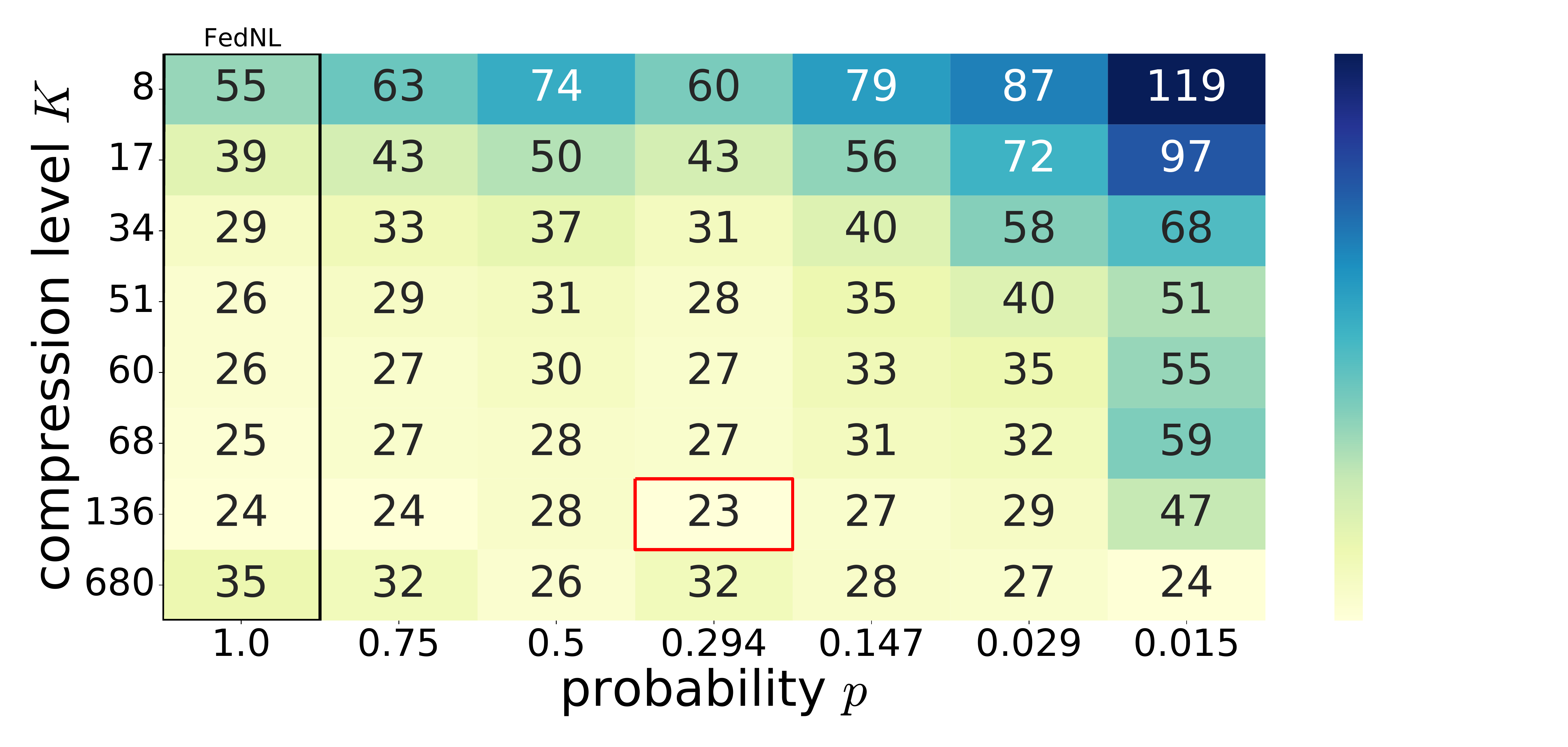} &
				\includegraphics[width=0.22\linewidth]{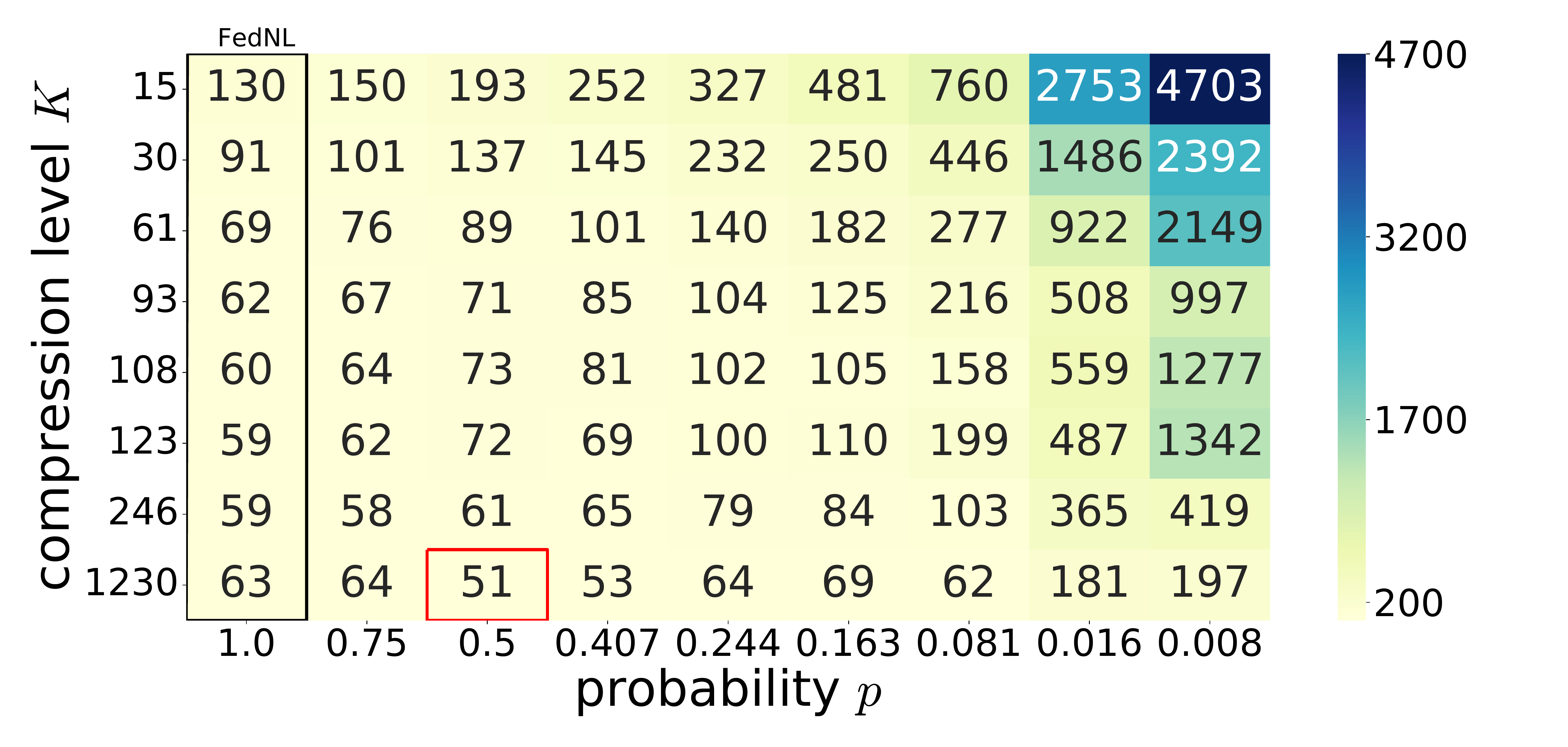} &
				\includegraphics[width=0.22\linewidth]{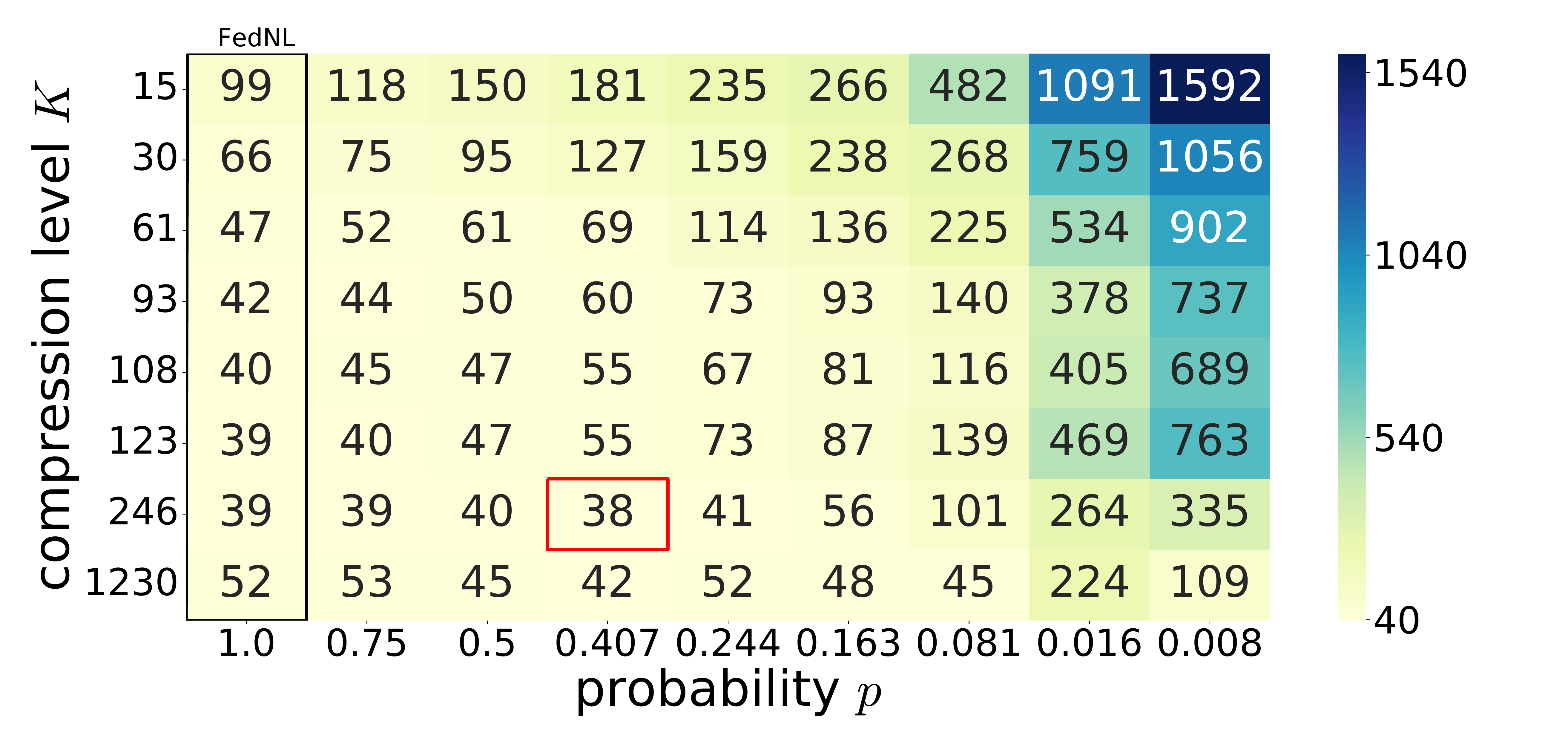} & 
				\includegraphics[width=0.22\linewidth]{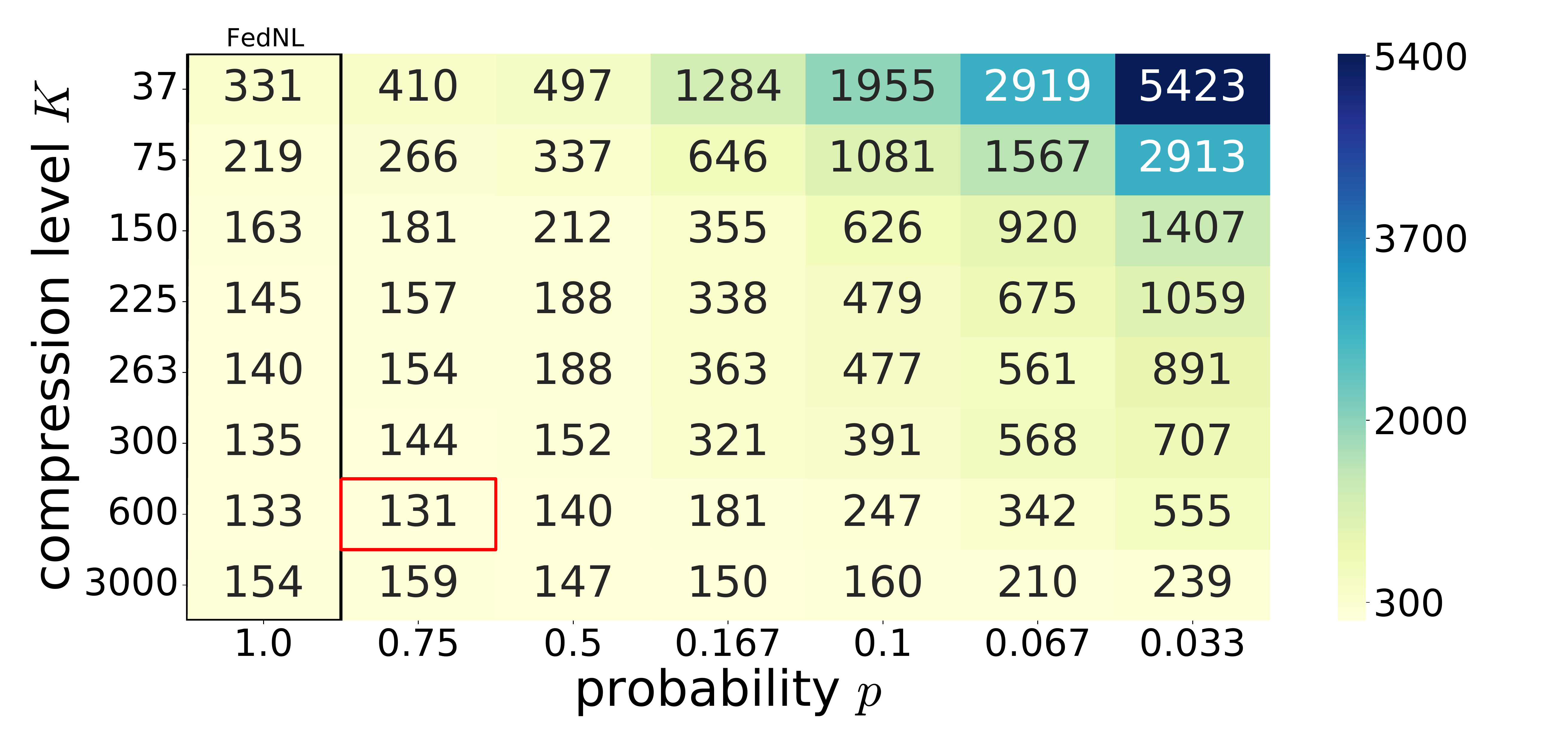}\\
				(e) \dataname{phishing}, {\scriptsize$ \lambda=10^{-3}$} &
				(f) \dataname{a1a}, {\scriptsize $\lambda=10^{-4}$} &
				(g) \dataname{a9a}, {\scriptsize$ \lambda=10^{-3}$} &
				(h) \dataname{w2a}, {\scriptsize$ \lambda=10^{-4}$} \\
				\includegraphics[width=0.22\linewidth]{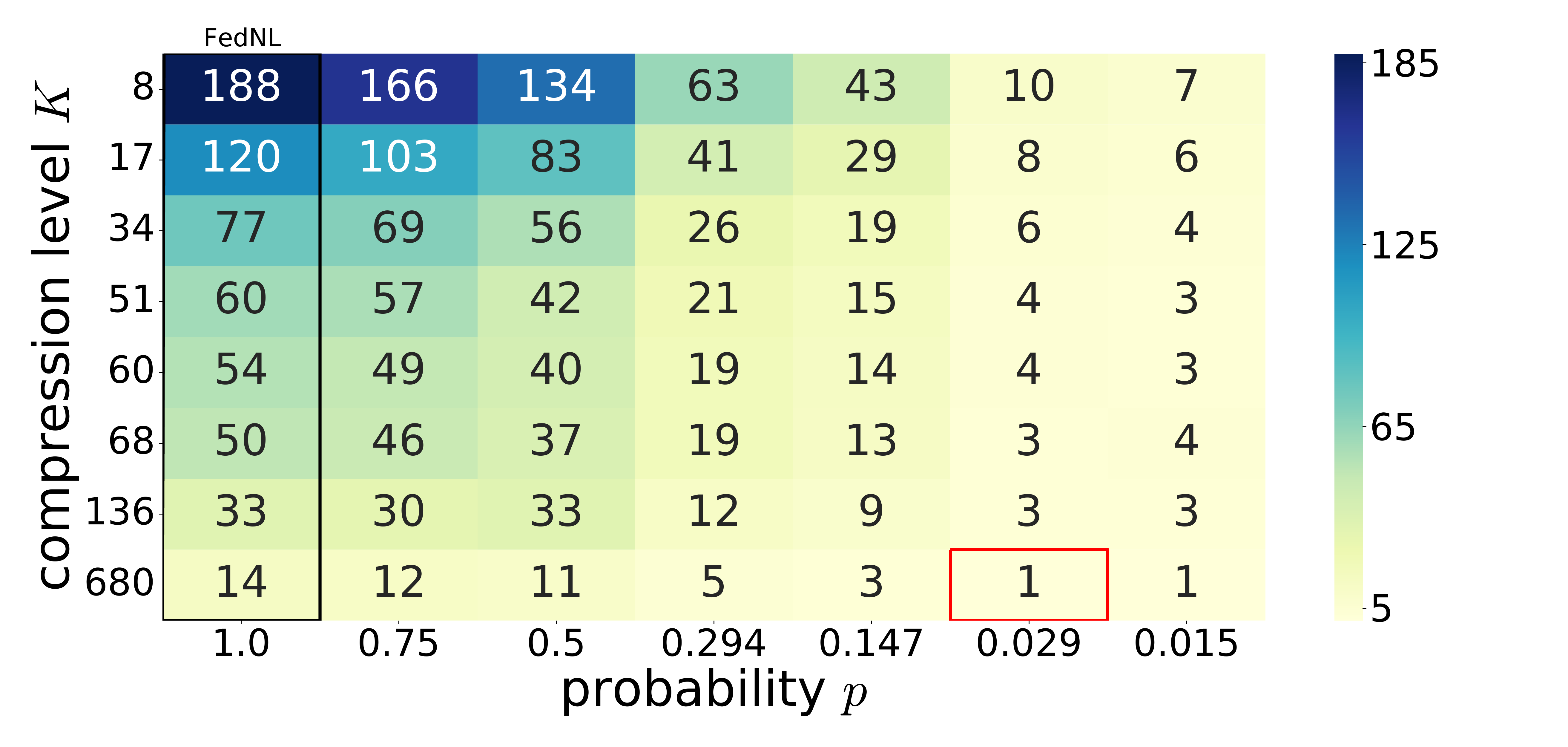} &
				\includegraphics[width=0.22\linewidth]{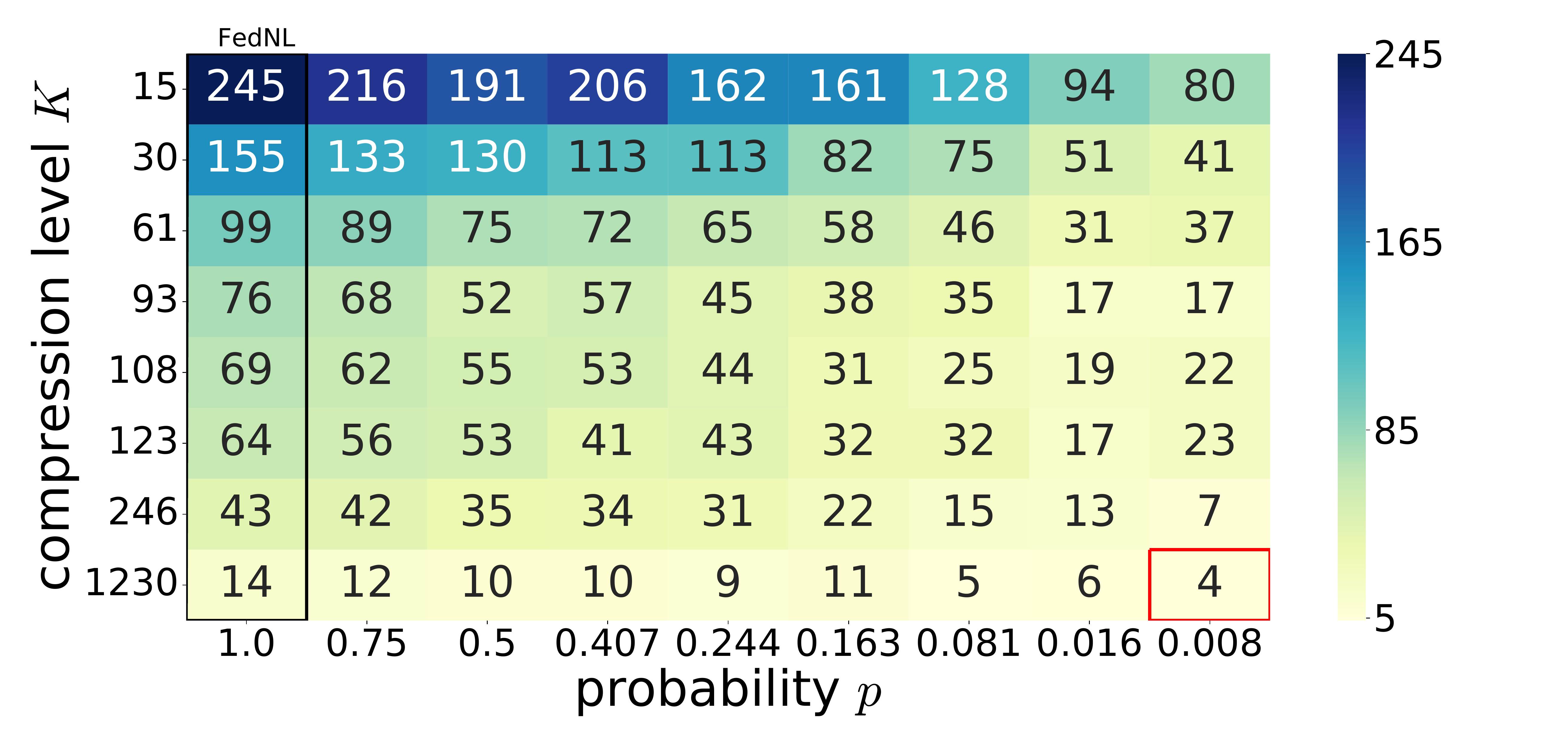} &
				\includegraphics[width=0.22\linewidth]{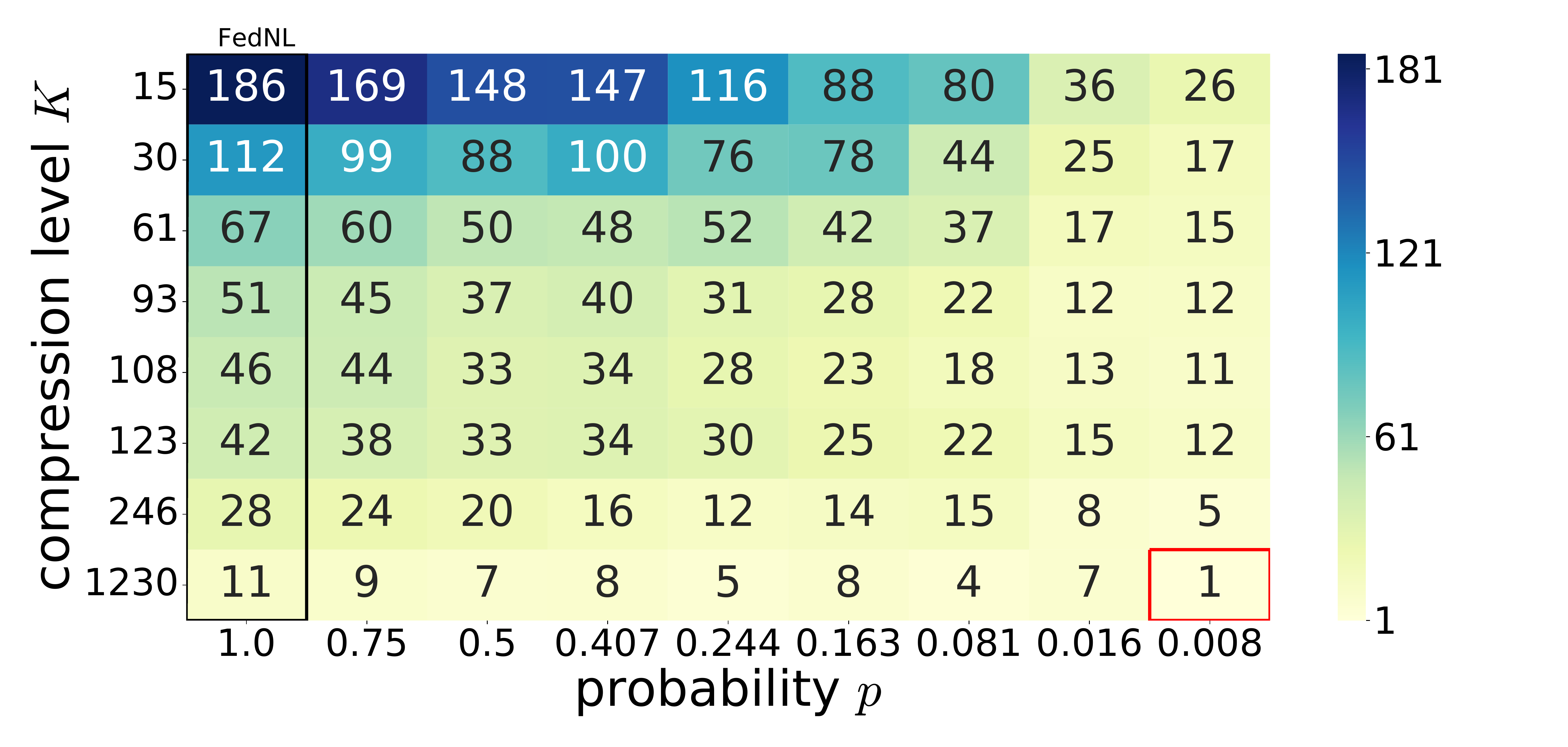} & 
				\includegraphics[width=0.22\linewidth]{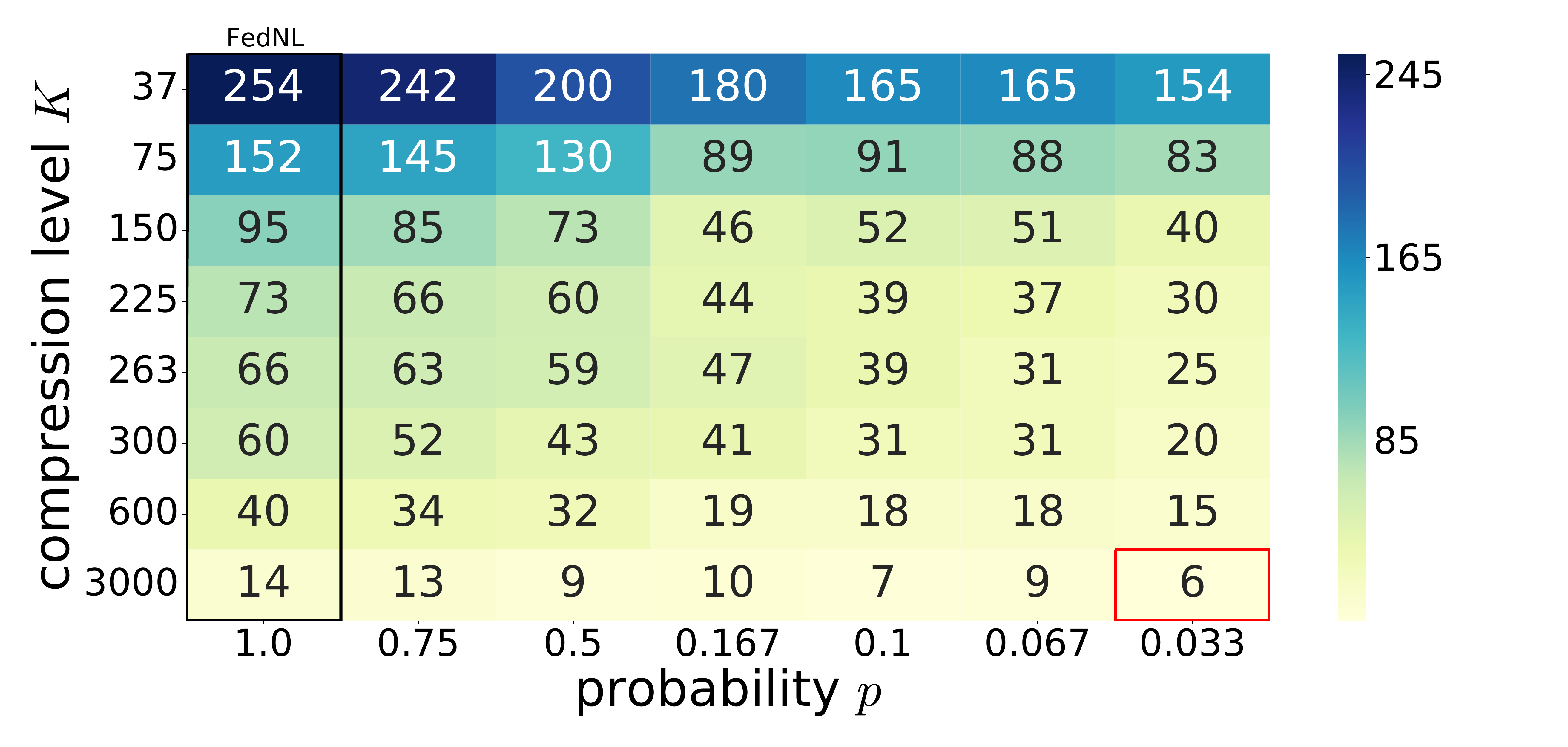}\\
				(i) \dataname{phishing}, {\scriptsize$ \lambda=10^{-3}$} &
				(j) \dataname{a1a}, {\scriptsize $\lambda=10^{-4}$} &
				(k) \dataname{a9a}, {\scriptsize$ \lambda=10^{-3}$} &
				(l) \dataname{w2a}, {\scriptsize$ \lambda=10^{-4}$} \\
			\end{tabular}				
		\end{center}
		\caption{Comparison of \algname{Newton-CBAG} with Top-$d$ compressor and probability $p=0.75$, \algname{Newton-EF21} (equivalent to \algname{FedNL}) with Rank-$1$ compressor, \algname{NL1} with Rand-$1$ compressor, and \algname{DINGO} ({\bf first row}). The performance of \algname{Newton-CBAG} with Top-$d$ in terms of communication complexity ({\bf second row}, in Mbytes) and the number of local Hessian computations ({\bf third row}).}
		\label{fig:Newton-ProbCLAG-three-in-one}
	\end{figure}
	
	\subsection{Comparison of \algname{Newton-3PC} with other second-order methods}
	According to \citep{FedNL2021}, \algname{FedNL} (equivalent to \algname{Newton-EF21}) with Rank-$1$ compressor outperforms other second-order methods in all cases in terms of communication complexity. Thus, we compare in Figure~\ref{fig:Newton-ProbCLAG-three-in-one} (first row) \algname{Newton-CBAG} (based on Top-$d$ compressor and probability $p=0.75$), \algname{Newton-EF21} with Rank-$1$, \algname{NL1} with Rand-$1$, \algname{DINGO}, and \algname{Fib-IOS} indicating how many bits are transmitted by each client in both uplink and downlink directions. We clearly see that \algname{Newton-CBAG} is much more communication efficient than \algname{NL1}, \algname{Fib-IOS} and \algname{DINGO}. Besides, it outperforms \algname{FedNL} in all cases (in the case of \dataname{a1a} data set speedup is almost $2$ times). On top of that, we achieve improvement not only in communication complexity, but also in computational cost with \algname{Newton-CBAG}. Indeed, when clients do not send compressed Hessian differences to the server there is no need to compute local Hessians. Consequently, computational costs goes down.
	
	We decided not to compare \algname{Newton-3PC} with first-order methods since \algname{FedNL} already outperforms them in terms of communication complexity in a variety of experiments in \citep{FedNL2021}.

	\subsection{Does Bernoulli aggregation brings any advantage?}
	Next, we investigate the performance of \algname{Newton-CBAG} based on Top-$K$. We report the results in 
	heatmaps (see Figure~\ref{fig:Newton-ProbCLAG-three-in-one}, second row) where we vary probability $p$ along rows and compression level $K$ along columns. Notice that \algname{Newton-CBAG} reduces to \algname{FedNL} when $p=1$ (left column). We observe that Bernoulli aggregation (BAG) is indeed beneficial since the communication complexity reduces when $p$ becomes smaller than $1$ (in case of \dataname{a1a} data set the improvement is significant). We can conclude that BAG leads to better communication complexity of \algname{Newton-3PC} over \algname{FedNL} (is equivalent to \algname{Newton-EF21}).
	
	On top of that, we claim that \algname{Newton-CBAG} is also computationally more efficient than \algname{FedNL}; see Figure~\ref{fig:Newton-ProbCLAG-three-in-one} (third row) that indicates the number of Hessian computations. We observe that even if communication complexity in two regimes are close to each other, but computationally better the one with smaller $p$. Indeed, in the case when $p < 1$ we do not have to compute local Hessians with probability $1-p$ that leads to acceleration in terms of computation complexity.

	\subsection{ 3PC based on adaptive thresholding}
	Next we test the performance of \algname{Newton-3PC} using adaptive thresholding operator \eqref{def:AT}. We compare \algname{Newton-EF21} (equivalent to \algname{FedNL}), \algname{Newton-CBAG}, and \algname{Newton-CLAG} with adaptive thresholding against \algname{Newton-CBAG} with Top-$d$ compressor. We fix the probability $p=0.5$ for CBAG, the trigger $\zeta=2$ for CLAG, and thresholding parameter $\lambda=0.5$. According to the results presented in Figure~\ref{fig:Newton-ProbCLAG-two-in-one} (first row), adaptive thresholding can be beneficial since it improves the performance of \algname{Newton-3PC} in some cases. Moreover, it is computationally cheaper than Top-$K$ as we do not sort entries of a matrix as it is for Top-$K$.	
	
	\subsection{\algname{Newton-3PC-BC} against \algname{FedNL-BC}}
	In our next experiment, we study bidirectional compression. We compare \algname{Newton-3PC-BC} against \algname{FedNL-BC} (equivalent to \algname{Newton-3PC-BC} with EF21 update rule applied on Hessians and iterates). For \algname{Newton-3PC-BC} we fix CBAG with $p=0.75$ combined with Top-$d$ compressor applied on Hessians, BAG with $p=0.75$ applied on gradients, and 3PCv4 \citep{richtarik3PC} combined with (Top-$K_1$, Top-$K_2$) compressors on iterates. For \algname{FedNL-BC} we use Top-$d$ compressor on Hessians and BAG with $p=0.75$ on gradients, and Top-$K$ compressor on iterates. We choose different values for $K_1$ and $K_2$ such that it $K_1+K_2=K$ always hold. Such choice of parameters allows to make the iteration cost of both methods to be equal. Based on the results, we argue that the superposition of CBAG and 3PCv4 applied on Hessians and iterates respectively is more communication efficient than the combination of EF21 and EF21.

	\begin{figure}[t]
		\begin{center}
			\begin{tabular}{cccc}
				\includegraphics[width=0.22\linewidth]{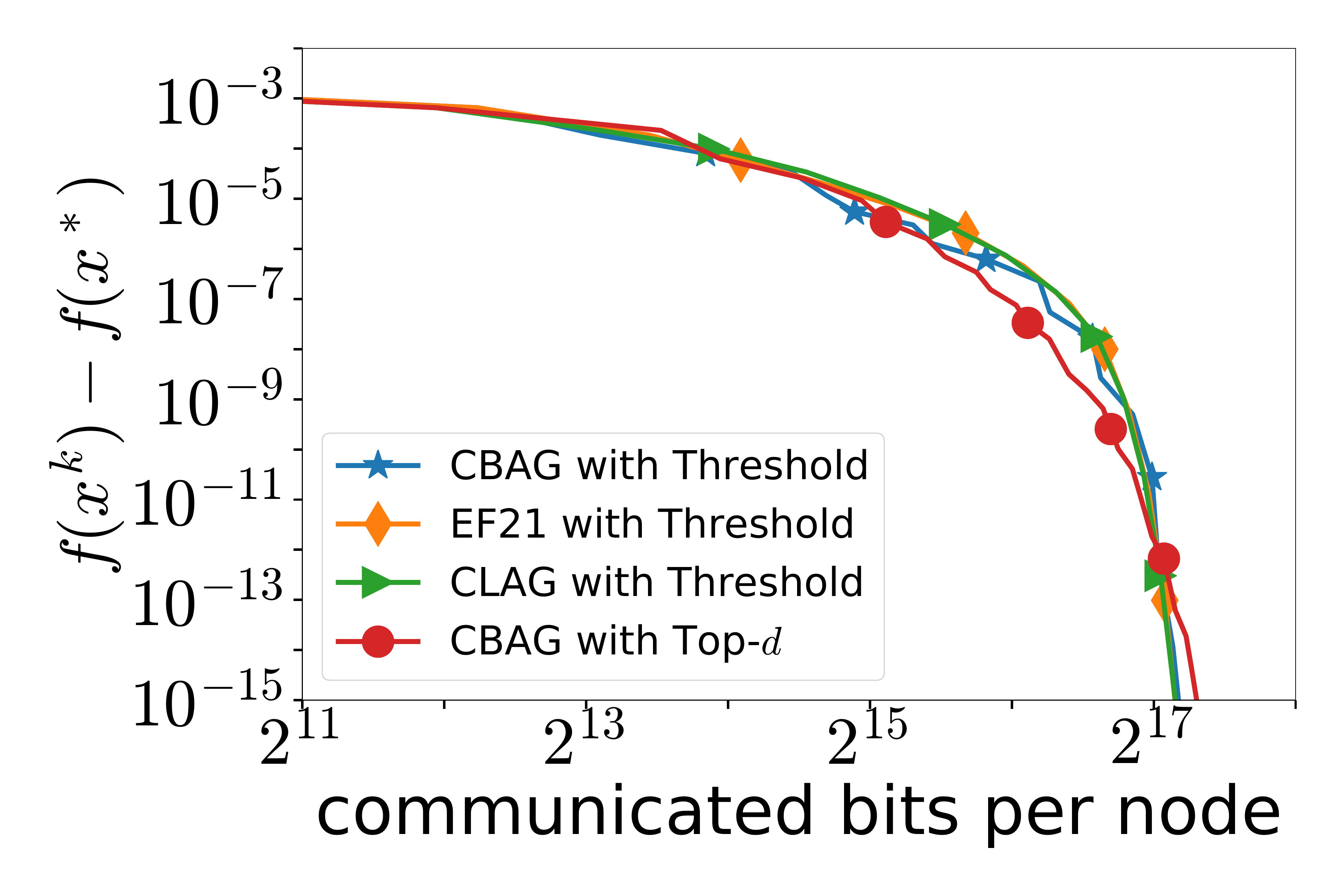} &
				\includegraphics[width=0.22\linewidth]{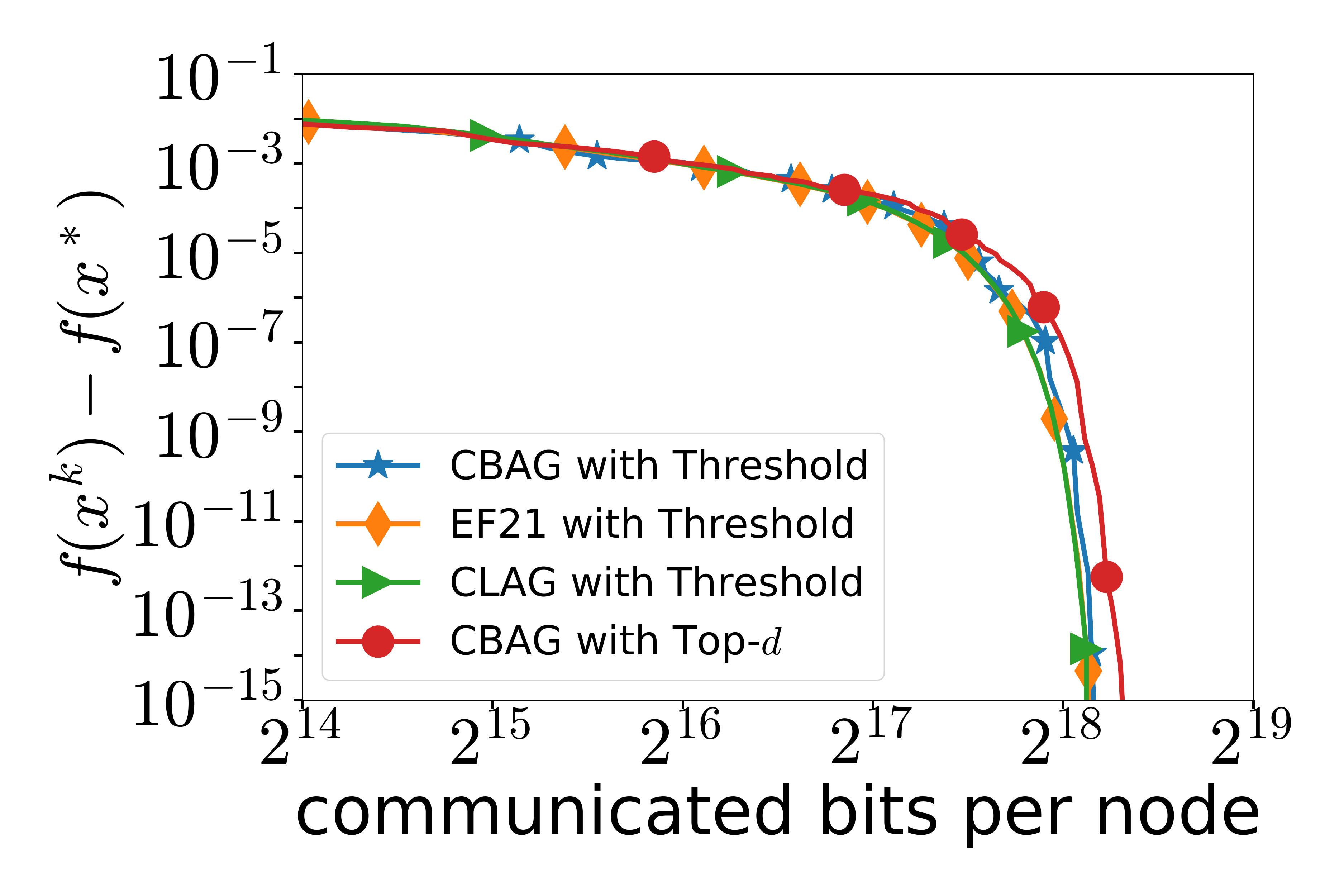} &
				\includegraphics[width=0.22\linewidth]{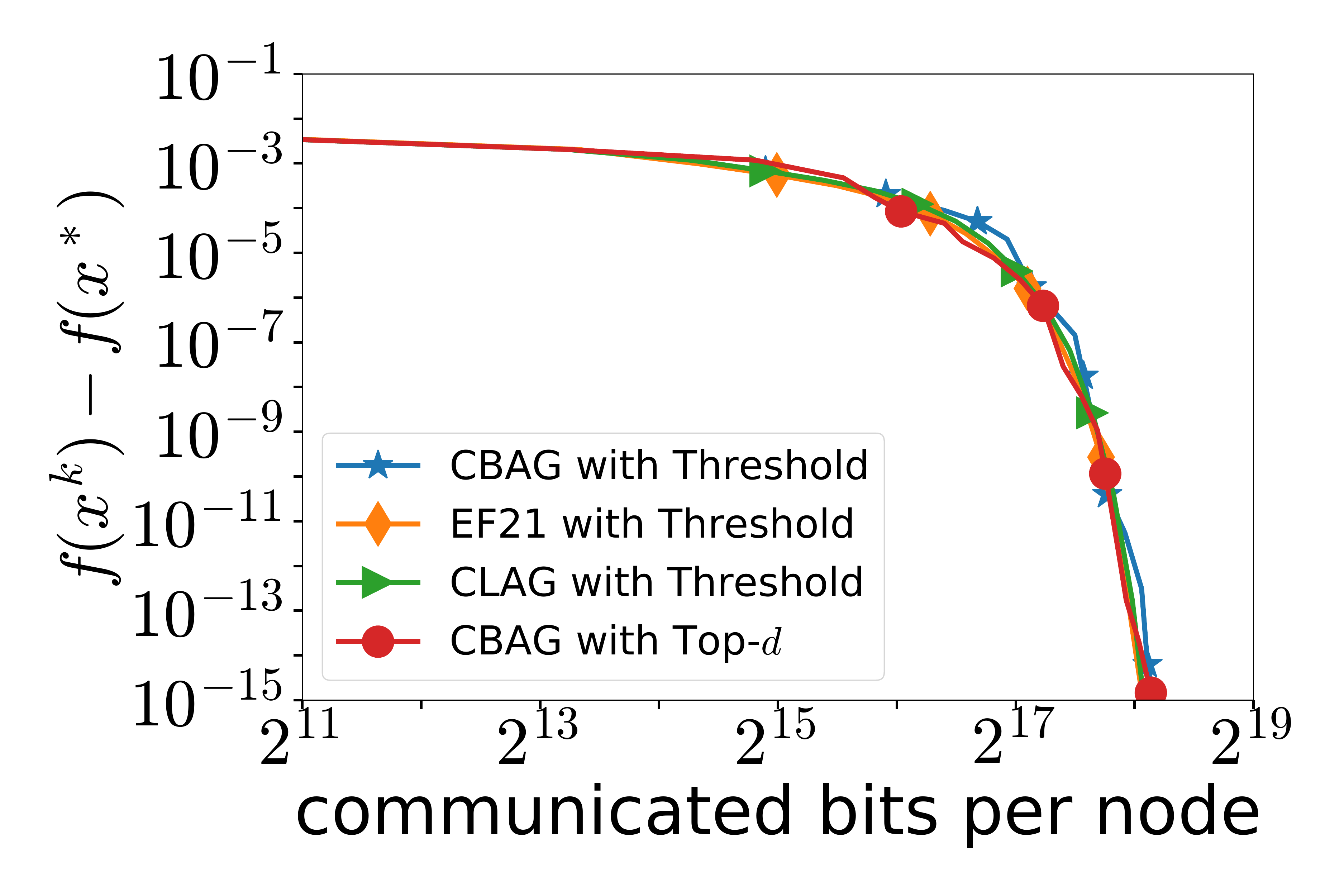} & 
				\includegraphics[width=0.22\linewidth]{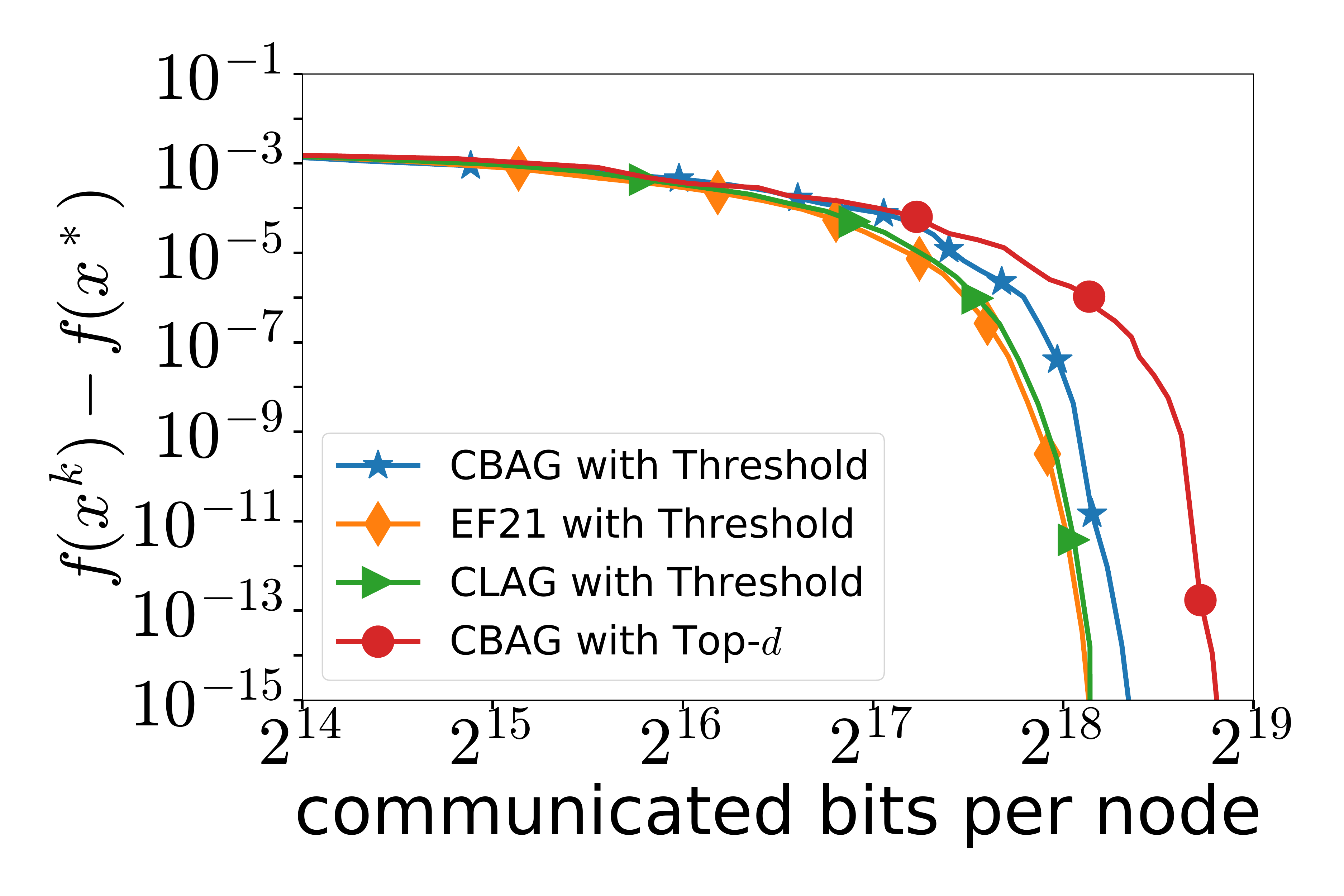}\\
				(a) \dataname{a9a}, {\scriptsize$ \lambda=10^{-3}$} &
				(b) \dataname{a1a}, {\scriptsize $\lambda=10^{-4}$} &
				(c) \dataname{w8a}, {\scriptsize$ \lambda=10^{-3}$} &
				(d) \dataname{w2a}, {\scriptsize$ \lambda=10^{-4}$} \\
				\includegraphics[width=0.22\linewidth]{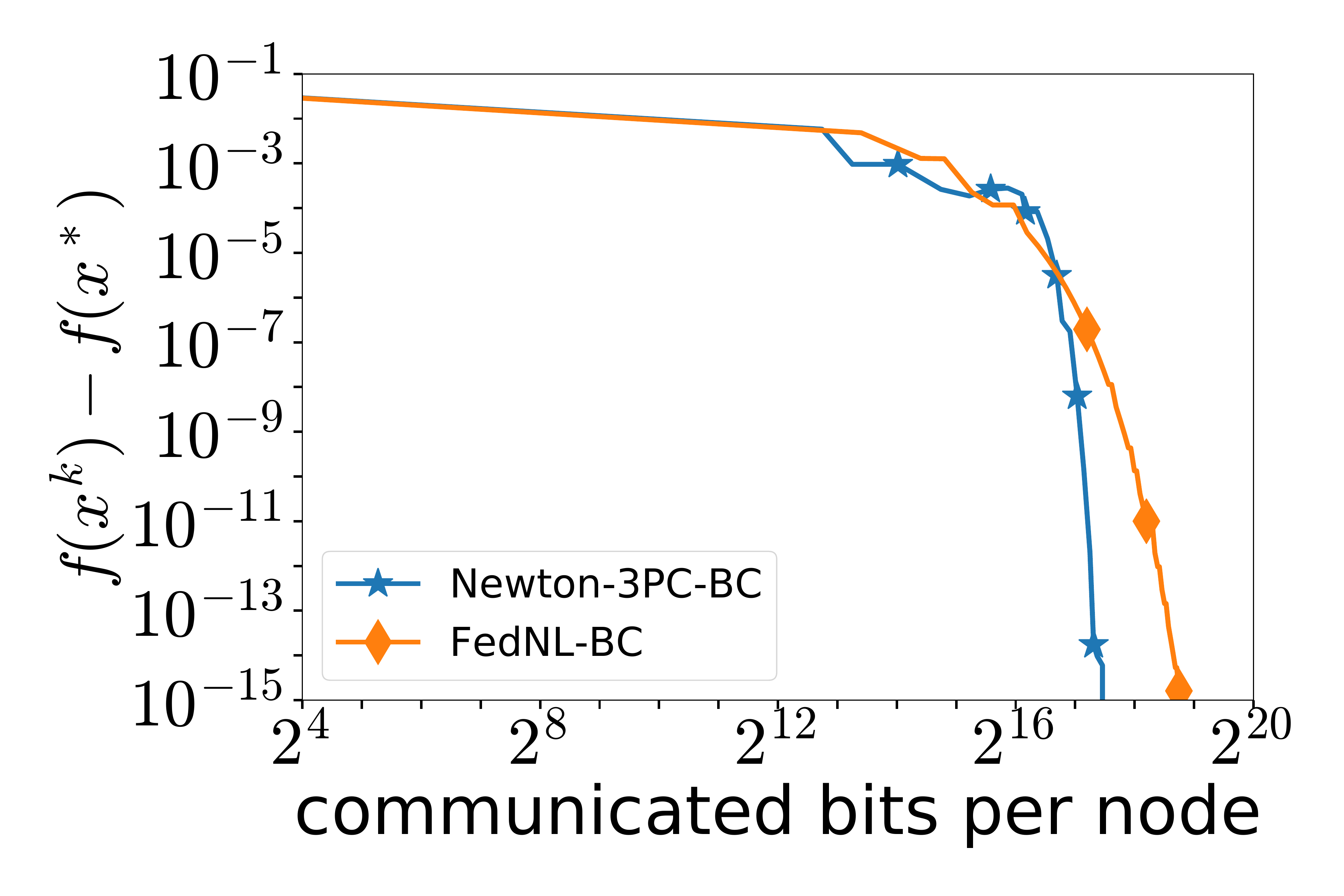} &
				\includegraphics[width=0.22\linewidth]{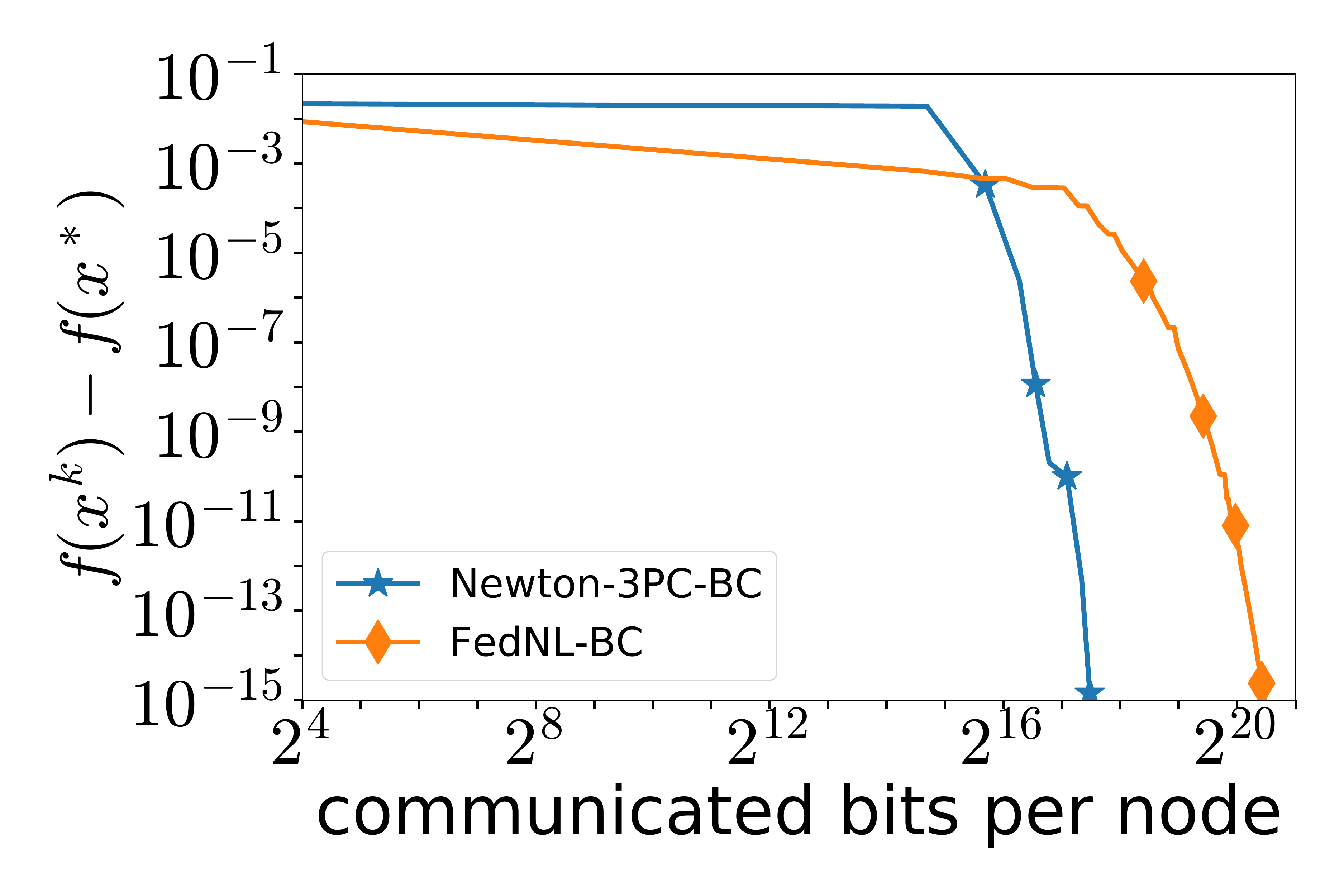} &
				\includegraphics[width=0.22\linewidth]{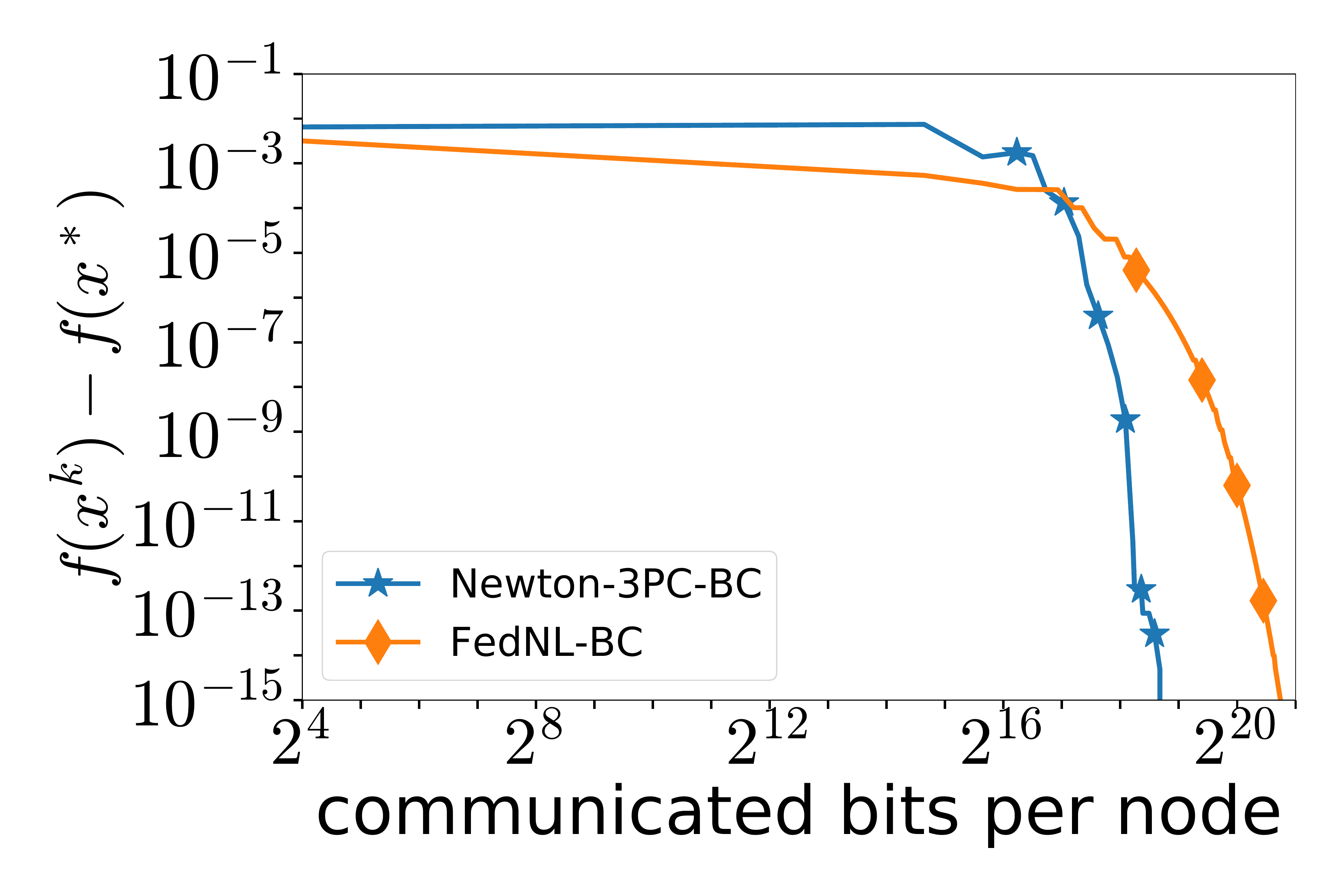} & 
				\includegraphics[width=0.22\linewidth]{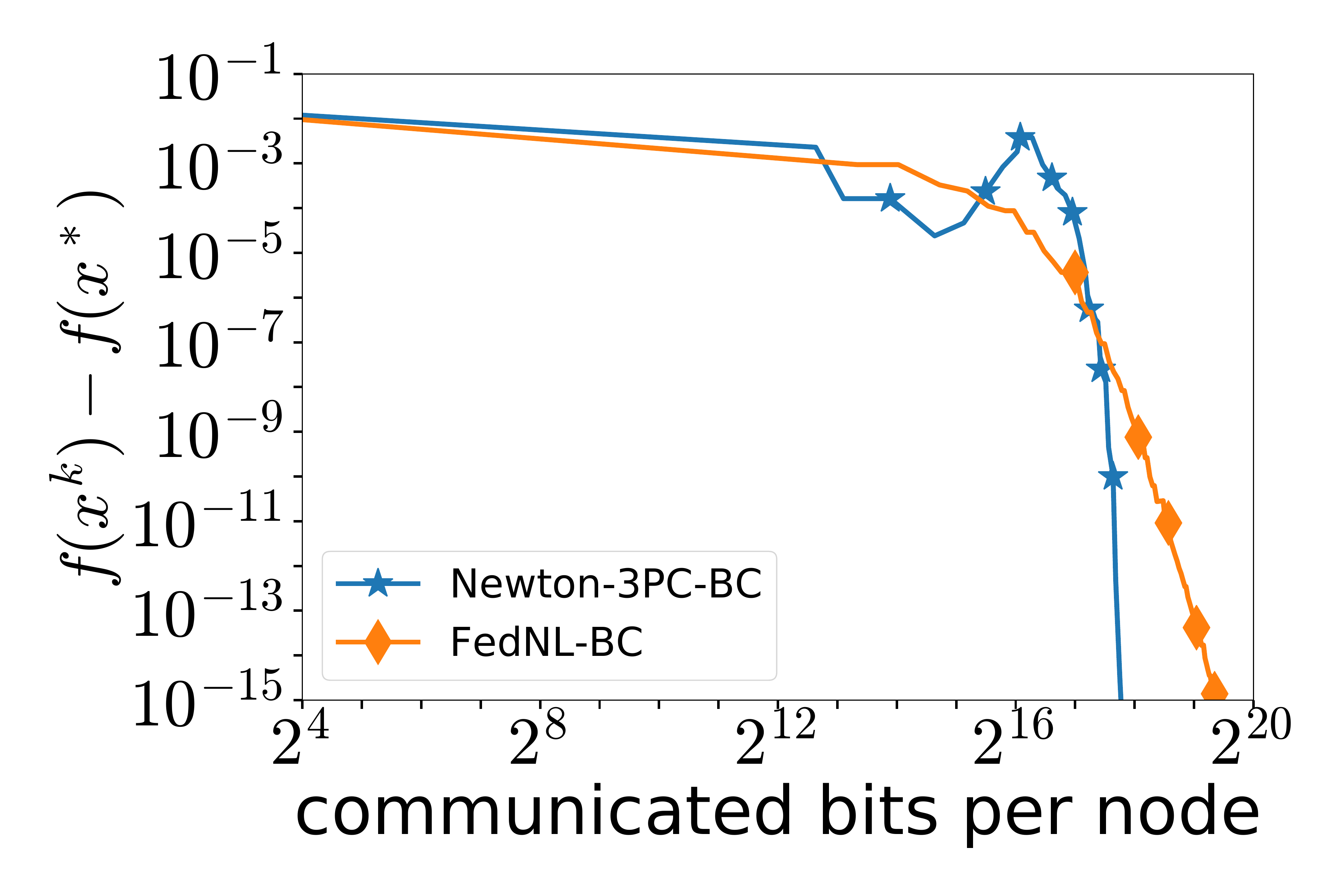}\\
				(e) \dataname{a9a}, {\scriptsize$ \lambda=10^{-3}$} &
				(f) \dataname{w2a}, {\scriptsize $\lambda=10^{-4}$} &
				(g) \dataname{w8a}, {\scriptsize$ \lambda=10^{-3}$} &
				(h) \dataname{a1a}, {\scriptsize$ \lambda=10^{-4}$} \\
			\end{tabular}				
		\end{center}
		\caption{Comparison of \algname{Newton-CBAG} with thresholding and Top-$d$ compressors and \algname{Newton-EF21} with thresholding compressor in terms of communication complexity ({\bf first row}). Comparison of \algname{Newton-3PC-BC} against \algname{FedNL-BC} in terms of communication complexity ({\bf second row}).}
		\label{fig:Newton-ProbCLAG-two-in-one}
	\end{figure}

	\clearpage
	\bibliography{references}
	\bibliographystyle{plainnat}

	\clearpage
	\appendix
	\part*{Appendix}
	
	\section{Deferred Proofs from Section \ref{sec:3PC4M} and New 3PC Compressors}

	\subsection{Proof of Lemma \ref{lem:3PC__AT}: Adaptive Thresholding}
	
	Basically, we show two upper bounds for the error and combine them to get the expression for $\alpha$. From the definition \eqref{def:AT}, we get
	\begin{equation*}
		\|\cC(\mX) - \mX\|^2_{\rm F}
		= \sum_{j,l: |\mX_{jl}|<\lambda\|\mX\|_{\infty}} \mX_{jl}^2
		\le d^2\lambda^2\|\mX\|_{\infty}^2
		\le d^2\lambda^2\|\mX\|_{\rm F}^2.
	\end{equation*}
	
	The second inequality is derived from the observation that at least on entry, the top one in magnitude, is selected always. Since the top entry is missing in the sum below, we imply that the average without the top one is smaller than the overall average.
	\begin{equation*}
		\|\cC(\mX) - \mX\|^2_{\rm F}
		= \sum_{j,l: |\mX_{jl}|<\lambda\|\mX\|_{\infty}} \mX_{jl}^2
		\le \frac{d^2-1}{d^2} \sum_{j,l=1}^d \mX_{jl}^2
		\le \(1-\frac{1}{d^2}\)\|\mX\|_{\rm F}^2.
	\end{equation*}
	
	\subsection{Proof of Lemma \ref{lem:3PC__CBAG}: Compressed Bernoulli AGgregation (CBAG)}
	
	As it was mentioned, CBAG has two independent sources of randomness: Bernoulli aggregation and possible random contractive compression. To show that CBAG is a 3PC mechanism, we consider these randomness one by one and upper bound the error as follows:
	\begin{eqnarray*}
		\E\left[\|\cC_{\mH,\mY}(\mX) - \mX\|^2 \right]
		&=& (1-p)\|\mH - \mX\|^2 + p \E\ll\|\cC(\mX-\mH) - (\mX-\mH)\|^2\rr \\
		&\le& (1-p)\|\mX-\mH\|^2 + p(1-\alpha)\|\mX-\mH\|^2 \\
		&=& (1-p\alpha)\|\mX-\mH\|^2 \\
		&\le& (1-p\alpha)(1+s)\|\mH-\mY\|^2 + (1-p\alpha)(1+\nicefrac{1}{s})\|\mX-\mY\|^2.
	\end{eqnarray*}

	\subsection{New 3PC: Adaptive Top-$K$}

	Assume that in our framework we are restricted by the number of floats we can send from clients to the server. For example, each client is able to broadcast $d_0 \leq d^2$ floats to the server. Besides, we want to use Top-$K$ compression operator with adaptive $K$, but due to the aforementioned restrictions we should control how $K$ evolves. Let $K_{\mH,\mY}$ be such that
	$$K_{\mH,\mY} = \min\left\{\left\lceil \frac{\|\mY-\mH\|^2_{\rm F}}{\|\mX-\mH\|^2_{\rm F}}d^2\right\rceil, d_0\right\}$$. We introduce the following compression operator 
	\begin{equation}\label{def:adaptive_topk}
		\cC_{\mH,\mY}(\mX) \eqdef \mH + \text{Top-}K_{\mH,\mY}\left(\mX-\mH\right).
	\end{equation}
	The next lemma shows that the described compressor satisfy~\eqref{def:3PC_comp}.
	\begin{lemma}
		The compressor $\cC_{\mY,\mH}$~\eqref{def:adaptive_topk} satisfy~\eqref{def:3PC_comp} with $$A=\frac{d_0}{2d^2}, \quad B = \max\left\{\left(1-\frac{d_0}{d^2}\right)\left(\frac{2d^2}{d_0}-1\right), 3\right\}.$$
	\end{lemma}
	\begin{proof}
		Recall that if $\cC$ is a Top-$K$ compressor, then for all $\mX\in \R^{d\times d}$ 
		$$\norm{\cC(\mX)-\mX}^2_{\rm F} \leq \left(1-\frac{K}{d^2}\right)\norm{\mX}^2_{\rm F},$$
		Using this property we get in the case when $K_{\mY,\mH} = d_0$
		\begin{align*}
			\norm{\cC_{\mH,\mY}(\mX) - \mX}^2_{\rm F} &= \norm{\mH + \text{Top-}K_{\mH,\mY}(\mX-\mH) - \mX}_{\rm F}^2\\
			&\leq \left(1-\frac{d_0}{d^2}\right)\norm{\mH-\mX}^2_{\rm F}\\
			&\leq \left(1-\frac{d_0}{2d^2}\right)\norm{\mH-\mY}^2_{\rm F} + \left(1-\frac{d_0}{d^2}\right)\frac{2d^2-d_0}{d_0}\norm{\mY-\mX}^2_{\rm F}.
		\end{align*}
		If $K_{\mH,\mY} = \left\lceil \frac{\|\mY-\mH\|^2_{\rm F}}{\|\mX-\mH\|^2_{\rm F}}d^2\right\rceil$, then $-K_{\mH,\mY}  \leq -\frac{\|\mY-\mH\|^2_{\rm F}}{\|\mX-\mH\|^2_{\rm F}}d^2$, and we have 
		\begin{align*}
			\norm{\cC_{\mH,\mY}(\mX) - \mX}^2_{\rm F} &= \norm{\mH+\text{Top-}K_{\mH,\mY}(\mX-\mH)-\mX}^2_{\rm F}\\
			&\leq \left(1-\frac{K_{\mH,\mY}}{d^2}\right)\norm{\mH-\mX}^2_{\rm F}\\
			&\leq \left(1-\frac{\norm{\mY-\mH}^2_{\rm F}}{\norm{\mX-\mH}^2_{\rm F}}\right)\norm{\mH-\mX}^2_{\rm F}\\
			&= \norm{\mH-\mX}^2_{\rm F} - \norm{\mY-\mH}^2_{\rm F}\\
			&\leq \frac{3}{2}\norm{\mH-\mY}^2_{\rm F} + 3\norm{\mY-\mX}^2_{\rm F}- \norm{\mY-\mH}^2_{\rm F}\\
			&=\frac{1}{2}\norm{\mY-\mH}^2_{\rm F} + 3\norm{\mY-\mX}^2_{\rm F},
		\end{align*}
		where in the last inequality we use Young's inequality. Since we always have $\frac{d_0}{2d^2}$ (because $d_0 \leq d^2$), then $A = \frac{d_0}{2d^2}.$
	\end{proof}
	
	\subsection{New 3PC: Rotation Compression}\label{apx:rot-comp}
	
	\citep{qian2021basis} proposed a novel idea to change the basis in the space of matrices that allows to apply more aggresive compression mechanism. Following Section~$2.3$ from \citep{qian2021basis} one can show that for Generalized Linear Models local Hessians can be represented as $\nabla^2f_i(x) = \mQ_i\Lambda_i(x)\mQ_i^\top,$ where $\mQ_i$ is properly designed basis matrix. This means that $\mQ_i$ is orthogonal matrix. Their idea is based on the fact that $\Lambda_i(x)$ is potentially sparser matrix than $\nabla^2f_i(x),$ and applying compression on $\Lambda_i(x)$ could require smaller compression level to obtain the same results than applying compression on dense standard representation $\nabla^2f_i(x).$ We introduce the following compression based on this idea. Let $\cC$ be an arbitrary contractive compressor with parameter $\alpha,$ and $\mQ$ be an orthogonal matrix, then our new compressor is defined as follows 
	\begin{equation}\label{def:rotation_comp}
		\cC_{\mH,\mY}(\mX) \eqdef \mH + \mQ\cC\left(\mQ^\top(\mX-\mH)\mQ\right)\mQ^\top.
	\end{equation}
	Now we prove that this compressor satisfy~\eqref{def:3PC_comp}.
	\begin{lemma}
		The compressor $\cC_{\mH,\mQ}$~\eqref{def:rotation_comp} based on a contractive compressor $\cC$ with parameter $\alpha\in(0,1]$ satisfy~\eqref{def:3PC_comp} with $A=\nicefrac{\alpha}{2}$ and $B=(1-\alpha)\left(\nicefrac{(2-\alpha)}{\alpha}\right)$.
	\end{lemma}
	\begin{proof}
		From the definition of contractive compressor
		$$\ExpBr{\norm{\cC(\mX)-\mX}^2_{\rm F}} \leq (1-\alpha)\norm{\mX}^2_{\rm F}.$$
	\end{proof}
	Thus, we get
	\begin{align*}
		\ExpBr{\norm{\cC_{\mH,\mY}(\mX) - \mX}^2_{\rm F}} &= \ExpBr{\norm{\mQ\cC\left(\mQ^\top(\mX-\mH)\mQ\right)\mQ^\top - (\mX-\mH)}^2_{\rm F}}\\
		&=\ExpBr{\norm{\mQ\cC\left(\mQ^\top(\mX-\mH)\mQ\right)\mQ^\top - \mQ\mQ^\top(\mX-\mH)\mQ\mQ^\top}^2_{\rm F}}\\
		&=\ExpBr{\norm{\cC\left(\mQ^\top(\mX-\mH)\mQ\right) -\mQ^\top(\mX-\mH)\mQ}^2_{\rm F}}\\
		&\leq (1-\alpha)\norm{\mQ^\top(\mX-\mH)\mQ}^2_{\rm F}\\
		&= (1-\alpha)\norm{\mX-\mH}^2_{\rm F}\\
		&\leq (1-\alpha)(1+\beta)\norm{\mY-\mH}^2_{\rm F} + (1-\alpha)(1+\beta^{-1})\norm{\mY-\mX}^2_{\rm F},
	\end{align*}
	where we use the fact that an orthogonal matrix doesn't change a norm. Let $\beta=\frac{\alpha}{2(1-\alpha)}$, then
	\begin{align}
		\ExpBr{\norm{\cC_{\mH,\mY}(\mX) - \mX}^2_{\rm F}} &\leq \left(1-\frac{\alpha}{2}\right)\norm{\mY-\mH}^2_{\rm F} + (1-\alpha)\left(\frac{2-\alpha}{\alpha}\right)\norm{\mY-\mX}^2_{\rm F}.
	\end{align}

	\section{Deferred Proofs from Section \ref{sec:N3PC} (\algname{Newton-3PC})}
	
	\subsection{Auxiliary lemma}
	
	Denote by $\E_{k+1}[\cdot]$ the conditional expectation given $(k+1)^{th}$ iterate $x^{k+1}$. We first develop a lemma to handle the mismatch $\mathbb{E}_k \|\mH_i^{k+1} - \nabla^2 f_i(x^{*})\|^2_{\rm F}$ of the estimate $\mH_i^{k+1}$ defined via 3PC compressor.
	
	\begin{lemma}\label{lm:threecomp}
		Assume that $\norm{x^{k+1}-x^*}^2\leq \frac{1}{2}\norm{x^k-x^*}^2$ for all $k\geq 0$. Then
		\begin{eqnarray*}
			\E_{k+1} \ll\|\mH_i^{k+1} - \nabla^2 f_i(x^{*})\|^2_{\rm F}\rr
			\le \(1-\frac{A}{2}\) \|\mH_i^{k} - \nabla^2 f_i(x^{*})\|^2_{\rm F} + \(\frac{1}{A}+3B\)\HF^2 \|x^{k} - x^*\|^2_{\rm F}
		\end{eqnarray*}
	\end{lemma}
	\begin{proof}
		Using the defining inequality of 3PC compressor and the assumption of the error in terms of iterates, we expand the approximation error of the estimate $\mH_i^{k+1}$ as follows:
		\begin{eqnarray*}
			&& \E_{k+1} \ll\|\mH_i^{k+1} - \nabla^2 f_i(x^{*})\|^2_{\rm F}\rr \\
			&=&  \E_{k+1} \ll\|\cC_{\mH_i^k,\nabla^2f_i(x^k)}\left(\nabla^2f_i(x^{k+1})\right) - \nabla^2 f_i(x^*)\|^2_{\rm F}\rr \\ 
			&\le& (1+\beta)\E_{k+1} \ll\|\cC_{\mH_i^k,\nabla^2f_i(x^k)}\left(\nabla^2f_i(x^{k+1})\right) - \nabla^2 f_i(x^{k+1})\|^2_{\rm F}\rr + (1+\nicefrac{1}{\beta})\|\nabla^2f_i(x^{k+1}) - \nabla^2 f_i(x^*)\|^2_{\rm F} \\
			&\le& (1+\beta)(1-A) \|\mH_i^k - \nabla^2f_i(x^k)\|^2_{\rm F} + B\|\nabla^2f_i(x^{k+1}) - \nabla^2 f_i(x^*)\|^2_{\rm F} + (1+\nicefrac{1}{\beta}) \|\nabla^2f_i(x^{k+1}) - \nabla^2 f_i(x^*)\|^2_{\rm F} \\
			&\le& (1+\beta)(1-A) \|\mH_i^k - \nabla^2f_i(x^k)\|^2_{\rm F} \\
			&&\quad + 2B\|\nabla^2f_i(x^k) - \nabla^2 f_i(x^*)\|^2_{\rm F} + (1+\nicefrac{1}{\beta} + 2B) \|\nabla^2f_i(x^{k+1}) - \nabla^2 f_i(x^*)\|^2_{\rm F} \\
			&\le& (1+\beta)(1-A) \|\mH_i^k - \nabla^2f_i(x^k)\|^2_{\rm F} \\
			&&\quad + 2B\HF^2\|x^k - x^*\|^2_{\rm F} + (1+\nicefrac{1}{\beta} + 2B)\HF^2 \|x^{k+1} - x^*\|^2_{\rm F} \\
			&\le& (1+\beta)(1-A) \|\mH_i^k - \nabla^2f_i(x^k)\|^2_{\rm F} + \(\frac{\beta+1}{2\beta} + 3B\)\HF^2 \|x^{k} - x^*\|^2_{\rm F}.
		\end{eqnarray*}
		where we use Young's inequality for some $\beta>0$. By choosing $\beta = \frac{A}{2(1-A)}$ when $0<A<1$, we get 
		\begin{eqnarray*}
			\E_{k+1} \ll\|\mH_i^{k+1} - \nabla^2 f_i(x^{*})\|^2_{\rm F}\rr
			\le \(1-\frac{A}{2}\) \|\mH_i^{k} - \nabla^2 f_i(x^{*})\|^2_{\rm F} + \(\frac{1}{A}+3B-\frac{1}{2}\)\HF^2 \|x^{k} - x^*\|^2_{\rm F}
		\end{eqnarray*}
		When $A=1$, we can choose $\beta=1$ and have
		\begin{eqnarray*}
			\E_{k+1} \ll\|\mH_i^{k+1} - \nabla^2 f_i(x^{*})\|^2_{\rm F}\rr \le \(3B+1\)\HF^2 \|x^{k} - x^*\|^2_{\rm F}.
		\end{eqnarray*}
		Thus, for all $0<A\le 1$ we get the desired bound.
	\end{proof}
	
	\subsection{Proof of Theorem \ref{th:NLU}}
	
	The proof follows the same steps as for \algname{FedNL} until the appearance of 3PC compressor. We derive recurrence relation for $\|x^k-x^*\|^2$ covering both options of updating the global model. If {\em Option 1.} is used in \algname{FedNL}, then
	
	\begin{eqnarray}
		\|x^{k+1} - x^*\|^2
		&=&   \left\|x^k-x^* - \ll\mH^{k}_{\mu}\rr^{-1} \nabla f(x^k) \right\|^2 \notag \\
		&\le& \left\| \ll\mH^{k}_{\mu}\rr^{-1} \right\|^2 \left\|\mH^{k}_{\mu}(x^k-x^*) - \nabla f(x^k))\right\|^2 \notag \\
		&\le& \frac{2}{\mu^2}\( \left\|\(\mH_{\mu}^{k} - \nabla^2 f(x^*)\)(x^k-x^*) \right\|^2 + \left\|\nabla^2 f(x^*)(x^k-x^*) - \nabla f(x^k) + \nabla f(x^*) \right\|^2\) \notag \\
		&=& \frac{2}{\mu^2}\( \left\|\(\mH_{\mu}^{k} - \nabla^2 f(x^*)\)(x^k-x^*) \right\|^2 + \left\| \nabla f(x^k) - \nabla f(x^*) - \nabla^2 f(x^*)(x^k-x^*) \right\|^2\) \notag \\
		&\le& \frac{2}{\mu^2}\(
		\left\|\mH_{\mu}^{k} - \nabla^2 f(x^*)\right\|^2 \|x^k-x^*\|^2
		+ \frac{\HS^2}{4}\|x^k-x^*\|^4
		\) \notag \\
		&=&   \frac{2}{\mu^2}\|x^k-x^*\|^2 \(
		\left\|\mH_{\mu}^{k} - \nabla^2 f(x^*)\right\|^2
		+ \frac{\HS^2}{4}\|x^k-x^*\|^2
		\) \notag \\
		&\le& \frac{2}{\mu^2}\|x^k-x^*\|^2 \(
		\left\|\mH^{k} - \nabla^2 f(x^*)\right\|^2
		+ \frac{\HS^2}{4}\|x^k-x^*\|^2
		\) \notag \\
		&\le& \frac{2}{\mu^2}\|x^k-x^*\|^2 \(
		\left\|\mH^{k} - \nabla^2 f(x^*)\right\|^2_{\rm F}
		+ \frac{\HS^2}{4}\|x^k-x^*\|^2
		\)  \notag, 
	\end{eqnarray}
	where we use $\mH_\mu^k \succeq \mu \mI$ in the second inequality, and $\nabla^2 f(x^*) \succeq \mu \mI$ in the fourth inequality. From the convexity of $\|\cdot \|^2_{\rm F}$, we have 
	$$
	\|\mH^k - \nabla^2 f(x^*)\|^2_{\rm F} = \left\| \frac{1}{n}\sum_{i=1}^n \left(  \mH_i^k - \nabla^2 f_i(x^*)  \right) \right\|^2_{\rm F} \leq \frac{1}{n}\sum_{i=1}^n \|\mH_i^k - \nabla^2 f_i(x^*)\|^2_{\rm F} = {\cal H}^k. 
	$$
	
	Thus, 
	\begin{equation}\label{eq:xk+1option1}
		\|x^{k+1} - x^*\|^2 \leq \frac{2}{\mu^2}\|x^k-x^*\|^2 {\cal H}^k + \frac{\HS^2}{2\mu^2} \|x^k-x^*\|^4. 
	\end{equation}
	
	If {\em Option 2.} is used in \algname{FedNL}, then as $\mH^k + l^k\mI \succeq \nabla^2 f(x^k) \succeq \mu \mI$ and $\nabla f(x^*) = 0$, we have 
	\begin{align*}
		\|x^{k+1} - x^*\| &= \|x^k - x^* - [\mH^k + l^k\mI]^{-1} \nabla f(x^k) \| \\
		& \leq \|[\mH^k + l^k \mI]^{-1}\| \cdot \|(\mH^k + l^k \mI) (x^k-x^*) - \nabla f(x^k) + \nabla f(x^*)\| \\ 
		& \leq \frac{1}{\mu} \|(\mH^k + l^k \mI - \nabla^2 f(x^*))(x^k-x^*)\| + \frac{1}{\mu} \|\nabla f(x^k) - \nabla f(x^*) - \nabla^2 f(x^*) (x^k-x^*)\| \\ 
		& \leq \frac{1}{\mu} \|\mH^k + l^k\mI - \nabla^2 f(x^*)\| \|x^k-x^*\| + \frac{\HS}{2\mu}\|x^k-x^*\|^2 \\ 
		& \leq \frac{1}{n\mu} \sum_{i=1}^n \|\mH_i^k + l_i^k\mI - \nabla^2 f_i(x^*)\| \|x^k-x^*\| + \frac{\HS}{2\mu}\|x^k-x^*\|^2 \\ 
		& \leq \frac{1}{n\mu} \sum_{i=1}^n (\|\mH_i^k - \nabla^2 f_i(x^*)\| + l_i^k )\|x^k-x^*\| +  \frac{\HS}{2\mu}\|x^k-x^*\|^2. 
	\end{align*}
	
	From the definition of $l_i^k$, we have 
	$$
	l_i^k = \|\mH_i^k - \nabla^2 f_i(x^k)\|_{\rm F} \leq \|\mH_i^k - \nabla^2 f_i(x^*)\|_{\rm F} + \HF \|x^k-x^*\|. 
	$$
	Thus, 
	$$
	\|x^{k+1} - x^*\|  \leq \frac{2}{n\mu} \sum_{i=1}^n \|\mH_i^k - \nabla^2 f_i(x^*)\|_{\rm F} \|x^k-x^*\| + \frac{\HS+2\HF}{2\mu}\|x^k-x^*\|^2. 
	$$
	From Young's inequality, we further have 
	\begin{align}
		\|x^{k+1} - x^*\|^2 & \leq \frac{8}{\mu^2} \left(  \frac{1}{n} \sum_{i=1}^n \|\mH_i^k - \nabla^2 f_i(x^*)\|_{\rm F} \|x^k-x^*\|   \right)^2 + \frac{(\HS+2\HF)^2}{2\mu^2} \|x^k-x^*\|^4 \nonumber \\ 
		& \leq \frac{8}{\mu^2} \|x^k-x^*\|^2 \left(  \frac{1}{n} \sum_{i=1}^n \|\mH_i^k - \nabla^2 f_i(x^*)\|^2_{\rm F}  \right) +  \frac{(\HS+2\HF)^2}{2\mu^2} \|x^k-x^*\|^4 \nonumber \\ 
		& = \frac{8}{\mu^2} \|x^k-x^*\|^2 {\cal H}^k + \frac{(\HS+2\HF)^2}{2\mu^2} \|x^k-x^*\|^4,  \label{eq:xk+1option2}
	\end{align}
	where we use the convexity of $\|\cdot\|^2_{\rm F}$ in the second inequality. 
	
	Thus, from (\ref{eq:xk+1option1}) and (\ref{eq:xk+1option2}), we have the following unified bound for both {\em Option 1} and {\em Option 2}:
	\begin{equation}\label{eq:xk+1U}
		\|x^{k+1} - x^*\|^2 \leq \frac{C}{\mu^2} \|x^k-x^*\|^2 {\cal H}^k + \frac{D}{2\mu^2} \|x^k-x^*\|^4. 
	\end{equation}
	
	Assume $\|x^0-x^*\|^2 \leq \frac{\mu^2}{2D}$ and ${\cal H}^k \leq \frac{\mu^2}{4C}$ for all $k\geq 0$. Then we show that $\|x^k-x^*\|^2 \leq \frac{\mu^2}{2D}$ for all $k\geq 0$ by induction. Assume  $\|x^k-x^*\|^2 \leq \frac{\mu^2}{2D}$ for all $k \leq K$. Then from (\ref{eq:xk+1U}), we have 
	\begin{align*}
		\|x^{K+1} - x^*\|^2 & \leq \frac{1}{4}\|x^K-x^*\|^2 + \frac{1}{4}\|x^K-x^*\|^2 \leq \frac{\mu^2}{2D}. 
	\end{align*} 
	Thus we have $\|x^k-x^*\|^2 \leq \frac{\mu^2}{2D}$ and ${\cal H}^k \leq \frac{\mu^2}{4C}$ for $k\geq 0$. Using (\ref{eq:xk+1U}) again, we obtain 
	\begin{equation}\label{eq:xk+1Ufix}
		\|x^{k+1} - x^*\|^2 \leq \frac{1}{2} \|x^k-x^*\|^2. 
	\end{equation}
	
	Assume $\|x^0-x^*\|^2 \leq \frac{\mu^2}{2D}$ and ${\cal H}^k \leq \frac{\mu^2}{4C}$ for all $k\geq 0$. Then we show that $\|x^k-x^*\|^2 \leq \frac{\mu^2}{2D}$ for all $k\geq 0$ by induction. Assume  $\|x^k-x^*\|^2 \leq \frac{\mu^2}{2D}$ for all $k \leq K$. Then from (\ref{eq:xk+1U}), we have 
	\begin{align*}
		\|x^{K+1} - x^*\|^2 & \leq \frac{1}{4}\|x^K-x^*\|^2 + \frac{1}{4}\|x^K-x^*\|^2 \leq \frac{\mu^2}{2D}. 
	\end{align*} 
	Thus we have $\|x^k-x^*\|^2 \leq \frac{\mu^2}{2D}$ and ${\cal H}^k \leq \frac{\mu^2}{4C}$ for $k\geq 0$. Using (\ref{eq:xk+1U}) again, we obtain 
	\begin{equation}\label{eq:xk+1Ufix}
		\|x^{k+1} - x^*\|^2 \leq \frac{1}{2} \|x^k-x^*\|^2. 
	\end{equation}
	
	Thus, we derived the first rate of the theorem. Next, we invoke Lemma \ref{lm:threecomp} to have an upper bound for $\cH^{k+1}$:
	$$
	\E_k[{\cal H}^{k+1}] \leq \(1-\frac{A}{2}\) {\cal H}^k + \(\frac{1}{A} + 3B\)\HF^2 \|x^k-x^*\|^2. 
	$$
	Using the above inequality and (\ref{eq:xk+1Ufix}), for Lyapunov function $\Phi^k$ we deduce
	\begin{align*}
		\E_k[\Phi^{k+1}] & \leq \(1-\frac{A}{2}\) {\cal H}^k + \(\frac{1}{A}+3B\)\HF^2 \|x^k-x^*\|^2 + 3\(\frac{1}{A}+3B\)\HF^2 \|x^k-x^*\|^2 \\ 
		& =  \(1-\frac{A}{2}\) {\cal H}^k + \(1 - \frac{1}{3}\)6\(\frac{1}{A}+3B\)\HF^2 \|x^k-x^*\|^2 \\ 
		& \leq \(1 - \min\left\{  \frac{A}{2}, \frac{1}{3}  \right\}  \) \Phi^k. 
	\end{align*}
	Hence $\E_k[\Phi^k] \leq \left(  1 - \min\left\{  \frac{A}{2}, \frac{1}{3}  \right\}  \right)^k \Phi^0$. Clearly, we further have $\E[{\cal H}^k] \leq \left(  1 - \min\left\{  \frac{A}{2}, \frac{1}{3}  \right\}  \right)^k \Phi^0$ and $\mathbb{E}[\|x^k-x^*\|^2] \leq \frac{A}{6(1+3AB)\HF^2} \left(  1 - \min\left\{  \frac{A}{2}, \frac{1}{3}  \right\}  \right)^k \Phi^0$ for $k\geq 0$. Assume $x^k\neq x^*$ for all $k$. Then from (\ref{eq:xk+1U}), we have 
	$$
	\frac{\|x^{k+1}-x^*\|^2}{\|x^k-x^*\|^2} \leq \frac{C}{\mu^2}{\cal H}^k + \frac{D}{2\mu^2}\|x^k-x^*\|^2, 
	$$
	and by taking expectation, we have 
	\begin{align*}
		\mathbb{E} \left[  \frac{\|x^{k+1}-x^*\|^2}{\|x^k-x^*\|^2}  \right] & \leq \frac{C}{\mu^2} \mathbb{E}[{\cal H}^k] + \frac{D}{2\mu^2} \mathbb{E}[\|x^k-x^*\|^2] \\ 
		& \leq  \left(  1 - \min\left\{  \frac{A}{2}, \frac{1}{3}  \right\}  \right)^k \left(  C + \frac{AD}{12(1+3AB)\HF^2}  \right) \frac{\Phi^0}{\mu^2},
	\end{align*}
	which concludes the proof.
	
	\subsection{Proof of Lemma \ref{lm:boundforbiased}}
	
	We prove this by induction. Assume $\|\mH_i^k - \nabla^2 f_i(x^*)\|^2_{\rm F}  \leq \frac{\mu^2}{4C}$ and $\|x^k-x^*\|^2 \leq e_1^2$ for $k\leq K$. Then we also have ${\cal H}^k \leq \frac{\mu^2}{4C}$ for $k\leq K$. From (\ref{eq:xk+1U}), we can get 
	\begin{align*}
		\|x^{K+1} - x^*\|^2 & \leq \frac{C}{\mu^2} \|x^K-x^*\|^2 {\cal H}^K + \frac{D}{2\mu^2} \|x^K-x^*\|^4 \\ 
		& \leq \frac{1}{4}\|x^K-x^*\|^2 + \frac{1}{4} \|x^K-x^*\|^2 \\ 
		& \leq \|x^{K} - x^*\|^2 \le e_1^2. 
	\end{align*}
	Using Lemma \ref{lm:threecomp} and the assumptions that we use non-random 3PC compressor, we have 
	\begin{align*}
		\|\mH_i^{K+1} - \nabla^2 f_i(x^{*})\|^2_{\rm F} 
		& \leq \(1-\frac{A}{2}\) \|\mH_i^K - \nabla^2 f_i(x^*) \|_{\rm F}^2 + \frac{1+3AB}{A} \HF^2 \|x^K-x^*\|^2 \\ 
		& \leq \(1-\frac{A}{2}\) \frac{\mu^2}{4C} + \frac{1+3AB}{A}\HF^2 \cdot \frac{A^2\mu^2}{8(1+3AB)C\HF^2} \\ 
		& = \frac{\mu^2}{4C}. 
	\end{align*}

	\subsection{Proof of Lemma \ref{lm:boundforcbag}}
	
	We prove this by induction. Assume $\|x^k-x^*\| \le e_1$ and $\|\mH_i^k - \nabla^2 f_i(x^*)\|^2_{\rm F}  \leq \frac{\mu^2}{4C}$ for $k\leq K$. Then we also have ${\cal H}^k \leq \frac{\mu^2}{4C}$ for $k\leq K$. From (\ref{eq:xk+1U}), we can get 
	\begin{align*}
		\|x^{K+1} - x^*\|^2 & \leq \frac{C}{\mu^2} \|x^K-x^*\|^2 {\cal H}^K + \frac{D}{2\mu^2} \|x^K-x^*\|^4 \\ 
		& \leq \frac{1}{4}\|x^K-x^*\|^2 + \frac{1}{4} \|x^K-x^*\|^2 \le e_1^2.
	\end{align*}
	From the definition
	\begin{equation}
		\mH_i^{k+1} =
		\begin{cases}
			\mH_i^k + \cC(\nabla^2 f_i(x^{k+1}) - \mH_i^k) & \text{with probability } p, \\
			\mH_i^k & \text{with probability } 1-p.
		\end{cases}
	\end{equation}
	we have two cases for $\mH_i^{k+1}$ we need to upper bound individually instead of in expectation. Note that the case $\mH_i^{k+1} = \mH_i^{k}$ is trivial as $\|\mH_i^{k+1} - \nabla^2 f_i(x^*)\|_{\rm F} = \|\mH_i^{k} - \nabla^2 f_i(x^*)\|_{\rm F} \le \frac{\mu}{2\sqrt{C}}$. For the other case when $\mH_i^{k+1} = \mH_i^k + \cC(\nabla^2 f_i(x^{k+1}) - \mH_i^k)$, we have
	\begin{eqnarray*}
		&&\|\mH_i^{k+1} - \nabla^2 f_i(x^*)\|_{\rm F} \\
		&=& \|\mH_i^k + \cC(\nabla^2 f_i(x^{k+1}) - \mH_i^k) - \nabla^2 f_i(x^*)\|_{\rm F} \\
		&\le& \|\cC(\nabla^2 f_i(x^{k+1}) - \mH_i^k) - (\nabla^2 f_i(x^{k+1}) - \mH_i^k)\|_{\rm F} + \|\nabla^2 f_i(x^{k+1}) - \nabla^2 f_i(x^*)\|_{\rm F} \\
		&\le& \sqrt{1-\alpha}\|\nabla^2 f_i(x^{k+1}) - \mH_i^k\|_{\rm F} + \HF\|x^{k+1} - x^*\| \\
		&\le& \sqrt{1-\alpha}\|\mH_i^k - \nabla^2 f_i(x^*)\|_{\rm F} + \sqrt{1-\alpha}\|\nabla^2 f_i(x^{k+1}) - \nabla^2 f_i(x^*)\|_{\rm F} + \HF\|x^{k+1} - x^*\| \\
		&\le& \sqrt{1-\alpha}\|\mH_i^k - \nabla^2 f_i(x^*)\|_{\rm F} + 2\HF\|x^{k+1} - x^*\| \\
		&\le& \sqrt{1-\alpha}\frac{\mu}{2\sqrt{C}} + 2\HF\cdot\frac{(1-\sqrt{1-\alpha})\mu}{4\sqrt{C}\HF} = \frac{\mu}{2\sqrt{C}},
	\end{eqnarray*}
	which completes our induction step and the proof.

	\section{Deferred Proofs from Section \ref{sec:N3PC-BC} (\algname{Newton-3PC-BC})}

	\subsection{Proof of Theorem \ref{th:3PCBL1}}
	
	First we have 
	\begin{align}
		\|x^{k+1} - x^*\|^2 & = \|z^k - x^* - [\mH^k]_\mu^{-1} g^k \|^2 \nonumber \\ 
		& = \left\| [\mH^k]_\mu^{-1} \left(  [\mH^k]_\mu (z^k - x^*) - (g^k - \nabla f(x^*))  \right)   \right\|^2 \nonumber \\ 
		& \leq \frac{1}{\mu^2} \left\|   [\mH^k]_\mu (z^k - x^*) - (g^k - \nabla f(x^*))   \right\|^2, \label{eq:11-BL1}
	\end{align}
	where we use $\nabla f(x^*) = 0$ in the second equality, and $\|[\mH^k]_\mu^{-1}\| \leq \frac{1}{\mu}$ in the last inequality. 
	
	If $\xi^k = 1$, then 
	\begin{align}
		& \quad \left\|   [\mH^k]_\mu (z^k - x^*) - (g^k - \nabla f(x^*))   \right\|^2 \nonumber \\ 
		& = \left\|  \nabla f(z^k) - \nabla f(x^*) - \nabla^2 f(x^*) (z^k-x^*) + (\nabla^2 f(x^*) - [\mH^k]_\mu) (z^k-x^*)  \right\|^2 \nonumber \\ 
		& \leq 2\left\|  \nabla f(z^k) - \nabla f(x^*) - \nabla^2 f(x^*) (z^k-x^*) \right\|^2  + 2\left\| (\nabla^2 f(x^*) - [\mH^k]_\mu) (z^k-x^*)  \right\|^2 \nonumber \\ 
		& \leq \frac{\HS^2}{2} \|z^k - x^*\|^4 + 2\| [\mH^k]_\mu - \nabla^2 f(x^*)\|^2 \cdot \|z^k-x^*\|^2 \nonumber \\ 
		& \leq \frac{\HS^2}{2} \|z^k - x^*\|^4 + 2\| \mH^k - \nabla^2 f(x^*)\|_{\rm F}^2 \|z^k-x^*\|^2 \nonumber \\ 
		& = \frac{\HS^2}{2} \|z^k - x^*\|^4 + 2 \left\| \frac{1}{n} \mH_i^k - \frac{1}{n} \nabla^2 f_i(x^*) \right\|^2_{\rm F} \|z^k - x^*\|^2 \nonumber \\
		& \leq \frac{\HS^2}{2} \|z^k - x^*\|^4 +  \frac{2}{n} \sum_{i=1}^n \| \mH_i^k - \nabla^2 f_i(x^*) \|^2_{\rm F} \|z^k-x^*\|^2, \label{eq:22-BL1}
	\end{align}
	where in the second inequality, we use the Lipschitz continuity of the Hessian of $f$, and in the last inequality, we use the convexity of $\|\cdot\|^2_{\rm F}$. 
	
	If $\xi^k = 0$, then 
	\begin{align}
		& \quad \left\|   [\mH^k]_\mu (z^k - x^*) - (g^k - \nabla f(x^*))   \right\|^2 \nonumber \\ 
		& = \left\|  [\mH^k]_\mu(z^k-w^k) + \nabla f(w^k) - \nabla f(x^*) - [\mH^k]_\mu (z^k - x^*)  \right\|^2 \nonumber \\ 
		& = \left\|  [\mH^k]_\mu(x^* - w^k) +   \nabla f(w^k) - \nabla f(x^*)  \right\|^2 \nonumber \\ 
		& = \left\| \nabla f(w^k) - \nabla f(x^*) - \nabla^2 f(x^*) (w^k-x^*) + (\nabla^2 f(x^*) - [\mH^k]_\mu) (w^k-x^*)   \right\|^2 \nonumber \\ 
		& \leq \frac{\HS^2}{2}\|w^k-x^*\|^4 +  2\| \mH^k - \nabla^2 f(x^*)\|_{\rm F}^2 \|w^k-x^*\|^2 \nonumber \\ 
		& \leq \frac{\HS^2}{2}\|w^k-x^*\|^4 +  \frac{2}{n} \sum_{i=1}^n \| \mH_i^k - \nabla^2 f_i(x^*) \|^2_{\rm F} \|w^k-x^*\|^2. \label{eq:33-BL1}
	\end{align}
	
	For $k\geq 1$, from the above three inequalities, we can obtain 
	\begin{align}
		\mathbb{E}_k \|x^{k+1} - x^*\|^2 & \leq \frac{\HS^2p}{2\mu^2} \|z^k - x^*\|^4 + \frac{2p}{n \mu^2} \sum_{i=1}^n \| \mH_i^k - \nabla^2 f_i(x^*) \|^2_{\rm F} \|z^k-x^*\|^2 \nonumber \\ 
		& \quad + \frac{\HS^2(1-p)}{2\mu^2}\|w^k-x^*\|^4 +  \frac{2(1-p)}{n \mu^2} \sum_{i=1}^n \| \mH_i^k - \nabla^2 f_i(x^*) \|^2_{\rm F} \|w^k-x^*\|^2 \nonumber \\ 
		& = \frac{p}{2\mu^2} \left(  \HS^2 \|z^k-x^*\|^2 + 4 {\cal H}^k  \right) \|z^k-x^*\|^2 \nonumber \\ 
		& \quad  + \frac{(1-p)}{2\mu^2} \left(  \HS^2 \|w^k-x^*\|^2 + 4 {\cal H}^k  \right) \|w^k-x^*\|^2,  \label{eq:xk+1-BL1}
	\end{align}
	where we denote ${\cal H}^k \eqdef \frac{1}{n} \sum_{i=1}^n \|\mH_i^k - \nabla^2 f_i(x^*)\|_{\rm F}^2$. 
	
	For $k=0$, since $z^0=w^0$, it is easy to verify that the above equality also holds. 
	
	From the update rule of $z^k$, we have 
	
	\begin{align*}
		\mathbb{E}_k \|z^{k+1}-x^*\|^2 & \leq (1+\alpha) \mathbb{E}_k\|z^{k+1} - x^{k+1}\|^2 + \left(1+ \frac{1}{\alpha}\right) \mathbb{E}_k\|x^{k+1}-x^*\|^2 \\ 
		& \leq (1+\alpha) (1-A_M)\|z^k-x^k\|^2 + (1+\alpha)B_M\mathbb{E}_k\|x^{k+1}-x^k\|^2 +  \left(1+ \frac{1}{\alpha}\right) \mathbb{E}_k\|x^{k+1}-x^*\|^2 \\ 
		& \leq (1+\alpha)(1-A_M)(1+\beta) \|z^k-x^*\|^2 + (1+\alpha)(1-A_M)\left(  1 + \frac{1}{\beta}  \right) \|x^k-x^*\|^2 \\ 
		& \quad + 2(1+\alpha)B_M\|x^k-x^*\|^2 + \left(  2(1+\alpha)B_M + 1 + \frac{1}{\alpha}  \right) \mathbb{E}_k\|x^{k+1}-x^*\|^2,
	\end{align*}
	for any $\alpha>0$, $\beta>0$. By choosing $\alpha = \frac{A_M}{4}$ and $\beta = \frac{A_M}{4(1-\frac{3A_M}{4})}$, we arrive at 
	
	\begin{align}
		\mathbb{E}_k\|z^{k+1}-x^*\|^2
		& \leq \left(  1 - \frac{A_M}{2}  \right) \|z^k-x^*\|^2 + \left(  \frac{4}{A_M} - 3 + \frac{5B_M}{2}  \right) \|x^k-x^*\|^2 \nonumber  \\ 
		& \qquad + \left(  \frac{4}{A_M} + 1 + \frac{5B_M}{2}  \right) \mathbb{E}_k\|x^{k+1}-x^*\|^2 \nonumber \\
		& \leq \left(  1 - \frac{A_M}{2}  \right) \|z^k-x^*\|^2 + C_M \|x^k-x^*\|^2  + C_M \mathbb{E}_k\|x^{k+1}-x^*\|^2, \label{eq:zk+1nbor-3PCBL1}
	\end{align}
	where we denote $C_M \eqdef  \frac{4}{A_M} + 1 + \frac{5B_M}{2}$. Then we have 
	\begin{eqnarray*}
		&& \mathbb{E}_k [\|z^{k+1}-x^*\|^2 + 2C_M \|x^{k+1}-x^*\|^2 ] \\
		&\leq& \left(  1 - \frac{A_M}{2}  \right) \|z^k-x^*\|^2 + C_M \|x^k-x^*\|^2  +  3C_M \mathbb{E}_k\|x^{k+1}-x^*\|^2 \\ 
		&\overset{(\ref{eq:xk+1-BL1})}{\leq}&  \left(  1 - \frac{A_M}{2}  \right) \|z^k-x^*\|^2  + \frac{3C_Mp}{2\mu^2} \left(  \HS^2 \|z^k-x^*\|^2 + 4 {\cal H}^k  \right) \|z^k-x^*\|^2 \\ 
		&& \quad + \frac{3C_M(1-p)}{2\mu^2} \left(  \HS^2 \|w^k-x^*\|^2 + 4 {\cal H}^k  \right) \|w^k-x^*\|^2 + C_M \|x^k-x^*\|^2. 
	\end{eqnarray*}
	
	Assume $\|z^k-x^*\|^2 \leq \frac{A_M \mu^2}{24C_M \HS^2}$ and ${\cal H}^k \leq \frac{A_M \mu^2}{96C_M}$ for $k\geq 0$. Then from the update rule of $w^k$, we also have  $\|w^k-x^*\|^2 \leq \frac{A_M \mu^2}{24C_M \HS^2}$ for $k\geq 0$. Therefore, we have 
	\begin{multline}\label{eq:zk+1-3PCBL1}
		\mathbb{E}_k [\|z^{k+1}-x^*\|^2 + 2C_M \|x^{k+1}-x^*\|^2 ] \leq \left(  1 - \frac{A_M}{2} + \frac{A_M p}{8}  \right) \|z^k-x^*\|^2 \\ + \frac{A_M (1-p)}{8} \|w^k-x^*\|^2 + C_M \|x^k-x^*\|^2. 
	\end{multline}
	
	From the update rule of $w^k$, we have 
	\begin{equation}\label{eq:wk+1-3PCBL1}
		\mathbb{E}_k\|w^{k+1} - x^*\|^2 = p\|z^{k+1}-x^*\|^2 + (1-p) \|w^k-x^*\|^2. 
	\end{equation}
	
	Define $\Phi_1^k \eqdef \|z^k-x^*\|^2 + C_M\|x^k-x^*\|^2 + \frac{A_M(1-p)}{4p} \|w^k-x^*\|^2$. Then we have 
	\begin{align*}
		\mathbb{E}_k[\Phi_1^{k+1}] & = \mathbb{E}_k  [\|z^{k+1}-x^*\|^2 + 2C_M \|x^{k+1}-x^*\|^2 ] + \frac{A_M(1-p)}{4p} \mathbb{E}_k\|w^{k+1}-x^*\|^2 \\ 
		& \overset{(\ref{eq:wk+1-3PCBL1})}{\leq} \left(  1 + \frac{A_M(1-p)}{4}  \right) \mathbb{E}_k  [\|z^{k+1}-x^*\|^2 + 2C_M \|x^{k+1}-x^*\|^2 ] + \frac{A_M(1-p)^2}{4p} \|w^k-x^*\|^2 \\ 
		& \overset{(\ref{eq:zk+1-3PCBL1})}{\leq} \left(  1 + \frac{A_M(1-p)}{4}  \right) \left(  1 - \frac{A_M}{2} + \frac{A_M p}{8}  \right) \|z^k-x^*\|^2 + \left(  1 + \frac{A_M(1-p)}{4}  \right) C_M \|x^k-x^*\|^2 \\ 
		& \qquad + \left(  \left(  1 + \frac{A_M(1-p)}{4}  \right) \frac{A_M(1-p)}{8} +  \frac{A_M(1-p)^2}{4p}  \right) \|w^k-x^*\|^2 \\ 
		& \leq \left(  1 - \frac{A_M}{4}  \right) \|z^k-x^*\|^2 + \left(  1 - \frac{3}{8}  \right) 2C_M \|x^k-x^*\|^2 + \frac{A_M(1-p)}{4p} \left(  1- \frac{3p}{8}  \right) \|w^k-x^*\|^2 \\ 
		& \leq \left(  1 - \frac{\min\{2A_M, 3p\}}{8}  \right) \Phi_1^k. 
	\end{align*}
	
	By applying the tower property, we have 
	$$
	\mathbb{E} [\Phi^{k+1}_1]  \leq \left(  1 - \frac{\min\{  2A_{\rm M}, 3p  \}}{8}  \right) \mathbb{E}[\Phi^k_1]. 
	$$
	Unrolling the recursion, we can get the result.

	\subsection{Proof of Lemma \ref{lm:nbor-N3PCBC-det}}
	
	We prove the results by mathematical induction. Assume the results hold for $k\leq K$. From the update rule of $w^k$, we know $\|w^k - x^*\|^2 \leq \min\{  \frac{A_M \mu^2}{24C_M \HS^2}, \frac{A_WA_M \mu^2}{384C_MC_W\HF^2}  \}$ for $k\leq K$. If $\xi^K=1$, from (\ref{eq:11-BL1}) and (\ref{eq:22-BL1}), we have 
	\begin{align}
		\|x^{K+1} - x^*\|^2 & \leq \frac{1}{\mu^2} \left(  \frac{\HS^2}{2} \|z^K-x^*\|^2 + 2{\cal H}^K  \right) \|z^K-x^*\|^2 \label{eq:xk+1-3PCBL1} \\ 
		& \leq \frac{A_M}{24C_M} \|z^K-x^*\|^2. \nonumber
	\end{align}
	If $\xi^K=0$, from $\|w^K - x^*\|^2 \leq  \min\{  \frac{A_M \mu^2}{24C_M \HS^2}, \frac{A_WA_M \mu^2}{384C_MC_W\HF^2}  \}$ and (\ref{eq:33-BL1}), we can obtain the above inequality similarly. From the upper bound of $\|z^K-x^*\|^2$, we further have $\|x^{K+1} -x^*\| \leq \frac{11A_M}{24C_M} \min\{  \frac{A_M \mu^2}{24C_M^2 \HS^2}, \frac{A_WA_M \mu^2}{384C_MC_W\HF^2}  \}$. Then from (\ref{eq:zk+1nbor-3PCBL1}) and the fact that $\cC^M_{z^k, x^k}(x^{k+1})$ is deterministic, we have 
	\begin{align*}
		\|z^{K+1}-x^*\|^2 & \leq \left(  1 - \frac{A_M}{2}  \right) \|z^K-x^*\|^2 + C_M\|x^K-x^*\|^2 + C_M\|x^{K+1}-x^*\|^2 \\ 
		& \leq \left(   1 - \frac{A_M}{2}  + \frac{A_M}{24}  \right) \|z^K-x^*\|^2 + C_M \cdot  \frac{11A_M}{24C_M} \min\left\{  \frac{A_M \mu^2}{24C_M^2 \HS^2}, \frac{A_WA_M \mu^2}{384C_MC_W\HF^2}  \right\} \\ 
		& \leq \min\left\{  \frac{A_M \mu^2}{24C_M^2 \HS^2}, \frac{A_WA_M \mu^2}{384C_MC_W\HF^2}  \right\}. 
	\end{align*}

	For $\|\mH_i^{k+1} - \nabla^2 f_i(x^*)\|_{\rm F}^2$, we have 
	\begin{align*}
		& \mathbb{E}_k \|\mH_i^{k+1} - \nabla^2 f_i(x^*)\|_{\rm F}^2 \\
		& \leq (1 + \alpha) \mathbb{E}_k \|\mH_i^k - \nabla^2 f_i (z^{k+1}) \|_{\rm F}^2 + \left(  1 + \frac{1}{\alpha}  \right) \mathbb{E}_k \| \nabla^2 f_i(z^{k+1}) - \nabla^2 f_i(x^*)\|_{\rm F}^2 \\ 
		& \leq (1+\alpha) (1-A_W) \|\mH_i^k - \nabla^2 f_i(z^k)\|_{\rm F}^2 + (1+\alpha)B_W \mathbb{E}_k\|\nabla^2 f_i(z^k) - \nabla^2 f_i(z^{k+1})\|_{\rm F}^2 \\ 
		& \quad + \left(  1 + \frac{1}{\alpha}  \right) \mathbb{E}_k \| \nabla^2 f_i(z^{k+1}) - \nabla^2 f_i(x^*)\|_{\rm F}^2 \\ 
		& \leq (1+\alpha) (1-A_W) \|\mH_i^k - \nabla^2 f_i(z^k)\|_{\rm F}^2 + (1+\alpha) B_W\HF^2 \mathbb{E}_k\|z^k-z^{k+1}\|^2 \\ 
		& \quad +  \left(  1 + \frac{1}{\alpha}  \right) \HF^2 \mathbb{E}_k \|z^{k+1}-x^*\|^2 \\ 
		& \leq (1+\alpha) (1-A_W) (1+\beta) \|\mH_i^k - \nabla^2 f_i(x^*)\|_{\rm F}^2 + (1+\alpha) (1-A_W) \left(  1 + \frac{1}{\beta}  \right) \HF^2 \|z^k-x^*\|^2 \\ 
		& \quad + 2(1+\alpha) B_W\HF^2 \|z^k-x^*\|^2 + \left(  2(1+\alpha) B_W + 1 + \frac{1}{\alpha}  \right) \HF^2 \|z^{k+1}-x^*\|^2, 
	\end{align*}
	for any $\alpha>0$, $\beta>0$. By choosing $\alpha = \frac{A_W}{4}$ and $\beta = \frac{A_W}{4(1-\frac{3A_W}{4})}$, we arrive at 
	\begin{equation}\label{eq:Hk+1nbor-3PCBL1}
		\mathbb{E}_k \|\mH_i^{k+1} - \nabla^2 f_i(x^*)\|_{\rm F}^2 \leq \left(  1 - \frac{A_W}{2}  \right) \|\mH_i^k - \nabla^2 f_i(x^*)\|_{\rm F}^2 + C_W \HF^2 \|z^k-x^*\|^2 + C_W \HF^2 \mathbb{E}_k\|z^{k+1}-x^*\|^2, 
	\end{equation}
	where we denote $C_W \eqdef \frac{4}{A_W} + 1 + \frac{5B_W}{2}$. Since $\cC^W_{\mH_i^k, \nabla^2 f_i(z^k)} (z^{k+1})$ is disterministic, from (\ref{eq:Hk+1nbor-3PCBL1}), we have 
	\begin{align*}
		{\cal H}^{K+1} & \leq \left(  1 - \frac{A_W}{2}  \right) {\cal H}^K + C_W\HF^2 \|z^K-x^*\|^2 + C_W\HF^2 \|z^{K+1}-x^*\|^2 \\ 
		& \leq \left(  1 - \frac{A_W}{2}  \right) \frac{A_M \mu^2}{96C_M} + 2C_W\HF^2 \cdot \frac{A_WA_M \mu^2}{384C_MC_W\HF^2}  \\ 
		& \leq \frac{A_M \mu^2}{96C_M}. 
	\end{align*}

	\subsection{Proof of Lemma \ref{lm:nbor-N3PCBC-conv}}
	
	We prove the results by mathematical induction. From the assumption on $\mH_i^k$, we have 
	\begin{align}
		{\cal H}^k & = \frac{1}{n} \sum_{i=1}^n \|\mH_i^k - \nabla^2 f_i(x^*)\|^2 \nonumber \\ 
		& \leq \frac{1}{n} \sum_{i=1}^n d^2 \max_{jl} \{  | (\mH_i^k)_{jl} - (\nabla^2 f(x^*))_{jl} |^2  \} \nonumber \\ 
		& \leq d^2 \HM^2 \max_{0\leq t \leq k} \|z^t-x^*\|^2. \label{eq:Gk-3PCBL1}
	\end{align}
	Then from $\|x^0-x^*\|^2 \leq {\tilde c}_1$, we have ${\cal H}^0 \leq \min \{  \frac{A_M\mu^2}{96C_M}, \frac{\mu^2}{4d}  \}$. Assume the results hold for all $k\leq K$. If $\xi^K=1$, from (\ref{eq:xk+1-3PCBL1}), we have 
	
	\begin{align*}
		\|x^{K+1} - x^*\|^2 & \leq \frac{1}{\mu^2} \left(  \frac{\HS^2}{2} \|z^K-x^*\|^2 + 2{\cal H}^K  \right) \|z^K-x^*\|^2 \\ 
		& \leq \frac{1}{d}\|z^K-x^*\|^2 \\ 
		& \leq {\tilde c}_1. 
	\end{align*}
	If $\xi^K=0$, from $\|w^K-x^*\|^2 \leq d {\tilde c}_1$ and (\ref{eq:33-BL1}), we can obtain the above inequality similarly. From the assumption on $z^k$, we have 
	\begin{align*}
		\|z^{K+1} - x^*\|^2 & \leq d \max_{j} | z_j^{K+1} -x_j^* |^2 \\ 
		& \leq d \max_{0\leq t\leq K+1} \|x^t-x^*\|^2 \\ 
		& \leq d {\tilde c}_1. 
	\end{align*}
	
	At last, using (\ref{eq:Gk-3PCBL1}), we can get ${\cal H}^{K+1} \leq \min \{  \frac{A_M\mu^2}{96C_M}, \frac{\mu^2}{4d}  \}$, which completes the proof.

	\section{Extension to Bidirectional Compression and Partial Participation}

	In this section, we unify the bidirectional compression and partial participation in Algorithm \ref{alg:BL2}. The algorithm can also be regarded as an extension of \algname{BL2} in \citep{qian2021basis} by the three point compressor. Here the symmetrization operator $[\cdot]_s$ is defined as
	$$
	[\mA]_s \eqdef \frac{\mA + \mA^\top}{2}
	$$
	
	for any $\mA \in \R^{d\times d}$. The update of the global model at $k$-th iteration is 
	$$
	x^{k+1} = \left(  [\mH^k]_s + l^k\mI  \right)^{-1} g^k, 
	$$
	where $\mH^k$, $l^k$, and $g^k$ are the average of $\mH_i^k$, $l_i^k$, and $g_i^k$ respectively. This update is based on the following step in Stochastic Newton method \citep{SN2019}
	\begin{align*}
		x^{k+1} &= \left[  \tfrac{1}{n} \sum_{i=1}^n \nabla^2 f_i(w_i^k)  \right]^{-1}  \left[  \tfrac{1}{n} \sum_{i=1}^n \left(  \nabla^2 f_i(w_i^k) w_i^k - \nabla f_i(w_i^k)  \right)  \right]. 
	\end{align*}
	We use $[\mH_i^k]_s + l_i^k\mI$ to estimate $\nabla^2 f_i(w_i^k)$, and $g_i^k$ to estimate $\nabla^2 f_i(w_i^k) w_i^k - \nabla f_i(w_i^k)$, where $l_i^k = \|[\mH_i^k]_s-\nabla^2 f_i(z_i^k)\|_{\rm F}$ is adopted to guarantee the positive definiteness of $[\mH^k]_s+l^k \mI$. Hence, like \algname{BL2} in \citep{qian2021basis}, we maintain the key relation 
	\begin{equation}\label{eq:gik-3PCBL2}
		g_i^k = ([\mH_i^k]_s + l_i^k\mI) w_i^k - \nabla f_i(w_i^k). 
	\end{equation}
	Since each node has a local model $w_i^k$, we introduce $z_i^k$ to apply the bidirectional compression with the three point compressor and $\mH_i^k$ is expected to learn $h^i(\nabla^2 f_i(z_i^k))$ iteratively.
	For the update of $g_i^k$ on the server when $\xi_i^k=0$, from (\ref{eq:gik-3PCBL2}), it is natural to let $$g_i^{k+1} - g_i^k = ([\mH_i^{k+1}]_s - [\mH_i^k]_s + l_i^{k+1}\mI - l_i^k\mI) w_i^{k+1},$$ since we have $w_i^{k+1} = w_i^k$ when $\xi_i^k=0$. The convergence results of \algname{Newton-3PC-BC-{\color{blue}PP}} are stated in the following two theorems.

	For $k\geq 0$, define Lyapunov function $$\Phi_3^k \eqdef {\cal Z}^k + \frac{2\tau C_M}{n} \|x^k-x^*\|^2 + \frac{A_M}{4p} {\cal W}^k,$$ 
	where $\tau\in[n]$ is the number of devices participating in each round.

	\begin{algorithm}[h!]
		\caption{\algname{Newton-3PC-BC-{\color{blue}PP}} (Newton's method with 3PC, BC and {\color{blue} Partial Participation})}
		\label{alg:BL2}
		\begin{algorithmic}[1]
			\STATE {\bfseries Parameters:} Worker's ($\cC^W$) and Master's ($\cC^M$) 3PC; probability $p\in(0, 1]$; ${\color{blue}0<\tau \leq n}$
			\STATE {\bfseries Initialization:}
			$w^0_i = z^0_i = x^0 \in \R^d$; $\mH_i^0 \in \R^{d\times d}$; $l_i^0 = \|[\mH_i^{0}]_s - \nabla^2 f_i(w_i^{0})\|_{\rm F}$; $g_i^0 = ([\mH_i^{0}]_s + l_i^{0} \mI)w_i^{0} - \nabla f_i(w_i^{0})$; Moreover: $\mH^0 = \tfrac{1}{n} \sum_{i=1}^n \mH_i^0$; $l^0 = \tfrac{1}{n} \sum_{i=1}^n l_i^0$; $g^0 = \tfrac{1}{n} \sum_{i=1}^n g_i^0$
			\STATE \textbf{on} server
			\STATE ~~~$x^{k+1} = \left(  [\mH^k]_s + l^k\mI  \right)^{-1} g^k$,
			\STATE ~~~{\color{blue} choose a subset $S^{k} \subseteq [n]$ such that $\mathbb{P}[ i \in S^k] = \nicefrac{\tau}{n}$ for all $i\in [n]$}
			\STATE ~~~$z_i^{k+1} = \cC^M_{z_i^k, x^k} (x^{k+1})$ for $i \in S^k$ 
			\STATE ~~~$z_i^{k+1} = z_i^k$, \quad $w_i^{k+1} = w_i^k$ for $i \notin S^k$ 
			\STATE ~~~Send $\cC^M_{z_i^k, x^k} (x^{k+1})$ to {\color{blue} the selected devices $i\in S^k$} 
			\FOR{each device $i = 1, \dots, n$ in parallel}
			\STATE {\color{blue} {\bf for participating devices} $i \in S^k$ {\bf do} }
			\STATE $z_i^{k+1} = \cC^M_{z_i^k, x^k} (x^{k+1})$
			\STATE $\mH_i^{k+1} = \cC^W_{\mH_i^k, \nabla^2 f_i(z_i^k)} (\nabla^2 f_i(z_i^{k+1}))$ 
			\STATE $l_i^{k+1} = \|[\mH_i^{k+1}]_s - \nabla^2 f_i(z_i^{k+1})\|_{\rm F}$ 
			\STATE Sample $\xi_i^{k+1} \sim \text{Bernoulli}(p)$
			\STATE {\color{blue}\textbf{if} $\xi_i^{k+1}=1$ }
			\STATE ~~~$w_i^{k+1} = z_i^{k+1}$, $g_i^{k+1} = ([\mH_i^{k+1}]_s + l_i^{k+1} \mI)w_i^{k+1} - \nabla f_i(w_i^{k+1})$, send $g_i^{k+1}-g_i^k$ to server 
			\STATE {\color{blue}\textbf{if} $\xi_i^{k+1}=0$ }
			\STATE ~~~$w_i^{k+1} = w_i^k$, $g_i^{k+1} = ([\mH_i^{k+1}]_s + l_i^{k+1} \mI)w_i^{k+1} - \nabla f_i(w_i^{k+1})$ 
			\STATE Send $\mH_i^{k+1}$, $l_i^{k+1} - l_i^k$, and $\xi_i^{k+1}$ to the server 
			\STATE {\color{blue} {\bf for non-participating devices} $i \notin S^k$ {\bf do} }
			\STATE $z_i^{k+1} = z_i^k$, $w_i^{k+1} = w_i^k$, $\mH_i^{k+1} = \mH_i^k$, $l_i^{k+1} = l_i^k$, $g_i^{k+1} = g_i^k$ 
			\ENDFOR
			
			\STATE \textbf{on} server
			\STATE ~~~{\color{blue}\textbf{if} $\xi_i^{k+1}=1$ }
			\STATE \quad \quad $w_i^{k+1} = z_i^{k+1}$, receive $g_i^{k+1}-g_i^k$
			\STATE ~~~{\color{blue}\textbf{if} $\xi_i^{k+1}=0$ } 
			\STATE \quad \quad $w_i^{k+1} = w_i^k$, $g_i^{k+1}-g_i^k = \left[ \mH_i^{k+1} - \mH_i^k \right]_s w_i^{k+1} +  (l_i^{k+1} - l_i^k) w_i^{k+1}$ 
			\STATE ~~~$g^{k+1} = g^k + \tfrac{1}{n}\sum_{i\in S^k} \left(  g_i^{k+1} - g_i^k  \right)$  
			\STATE ~~~$\mH^{k+1} = \frac{1}{n} \sum_{i=1}^n \mH_i^{k+1}$
			\STATE ~~~$l^{k+1} = l^k + \tfrac{1}{n}\sum_{i\in S^k} \left(  l_i^{k+1} - l_i^k  \right)$ 
		\end{algorithmic}
	\end{algorithm}

	\begin{theorem}\label{th:3PCBL2}
		Let Assumption \ref{asm:main}. Assume $\|z_i^k-x^*\|^2 \leq \frac{A_M \mu^2}{36(H^2 + 4\HF^2)C_M}$ and ${\cal H}^k \leq \frac{A_M\mu^2}{576C_M}$ for all $i\in [n]$ and $k\geq 0$. Then we have 
		$$
		\mathbb{E}[\Phi_3^k] \leq \left(  1 - \frac{\tau \min\{  2A_M, 3p  \} }{8n}  \right) ^k \Phi_3^0, 
		$$
		for $k\geq 0$. 
	\end{theorem}

	\begin{proof}
		
		First, similar to (30) in \citep{qian2021basis}, we can get 
		\begin{align}
			\|x^{k+1}-x^*\|^2 & \leq \frac{3\HS^2}{4\mu^2}({\cal W}^k)^2 + \frac{12{\cal W}^k}{n\mu^2} \sum_{i=1}^n \|\mH_i^k-\nabla^2 f_i(x^*)\|^2_{\rm F} + \frac{3\HF^2}{\mu^2} {\cal Z}^k {\cal W}^k \nonumber \\ 
			& = \frac{3\HS^2}{4\mu^2}({\cal W}^k)^2  +  \frac{12{\cal W}^k}{\mu^2} {\cal H}^k + \frac{3\HF^2}{\mu^2} {\cal Z}^k {\cal W}^k, \label{eq:xk+1-3PCBL2}
		\end{align}
		where ${\cal W}^k = \frac{1}{n} \sum_{i=1}^n \|w_i^k - x^*\|^2$ and ${\cal Z}^k = \frac{1}{n} \sum_{i=1}^n \|z_i^k-x^*\|^2$. For $i \in S^k$, we have $z_i^{k+1} = \cC^M_{z_i^k, x^k} (x^{k+1})$.  Then, similar to (\ref{eq:zk+1nbor-3PCBL1}), we have 
		$$
		\mathbb{E}_k \|z_i^{k+1} - x^*\|^2 \leq \left(  1 - \frac{A_M}{2}  \right) \|z_i^k - x^*\|^2 + C_M\|x^k-x^*\|^2 + C_M \|x^{k+1}-x^*\|^2. 
		$$
		
		Noticing that $\mathbb{P}[i \in S^k] = \nicefrac{\tau}{n}$ and $z_i^{k+1}=z_i^k$ for $i\notin S^k$, we further have 
		\begin{align*}
			\mathbb{E}_k\|z_i^{k+1} - x^*\|^2 & = \frac{\tau}{n} \mathbb{E}_k[\|z_i^{k+1} - x^*\|^2 \ | \  i\in S^k ] + \left(  1 - \frac{\tau}{n}  \right) \mathbb{E}_k[\|z_i^{k+1} - x^*\|^2 \ | \  i\notin S^k ] \\ 
			& \leq \frac{\tau}{n} \left(1 - \frac{A_M}{2} \right) \|z_i^k - x^*\|^2 + \frac{\tau C_M}{n} \|x^k-x^*\|^2 +  \frac{\tau C_M}{n} \|x^{k+1} - x^*\|^2 + \left(  1 - \frac{\tau}{n}  \right) \|z_i^k-x^*\|^2 \\ 
			& = \left(  1 - \frac{\tau A_{\rm M}}{2n}  \right) \|z_i^k-x^*\|^2 + \frac{\tau C_M}{n} \|x^k-x^*\|^2 +  \frac{\tau C_M}{n} \|x^{k+1} - x^*\|^2, 
		\end{align*}
		which implies that 
		\begin{align}
			\mathbb{E}_k [{\cal Z}^{k+1}] & = \frac{1}{n} \sum_{i=1}^n \mathbb{E}_k \|z_i^{k+1}-x^*\|^2  \nonumber \\ 
			& \leq \frac{1}{n} \sum_{i=1}^n \left(  1 - \frac{\tau A_{\rm M}}{2n}  \right) \|z_i^k-x^*\|^2 + \frac{\tau C_M}{n} \|x^k-x^*\|^2 +  \frac{\tau C_M}{n} \|x^{k+1} - x^*\|^2  \nonumber \\ 
			& =  \left(  1 - \frac{\tau A_{\rm M}}{2n}  \right) {\cal Z}^k +  \frac{\tau C_M}{n} \|x^k-x^*\|^2 +  \frac{\tau C_M}{n} \|x^{k+1} - x^*\|^2. \label{eq:Zk+1-3PCBL2}
		\end{align}
		
		Combining (\ref{eq:xk+1-3PCBL2}) and (\ref{eq:Zk+1-3PCBL2}), we have 
		\begin{align*}
			&\mathbb{E}_k [{\cal Z}^{k+1} + \frac{2\tau C_M}{n} \|x^{k+1} - x^*\|^2] \\
			& \leq \left(  1 - \frac{\tau A_{\rm M}}{2n}  \right) {\cal Z}^k +  \frac{\tau C_M}{n} \|x^k-x^*\|^2 +  \frac{3\tau C_M}{n} \|x^{k+1} - x^*\|^2 \\ 
			& \leq \left(  1 - \frac{\tau A_{\rm M}}{2n}  \right) {\cal Z}^k +  \frac{\tau C_M}{n} \|x^k-x^*\|^2 +  \frac{3\tau C_M}{n} \left(  \frac{3\HS^2}{4\mu^2} {\cal W}^k + \frac{12{\cal H}^k}{\mu^2} + \frac{3\HF^2{\cal Z}^k}{\mu^2}  \right) {\cal W}^k. 
		\end{align*}
		
		Assume $\|z_i^k-x^*\|^2 \leq \frac{A_M \mu^2}{36(\HS^2 + 4\HF^2)C_M}$ and ${\cal H}^k \leq \frac{A_M\mu^2}{576C_M}$ for all $i\in [n]$ and $k\geq 0$. Then we have 
		$$
		\frac{3\HS^2}{4\mu^2} {\cal W}^k + \frac{12{\cal H}^k}{\mu^2} + \frac{3\HF^2{\cal Z}^k}{\mu^2}  \leq \frac{A_M}{24C_M}, 
		$$
		which indicates that 
		\begin{equation}\label{eq:zk+1-3PCBL2}
			\mathbb{E}_k [{\cal Z}^{k+1} + \frac{2\tau C_M}{n} \|x^{k+1} - x^*\|^2] \leq \left(  1 - \frac{\tau A_{\rm M}}{2n}  \right) {\cal Z}^k +  \frac{\tau C_M}{n} \|x^k-x^*\|^2 + \frac{\tau A_M}{8n} {\cal W}^k. 
		\end{equation}

		For ${\cal W}^k$, similar to (32) in \citep{qian2021basis}, we have 
		$$
		\mathbb{E}_k [{\cal W}^{k+1}] = \left(  1 - \frac{\tau p}{n}  \right) {\cal W}^k + \frac{\tau p}{n} \mathbb{E} [{\cal Z}^{k+1}]. 
		$$
		
		Then from the above two inequalities we have 
		\begin{align*}
			&\mathbb{E}_k[\Phi_3^{k+1}] \\
			& \leq \left(  1 + \frac{\tau A_M}{4n}  \right) \mathbb{E}_k [{\cal Z}^{k+1} + \frac{2\tau C_M}{n}\|x^{k+1}-x^*\|^2] + \frac{A_M}{4p} \left(  1 - \frac{\tau p}{n}  \right) {\cal W}^k \\ 
			& \overset{(\ref{eq:zk+1-3PCBL2})}{\leq} \left(  1 - \frac{\tau A_M}{4n}  \right) {\cal  Z}^k + \left(  1 + \frac{\tau A_M}{4n}  \right) \frac{\tau C_M}{n} \|x^k-x^*\|^2 + \frac{A_M}{4p} \left(  1 - \frac{\tau p}{n} + \frac{\tau p}{2n} \left(  1 + \frac{\tau A_M}{4n}  \right)  \right) {\cal W}^k \\ 
			& \leq \left(  1 - \frac{\tau \min\{  2A_M, 3p  \} }{8n}  \right) \Phi_3^k. 
		\end{align*}
		
		By applying the tower property, we have 
		$$
		\mathbb{E}[\Phi_3^{k+1}] \leq \left(  1 - \frac{\tau \min\{  2A_M, 3p  \} }{8n}  \right) \mathbb{E}[\Phi_3^k]. 
		$$
		
		Unrolling the recursion, we can obtain the result. 
		
	\end{proof}

	Define $\Phi_4^{k} = {\cal H}^k+ \frac{16C_W\HF^2}{A_M} \|x^k-x^*\|^2$ for $k\geq 0$, where $C_W \eqdef \frac{4}{A} + 1 + \frac{5B}{2}$. 
	
	\begin{theorem}\label{th:supl-3PCBL2}
		Let Assumption \ref{asm:main} holds, $\xi^k \equiv 1$, $S^k\equiv [n]$, and $\cC_{z_i^k, x^k}^M (x^{k+1}) \equiv x^{k+1}$ for all $i\in[n]$ and $k\geq 0$. Assume $\|z_i^k-x^*\|^2 \leq \frac{A_M \mu^2}{36(\HS^2 + 4\HF^2)C_M}$ and ${\cal H}^k \leq \frac{A_M\mu^2}{576C_M}$ for all $i\in [n]$ and $k\geq 0$. Then we have 
		$$
		\mathbb{E}[\Phi_4^k] \leq \theta_2^k \Phi_4^0, 
		$$
		
		$$
		\mathbb{E} \left[   \frac{\|x^{k+1}-x^*\|^2}{\|x^k-x^*\|^2} \right]  \leq \theta_2^k \left( \frac{3(\HS^2+4\HF^2)A_{\rm M}}{64C_W\HF^2\mu^2} + \frac{12}{\mu^2}  \right) \Phi_4^0. 
		$$
		for $k\geq 0$, where $\theta_2 \eqdef \left(  1 - \frac{\min\{  2A_W, A_M  \}}{4}  \right)$. 
	\end{theorem}

	\begin{proof}
		
		Since $\xi^k\equiv 1$, $S^k \equiv [n]$, and $\cC^M_{z_i^k, x^k}(x^{k+1}) \equiv x^{k+1}$ for all $i\in [n]$ and $k\geq 0$, we have $z_i^k \equiv w_i^k \equiv x^k$ for all $i\in [n]$ and $k\geq 0$. Then from (\ref{eq:zk+1-3PCBL2}), we have 
		\begin{equation}\label{eq:xk+1-3PCBL2supp}
			\mathbb{E}_k\|x^{k+1}-x^*\|^2 \leq \left(  1 - \frac{3A_M}{8}  \right) \|x^k-x^*\|^2. 
		\end{equation}

		For $\|\mH_i^{k+1} - \nabla^2 f_i(x^*)\|_{\rm F}^2$, similar to (\ref{eq:Hk+1nbor-3PCBL1}), we have 
		$$
		\mathbb{E}_k \|\mH_i^{k+1} - \nabla^2 f_i(x^*)\|_{\rm F}^2 \leq \left(  1 - \frac{A_W}{2}  \right) \|\mH_i^k - \nabla^2 f_i(x^*)\|_{\rm F}^2 + C_W \HF^2 \|z_i^k-x^*\|^2 + C_W \HF^2 \mathbb{E}_k\|z_i^{k+1}-x^*\|^2. 
		$$
		
		Considering $z_i^k \equiv x^k$, we further have 
		\begin{align*}
			\mathbb{E}_k \|\mH_i^{k+1} - \nabla^2 f_i(x^*)\|_{\rm F}^2 & \leq \left(  1 - \frac{A_W}{2}  \right) \|\mH_i^k - \nabla^2 f_i(x^*)\|_{\rm F}^2 + C_W \HF^2 \|x^k-x^*\|^2 + C_W \HF^2 \mathbb{E}_k\|x^{k+1}-x^*\|^2 \\ 
			& \overset{(\ref{eq:xk+1-3PCBL2supp})}{\leq} \left(  1 - \frac{A_W}{2}  \right) \|\mH_i^k - \nabla^2 f_i(x^*)\|_{\rm F}^2 + 2 C_W \HF^2 \|x^k-x^*\|^2, 
		\end{align*}
		which implies that 
		\begin{equation}\label{eq:Gk+1-3PCBL2}
			\mathbb{E}_k [{\cal H}^{k+1}] \leq \left(  1 - \frac{A_W}{2}  \right) {\cal H}^k + 2C_W\HF^2 \|x^k-x^*\|^2. 
		\end{equation}

		Thus, we have 
		\begin{align*}
			\mathbb{E}_k[\Phi_4^{k+1}] & = \mathbb{E}_k [{\cal H}^{k+1}] + \frac{16C_W\HF^2}{A_M} \mathbb{E}_k\|x^{k+1}-x^*\|^2 \\ 
			& \leq \left(  1 - \frac{A_W}{2}  \right) {\cal H}^k + 2C_W\HF^2 \|x^k-x^*\|^2 + \frac{16C_W\HF^2}{A_M} \mathbb{E}_k\|x^{k+1}-x^*\|^2 \\ 
			& \overset{(\ref{eq:xk+1-3PCBL2supp})}{\leq} \left(  1 - \frac{\min\{  2A_W, A_M  \}}{4}  \right) \Phi_4^k. 
		\end{align*}
		
		By applying the tower property, we have $\mathbb{E}[\Phi_4^{k+1}] \leq \theta_1 \mathbb{E}[\Phi_4^k]$. Unrolling the recursion, we have $\mathbb{E}[\Phi_4^k] \leq \theta_2^k \Phi_4^0$. Then we further have $\mathbb{E}[{\cal H}^k] \leq \theta_2^k \Phi_4^0$ and $\mathbb{E}\|x^k-x^*\|^2 \leq \frac{A_{\rm M}}{16C_W\HF^2} \theta_2^k \Phi_4^0$.

		From (\ref{eq:xk+1-3PCBL2}), we can get 
		$$
		\|x^{k+1} - x^*\|^2 \leq \frac{1}{\mu^2} \left(  \frac{3(\HS^2+4\HF^2)}{4} \|x^k-x^*\|^2 + 12 {\cal H}^k  \right) \|x^k-x^*\|^2. 
		$$
		
		Assume $x^k\neq x^*$ for all $k\geq 0$. Then we have 
		$$
		\frac{\|x^{k+1}-x^*\|^2}{\|x^k-x^*\|^2} \leq \frac{1}{\mu^2} \left(  \frac{3(\HS^2+4\HF^2)}{4} \|x^k-x^*\|^2 + 12 {\cal H}^k  \right), 
		$$
		and by taking expectation, we arrive at  
		\begin{align*}
			\mathbb{E} \left[   \frac{\|x^{k+1}-x^*\|^2}{\|x^k-x^*\|^2} \right] & \leq \frac{3(\HS^2+4\HF^2)}{4\mu^2} \mathbb{E}\|x^k-x^*\|^2 + \frac{12}{\mu^2} \mathbb{E}[{\cal H}^k] \\ 
			& \leq \theta_2^k \left(  \frac{3(\HS^2+4\HF^2)A_{\rm M}}{64C_W\HF^2\mu^2} + \frac{12}{\mu^2}  \right) \Phi_4^0. 
		\end{align*}
		
	\end{proof}

	Next, we explore under what conditions we can guarantee the boundedness of $\|z_i^k-x^*\|^2$ and ${\cal H}^k$. 
	
	\begin{theorem}\label{th:nbor-3PCBL2}
		Let Assumption \ref{asm:main} holds. \\ 
		(i) Let $\cC^M$ and $\cC^W$ be deterministic. Assume $$\|x^0 - x^*\|^2 \leq \frac{11A_M}{24C_M}\min\left\{  \frac{A_M \mu^2}{36(\HS^2+4\HF^2)C_M}, \frac{A_WA_M \mu^2}{2304C_MC_W\HF^2}  \right\} \quad\text{and}\quad {\cal H}^0 \leq \frac{A_M \mu^2}{576C_M}.$$ Then we have $$\|x^k-x^*\| \leq \frac{11A_M}{24C_M} \min\{  \frac{A_M \mu^2}{36(\HS^2+4\HF^2)C_M}, \frac{A_WA_M \mu^2}{2304C_MC_W\HF^2}  \}$$, $$\|z_i^k - x^*\|^2 \leq \min\{  \frac{A_M \mu^2}{36(\HS^2+4\HF^2)C_M}, \frac{A_WA_M \mu^2}{2304C_MC_W\HF^2}  \}$$ and ${\cal H}^k \leq  \frac{A_M \mu^2}{576C_M}$ for all $i\in [n]$ and $k\geq 0$. \\ 
		
		(ii) Assume $(z_i^k)_j$ is a convex combination of $\{(x^t)_j\}_{t=0}^k$, and $(\mH_i^k)_{jl}$ is a convex combination of $\{  (\nabla^2 f_i(z_i^k))_{jl}  \}_{t=0}^k$ for all $i\in [n]$, $j,l \in [d]$, and $k\geq 0$. If $$\|x^0-x^*\|^2 \leq {\tilde c}_2 \eqdef \min\left\{ \frac{2\mu^2}{3d^2 (\HS^2+4\HF^2)}, \frac{A_M \mu^2}{36dC_M(\HS^2+4\HF^2)}, \frac{A_M\mu^2}{576d^3C_M\HM^2}, \frac{\mu^2}{24d^4\HM^2} \right\}$$, then $\|z_i^k-x^*\|^2 \leq d {\tilde c}_2$ and ${\cal H}^k \leq \min \{  \frac{A_M\mu^2}{576C_M}, \frac{\mu^2}{24d}  \}$ for all $i\in [n]$ and $k\geq 0$. 
		
	\end{theorem}
	
	\begin{proof}
		
		The proof is similar to that of  Lemmas \ref{lm:nbor-N3PCBC-det} and \ref{lm:nbor-N3PCBC-conv}. Hence we omit it. 
		
	\end{proof}

	\section{Globalization Through Cubic Regularization and Line Search Procedure}
	
	So far, we have discussed only the local convergence of our methods. To prove global rates, one must incorporate additional regularization mechanisms. Otherwise, global convergence cannot be guaranteed.
	Due to the smooth transition from contractive compressors to general 3PC mechanism, we can easily adapt two globalization strategies of \algname{FedNL} (equivalent to \algname{Newton-EF21}) to our \algname{Newton-3PC} algorithm.
	
	The two globalization strategies are {\em cubic regularization} and {\em line search procedure}. We only present the extension with cubic regularization \algname{Newton-3PC-CR} (Algorithm \ref{alg:N3PC-CR}) analogous to \algname{FedNL-CR} \citep{FedNL2021}. Similarly, line search procedure can be combined as it was done in \algname{FedNL-LS} \citep{FedNL2021}.
	
	\begin{algorithm}[H]
		\caption{\algname{Newton-3PC-CR} (Newton's method with 3PC and {\color{blue}Cubic Regularization})}
		\label{alg:N3PC-CR}
		\begin{algorithmic}[1]
			\STATE \textbf{Input:} $x^0\in\R^d,\, \mH_1^0, \dots, \mH_n^0 \in \R^{d\times d},\, \mH^0 \eqdef \frac{1}{n}\sum_{i=1}^n \mH_i^0,\, l^0 = \frac{1}{n}\sum_{i=1}^n \|\mH_i^0 - \nabla^2 f_i(x^0)\|_{\rm F}$
			\STATE {\bf on} master
			\STATE \quad {\color{blue}$h^k = \arg\min_{h\in\R^d}T_k(h)$, where $T_k(h) \eqdef \<\nabla f(x^k),h\> + \frac{1}{2}\<(\mH^k+l^k\mI)h,h\> + \frac{\HS}{6}\|h\|^3$}
			\STATE \quad Update global model to $x^{k+1} = x^k + {\color{blue}h^k}$ and send to the nodes
			\FOR{each device $i = 1, \dots, n$ in parallel} 
			\STATE Get $x^{k+1}$ and compute local gradient $\nabla f_i(x^{k+1})$ and local Hessian $\nabla^2 f_i(x^{k+1})$
			\STATE Take $\nabla^2f_i(x^k)$ from memory and update $\mH^{k+1}_i = \cC_{\mH_i^k, \nabla^2f_i(x^k)}(\nabla^2f_i(x^{k+1}))$
			\STATE Send $\nabla f_i(x^{k+1})$,\; $\mH^{k+1}_i$ and $l_i^{k+1} \eqdef \|\mH_i^{k+1} - \nabla^2 f_i(x^{k+1})\|_{\rm F}$ to the server
			\ENDFOR
			\STATE \textbf{on} server
			\STATE \quad Aggregate $ \nabla f(x^{k+1}) = \frac{1}{n}\sum_{i=1}^n \nabla f_i(x^{k+1}), \mH^{k+1} = \frac{1}{n}\sum_{i=1}^n\mH_i^{k+1}, l^{k+1} = \frac{1}{n}\sum_{i=1}^n l_i^{k+1}$
		\end{algorithmic}
	\end{algorithm}
	
	We omit theoretical analysis of these extension as they can be obtained directly from \algname{FedNL} approach with minor adaptations. In particular, one can get global linear rate for \algname{Newton-3PC-CR}, global $\cO(\frac{1}{k})$ rate for general convex case and the same fast local rates \eqref{rate:local-linear-iter} and \eqref{rate:local-superlinear-iter} of \algname{Newton-3PC}.

	\section{Additional Experiments and Extended Numerical Analysis}

	In this section we provide extended variety of experiments to analyze the empirical performance of \algname{Newton-3PC}. We study the efficiency of \algname{Newton-3PC} in different settings changing 3PC compressor and comparing with other second-order state-of-the-art algorithms. Tests were carried out on logistic regression problem with L2 regularization
	\begin{equation}\label{eq:logreg_problem}
		\min\limits_{x\in\R^d}\left\{\frac{1}{n}\sum\limits_{i=1}^n f_i(x) +\frac{\lambda}{2}\|x\|^2\right\}, \quad f_i(x) = \frac{1}{m}\sum_{j=1}^m\log\(1+\exp(-b_{ij}a_{ij}^\top x)\),
	\end{equation}
	where $\{a_{ij},b_{ij}\}_{j\in [m]}$ are data points at the $i$-th device. On top of that, we also consider L2 regularized Softmax problem  of the form
	\begin{equation}\label{eq:softmax_problem}
		\min\limits_{x\in\R^d}\left\{\frac{1}{n}\sum_{i=1}^n f_i(x) + \frac{\lambda}{2}\|x\|^2\right\}, \quad f_i(x) = \sigma\log\left(\sum\limits_{j=1}^m\exp\left(\frac{a_{ij}^\top x-b_{ij}}{\sigma}\right)\right),
	\end{equation}
	where $\sigma > 0$ is a smoothing parameter. One can show that this function has both Lipschitz continuous gradient and Lipschitz continuous Hessian (see example $2.1$ in \citep{Doikov2021}). Let $\tilde{a}_{ij}$ be initial data points, and $\tilde{f}_i$ be defined as in \eqref{eq:softmax_problem}
	$$
	\tilde{f}_i(x) = \sigma\log\left(\sum\limits_{j=1}^m\exp\left(\frac{\tilde{a}_{ij}^\top x-b_{ij}}{\sigma}\right)\right).
	$$
	Then data shift is performed as follows
	$$a_{ij} = \tilde{a}_{ij} - \tilde{f}_i(0), j \in [m], i \in [n].$$
	After such shift we may claim that $0$ is the optimum since $\nabla f(0)=0$. Note that this problem does not belong to the class of {\it generalized linear models.}
	
	\subsection{Datasets split}
	
	We use standard datasets from LibSVM library \citep{chang2011libsvm}. We shuffle and split each dataset into $n$ equal parts representing a local data of $i$-th client. Exact names of datasets and values of $n$ are shown in Table~\ref{tab:datasets}.

	\begin{table}[h]
		\caption{Datasets used in the experiments with the number of worker nodes $n$ used in each case.}
		\label{tab:datasets}
		\centering
		\begin{tabular}{|l|r|r|r|}
			\hline
			{\bf Data set} & {\bf \# workers} $n$ & {\bf total \# of data points} ($=nm$) & {\bf \# features} $d$                 \\
			\hline
			\dataname{a1a} & $16$ & $1600$ & $123$\\ \hline
			\dataname{a9a} & $80$ & $32560$ & $123$\\ \hline
			\dataname{w2a} & $50$ & $3450$ & $300$\\ \hline 
			\dataname{w8a} & $142$ & $49700$ & $300$\\ \hline
			\dataname{phishing} & $100$ & $11000$ & $68$\\
			\hline
		\end{tabular}
	\end{table}
	
	\subsection{Choice of parameters}
	
	We follow the authors' choice of DINGO \citep{DINGO} in choosing hyperparameters: $\theta=10^{-4}, \phi=10^{-6}, \rho=10^{-4}$. Besides, \algname{DINGO} uses a backtracking line search that selects the largest stepsize from $\{1,2^{-1},\dots,2^{-10}\}.$ The initialization of $\mH^0_i$ for \algname{Newton-3PC}, \algname{FedNL} \citep{FedNL2021} and its extensions, \algname{NL1} \citep{Islamov2021NewtonLearn} is $\nabla^2 f_i(x^0)$ if it is not specified directly. For \algname{Fib-IOS} \citep{IOSFabbro2022} we set $d_k^i = 1$. Local Hessians are computed following the partial sums of Fibonacci number and the parameter $\rho=\lambda_{q_{j+1}}$. This is stated in the description of the method. The parameters of backtracking line search for \algname{Fib-IOS} are $\alpha=0.5$ and $\beta=0.9$.
	
	We conduct experiments for two values of regularization parameter $\lambda\in \{10^{-3}, 10^{-4}\}$. In the figures we plot the relation of the optimality gap $f(x^k)-f(x^*)$ and the number of communicated bits per node. In the heatmaps numbers represent the communication complexity per client of \algname{Newton-3PC} for some specific choice of 3PC compression mechanism (see the description in corresponding section). The optimal value $f(x^*)$ is chosen as the function value at the $20$-th iterate of standard Newton's method.

	In our experiments we use various compressors for the methods. Examples of classic compression mechanisms include Top-$K$ and Rank-$R$. The parameters of these compressors are parsed in details in Section~A.3 of \citep{FedNL2021}; we refer a reader to this paper for disaggregated description of aforementioned compression mechanisms. Besides, we use various 3PC compressors introduced in \citep{richtarik3PC}.
	
	\subsection{Performance of \algname{Newton-3PC} on Softmax problem}
	
	In the main part we include the comparison of \algname{Newton-3PC} method against others. In this section we additionally compare them on Softmax problem \eqref{eq:softmax_problem}. We would like to note that since Softmax problem is no longer GLM, then \algname{NL1} \citep{Islamov2021NewtonLearn} can not be implemented for considered problem. 
	
	We compare \algname{Newton-CBAG} combined with Top-$d$ compressor and probability $p=0.75$, \algname{Newton-EF21} (equivalent to \algname{FedNL} \citep{FedNL2021}) with Rank-$1$ compressor, \algname{DINGO} \citep{DINGO}, and \algname{Fib-IOS} \citep{IOSFabbro2022}. As we can see in Figure~\ref{fig:Newton-3PC-softmax}, \algname{Newton-CBAG} and \algname{Newton-EF21} demonstrate almost equivalent performance: in some cases slightly better the first one (\dataname{a1a} dataset), in some cases~--- the second (\dataname{phishing} dataset). Furthermore, \algname{DINGO} and \algname{Fib-IOS} are significantly slower than \algname{Newton-3PC} methods in terms of communication complexity.
	
	\begin{figure}[t]
		\begin{center}
			\begin{tabular}{cccc}
				\includegraphics[width=0.22\linewidth]{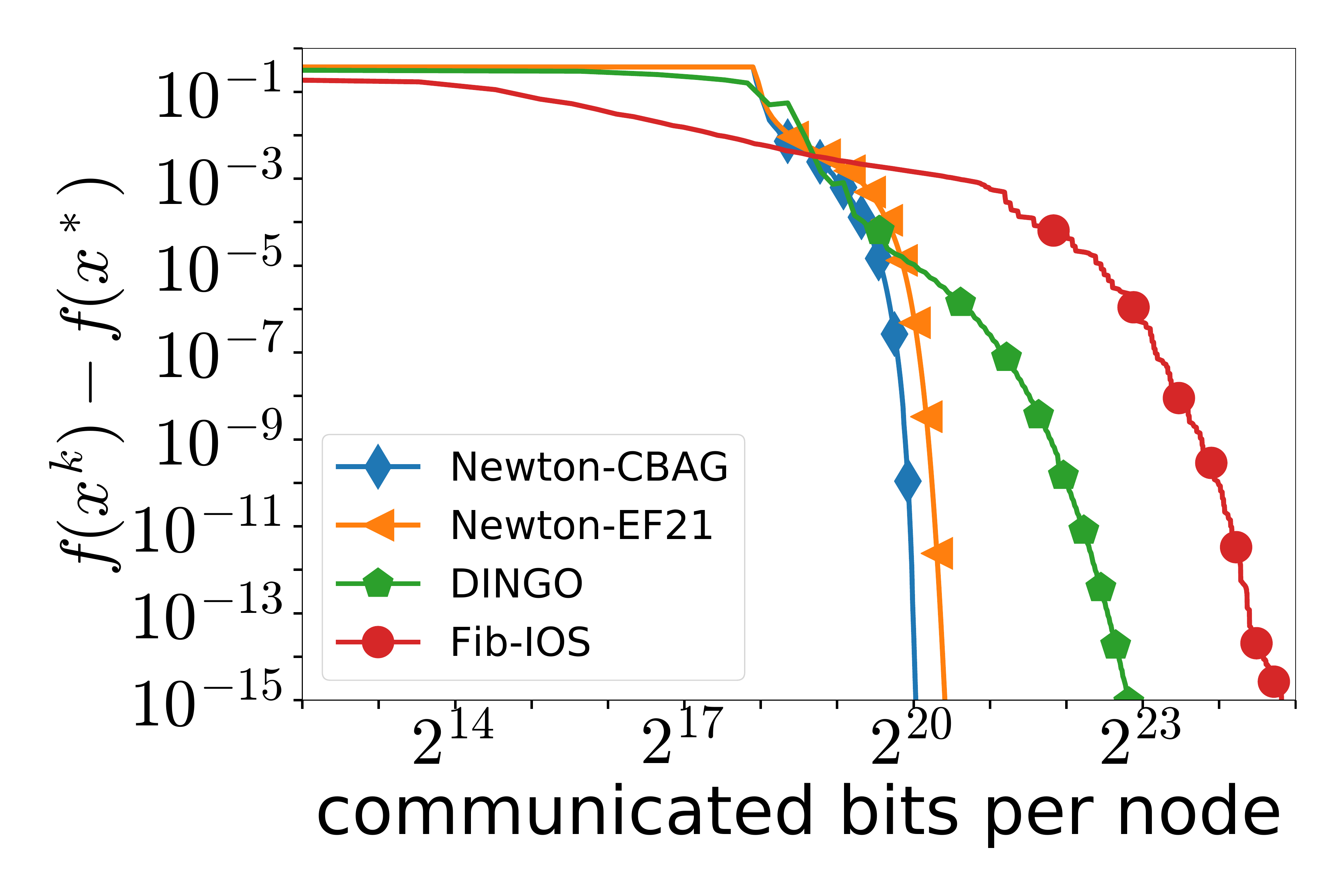} &
				\includegraphics[width=0.22\linewidth]{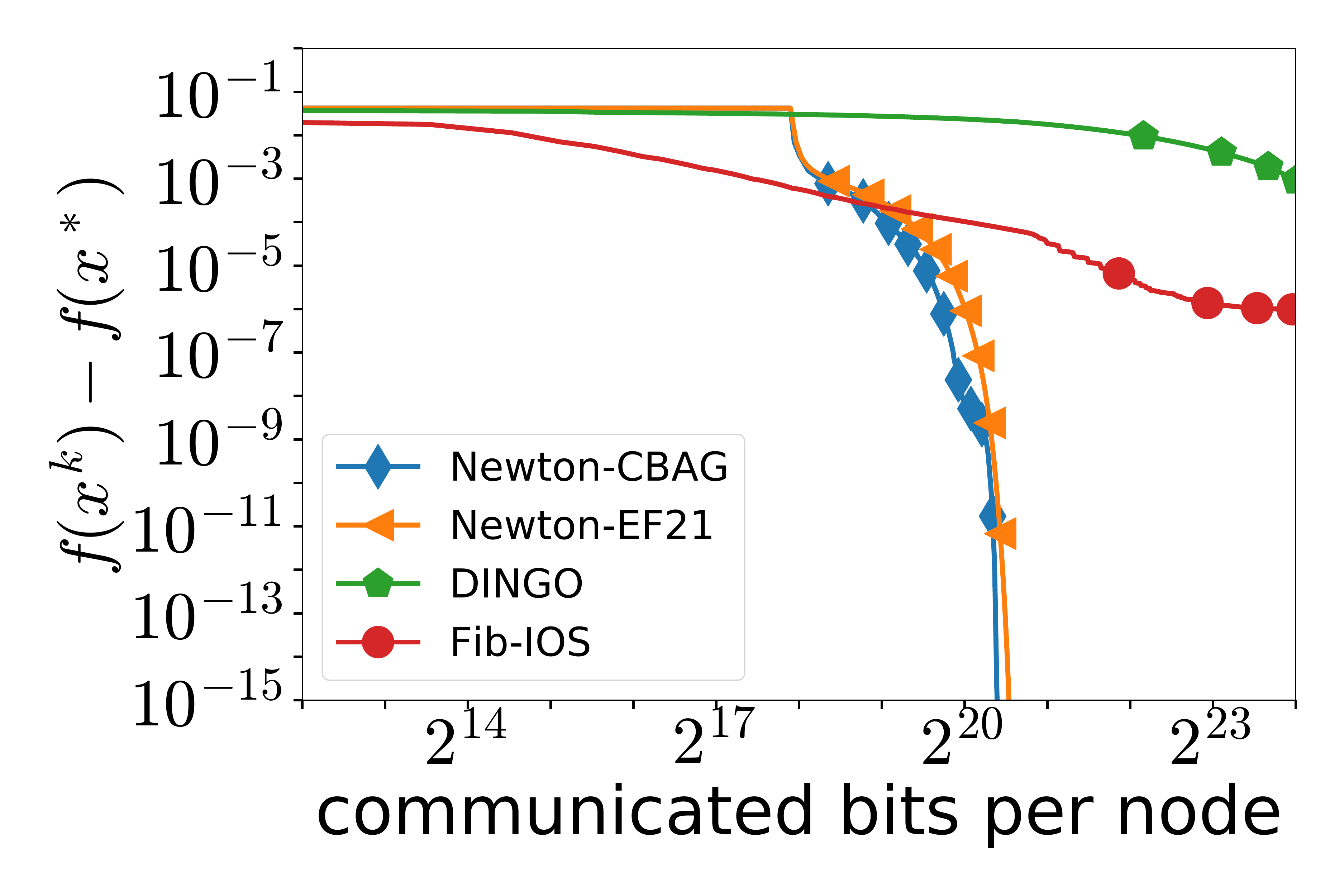} &
				\includegraphics[width=0.22\linewidth]{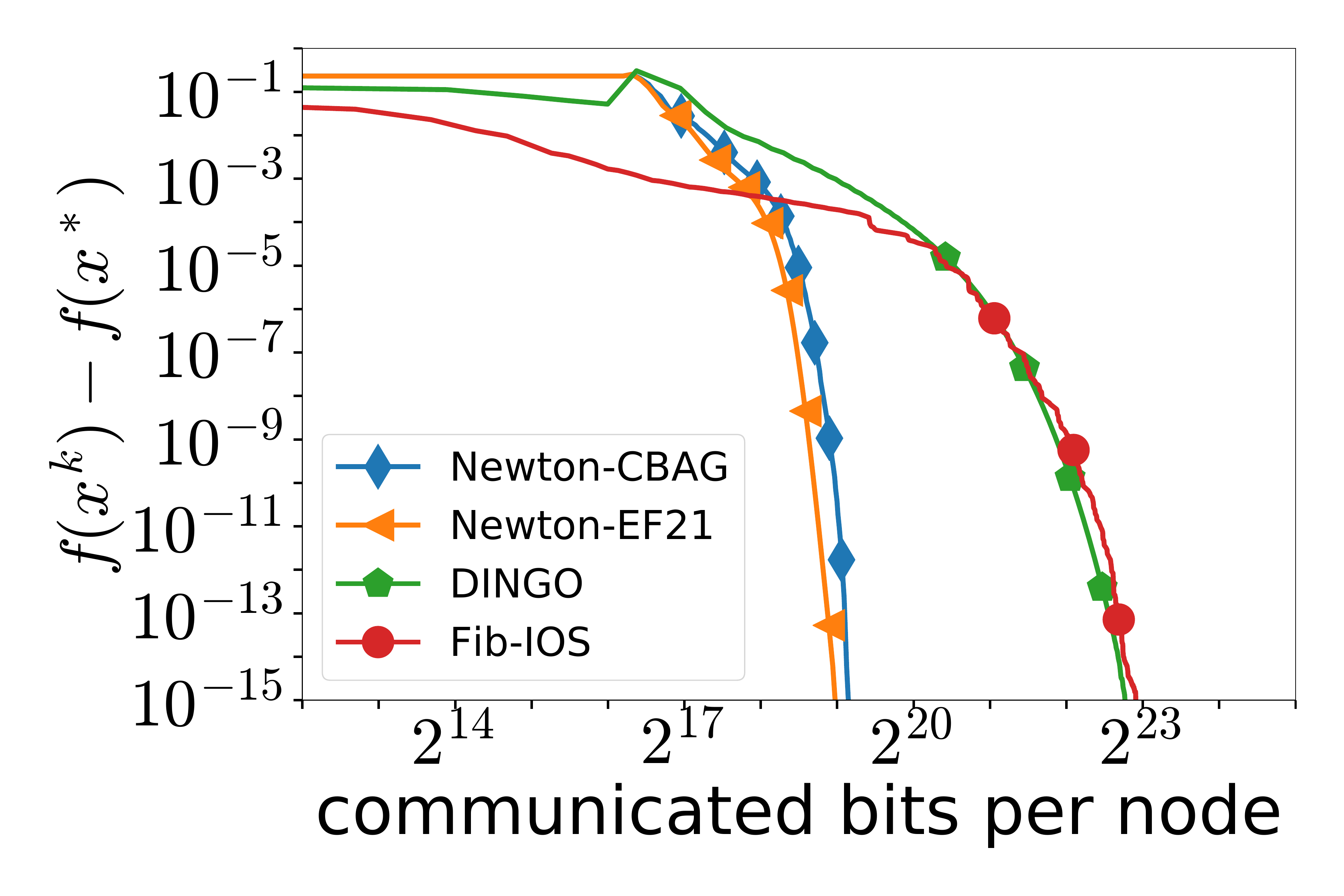} & 
				\includegraphics[width=0.22\linewidth]{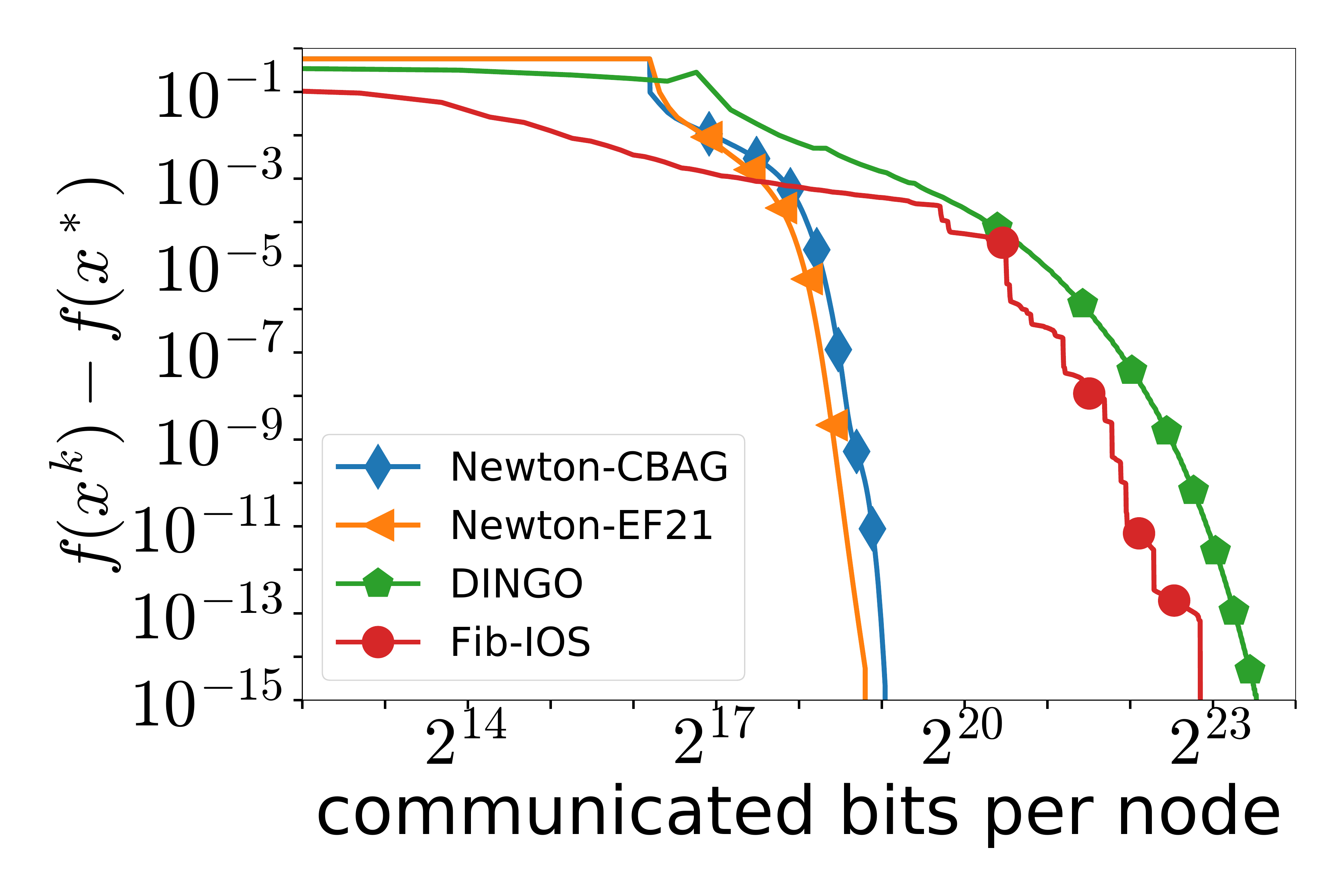}\\
				(a) \dataname{a1a} &
				(b) \dataname{a1a} &
				(c) \dataname{phishing} &
				(d) \dataname{phishing} \\
				{\scriptsize$\sigma=0.5, \lambda=10^{-3}$} & 
				{\scriptsize $\sigma=0.1, \lambda=10^{-4}$} &
				{\scriptsize$\sigma=0.1, \lambda=10^{-3}$} &
				{\scriptsize$\sigma=0.5, \lambda=10^{-3}$}
				
			\end{tabular}       
		\end{center}
		\caption{The performance of \algname{Newton-CBAG} combined with Top-$d$ compressor and probability $p=0.75$, \algname{Newton-EF21} with Rank-$1$ compressor, \algname{DINGO}, and \algname{Fib-ISO} in terms of communication complexity on Softmax problem.}
		\label{fig:Newton-3PC-softmax}
	\end{figure}

	\subsection{Behavior of \algname{Newton-CLAG} based on Top-$K$ and Rank-$R$ compressors}
	
	Next, we study how the performance of \algname{Newton-CLAG} changes when we vary parameters of biased compressor CLAG compression mechanism is based on. In particular, we test \algname{Newton-CLAG} combined with Top-$K$ and Rank-$R$ compressors modifying compression level (parameters $K$ and $R$ respectively) and trigger parameter $\zeta$. We present the results as heatmaps in Figure~\ref{fig:Newton-CLAG} indicating the communication complexity in Mbytes for particular choice of a pair of parameters (($K$, $\zeta$) or ($R$, $\zeta$) for CLAG based on Top-$K$ and Rank-$R$ respectively) .
	
	First, we can highlight that in special cases \algname{Newton-CLAG} reduces to \algname{FedNL} ($\zeta=0$, left column) and \algname{Newton-LAG} (compression is identity, bottom row). Second, we observe slight improvement from using the lazy aggregation. 
	
	\begin{figure}[t]
		\begin{center}
			\begin{tabular}{ccc}
				\includegraphics[width=0.22\linewidth]{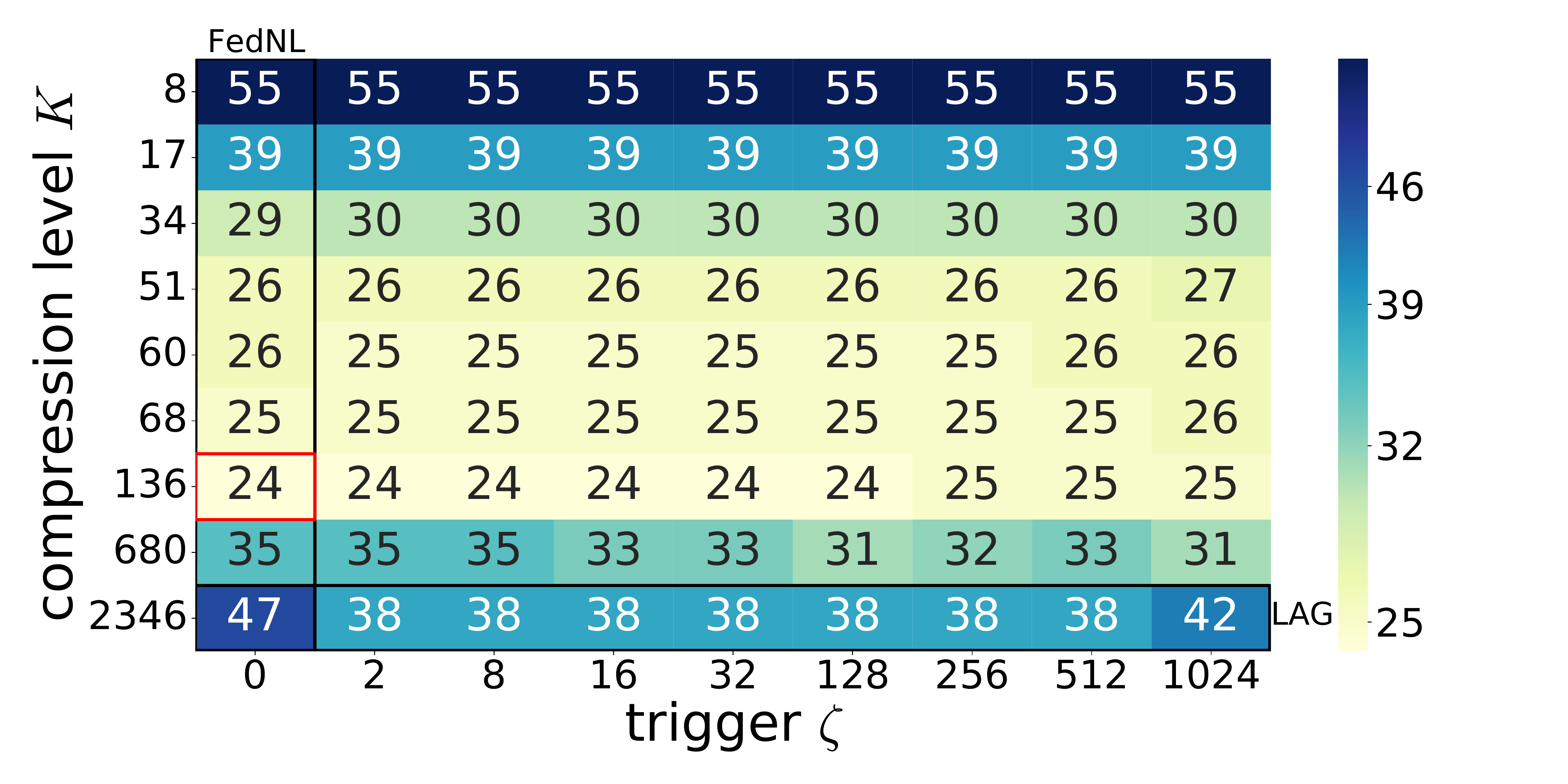} &
				\includegraphics[width=0.22\linewidth]{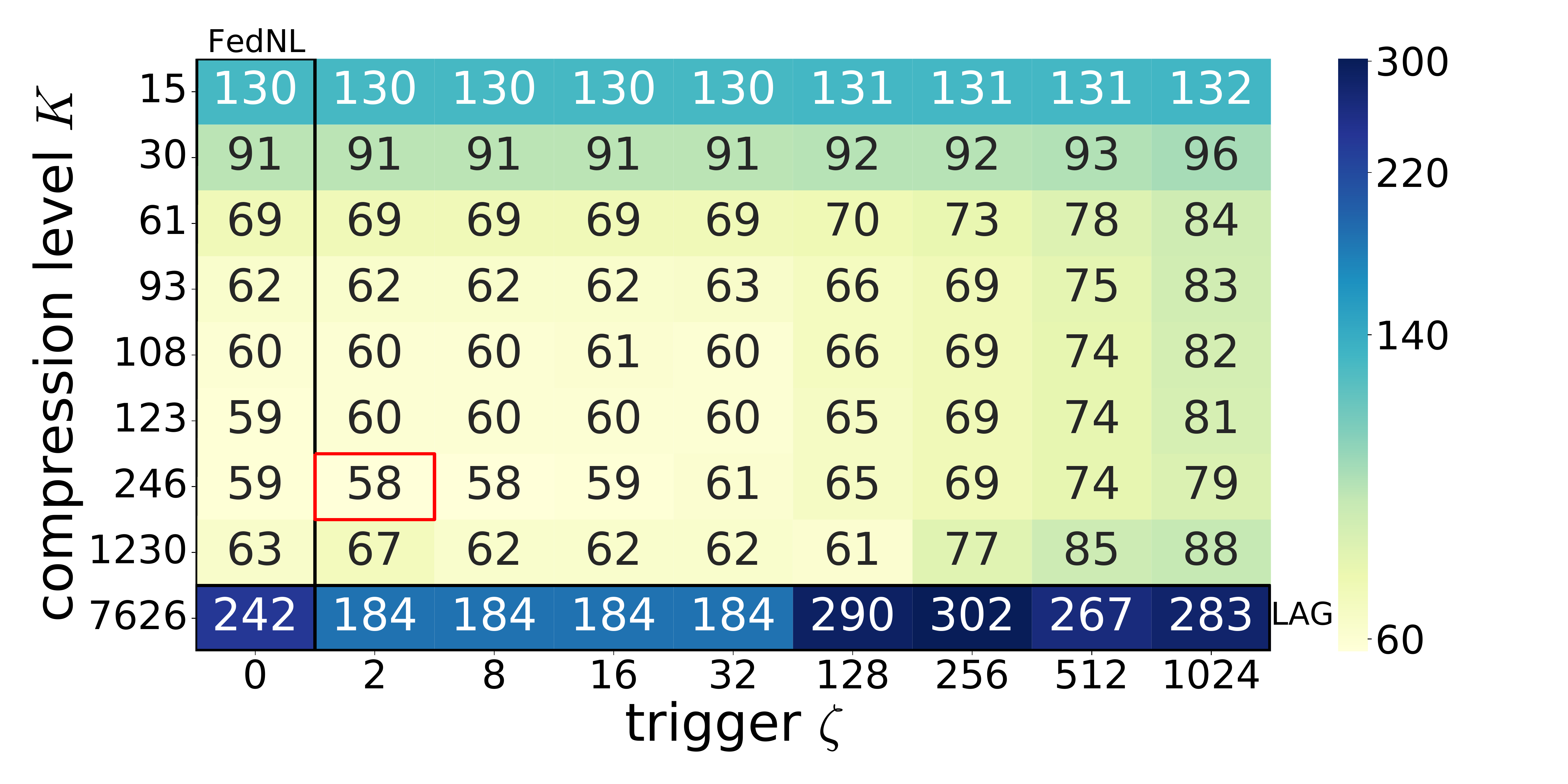} &
				\includegraphics[width=0.22\linewidth]{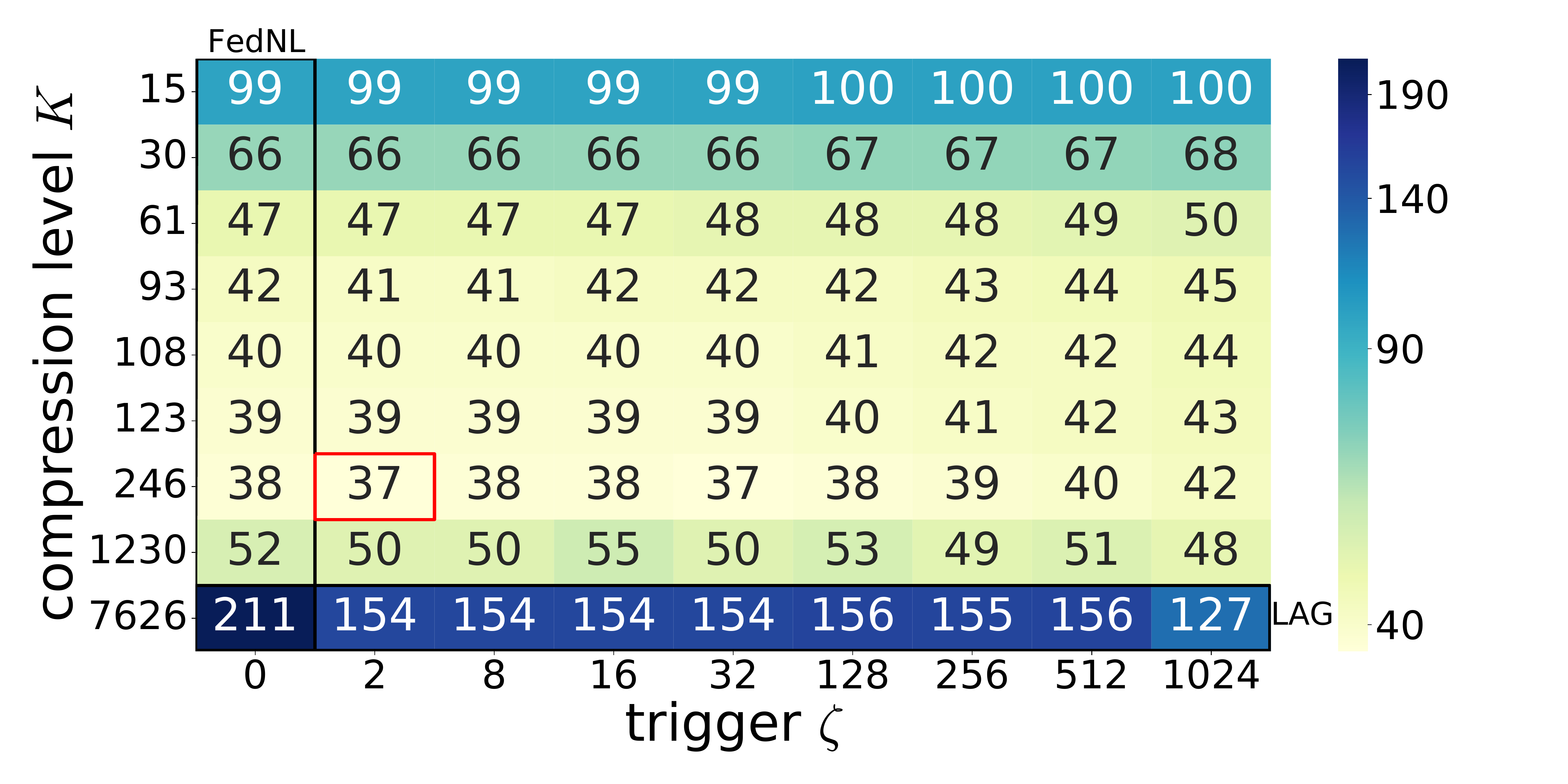}\\
				(a) \dataname{phishing}, {\scriptsize$ \lambda=10^{-3}, d=68$} &
				(b) \dataname{a1a}, {\scriptsize $\lambda=10^{-4}, d=123$} &
				(c) \dataname{a9a}, {\scriptsize$ \lambda=10^{-3}, d=123$}\\
				\includegraphics[width=0.22\linewidth]{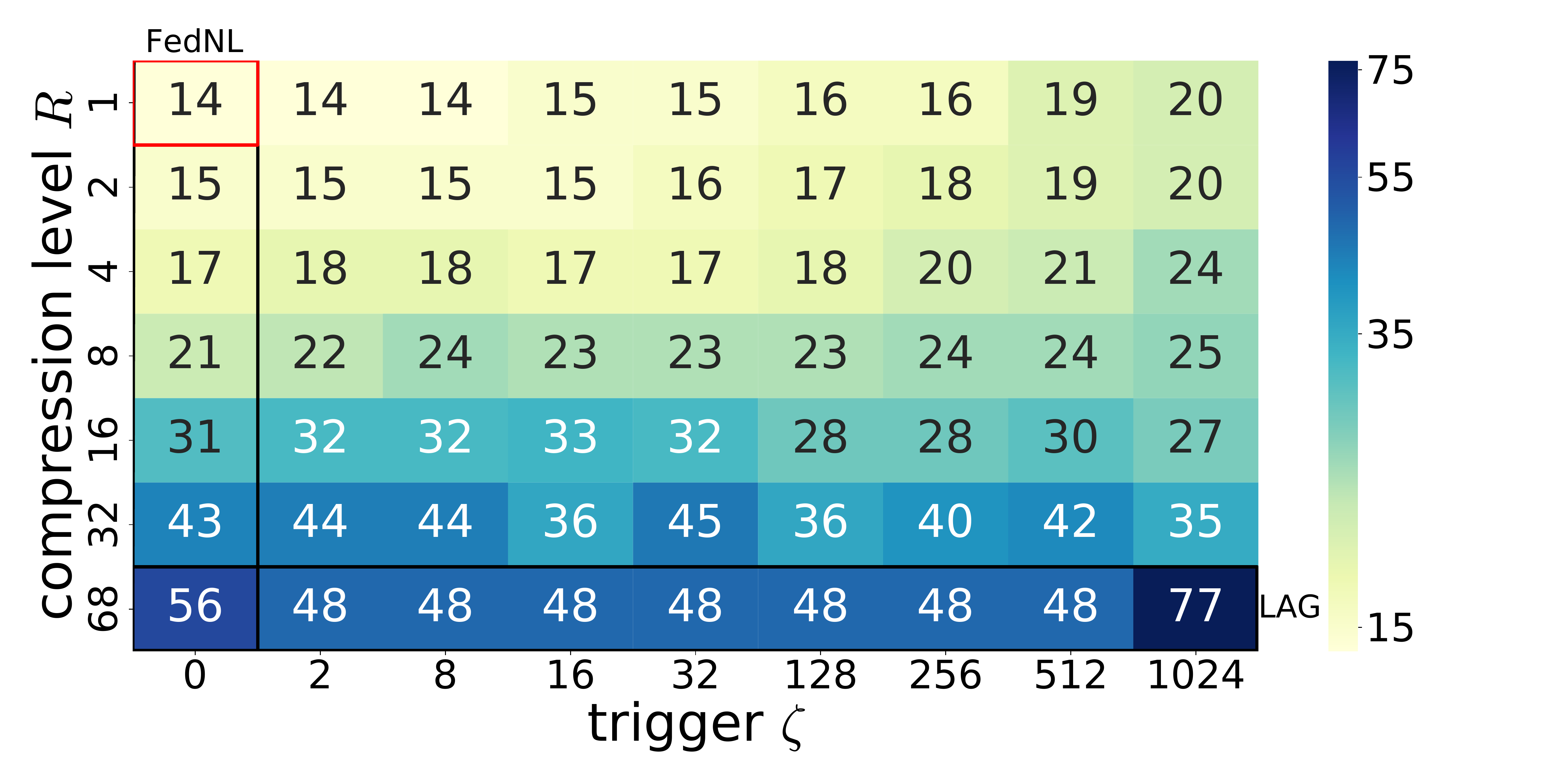} &
				\includegraphics[width=0.22\linewidth]{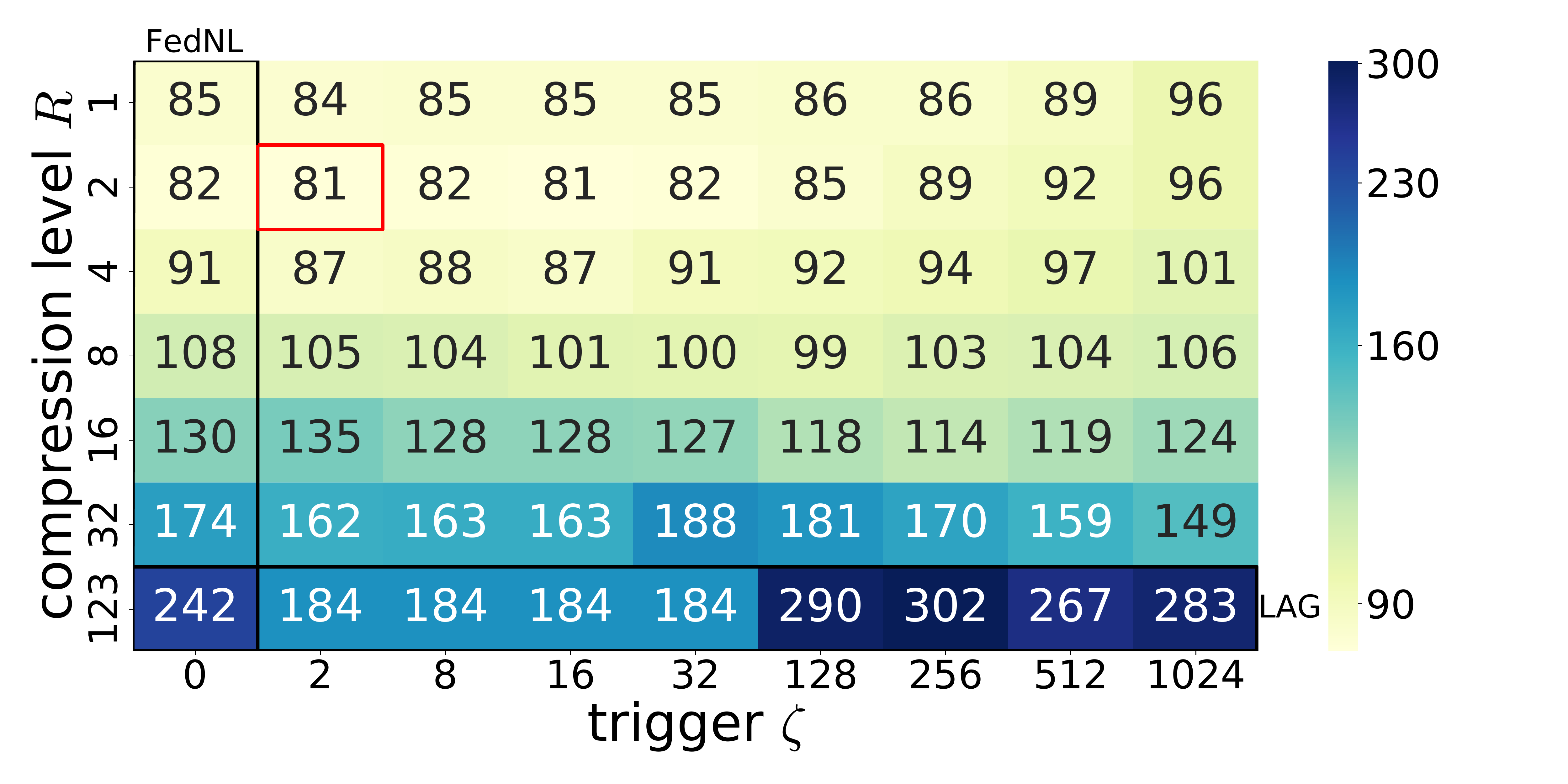} &
				\includegraphics[width=0.22\linewidth]{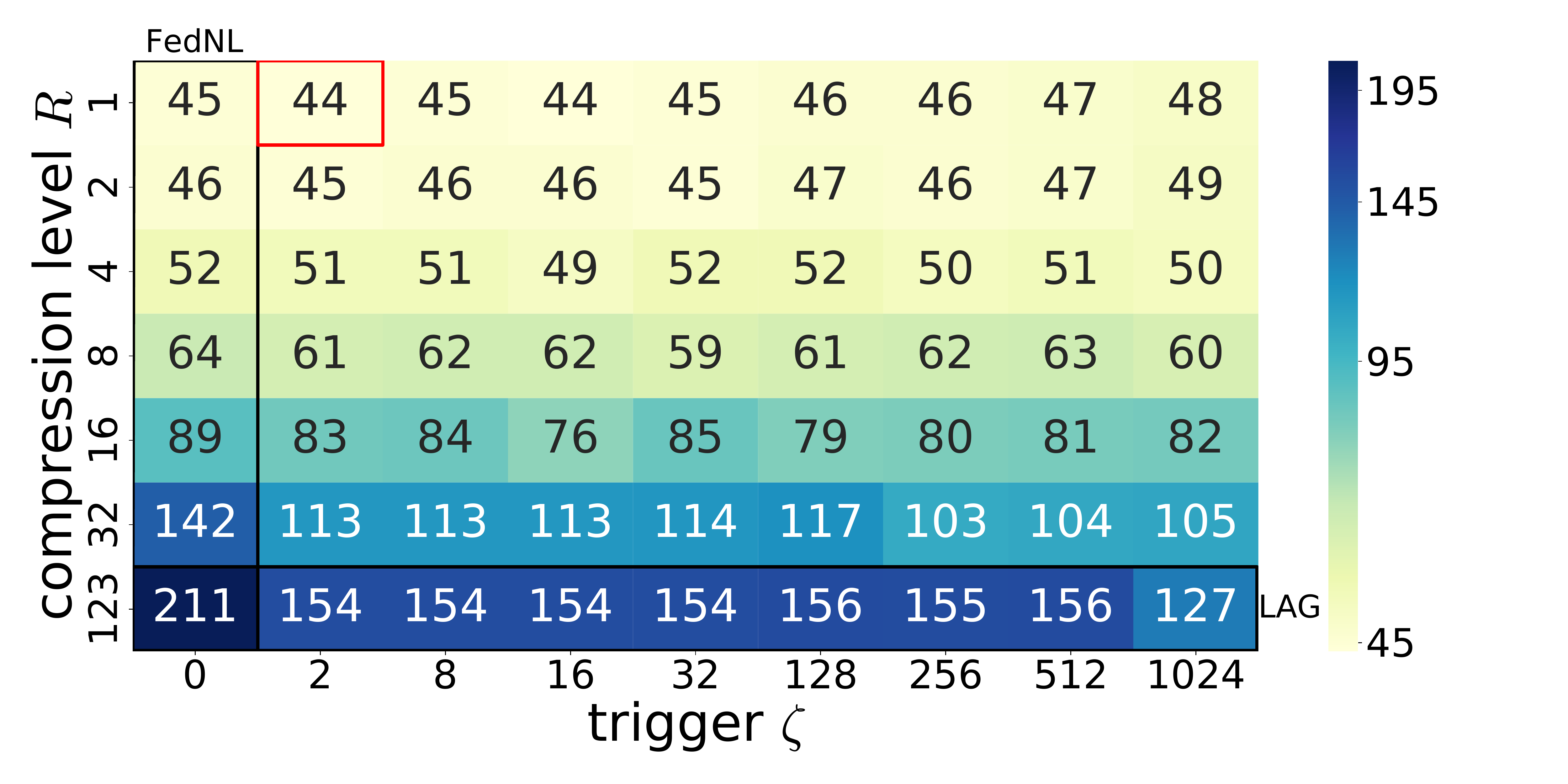}\\
				(e) \dataname{phishing}, {\scriptsize$ \lambda=10^{-3}, d=68$} &
				(f) \dataname{a1a}, {\scriptsize $\lambda=10^{-4},d=123$} &
				(g) \dataname{a9a}, {\scriptsize$ \lambda=10^{-3}, d=123$}\\ 
			\end{tabular}       
		\end{center}
		\caption{{\bf First row:} The performance of \algname{Newton-CLAG} based on Top-$K$ varying values of $(\zeta, K)$ in terms of communication complexity (in Mbytes). {\bf Second row:} The performance of \algname{Newton-CLAG} based on Rank-$R$  varying values of $(\zeta, R)$ in terms of communication complexity (in Mbytes). }
		\label{fig:Newton-CLAG}
	\end{figure}
	
	\subsection{Efficiency of \algname{Newton-3PCv2} under different compression levels}
	
	On the following step we study how \algname{Newton-3PCv2} behaves when the parameters of compressors 3PCv2 is based on are changing. In particular, in the first set of experiments we test the performance of \algname{Newton-3PCv2} assembled from Top-$K_1$ and Rand-$K_2$ compressors where $K_1+K_2=d$. Such constraint is forced to make the cost of one iteration to be $\cO(d)$. In the second set of experiments we choose $K_1=K_2=K$ and vary $K$. The results are presented in Figure~\ref{fig:Newton-3PCv2}. 
	
	For the first set of experiments, one can notice that randomness hurts the convergence since the larger the value of $K_2$, the worse the convergence in terms of communication complexity. In all cases a weaker level of randomness is preferable. For the second set of experiments, we observe that the larger $K$, the better communication complexity of \algname{Newton-3PCv2} except the case of \dataname{w8a} where the results for $K=150$  are slightly better than those for $K=300$. 
	
	\begin{figure}[t]
		\begin{center}
			\begin{tabular}{cccc}
				\includegraphics[width=0.22\linewidth]{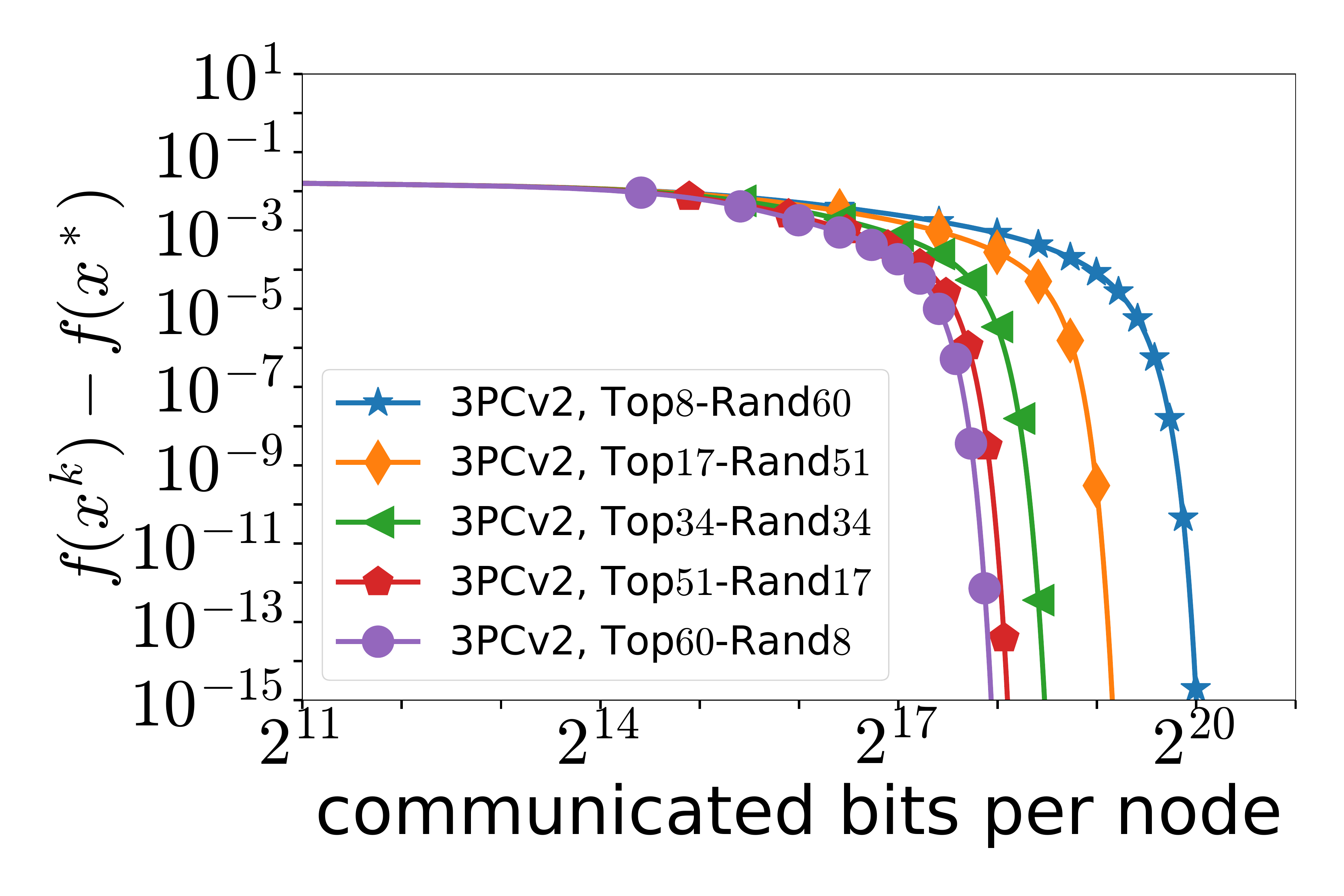} &
				\includegraphics[width=0.22\linewidth]{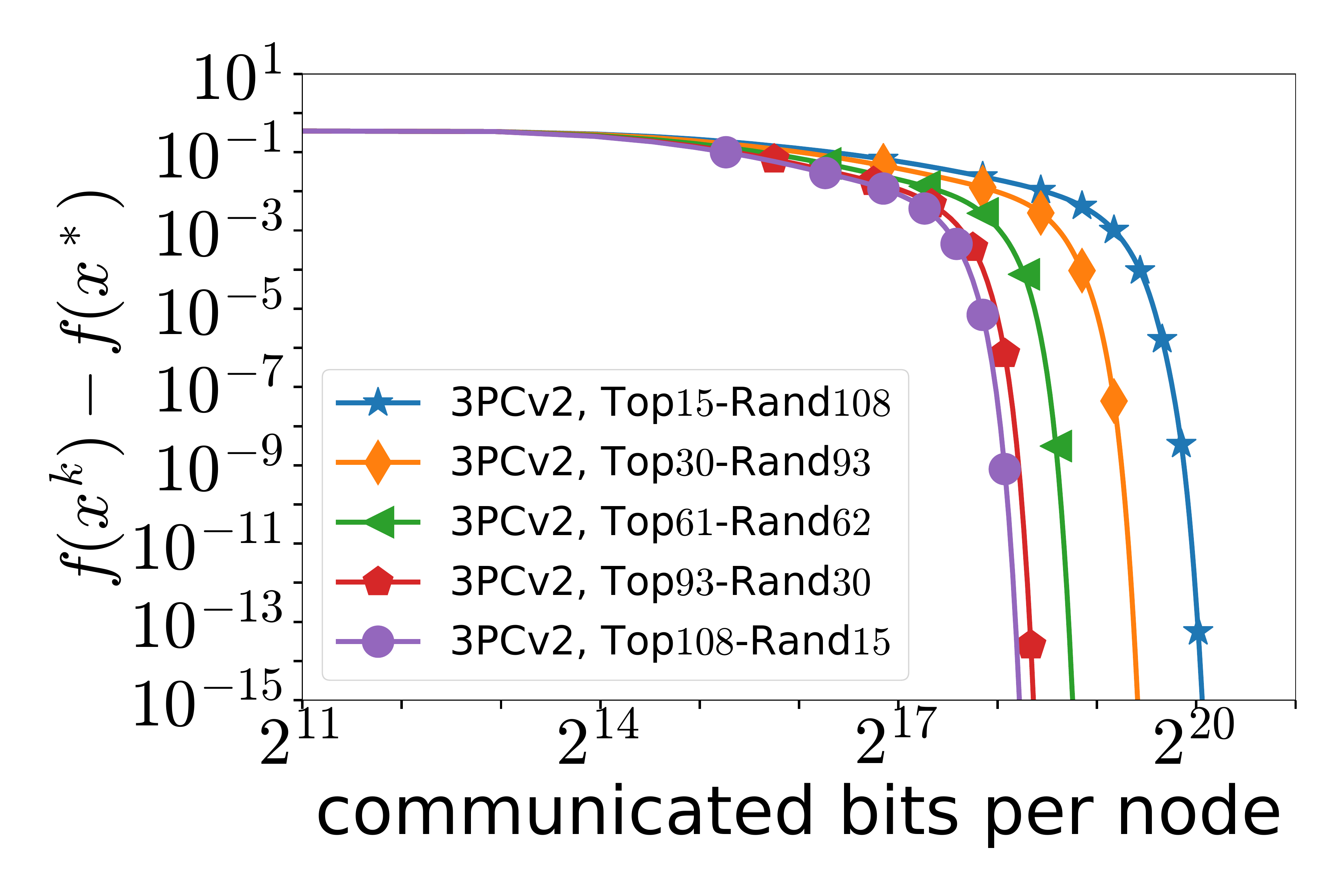} &
				\includegraphics[width=0.22\linewidth]{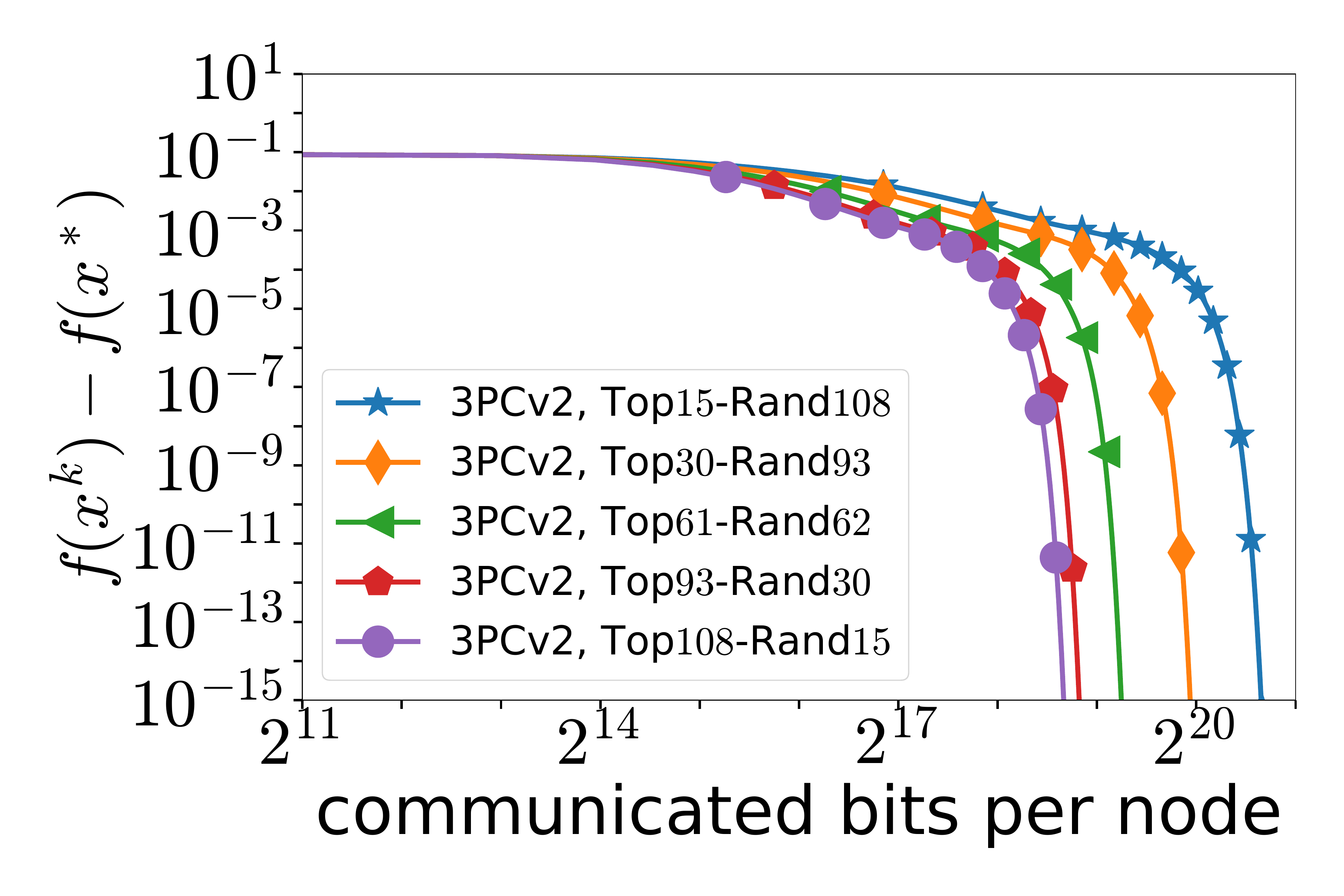} &
				\includegraphics[width=0.22\linewidth]{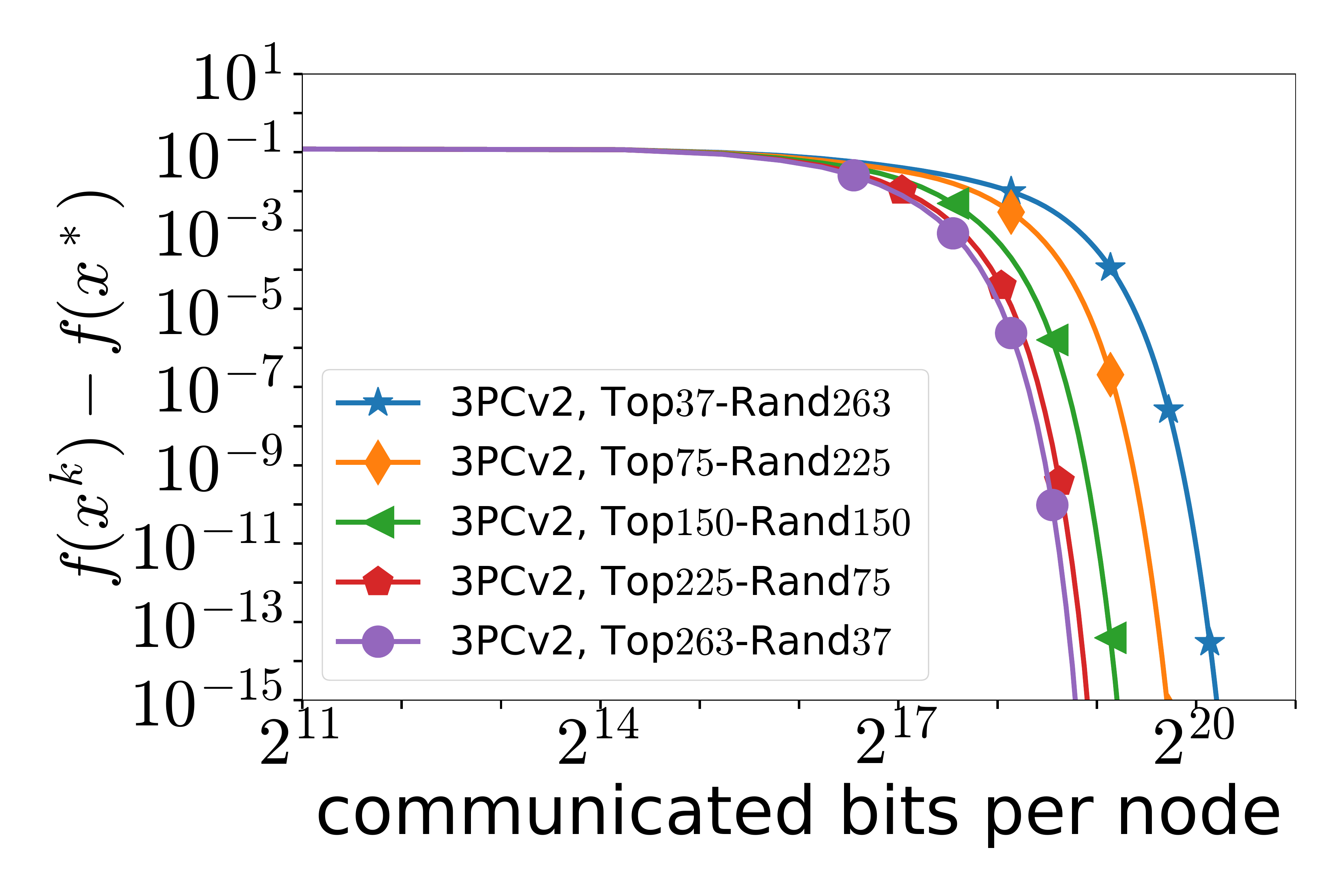}\\
				(a) \dataname{phishing}, {\scriptsize$ \lambda=10^{-4}$} &
				(b) \dataname{a1a}, {\scriptsize $\lambda=10^{-3}$} &
				(c) \dataname{a9a}, {\scriptsize$ \lambda=10^{-4}$} &
				(d) \dataname{w8a}, {\scriptsize$ \lambda=10^{-3}$} \\
				\includegraphics[width=0.22\linewidth]{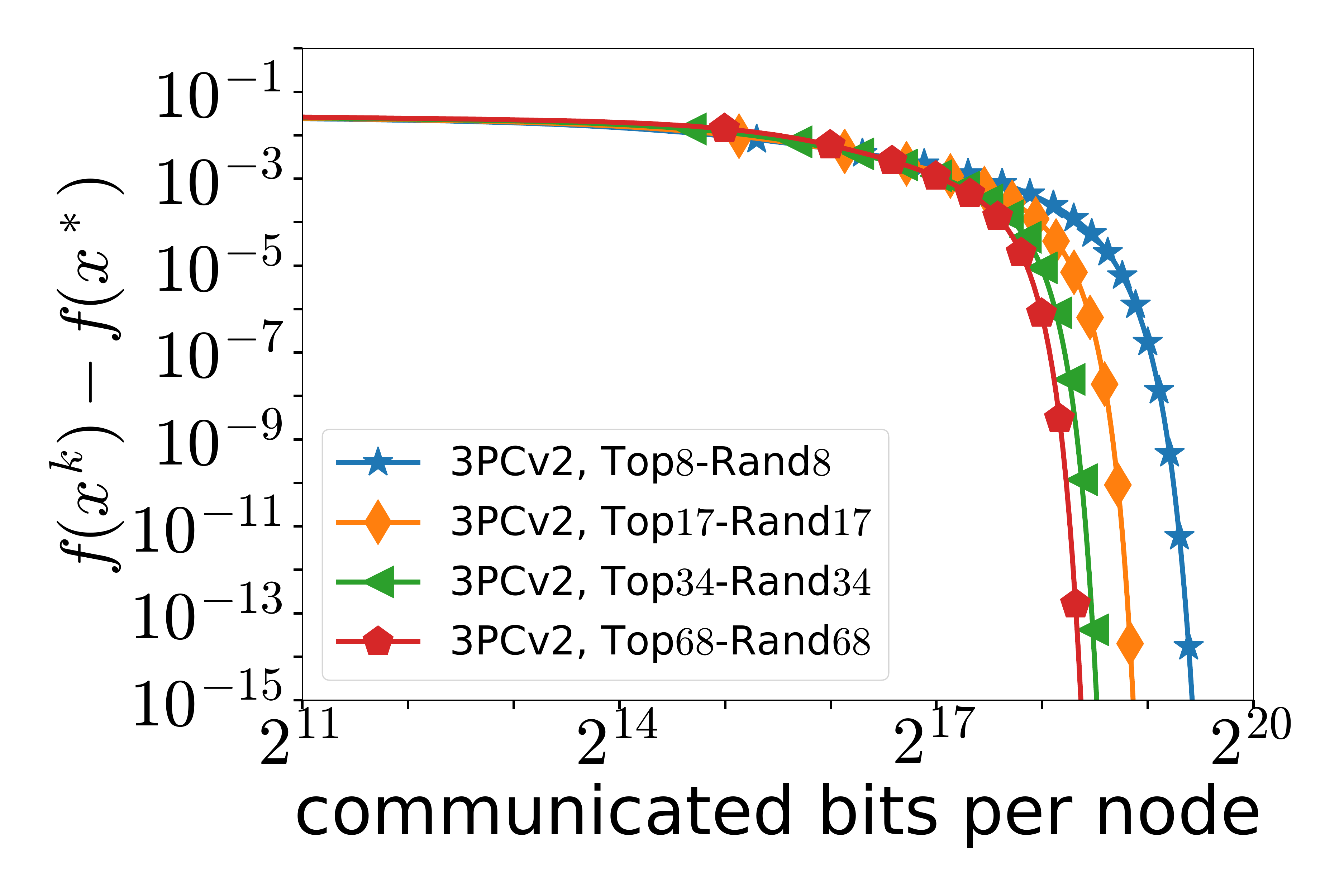} &
				\includegraphics[width=0.22\linewidth]{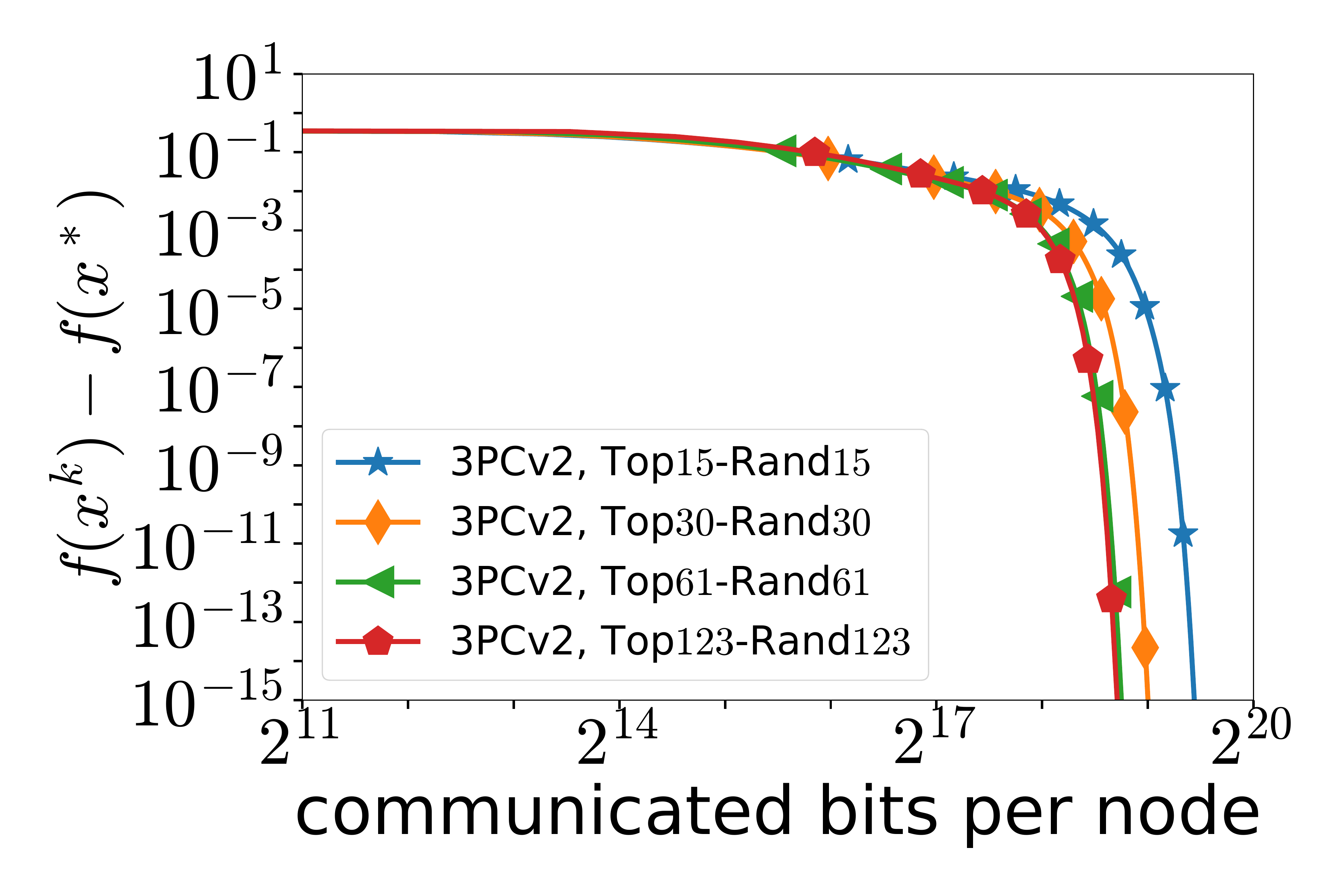} &
				\includegraphics[width=0.22\linewidth]{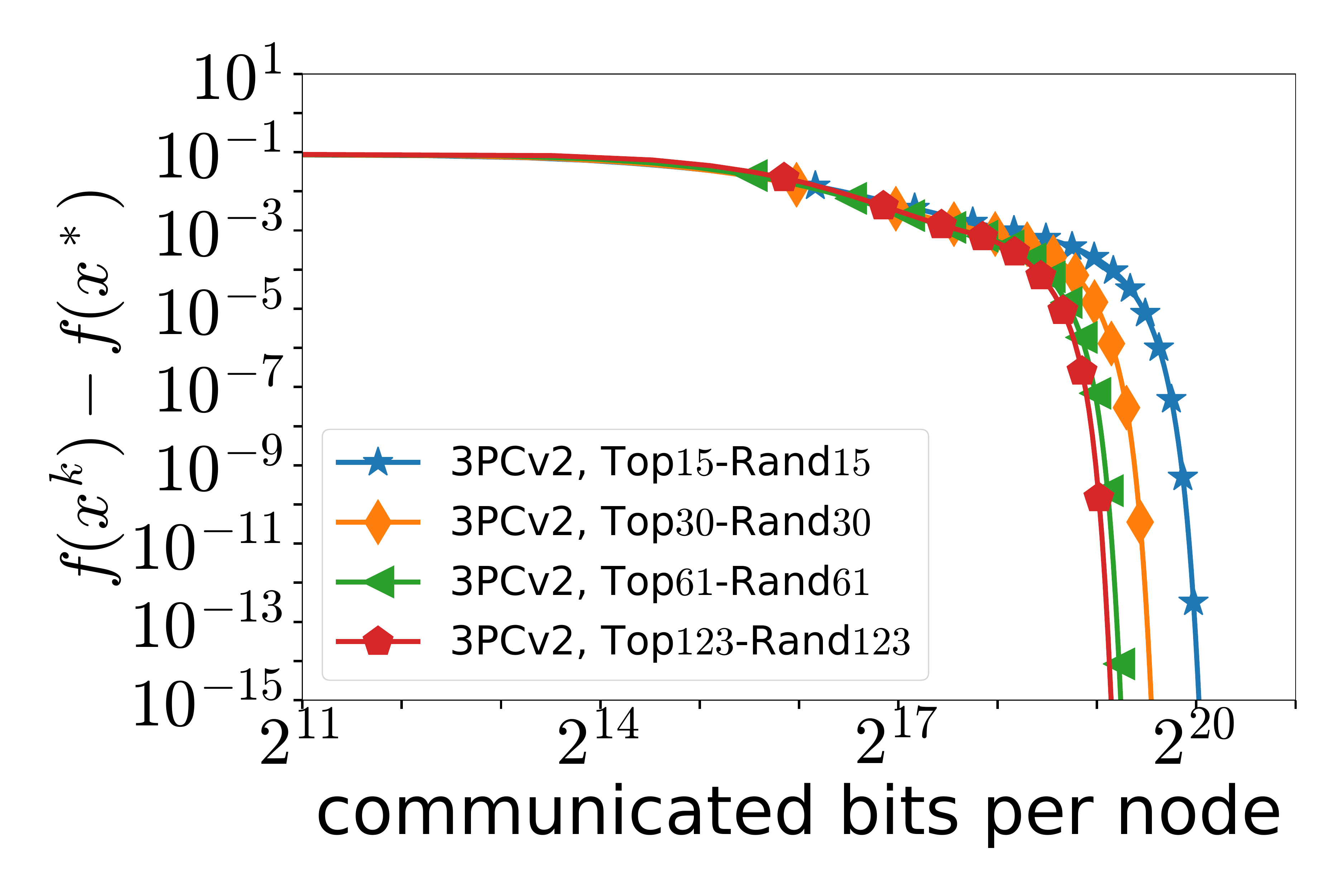} &
				\includegraphics[width=0.22\linewidth]{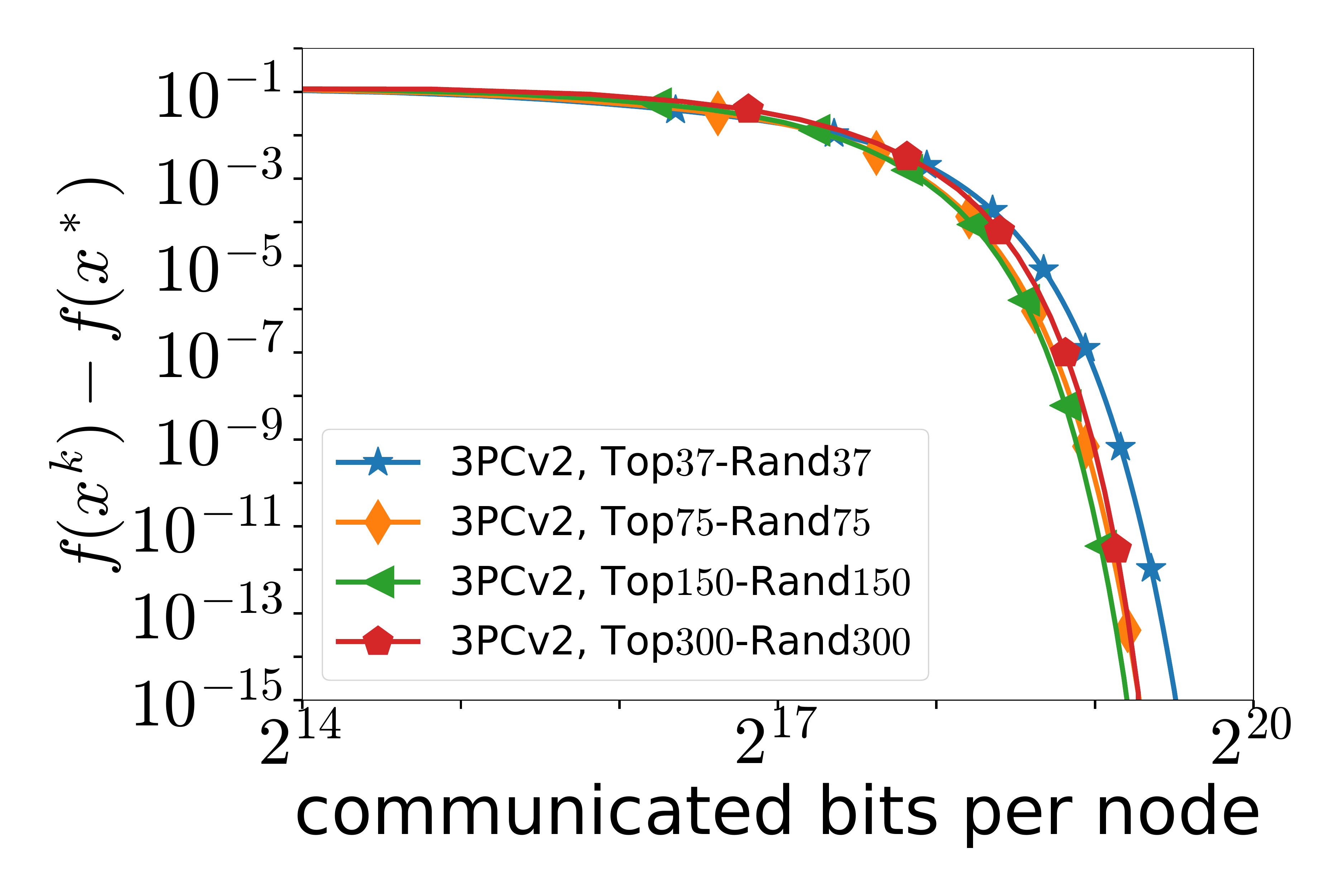}\\
				(e) \dataname{phishing}, {\scriptsize$ \lambda=10^{-4}$} &
				(f) \dataname{a1a}, {\scriptsize $\lambda=10^{-3}$} &
				(g) \dataname{a9a}, {\scriptsize$ \lambda=10^{-4}$} &
				(h) \dataname{w8a}, {\scriptsize$ \lambda=10^{-3}$} \\
			\end{tabular}       
		\end{center}
		\caption{{\bf First row:} The performance of \algname{Newton-3PCv2} where 3PCv2 compression mechanism is based on Top-$K_1$ and Rand-$K_2$ compressors with $K_1+K_2=d$ in terms of communication complexity. {\bf Second row:}  The performance of \algname{Newton-3PCv2} where 3PCv2 compression mechanism is based on Top-$K_1$ and Rand-$K_2$ compressors with $K_1=K_2 \in \{\nicefrac{d}{8}, \nicefrac{d}{4}, \nicefrac{d}{2}, d\}$ in terms of communication complexity. }
		\label{fig:Newton-3PCv2}
	\end{figure}

	\subsection{Behavior of \algname{Newton-3PCv4} under different compression levels}
	
	Now we test the behavior of \algname{Newton-3PCv4} where 3PCv4 is based on a pair (Top-$K_1$, Top-$K_2$) of compressors. Again, we have to sets of experiments: in the first one we examine the performance of \algname{Newton-3PCv4} when $K_1+K_2=d$; in the second one we check the efficiency of \algname{Newton-3PCv4} when $K_1=K_2=K$ varying $K$. In both cases we provide the behavior of \algname{Newton-EF21} (equivalent to \algname{FedNL}) for comparison. All results are presented in Figure~\ref{fig:Newton-3PCv4}. 
	
	As we can see, in the first set of experiments it does not matter how we distribute $d$ between $K_1$ and $K_2$ since it does not affect the performance. Regarding the second set of experiments, we can say that in some cases less aggressive compression ($K_1=K_2=d$) could be better than \algname{Newton-EF21}.

	\begin{figure}[t]
		\begin{center}
			\begin{tabular}{cccc}
				\includegraphics[width=0.22\linewidth]{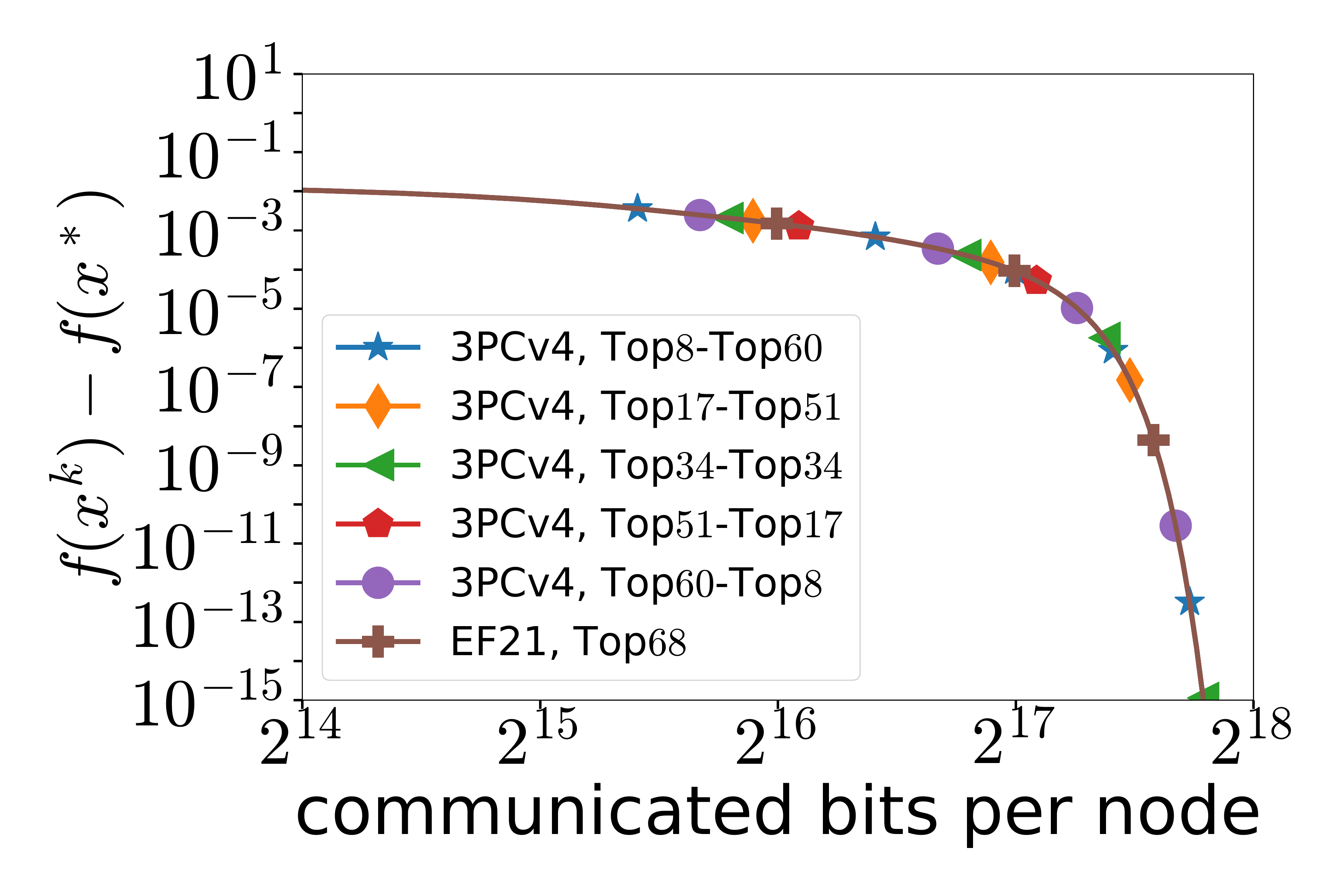} &
				\includegraphics[width=0.22\linewidth]{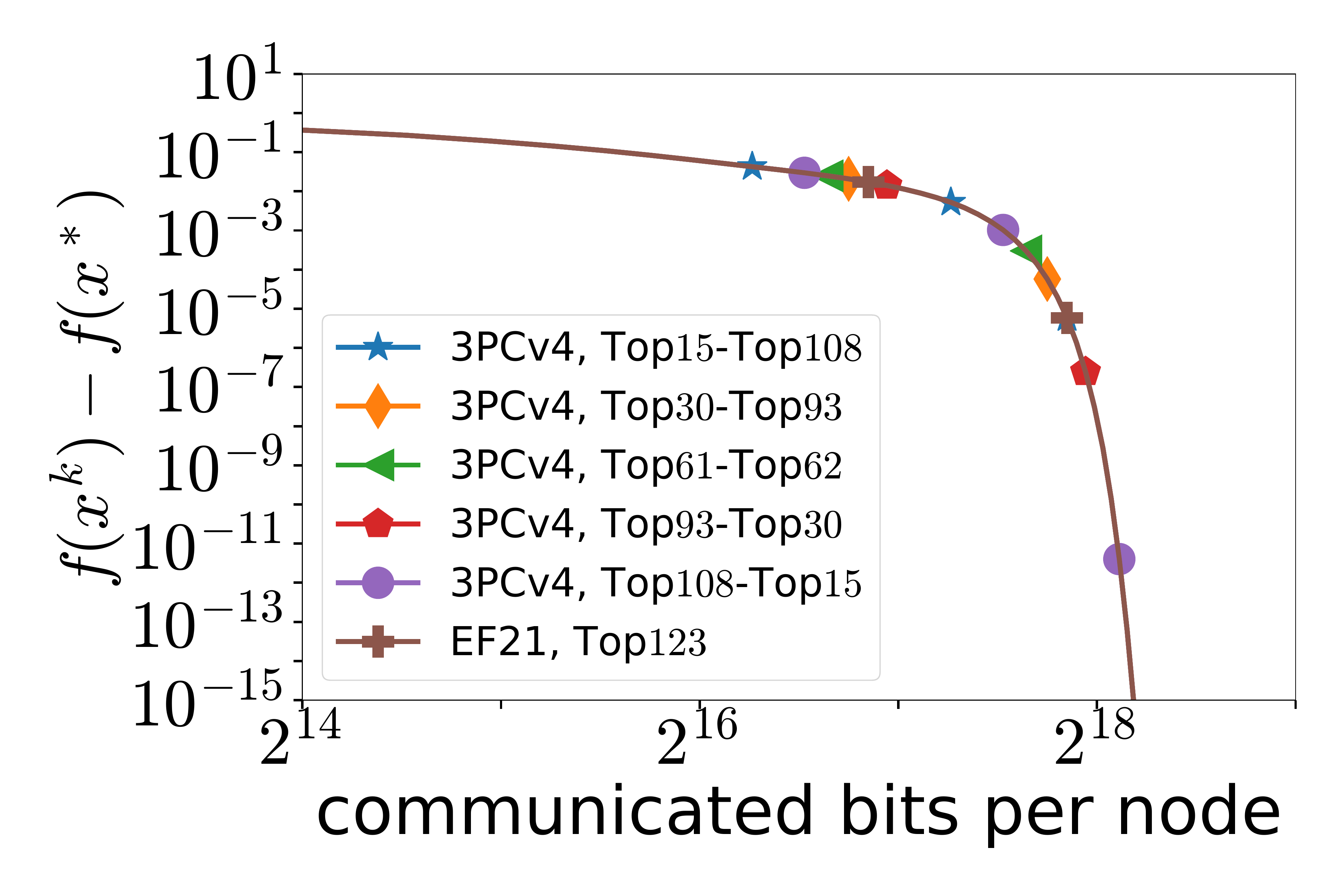} &
				\includegraphics[width=0.22\linewidth]{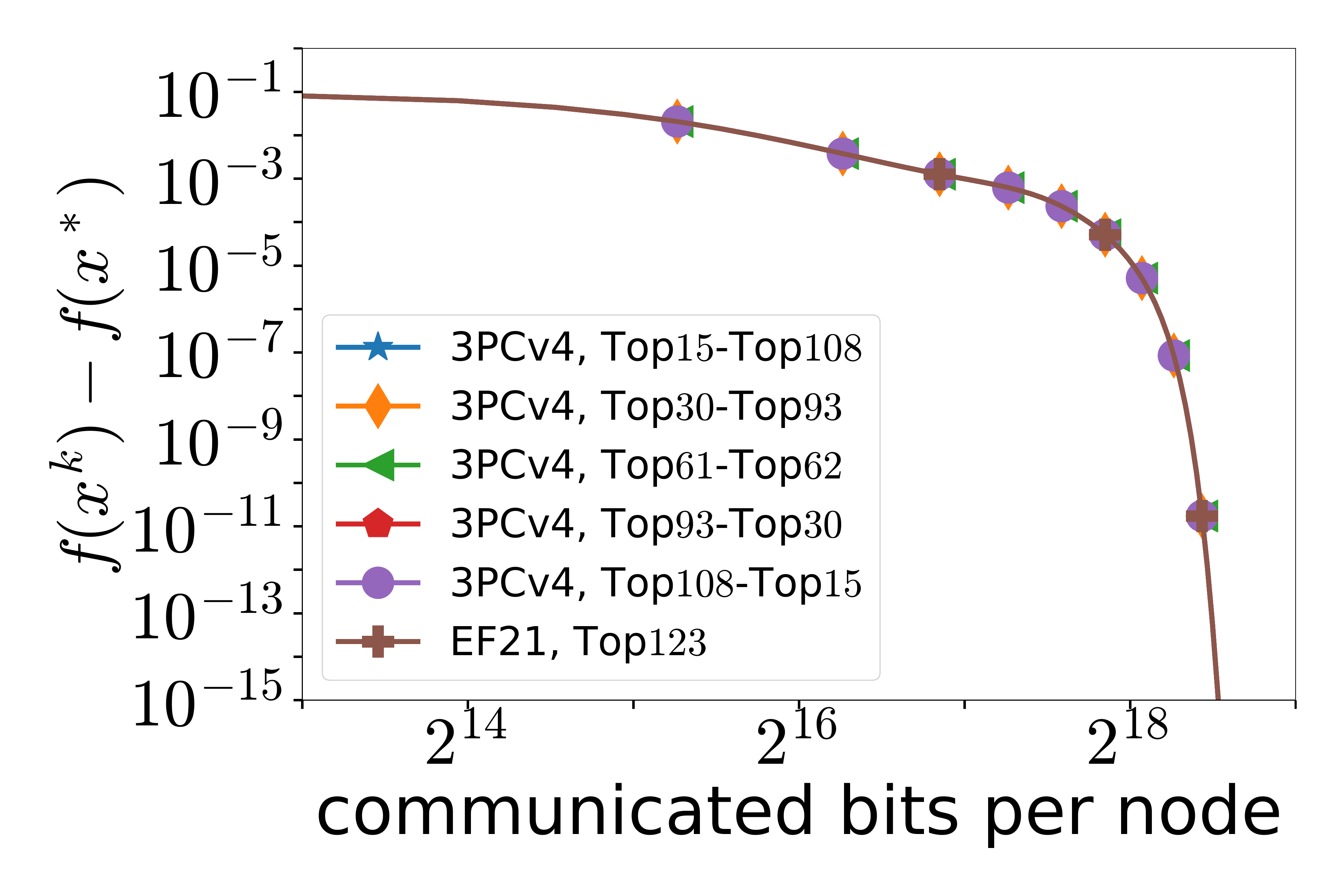} &
				\includegraphics[width=0.22\linewidth]{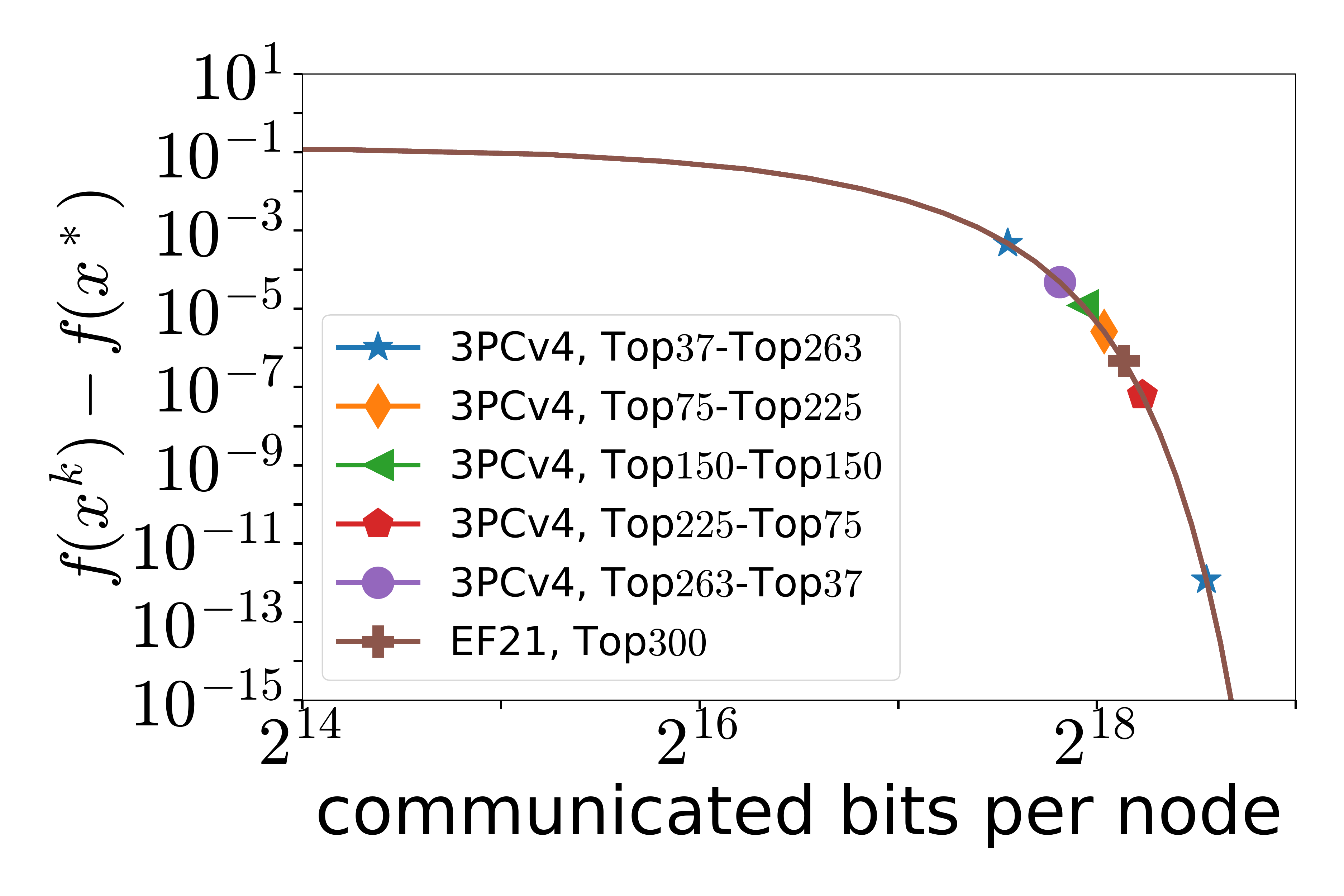}\\
				(a) \dataname{phishing}, {\scriptsize$ \lambda=10^{-4}$} &
				(b) \dataname{a1a}, {\scriptsize $\lambda=10^{-3}$} &
				(c) \dataname{a9a}, {\scriptsize$ \lambda=10^{-4}$} &
				(d) \dataname{w8a}, {\scriptsize$ \lambda=10^{-3}$} \\
				\includegraphics[width=0.22\linewidth]{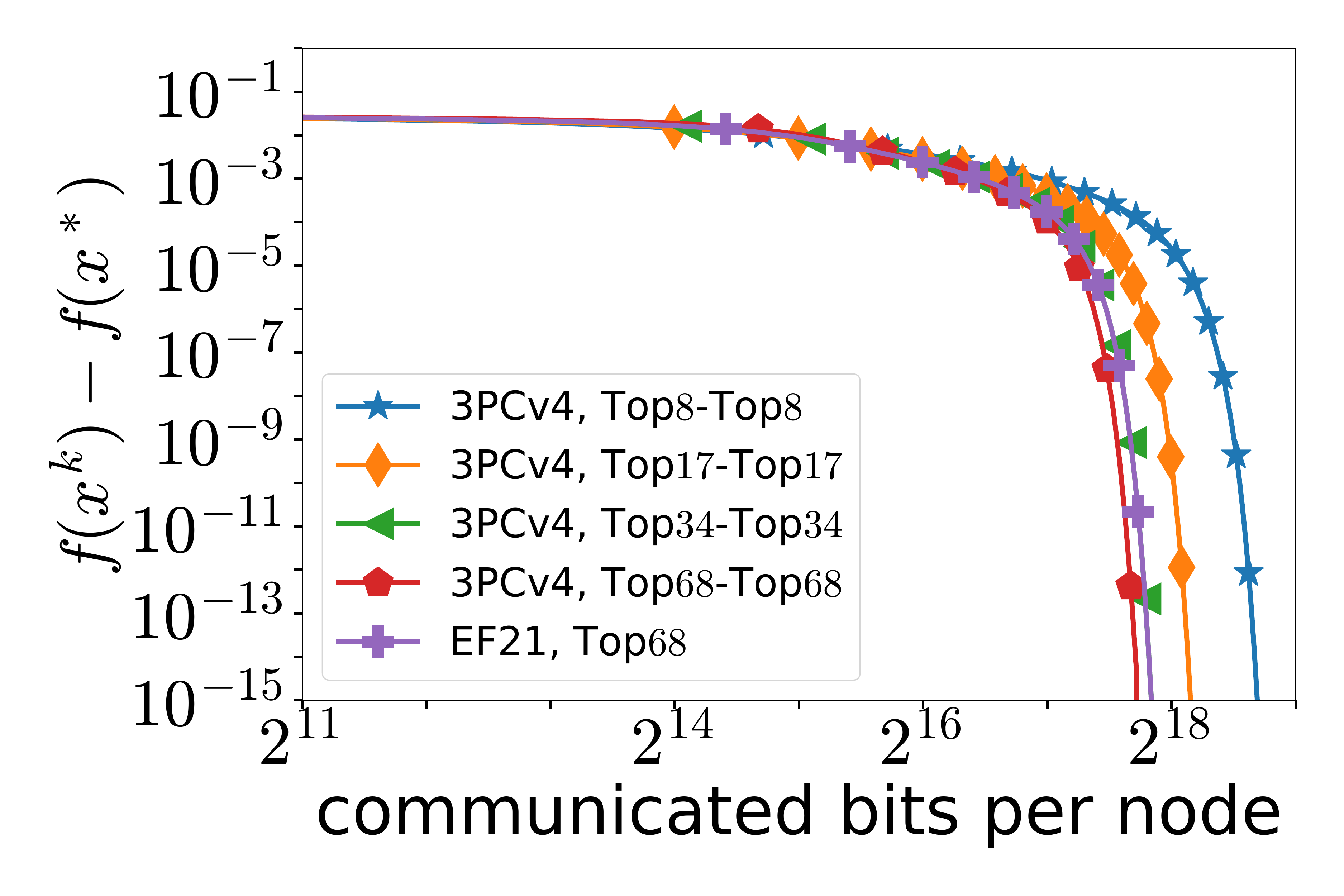} &
				\includegraphics[width=0.22\linewidth]{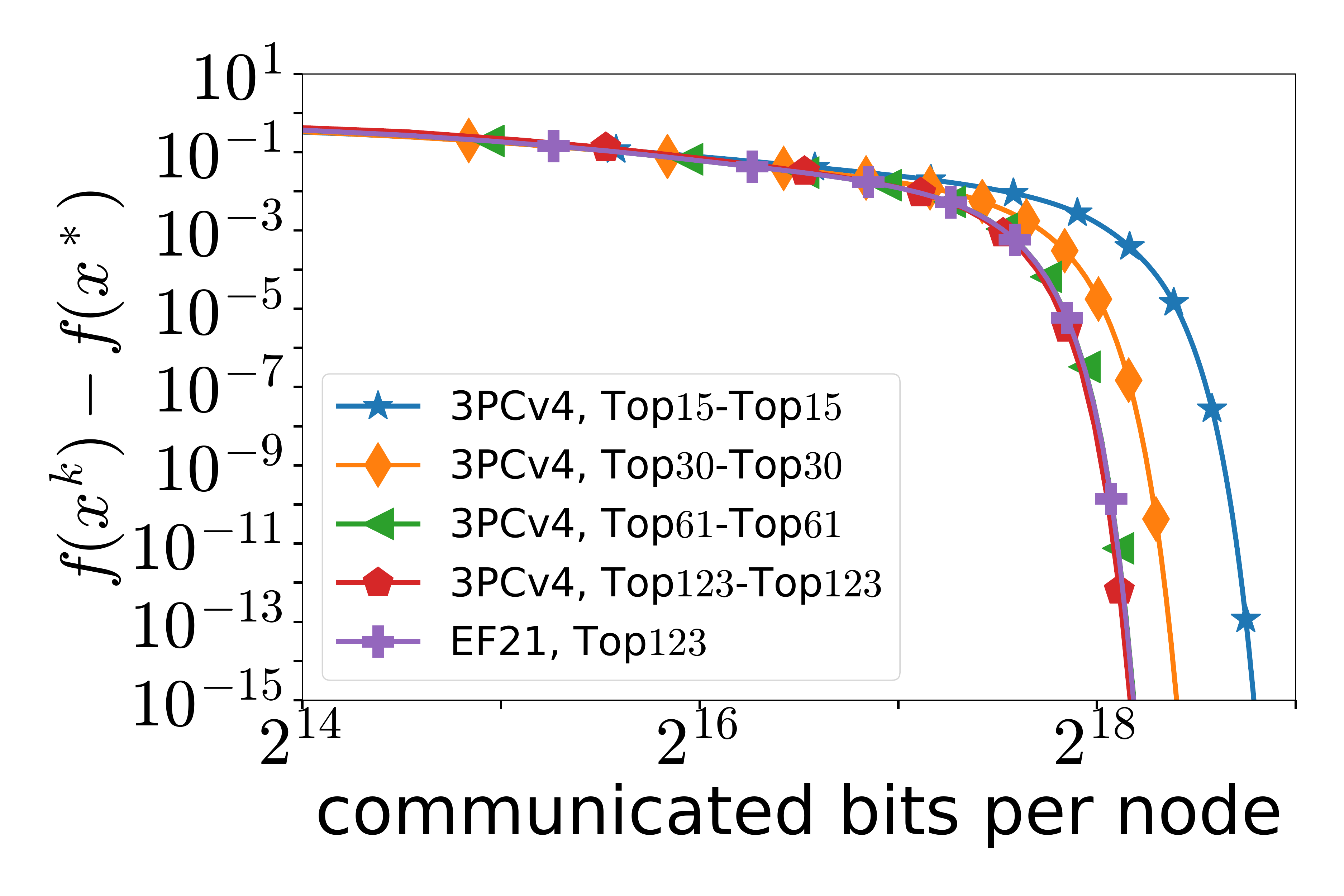} &
				\includegraphics[width=0.22\linewidth]{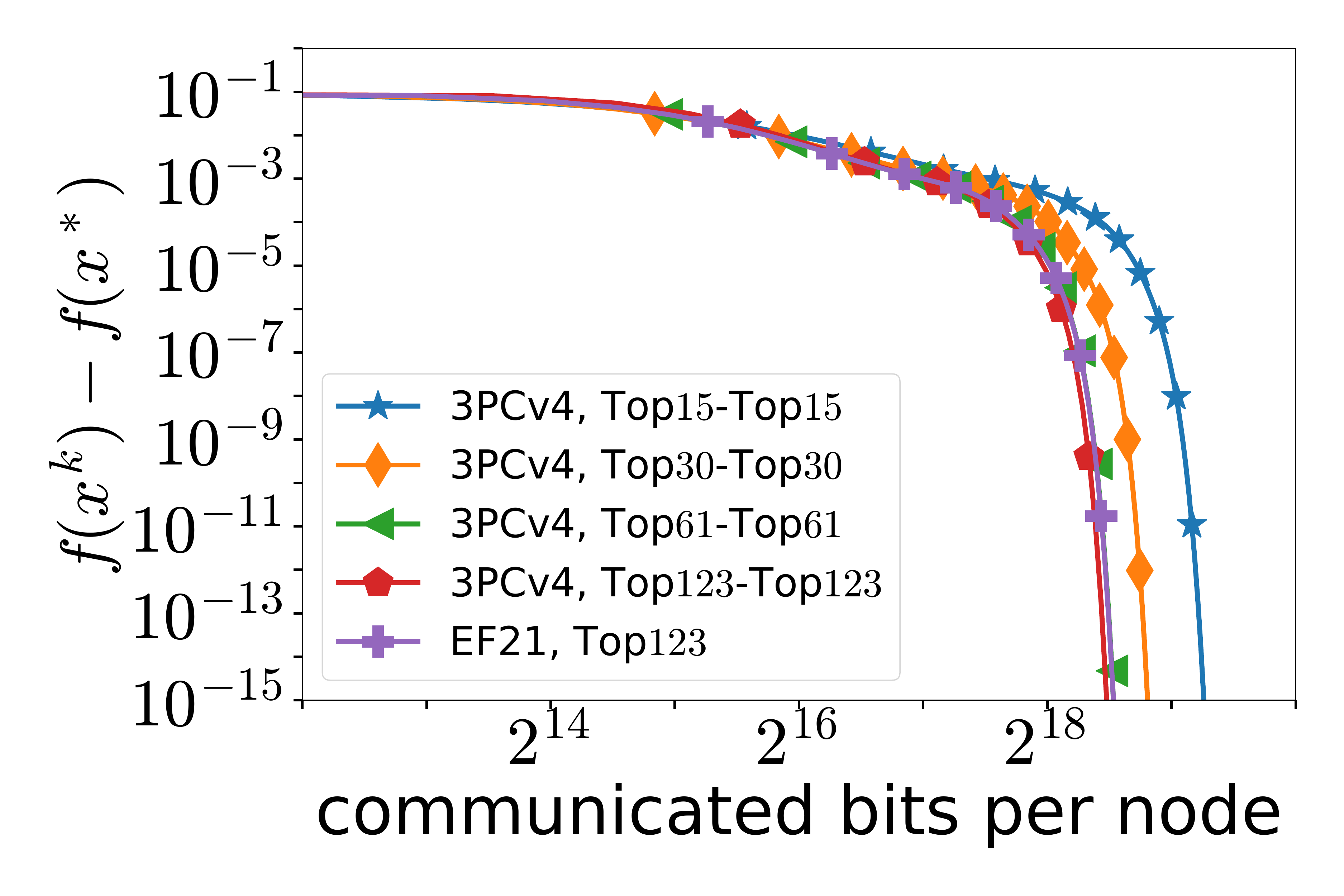} &
				\includegraphics[width=0.22\linewidth]{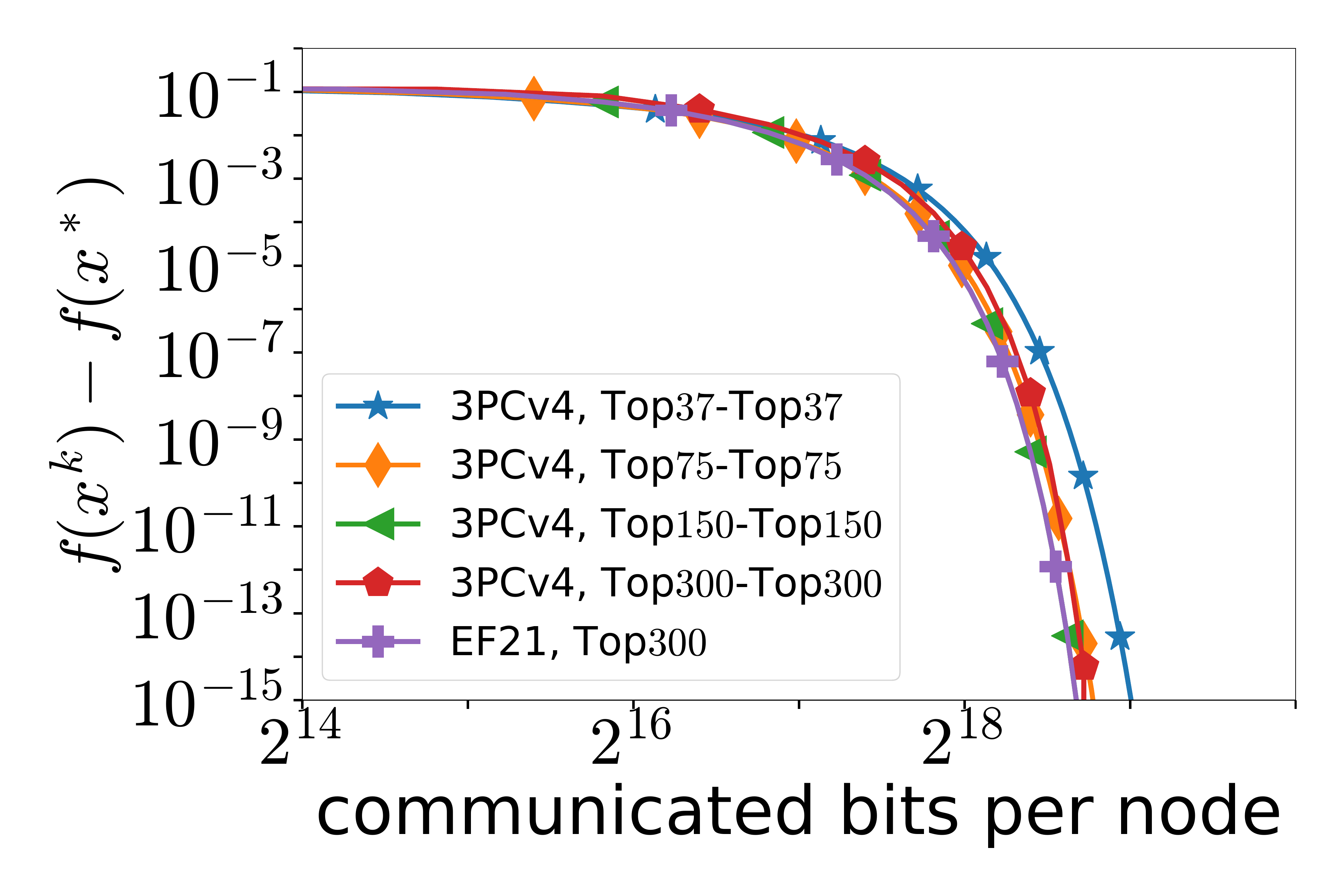}\\
				(e) \dataname{phishing}, {\scriptsize$ \lambda=10^{-4}$} &
				(f) \dataname{a1a}, {\scriptsize $\lambda=10^{-3}$} &
				(g) \dataname{a9a}, {\scriptsize$ \lambda=10^{-4}$} &
				(h) \dataname{w8a}, {\scriptsize$ \lambda=10^{-3}$} \\
			\end{tabular}       
		\end{center}
		\caption{{\bf First row:} The performance of \algname{Newton-3PCv4} where 3PCv4 compression mechanism is based on Top-$K_1$ and Top-$K_2$ compressors with $K_1+K_2=d$ in terms of communication complexity. {\bf Second row:}  The performance of \algname{Newton-3PCv4} where 3PCv4 compression mechanism is based on Top-$K_1$ and Top-$K_2$ compressors with $K_1=K_2 \in \{\nicefrac{d}{8}, \nicefrac{d}{4}, \nicefrac{d}{2}, d\}$ in terms of communication complexity. Performance of \algname{Newton-EF21}  with Top-$d$ is given for comparison.}
		\label{fig:Newton-3PCv4}
	\end{figure}

	\subsection{Study of \algname{Newton-3PCv1}}
	
	Next, we investigate the performance of \algname{Newton-3PCv1} where 3PC compression mechanism is based on Top-$K$. We compare its performance with \algname{Newton-EF21} (equivalent to \algname{FedNL}) with Top-$d$, \algname{NL1} with Rand-$1$, and \algname{DINGO}. We observe in Figure~\ref{fig:Newton-3PCv1} that \algname{Newton-3PCv1} is not efficient method since it fails in all cases.

	\begin{figure}[t]
		\begin{center}
			\begin{tabular}{cccc}
				\includegraphics[width=0.22\linewidth]{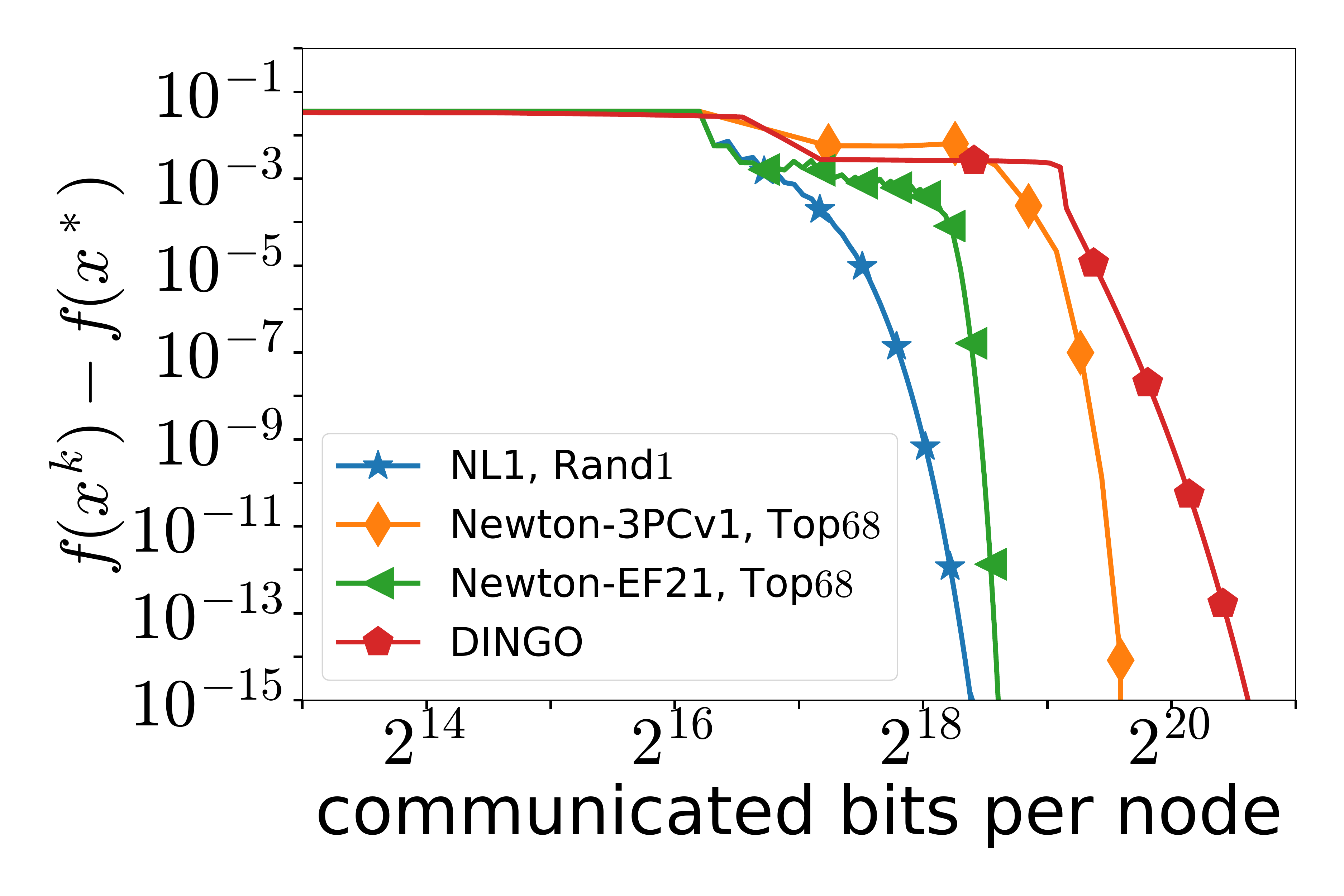} &
				\includegraphics[width=0.22\linewidth]{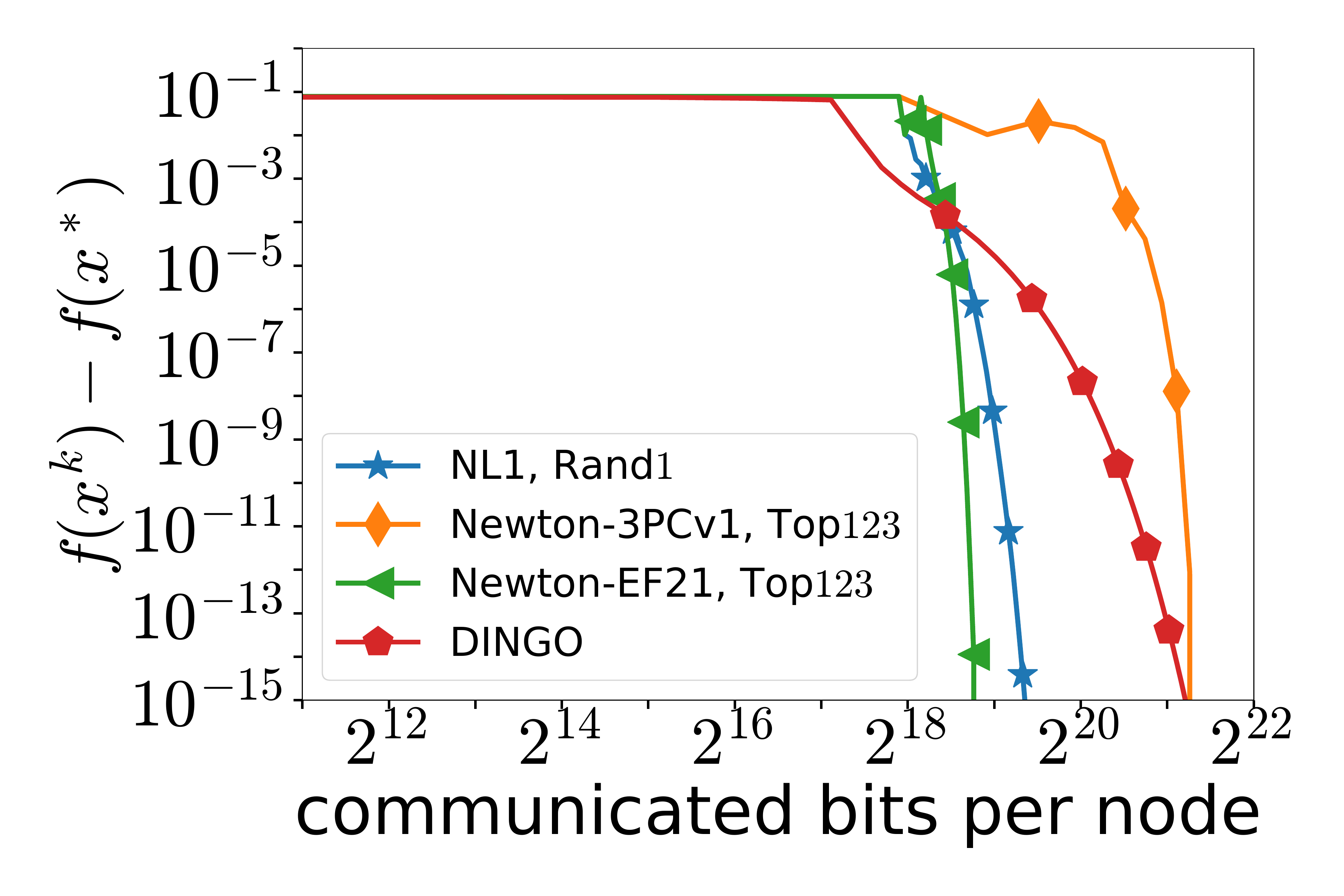} &
				\includegraphics[width=0.22\linewidth]{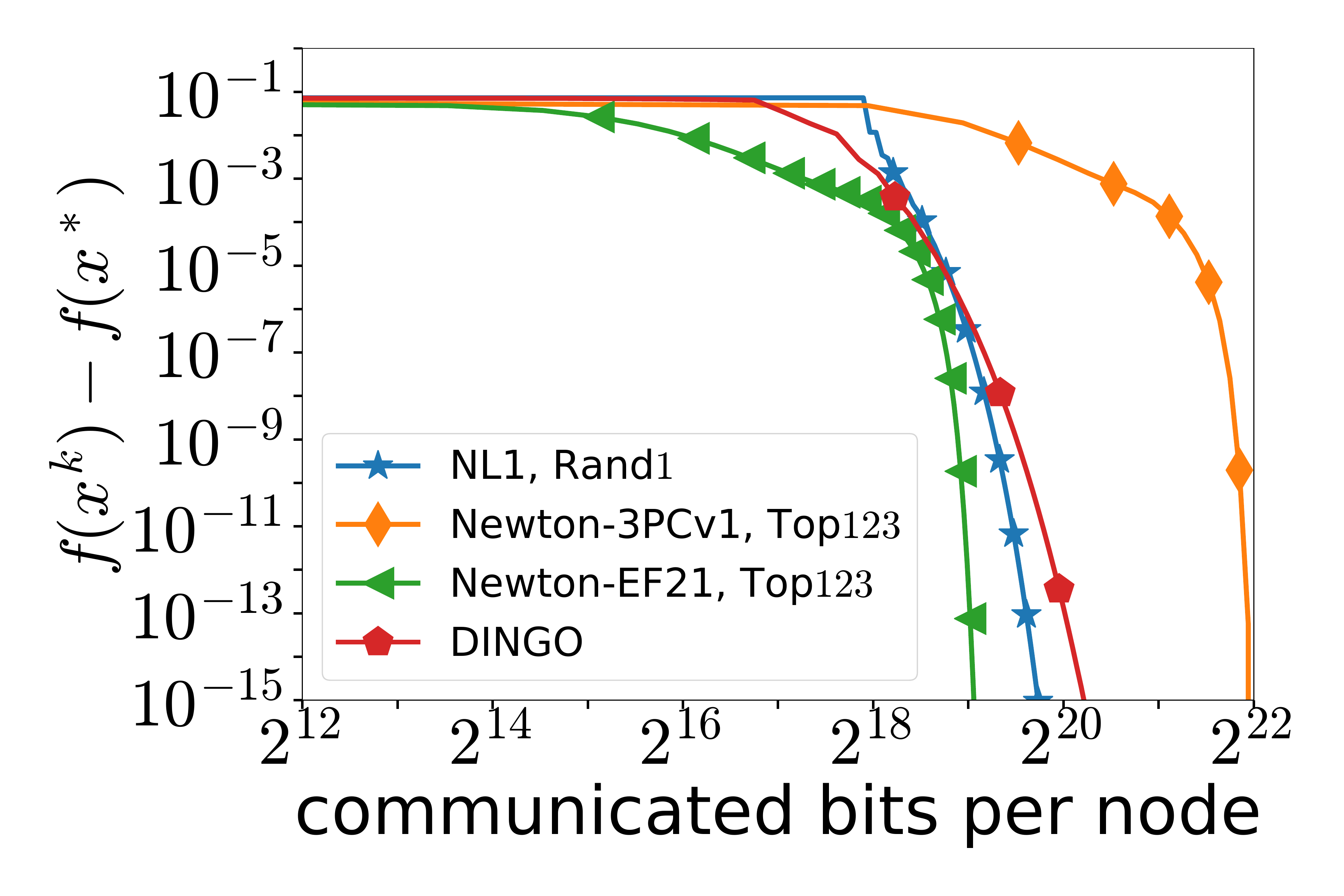} &
				\includegraphics[width=0.22\linewidth]{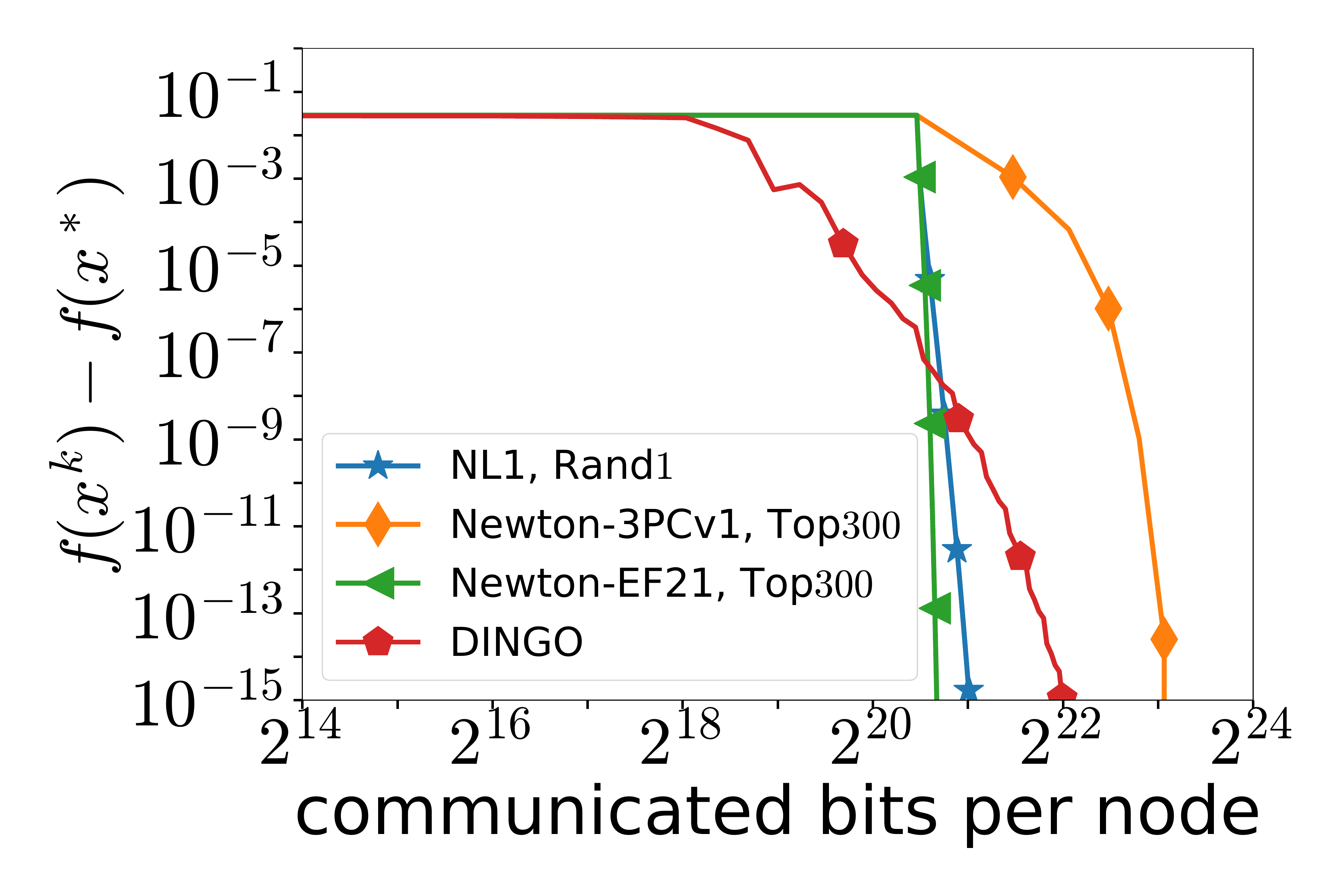}\\
				(a) \dataname{phishing}, {\scriptsize$ \lambda=10^{-4}$} &
				(b) \dataname{a1a}, {\scriptsize $\lambda=10^{-3}$} &
				(c) \dataname{a9a}, {\scriptsize$ \lambda=10^{-4}$} &
				(d) \dataname{w8a}, {\scriptsize$ \lambda=10^{-3}$} \\
			\end{tabular}       
		\end{center}
		\caption{The performance of \algname{Newton-3PCv1} with 3PCv1 based on Top-$d$, \algname{Newton-EF21} (equivalent to \algname{FedNL}) with Top-$d$, \algname{NL1} with Rand-$1$, and \algname{DINGO} in terms of communication complexity. }
		\label{fig:Newton-3PCv1}
	\end{figure}
	
	\subsection{Performance of \algname{Newton-3PCv5}}
	
	In this section we investigate the performance of \algname{Newton-3PCv5} where 3PC compression mechanism is based on Top-$K$. We compare its performance with \algname{Newton-EF21} (equivalent to \algname{FedNL}) with Top-$d$, \algname{NL1} with Rand-$1$, and \algname{DINGO}. According to the plots presented in Figure~\ref{fig:Newton-3PCv5}, we conclude that \algname{Newton-3PCv5} is not as effective as \algname{NL1} and \algname{Newton-EF21}, but it is comparable with \algname{DINGO}. The reason why \algname{Newton-3PCv5} is not efficient in terms of communication complexity is that we still need to send true Hessians with some nonzero probability which hurts the communication complexity of this method.
	
	\begin{figure}[t]
		\begin{center}
			\begin{tabular}{cccc}
				\includegraphics[width=0.22\linewidth]{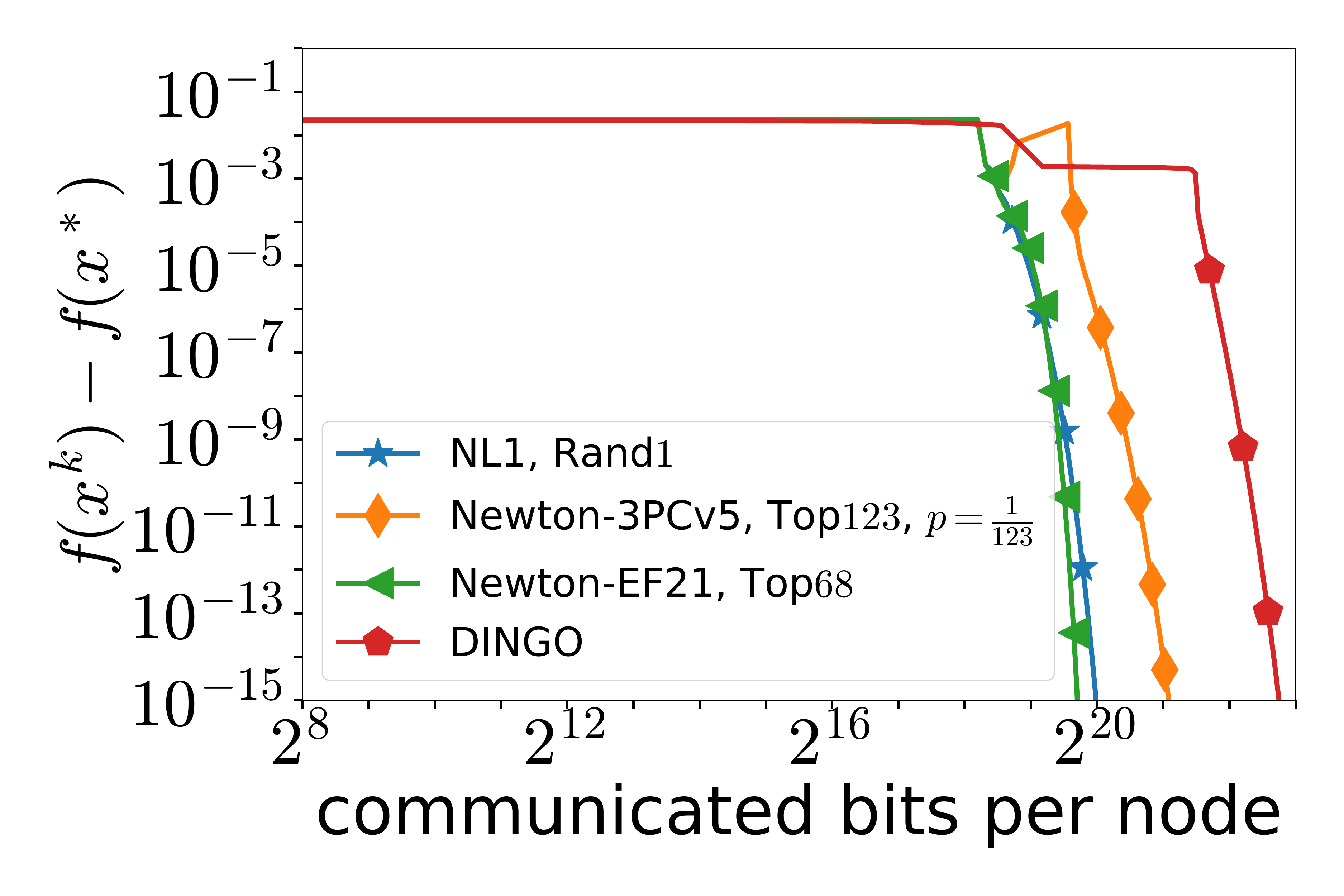} &
				\includegraphics[width=0.22\linewidth]{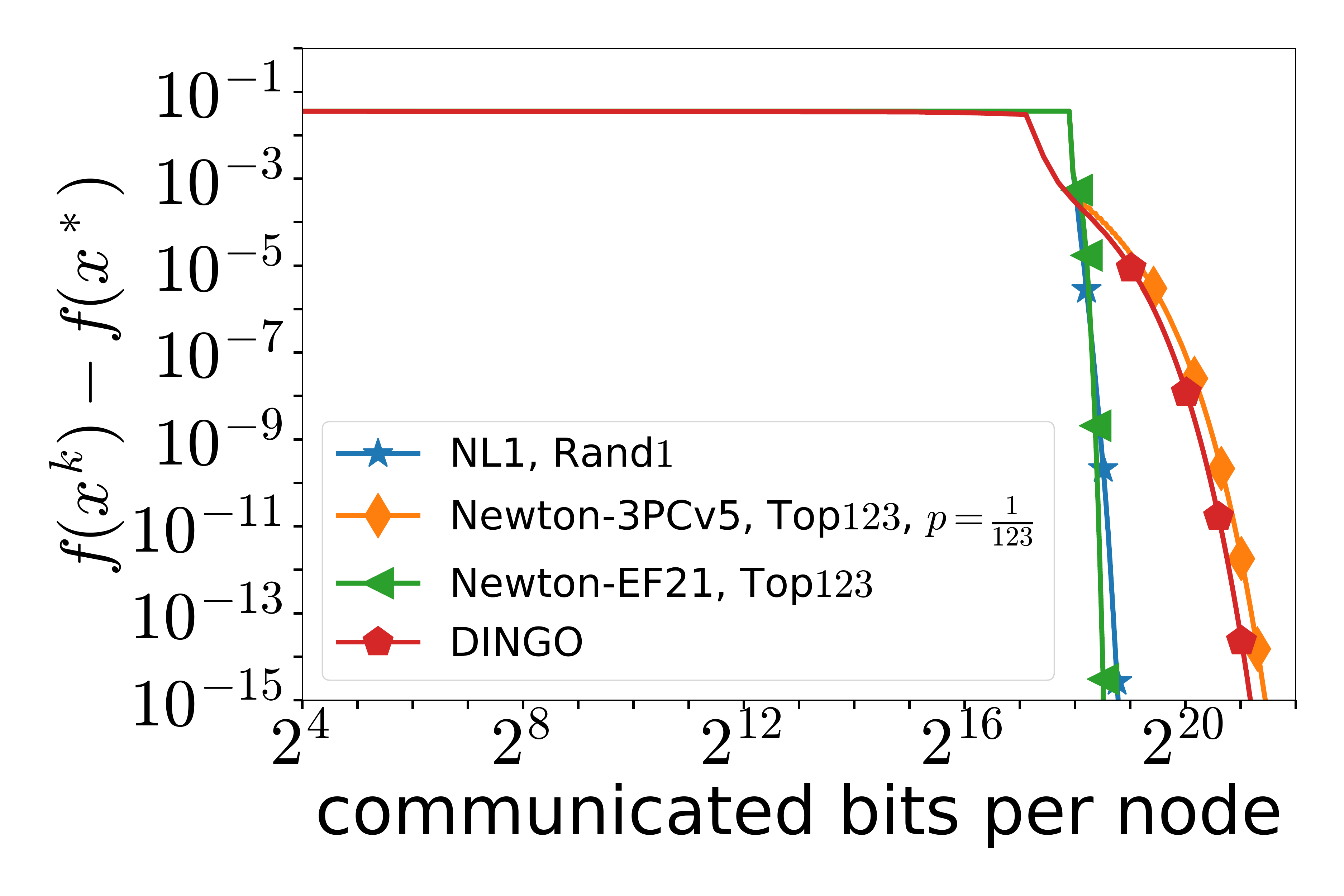} &
				\includegraphics[width=0.22\linewidth]{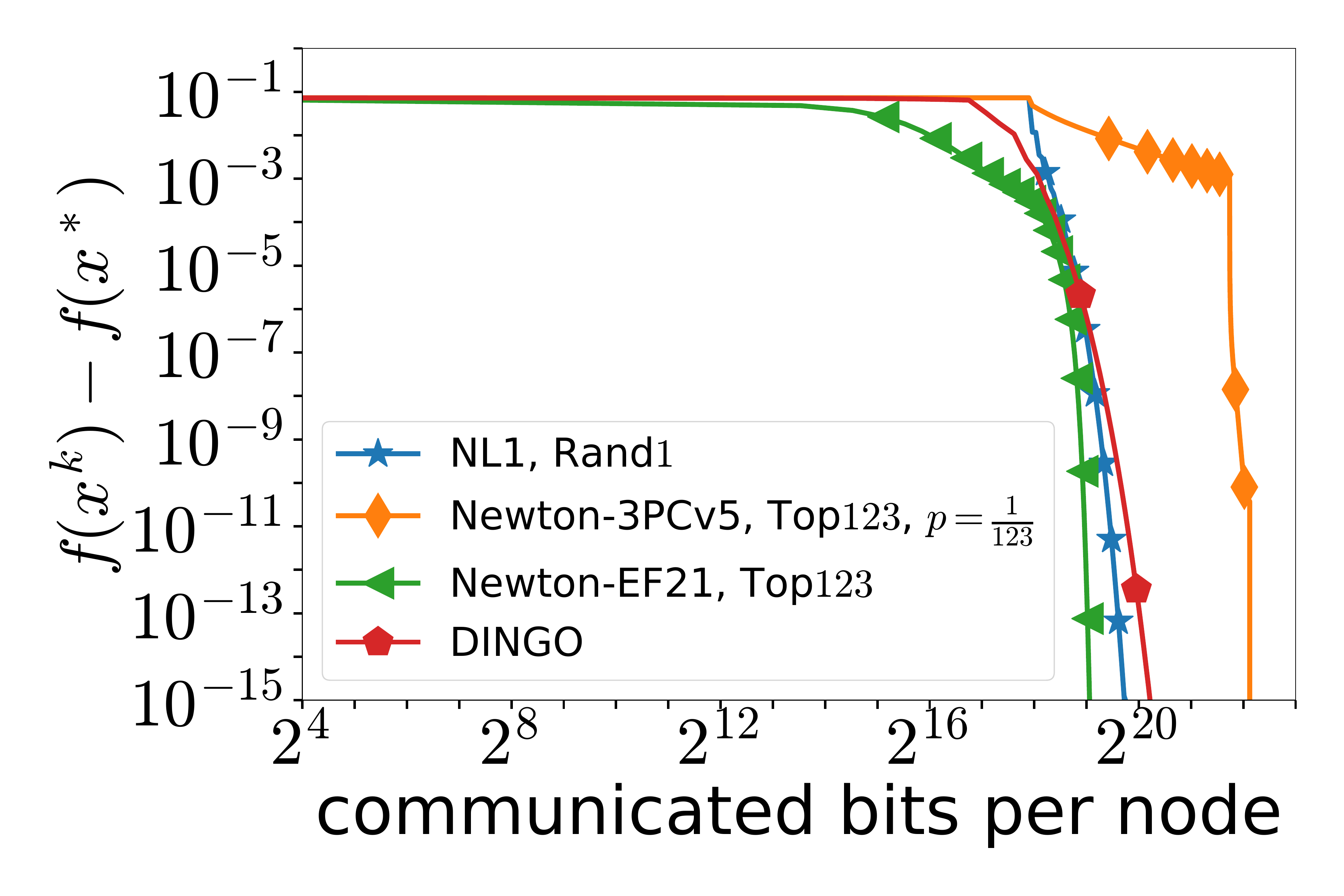} &
				\includegraphics[width=0.22\linewidth]{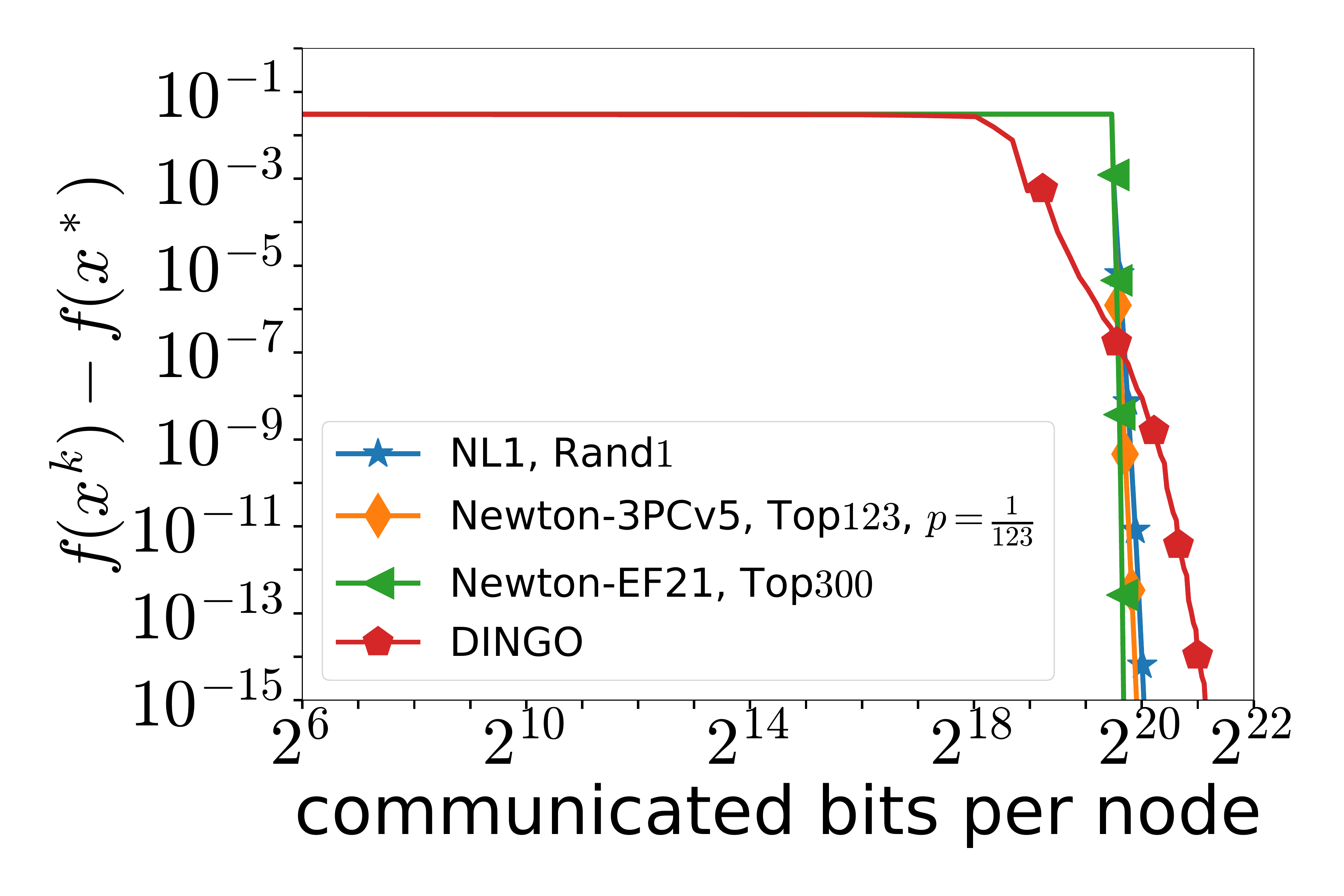}\\
				(a) \dataname{phishing}, {\scriptsize$ \lambda=10^{-4}$} &
				(b) \dataname{a1a}, {\scriptsize $\lambda=10^{-3}$} &
				(c) \dataname{a9a}, {\scriptsize$ \lambda=10^{-4}$} &
				(d) \dataname{w8a}, {\scriptsize$ \lambda=10^{-3}$} \\
			\end{tabular}       
		\end{center}
		\caption{The performance of \algname{Newton-3PCv5} with 3PCv5 based on Top-$d$, \algname{Newton-EF21} (equivalent to \algname{FedNL}) with Top-$d$, \algname{NL1} with Rand-$1$, and \algname{DINGO} in terms of communication complexity. }
		\label{fig:Newton-3PCv5}
	\end{figure}
	
	\subsection{\algname{Newton-3PC} with different choice of 3PC compression mechanism}
	
	Now we investigate how the choice of 3PC compressor influences the communication complexity of \algname{Newton-3PC}. We test the performance of \algname{Newton-3PC} with EF21, CLAG, LAG, 3PCv1 (based on Top-$K$), 3PCv2 (based on Top-$K_1$ and Rand-$K_2$), 3PCv4 (based on Top-$K_1$ and Top-$K_2$), and 3PCv5 (based on Top-$K$). We choose $p=\nicefrac{1}{d}$ for \algname{Newton-3PCv5} in order to make the communication cost of one iteration to be $\cO(d)$. The choice of $K$, $K_1$, and $K_2$ is justified by the same logic. 
	
	We clearly see that \algname{Newton-3PC} combined with EF21 (\algname{Newton-3PC} with this 3PC compressor reduces to \algname{FedNL}), CLAG, 3PCv2, 3PCv4 demonstrates almost identical results in terms of communication complexity. \algname{Newton-LAG} performs worse than previous methods except the case of \dataname{phishing} dataset. Surprisingly, \algname{Newton-3PCv1}, where only true Hessian differences is compressed, demonstrates the worst performance among all 3PC compression mechanisms. This probably caused by the fact that communication cost of one iteration of \algname{Newton-3PCv1} is significantly larger than those of other \algname{Newton-3PC} methods.

	\begin{figure}[t]
		\begin{center}
			\begin{tabular}{cccc}
				\includegraphics[width=0.22\linewidth]{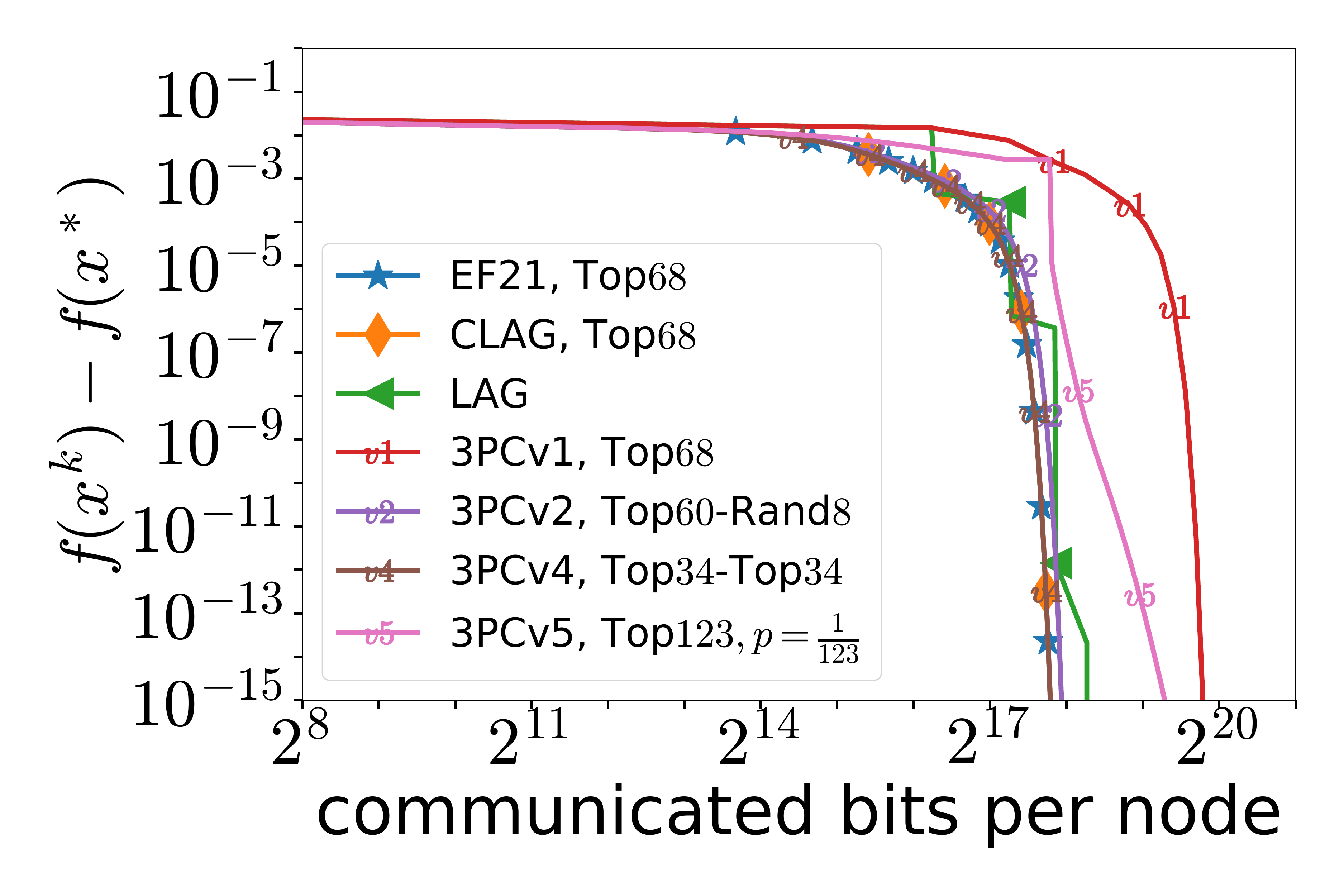} &
				\includegraphics[width=0.22\linewidth]{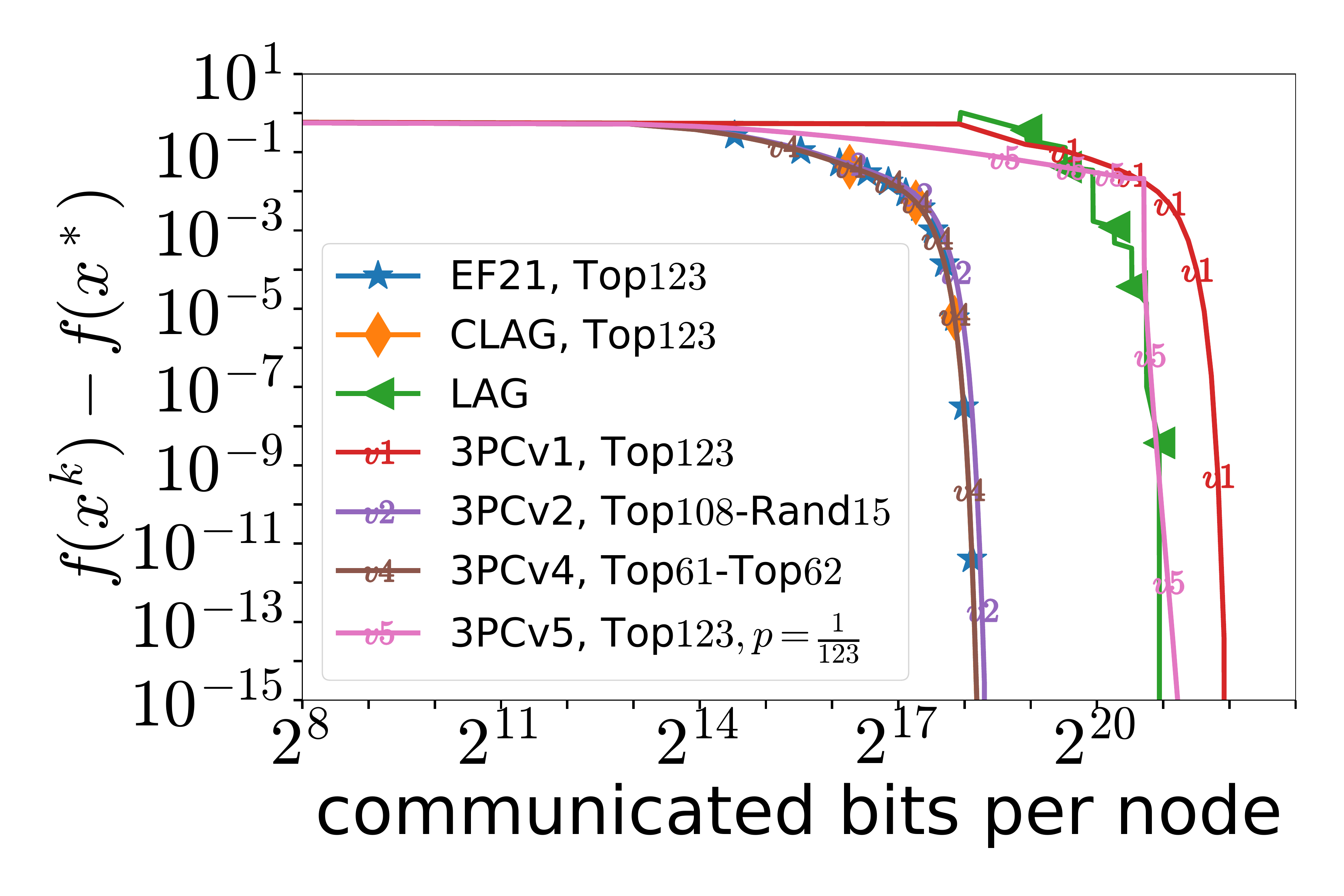} &
				\includegraphics[width=0.22\linewidth]{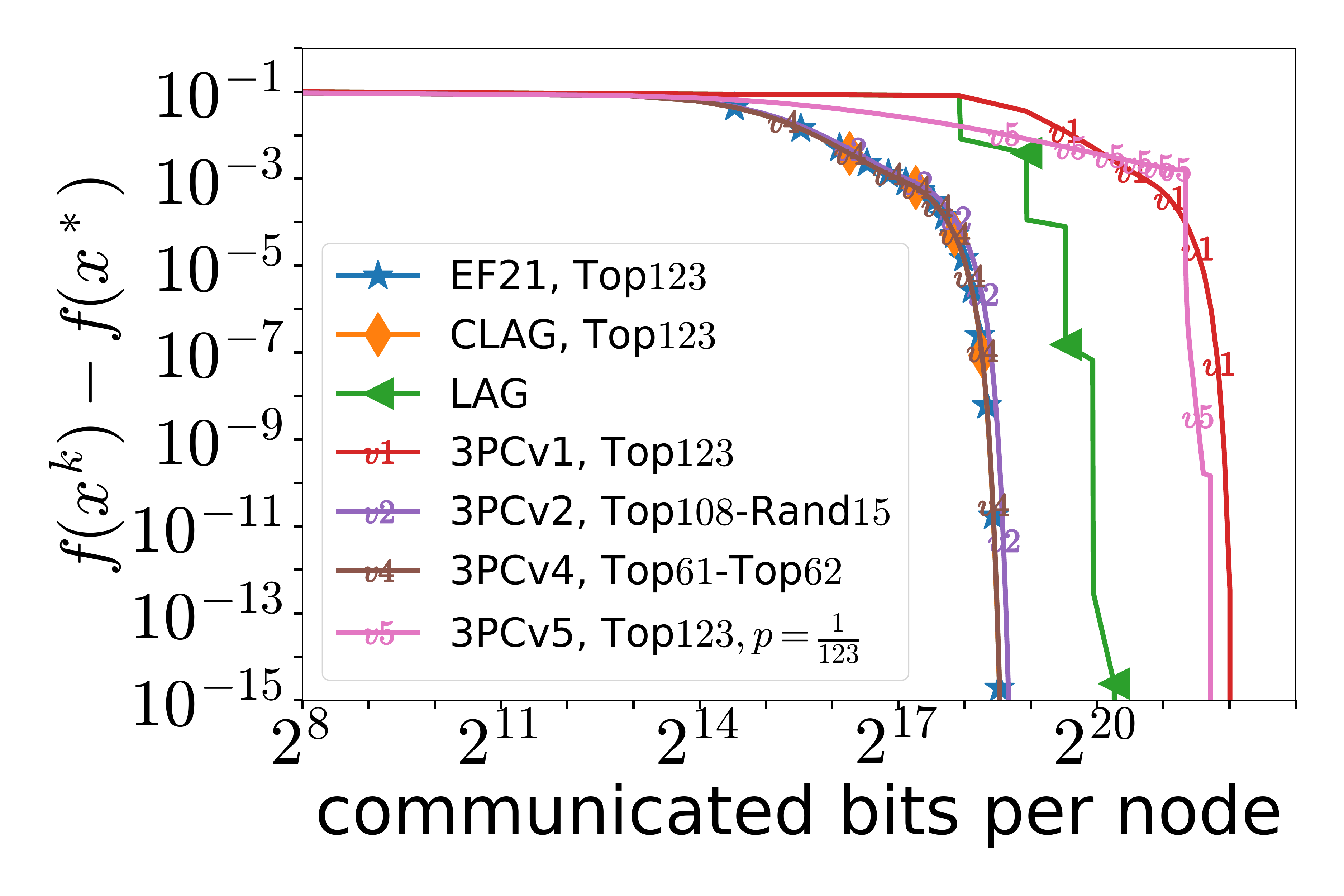} &
				\includegraphics[width=0.22\linewidth]{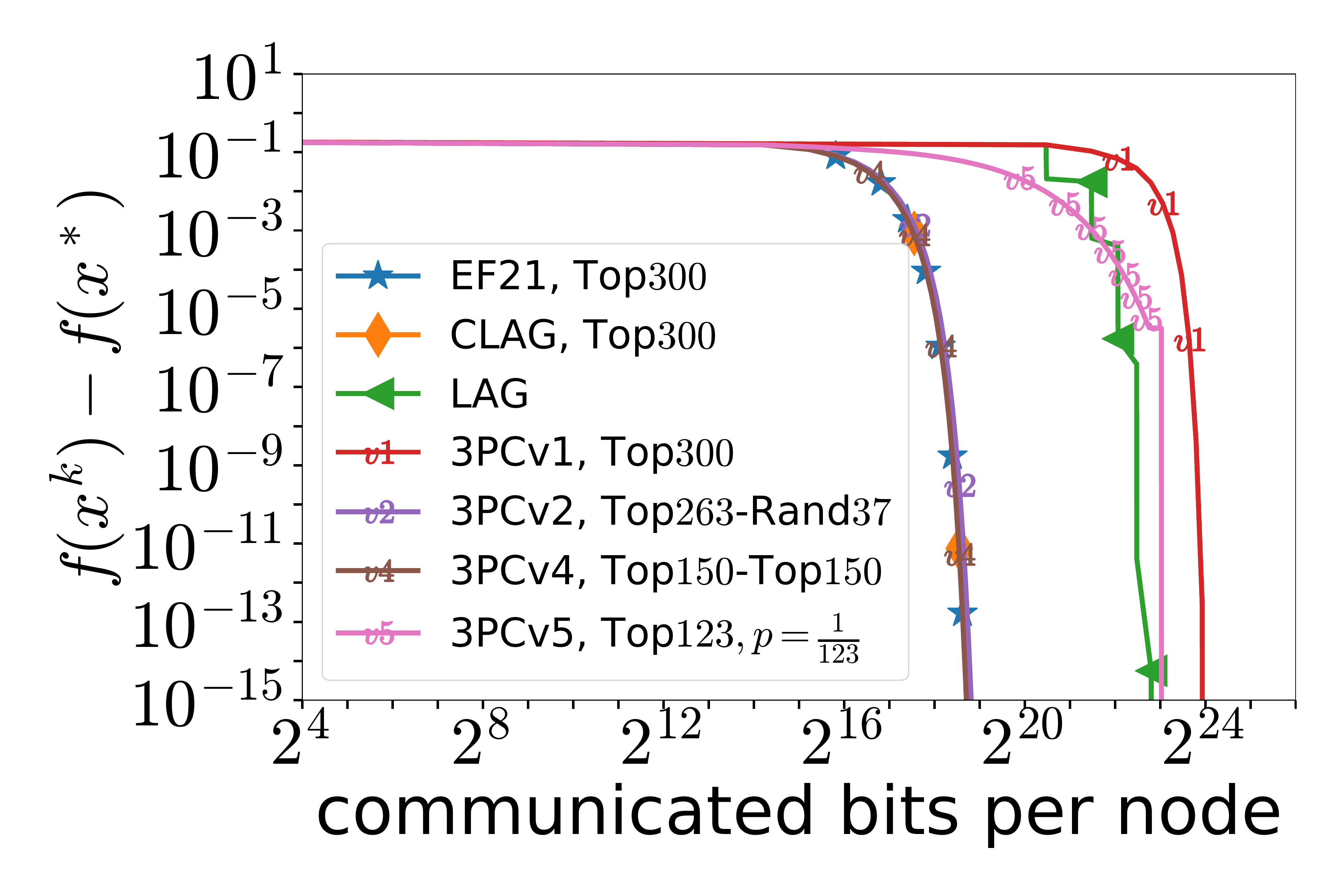}\\
				(a) \dataname{phishing}, {\scriptsize$ \lambda=10^{-4}$} &
				(b) \dataname{a1a}, {\scriptsize $\lambda=10^{-3}$} &
				(c) \dataname{a9a}, {\scriptsize$ \lambda=10^{-4}$} &
				(d) \dataname{w8a}, {\scriptsize$ \lambda=10^{-3}$} \\
			\end{tabular}       
		\end{center}
		\caption{The performance of \algname{Newton-3PC} with different choice of 3PC compression mechanism in terms of communication complexity. }
		\label{fig:Newton-3PC-different-3PC}
	\end{figure}

	\subsection{Analysis of Bidirectional \algname{Newton-3PC}}
	
	\subsubsection{EF21 compression mechanism}
	
	In this section we analyze how each type of compression (Hessians, iterates, and gradients) affects the performance of \algname{Newton-3PC}. In particular, we choose \algname{Newton-EF21} (equivalent to \algname{FedNL}) and change parameters of each compression mechanism. For Hessians and iterates we use Top-$K_1$ and Top-$K_2$ compressors respectively. In Figure~\ref{fig:Newton-EF21-BC} we present the results when we vary the parameter $K_1, K_2$ of Top-$K$ compressor and probability $p$ of Bernoulli Aggregation. The results are presented as heatmaps indicating the number of Mbytes transmitted in uplink and downlink directions by each client.
	
	In the first row in Figure~\ref{fig:Newton-EF21-BC} we test different combinations of compression parameters for Hessians and iterates keeping the probability $p$ of BAG for gradients to be equal $0.5$. In the second row we analyze various combinations of pairs of parameters $(K, p)$ for Hessians and gradients when the compression on iterates is not applied. Finally, the third row corresponds to the case when Hessians compression is fixed (we use Top-$d$), and we vary pairs of parameters $(K, p)$ for iterates and gradients compression.
	
	According to the results in the heatmaps, we can conclude that \algname{Newton-EF21} benefits from the iterates compression. Indeed, in both cases (when we vary compression level applied on Hessians or gradients) the best result is given in the case when we do apply the compression on iterates. This is not the case for gradients (see second row) since the best results are given for high probability $p$; usually for $p=1$ and rarely for $p=0.75$. Nevertheless, we clearly see that bidirectional compression is indeed useful in almost all cases.
	
	\begin{figure}[t]
		\begin{center}
			\begin{tabular}{ccc}
				\includegraphics[width=0.22\linewidth]{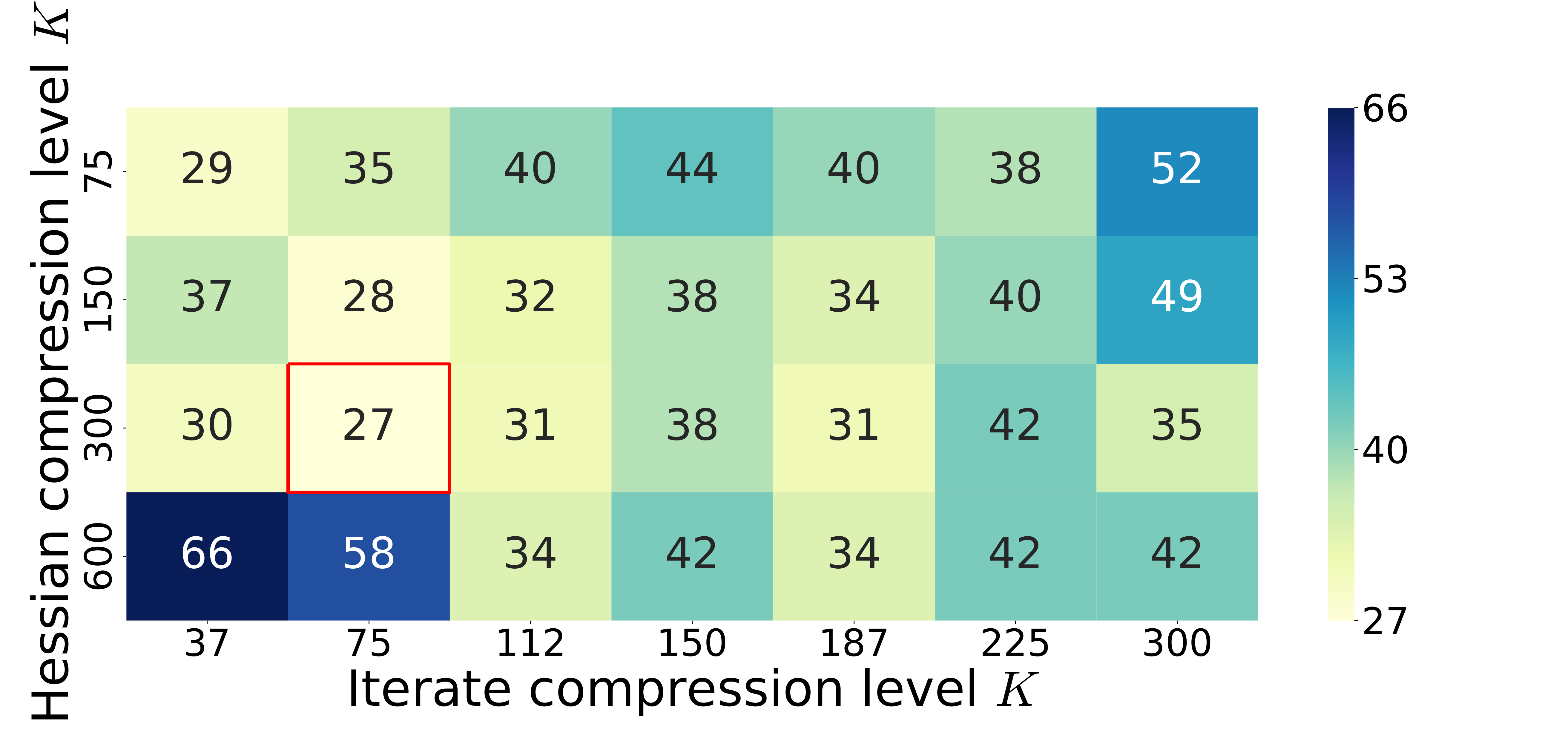} &
				\includegraphics[width=0.22\linewidth]{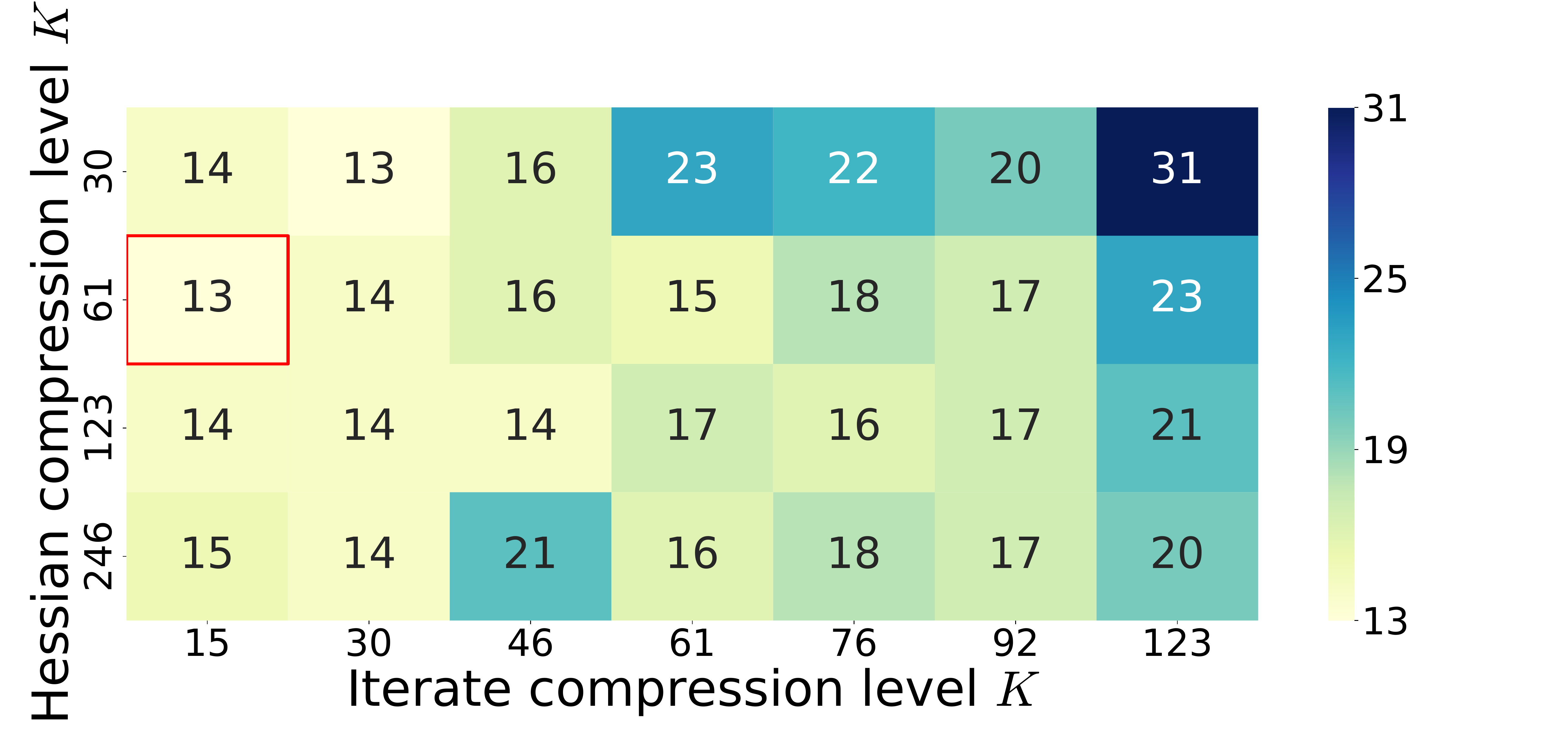} &
				\includegraphics[width=0.22\linewidth]{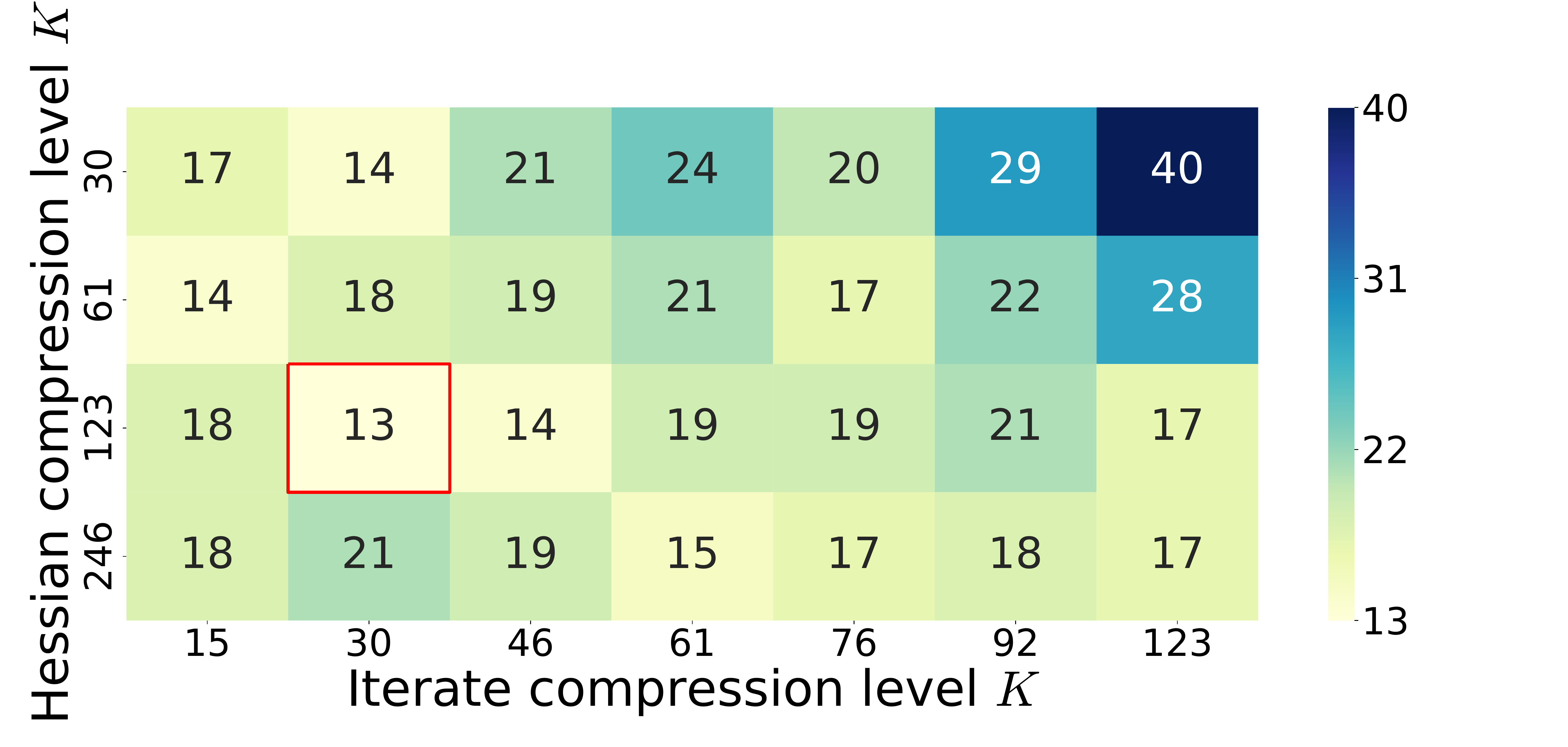}\\
				(a1) \dataname{w2a}, {\scriptsize$ \lambda=10^{-4}$} &
				(b1) \dataname{a9a}, {\scriptsize $\lambda=10^{-3}$} &
				(c1) \dataname{a1a}, {\scriptsize$ \lambda=10^{-4}$}\\
				\includegraphics[width=0.22\linewidth]{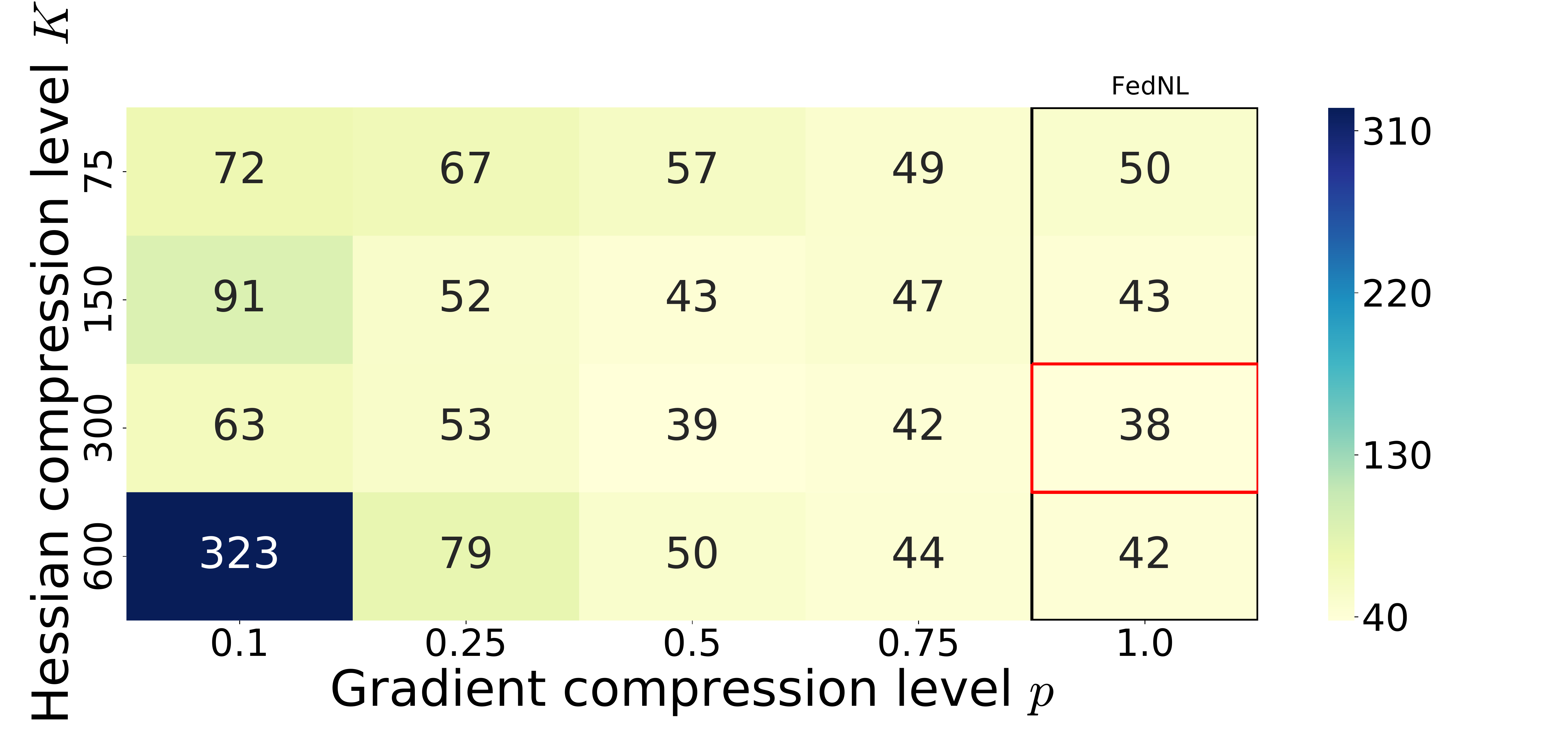} &
				\includegraphics[width=0.22\linewidth]{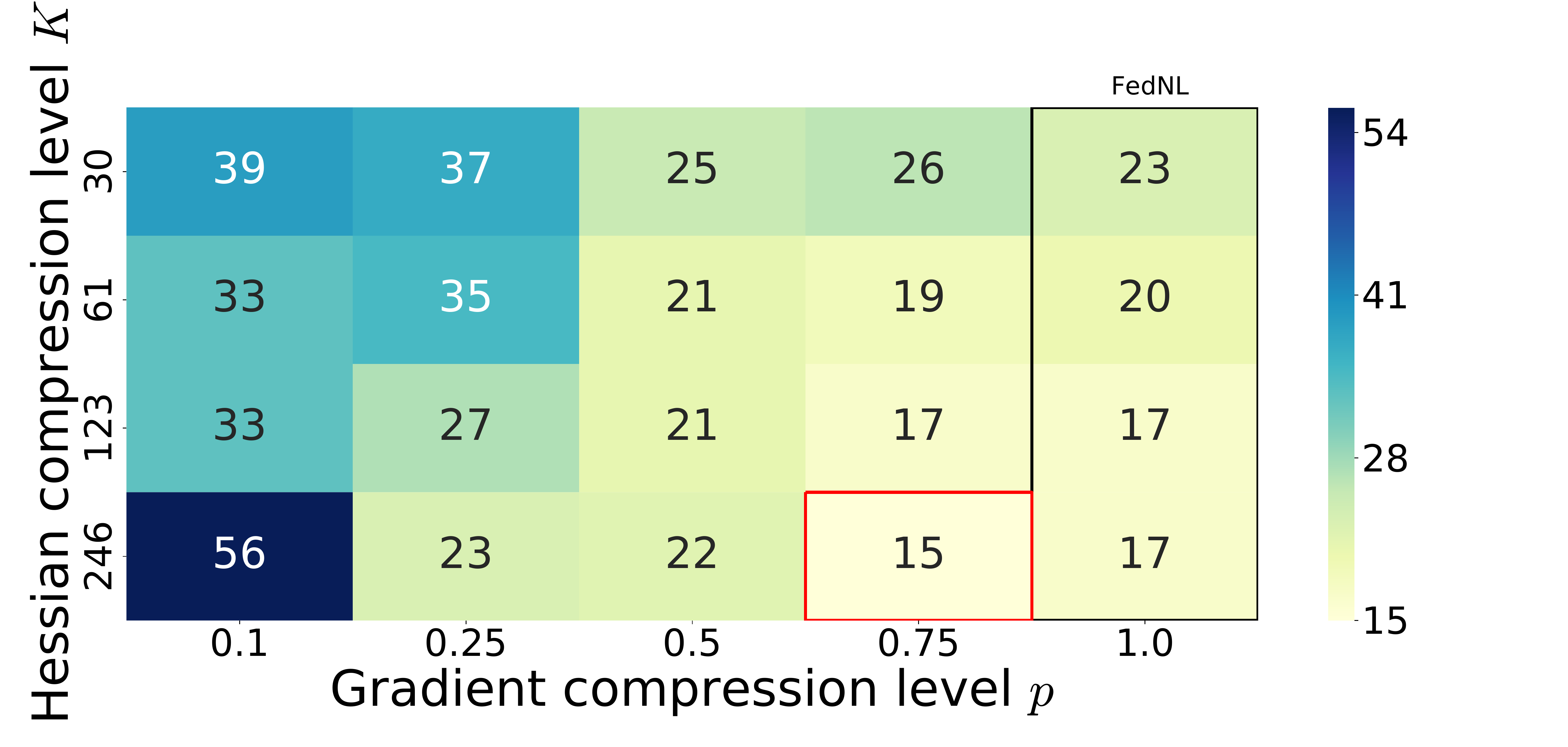} &
				\includegraphics[width=0.22\linewidth]{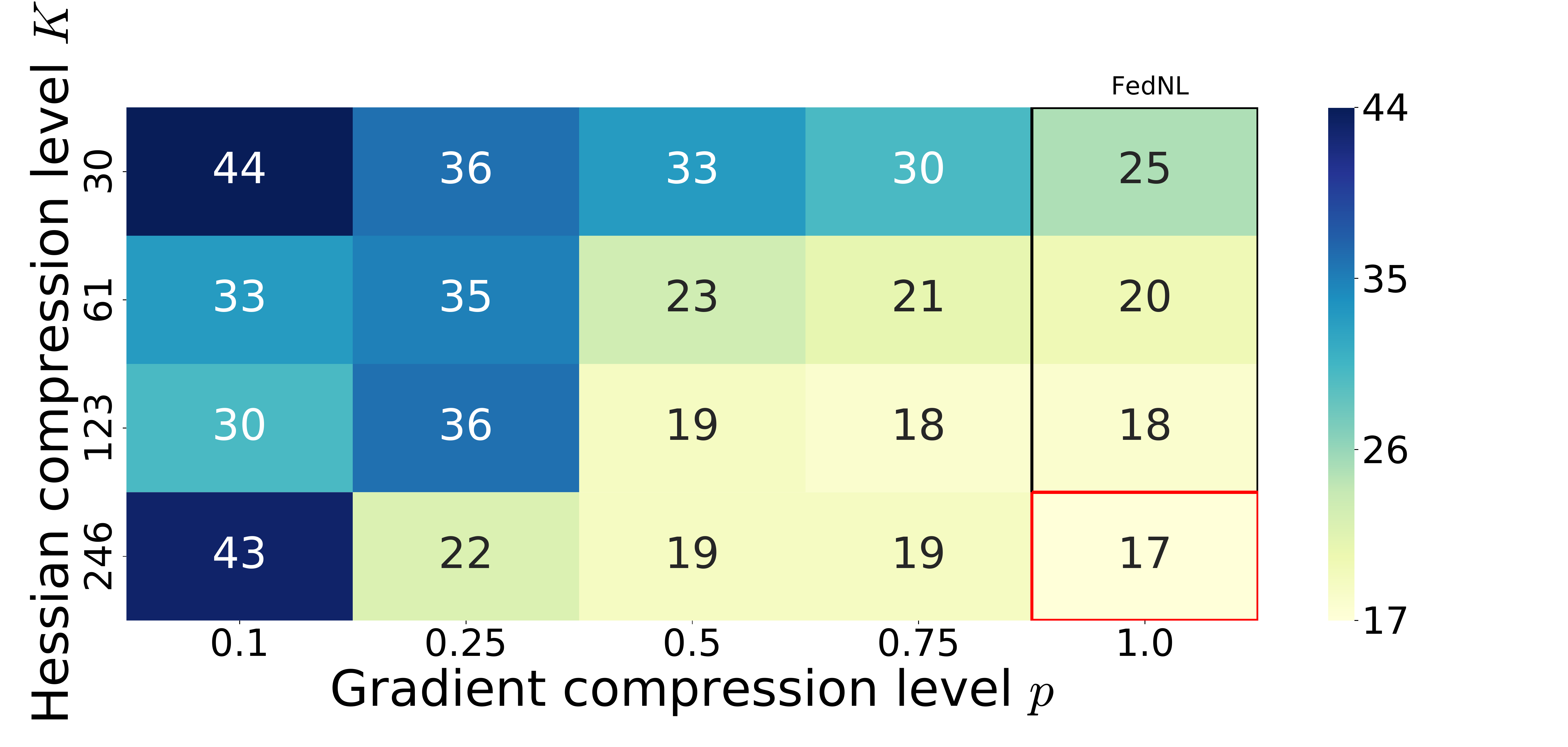}\\
				(a2) \dataname{w2a}, {\scriptsize$ \lambda=10^{-4}$} &
				(b2) \dataname{a9a}, {\scriptsize $\lambda=10^{-3}$} &
				(c2) \dataname{a1a}, {\scriptsize$ \lambda=10^{-4}$}\\
				\includegraphics[width=0.22\linewidth]{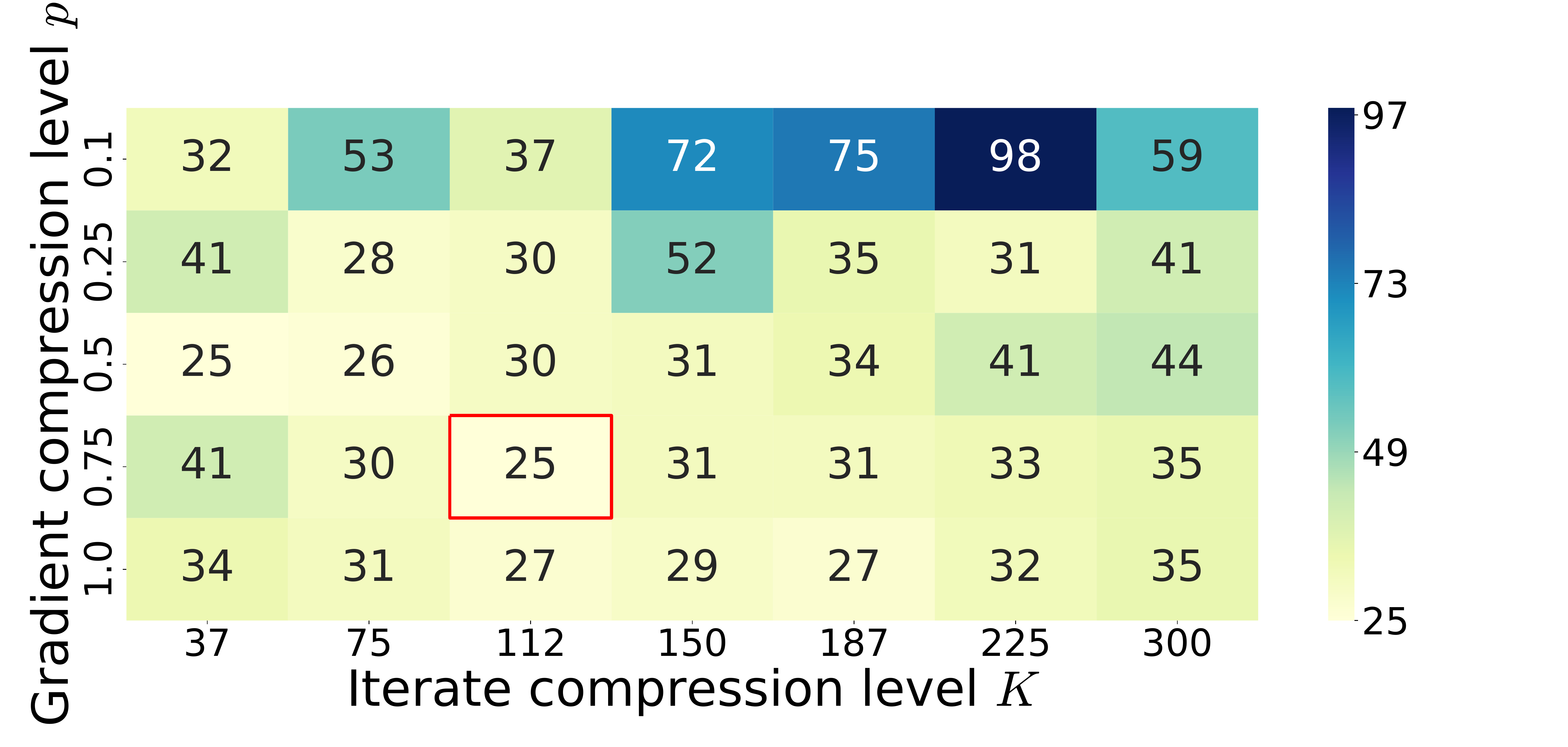} &
				\includegraphics[width=0.22\linewidth]{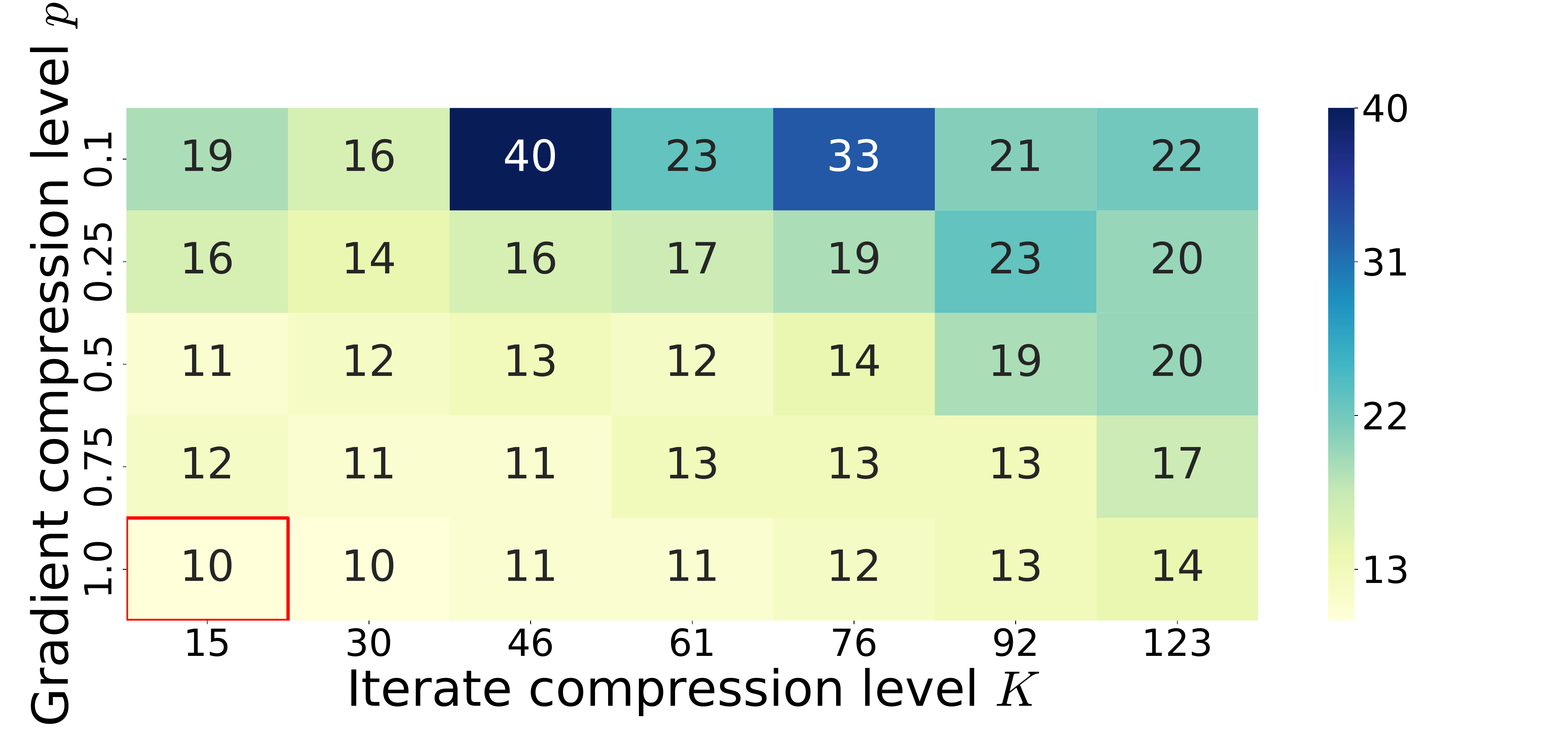} &
				\includegraphics[width=0.22\linewidth]{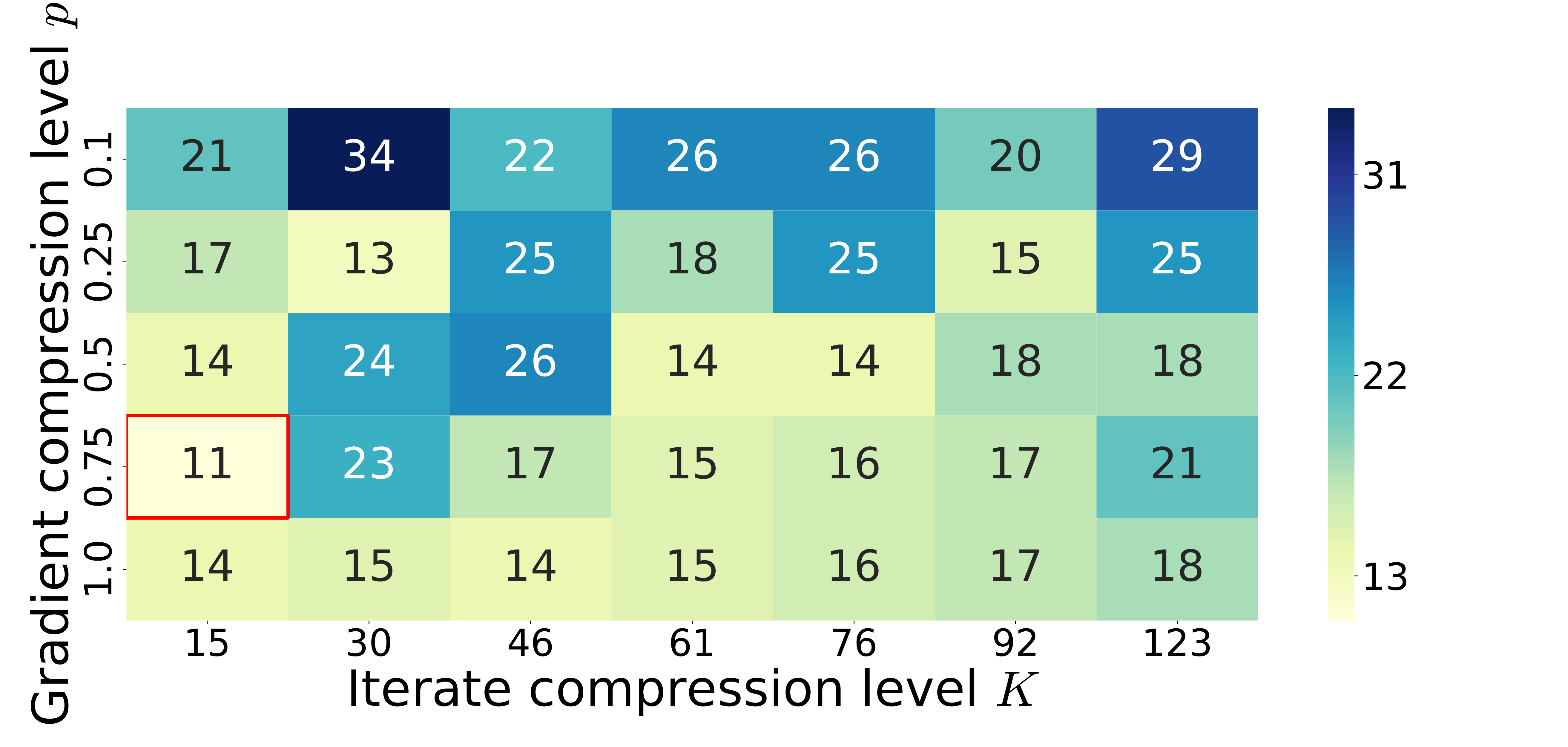}\\
				(a3) \dataname{w2a}, {\scriptsize$ \lambda=10^{-4}$} &
				(b3) \dataname{a9a}, {\scriptsize $\lambda=10^{-3}$} &
				(c3) \dataname{a1a}, {\scriptsize$ \lambda=10^{-4}$}\\
			\end{tabular}       
		\end{center}
		\caption{{\bf First row:} The performance of \algname{Newton-3PC-BC} in terms of communication complexity for different values of $(K_1, K_2)$ of Top-$K_1$ and Top-$K_2$ compressors applied on Hessians and iterates respectively while probability $p=0.75$ of BAG applied on gradients is fixed. {\bf Second row:} The performance of \algname{Newton-EF21} in terms of communication complexity for different values of $(K_1, p)$ of Top-$K_1$ compressor applied on Hessians and probability $p$ of BAG applied on gradients while $K_2=d$ parameter of Top-$K_2$ applied on iterates is fixed.  {\bf Third row:} The performance of \algname{Newton-EF21} in terms of communication complexity for different values of $(K_2, p)$ of Top-$K_2$ compressor applied on iterates and probability $p$ of BAG applied on gradients while $K_1=d$ parameter of Top-$K_1$ applied on Hessians is fixed.}
		\label{fig:Newton-EF21-BC}
	\end{figure}
	
	\subsubsection{3PCv4 compression mechanism}
	
	In our next set of experiments we fix EF21 compression mechanism based on Top-$d$ compressor applied on Hessians and probability $p=0.75$ of Bernoulli aggregation applied on gradients. Now we use 3PCv4 update rule on iterates based on outer and inner compressors (Top-$K_1$, Top-$K_2$) varying the values of pairs $(K_1, K_2)$. We report the results as heatmaps in Figure~\ref{fig:Newton-3PCv4-BC}.

	We observe that in all cases it is better to apply relatively smaller outer and inner compression levels as this leads to better performance in terms of communication complexity. Note that the first row in heatmaps corresponds to \algname{Newton-3PC-BC} when we apply just EF21 update rule on iterates. As a consequence, \algname{Newton-3PC-BC} reduces to \algname{FedNL-BC} method \citep{FedNL2021}. We obtain that 3PCv4 compression mechanism applied on iterates in this setting is more communication efficient than EF21. This implies the fact that \algname{Newton-3PC-BC} could be more efficient than \algname{FedNL-BC} in terms of communication complexity.
	
	\begin{figure}[t]
		\begin{center}
			\begin{tabular}{cccc}
				\includegraphics[width=0.22\linewidth]{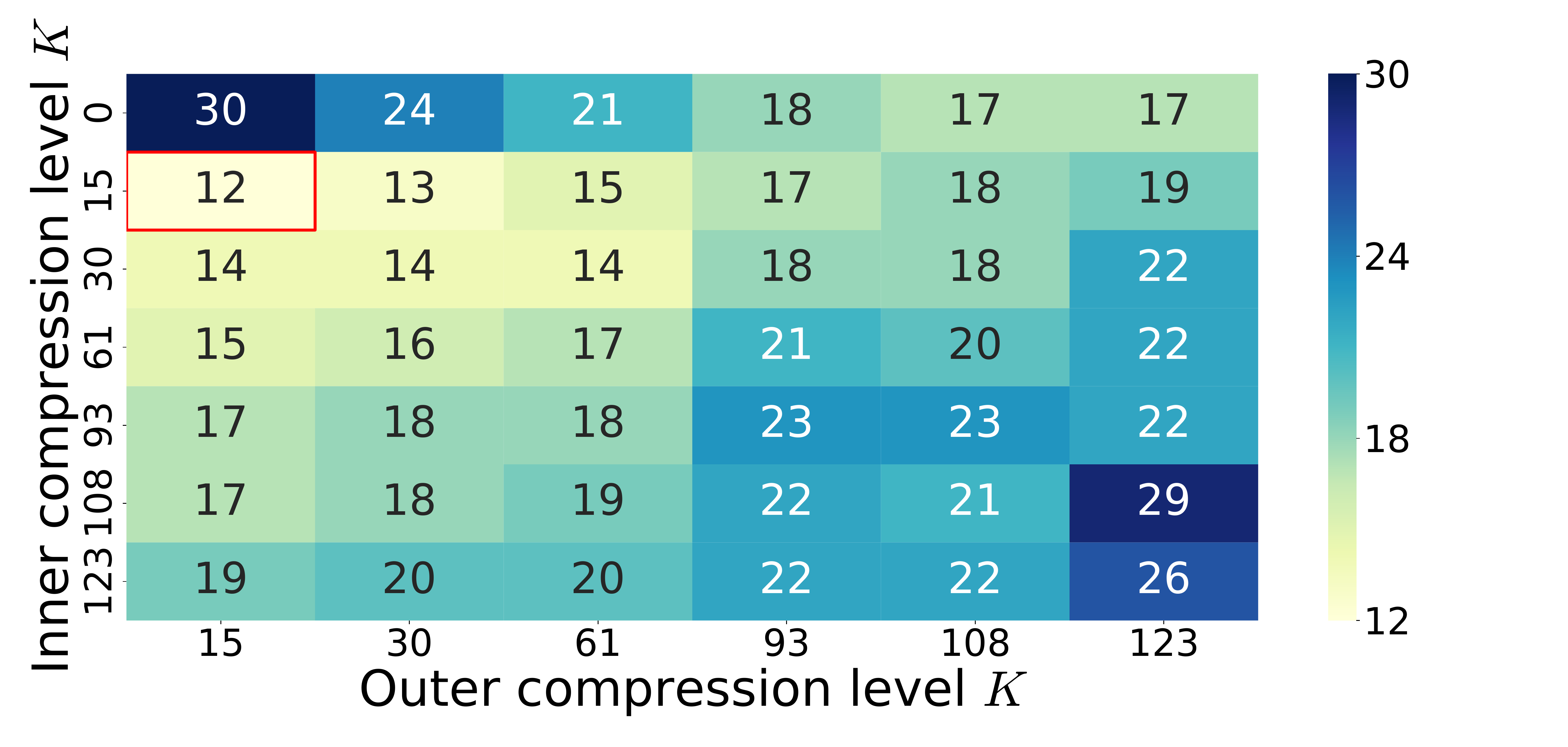} &
				\includegraphics[width=0.22\linewidth]{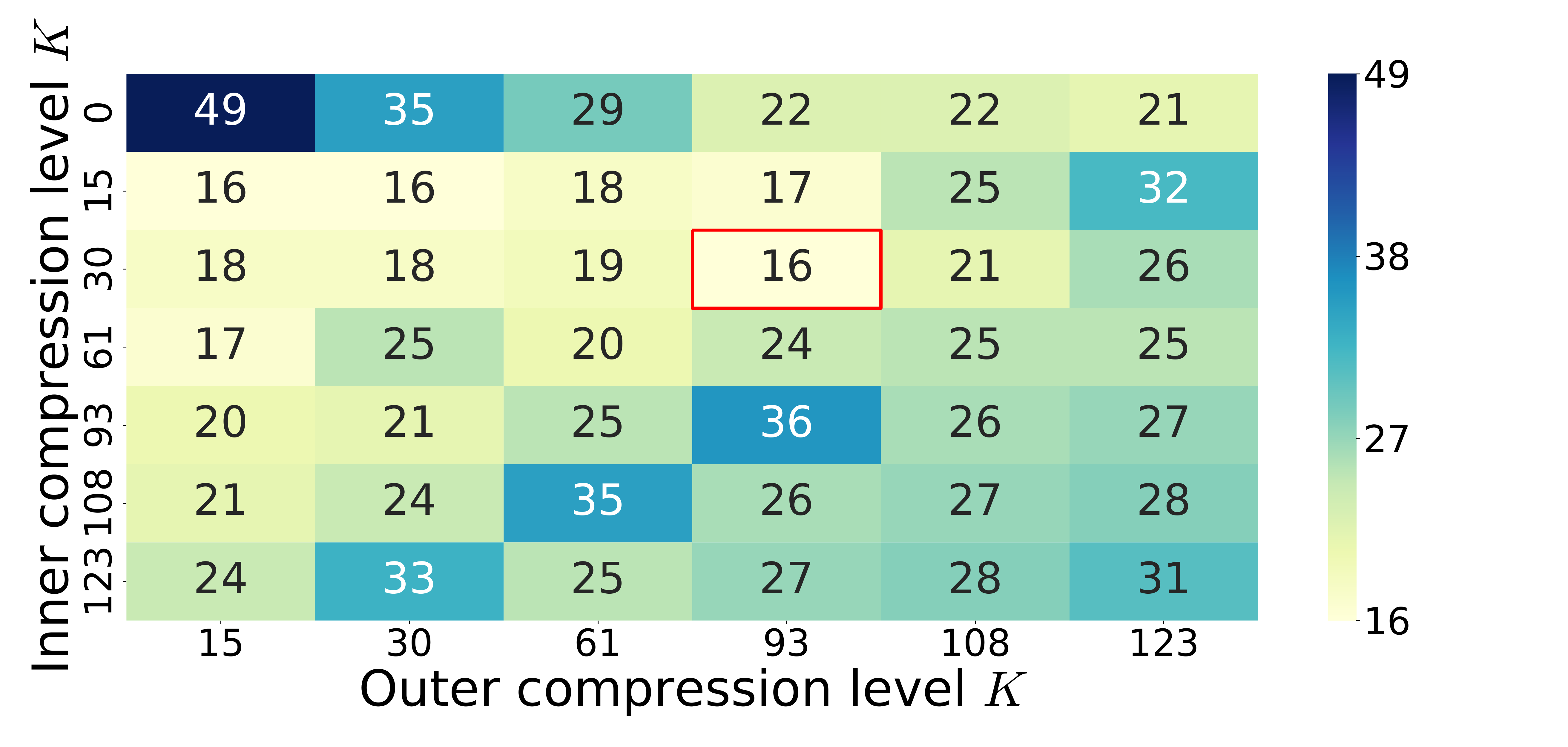} &
				\includegraphics[width=0.22\linewidth]{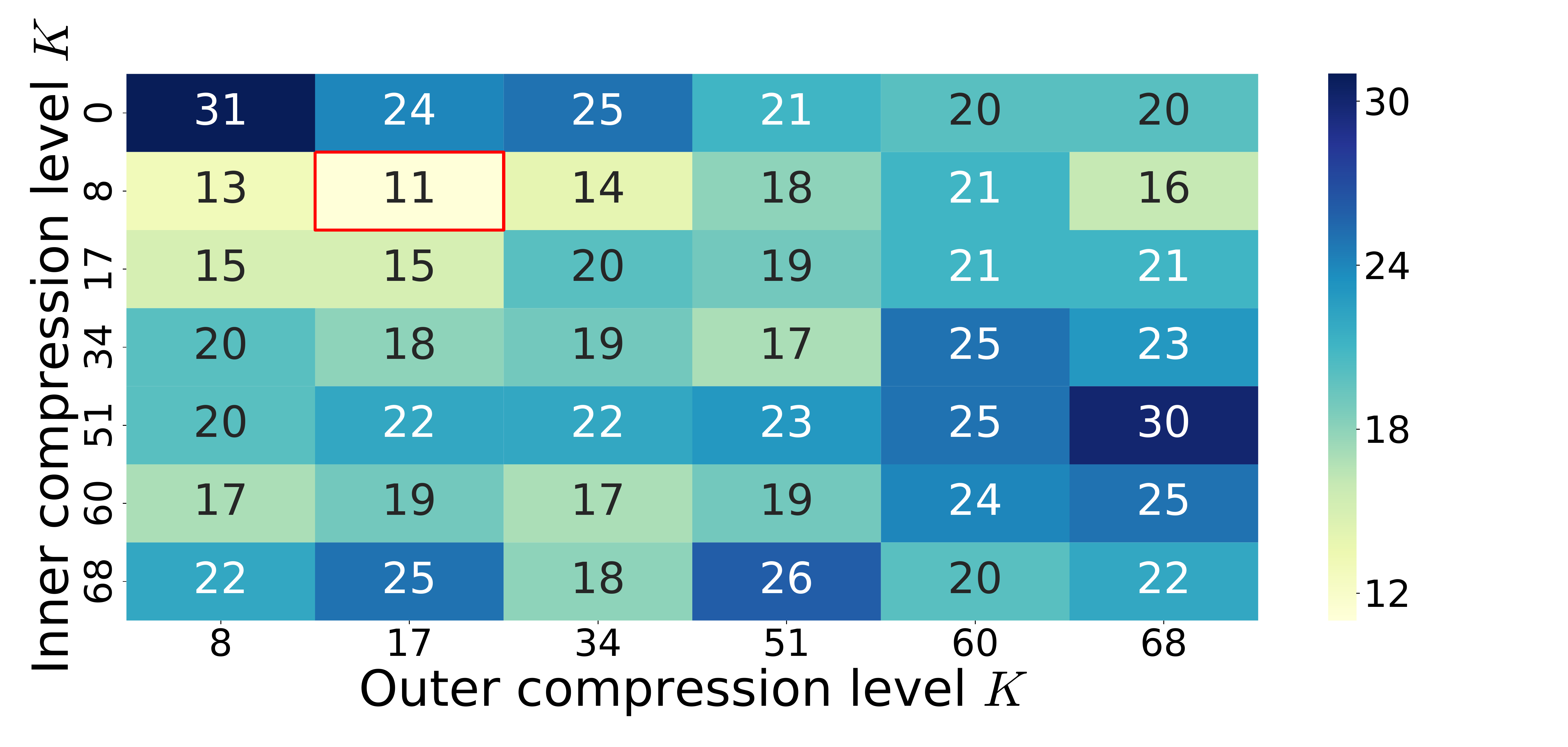} &
				\includegraphics[width=0.22\linewidth]{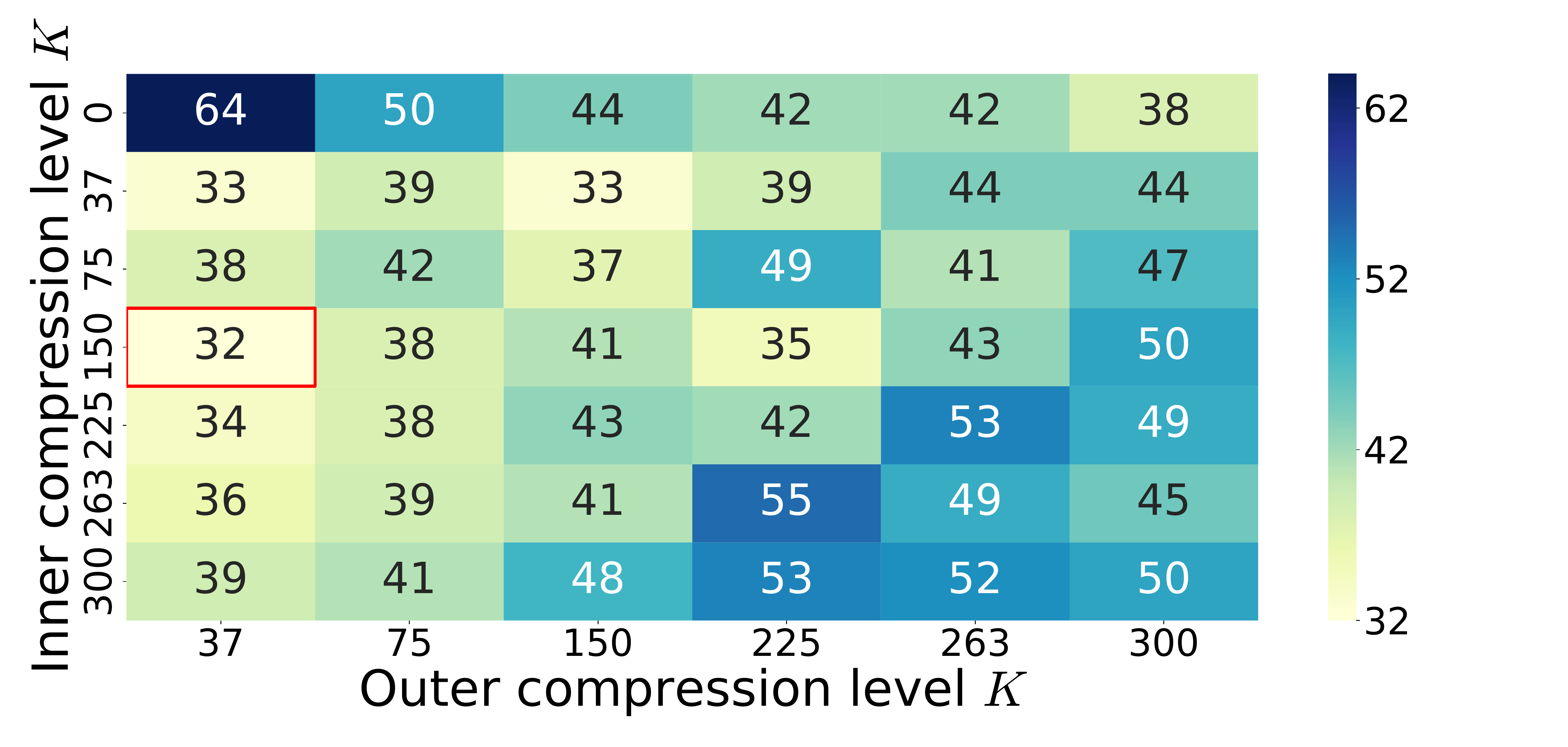}\\
				(a) \dataname{a9a}, {\scriptsize$ \lambda=10^{-3}$} &
				(b) \dataname{a1a}, {\scriptsize$ \lambda=10^{-4}$} &
				(c) \dataname{phishing}, {\scriptsize $\lambda=10^{-3}$} &
				(d) \dataname{a1a}, {\scriptsize$ \lambda=10^{-4}$}\\
			\end{tabular}       
		\end{center}
		\caption{The performance of \algname{Newton-3PC-BC} with EF21 update rule based on Top-$d$ compressor applied on Hessians, BAG update rule with probability $p=0.75$ applied on gradients, and 3PCv4 update rule based on (Top-$K_1$, Top-$K_2$) compressors applied on iterates for different values of pairs $(K_1, K_2)$. }
		\label{fig:Newton-3PCv4-BC}
	\end{figure}

	\subsection{\algname{BL1} \citep{qian2021basis} with 3PC compressor}
	
	As it was stated in Section $4.1$ \algname{Newton-3PC} covers methods introduced in \citep{qian2021basis} as a special case. Indeed, in order to run, for example, \algname{BL1} method we need to use rotation compression operator \ref{def:rotation_comp}. The role of orthogonal matrix in the definition plays the basis matrix. 
	
	In this section we test the performance of \algname{BL1} in terms of communication complexity with different 3PC compressors: EF21, CBAG, CLAG. For CBAG update rule the probability $p=0.5$, and for CLAG the trigger $\zeta=2$. All aforementioned 3PC compression operators are based on Top-$\tau$ compressor where $\tau$ is the dimension of local data (see Section~$2.3$ of \citep{qian2021basis} for detailed description).
	
	Observing the results in Figure~\ref{fig:BL1_3PC}, we can notice that there is no improvement of one update rule over another. However, in EF21 is slightly better than other 3PC compressors in a half of the cases, and CBAG insignificantly outperform in other cases. This means that even if the performance of \algname{BL1} with EF21 and CBAG are almost identical, CBAG is still preferable since it is computationally less expensive since we do not need to compute local Hessians and their representations in new basis.
	
	\begin{figure}[t]
		\begin{center}
			\begin{tabular}{cccc}
				\includegraphics[width=0.22\linewidth]{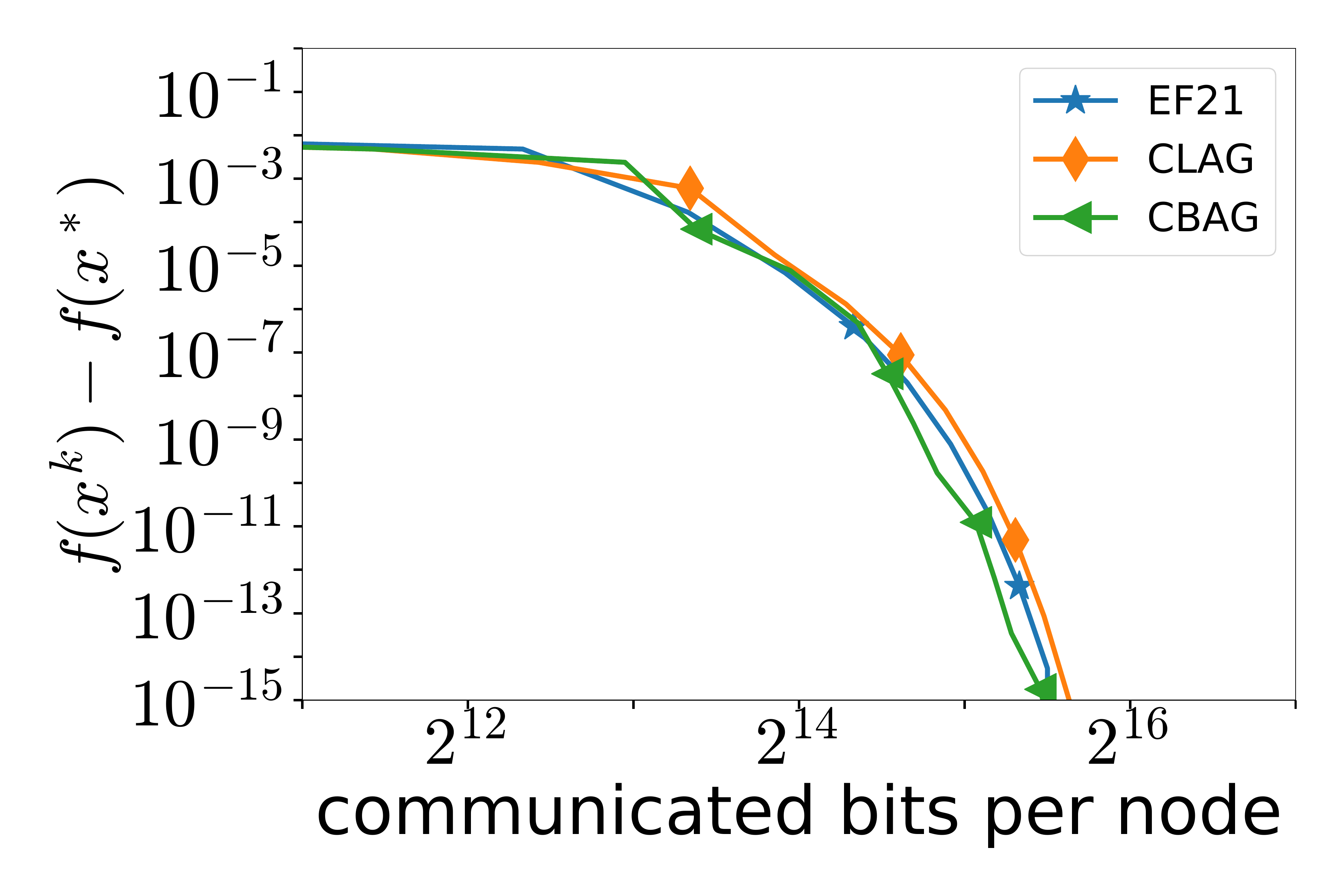} &
				\includegraphics[width=0.22\linewidth]{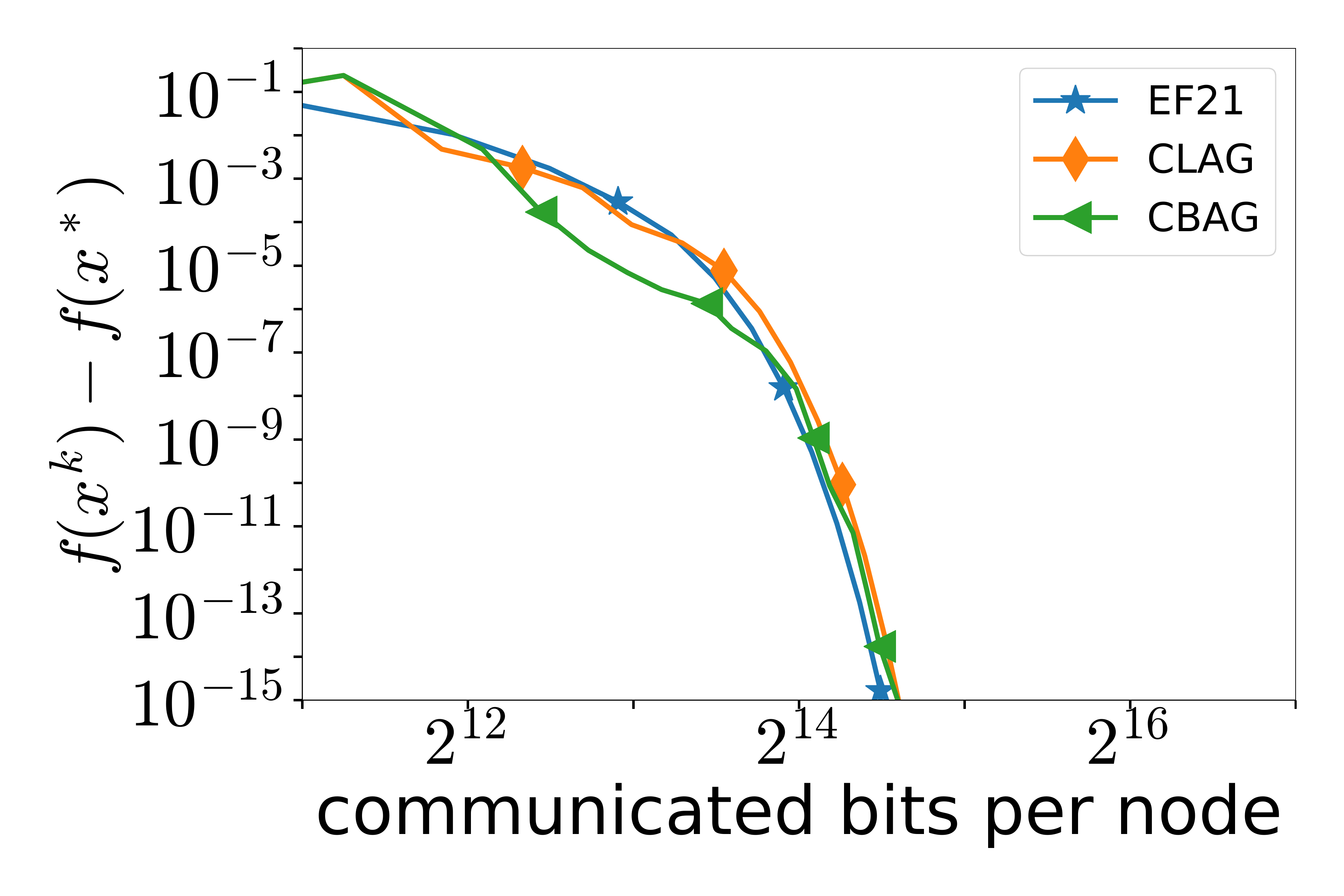} &
				\includegraphics[width=0.22\linewidth]{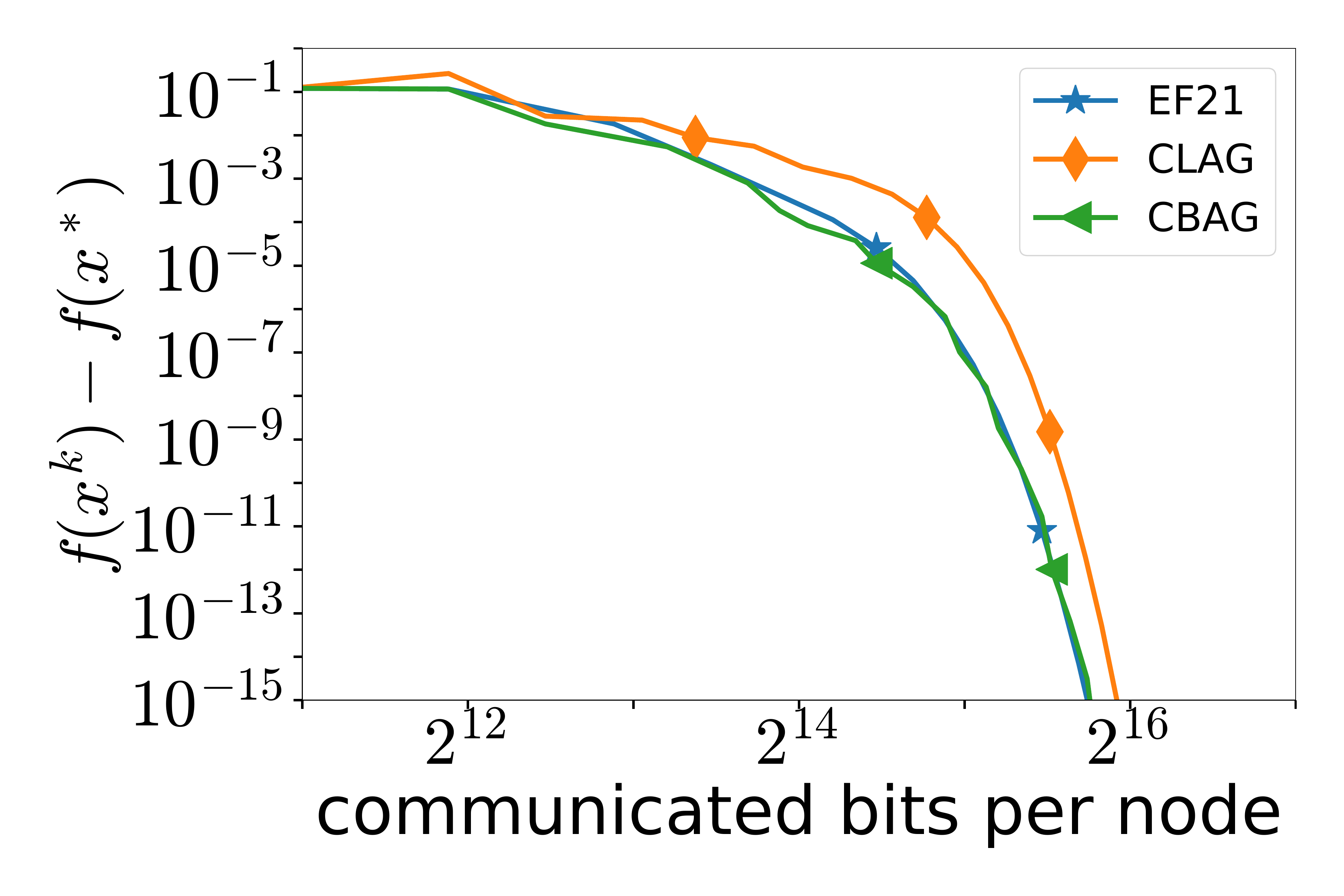} &
				\includegraphics[width=0.22\linewidth]{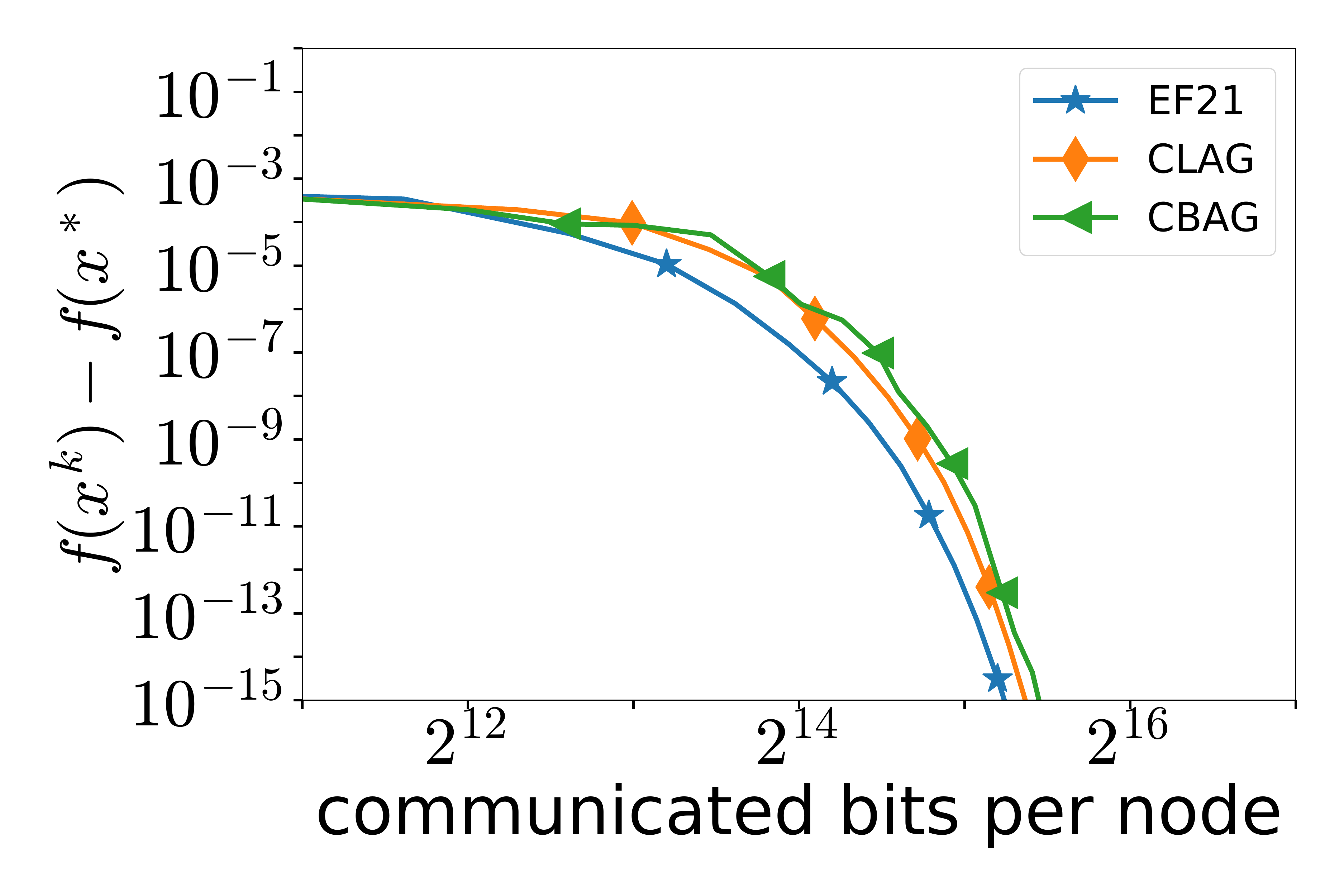}\\
				(a) \dataname{a9a}, {\scriptsize$ \lambda=10^{-3}$} &
				(b) \dataname{phishing}, {\scriptsize$ \lambda=10^{-4}$} &
				(c) \dataname{a1a}, {\scriptsize $\lambda=10^{-3}$} &
				(d) \dataname{w2a}, {\scriptsize$ \lambda=10^{-4}$}\\
			\end{tabular}       
		\end{center}
		\caption{The performance of \algname{BL1} with EF21, CBAG and CLAG 3PC compression mechanisms in terms of communication complexity.}
		\label{fig:BL1_3PC}
	\end{figure}

	\subsection{Analysis of \algname{Newton-3PC-BC-PP}}
	
	\subsubsection{3PC's parameters fine-tuning for \algname{Newton-3PC-BC-PP}}
	
	On the following step we study how the choice of parameters of 3PC compression mechanism and the number of active clients influence the performance of \algname{Newton-3PC-BC-PP}. 
	
	In the first series of experiments we test \algname{Newton-3PC-BC-PP} with CBAG compression combined with Top-$2d$ compressor and probability $p$ applied on Hessians; EF21 with Top-$\nicefrac{2d}{3}$ compressor applied on iterates; BAG update rule with probability $p=0.75$ applied on gradients. We vary aggregation probability $p$ of Hessians and the number of active clients $\tau$. Looking at the numerical results in Figure~\ref{fig:Newton-3PC-BC-PP} (first row), we may claim that the more clients are involved in the optimization process in each communication round, the faster the convergence since the best results in each case always belongs the first column. However, we do observe that lazy aggregation rule with probability $p < 1$ is still beneficial.
	
	In the second row of Figure~\ref{fig:Newton-3PC-BC-PP} we investigate \algname{Newton-3PC-BC-PP} with CBAG compression based on Top-$d$ and probability $p=0.75$ applied on Hessians; 3PCv5 update rule combined with Top-$\nicefrac{2d}{3}$ and probability $p$ applied on iterates; BAG lazy aggregation rule with probability $p-0.75$ applied gradients. In this case we modify iterate aggregation probability $p$ and the number of clients participating in the training. We observe that again the fastest convergence is demonstrated when all clients are active, but aggregation parameter $p$ of iterates smaller than $1$.
	
	Finally, we study the effect of BAG update rule on the communication complexity of \algname{Newton-3PC-BC-PP}. As in previous cases, \algname{Newton-3PC-BC-PP} is more efficient when all clients participate in the training process. Nevertheless, lazy aggregation rule of BAG still brings the benefit to communication complexity of the method. 
	
	\begin{figure}[t]
		\begin{center}
			\begin{tabular}{ccc}
				\includegraphics[width=0.22\linewidth]{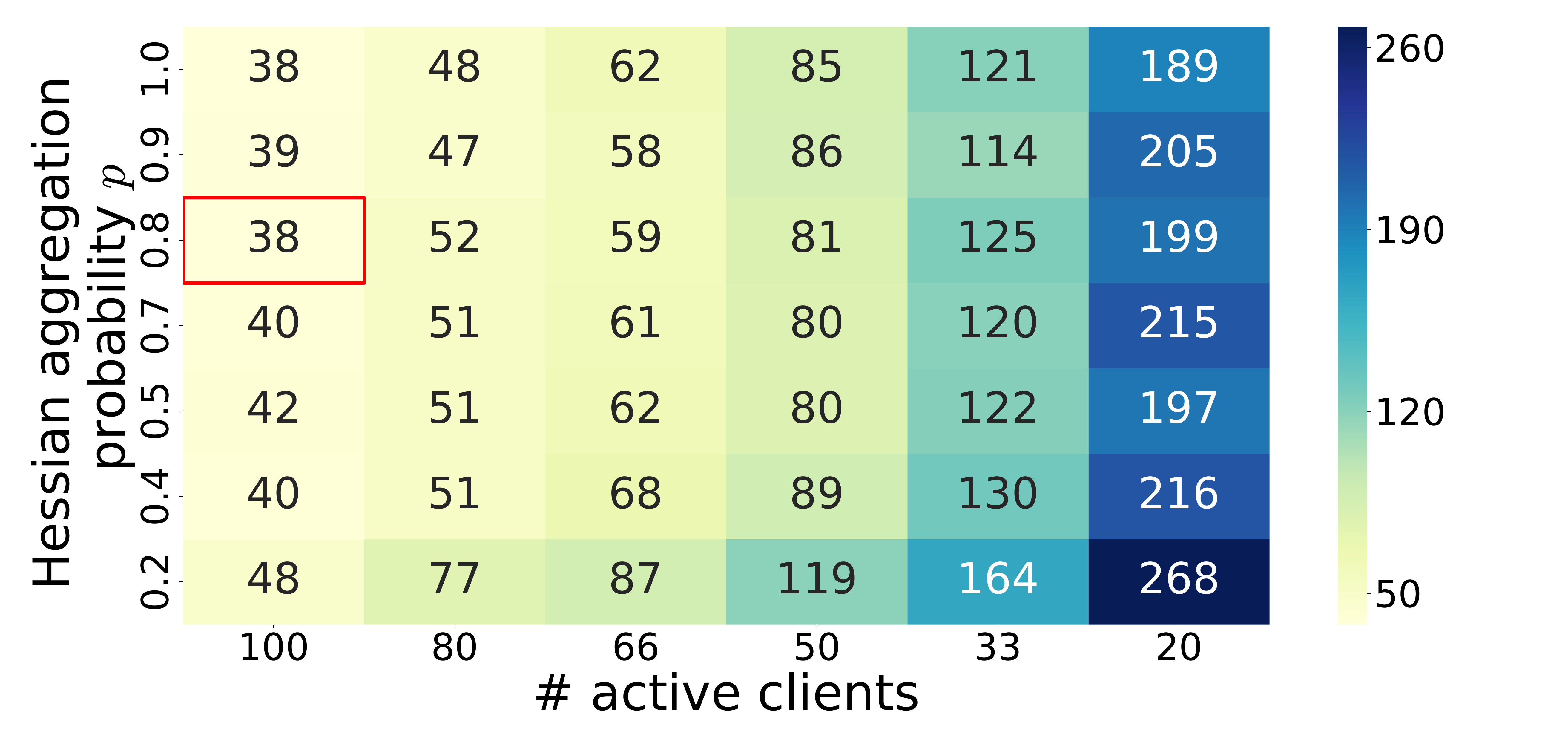} &
				\includegraphics[width=0.22\linewidth]{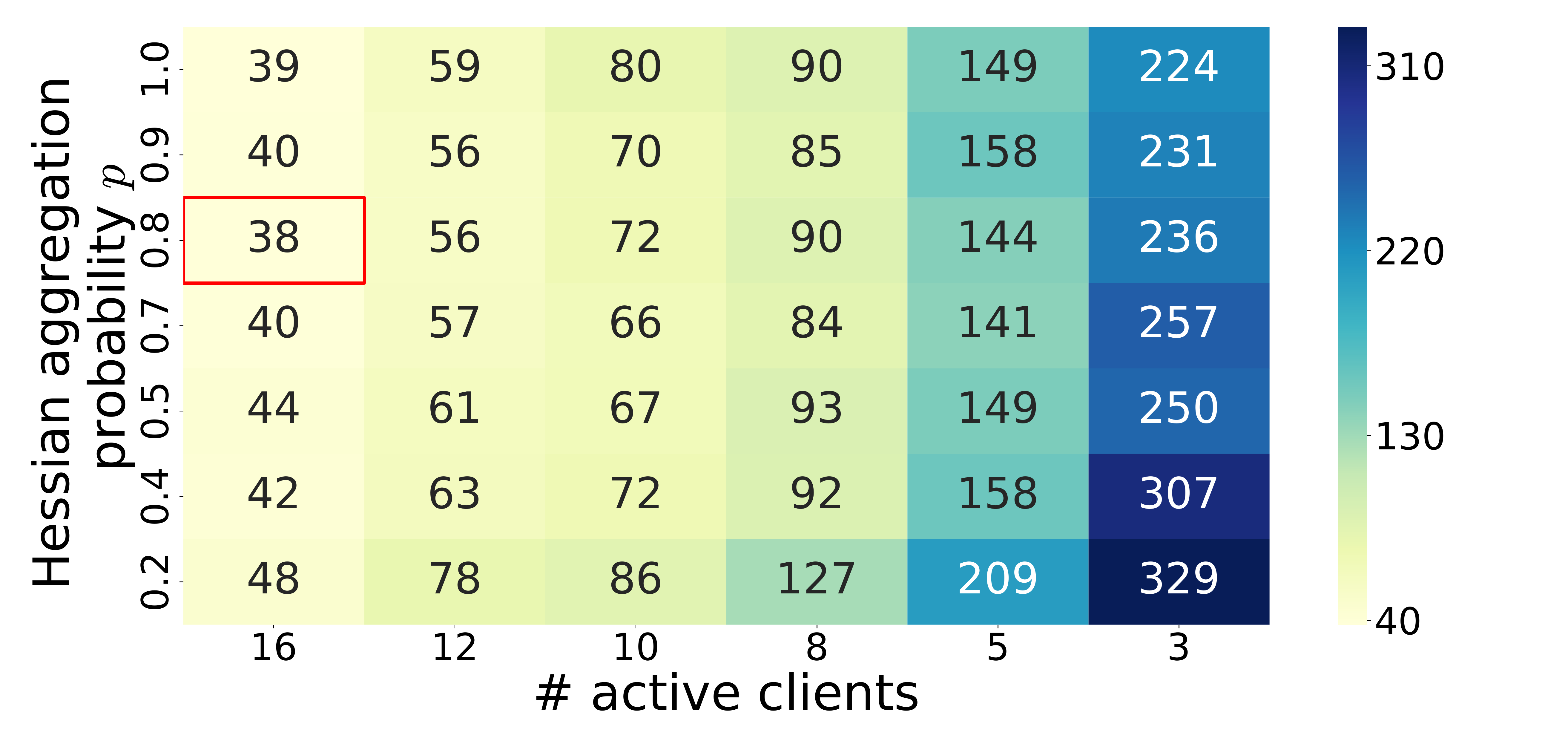} &
				\includegraphics[width=0.22\linewidth]{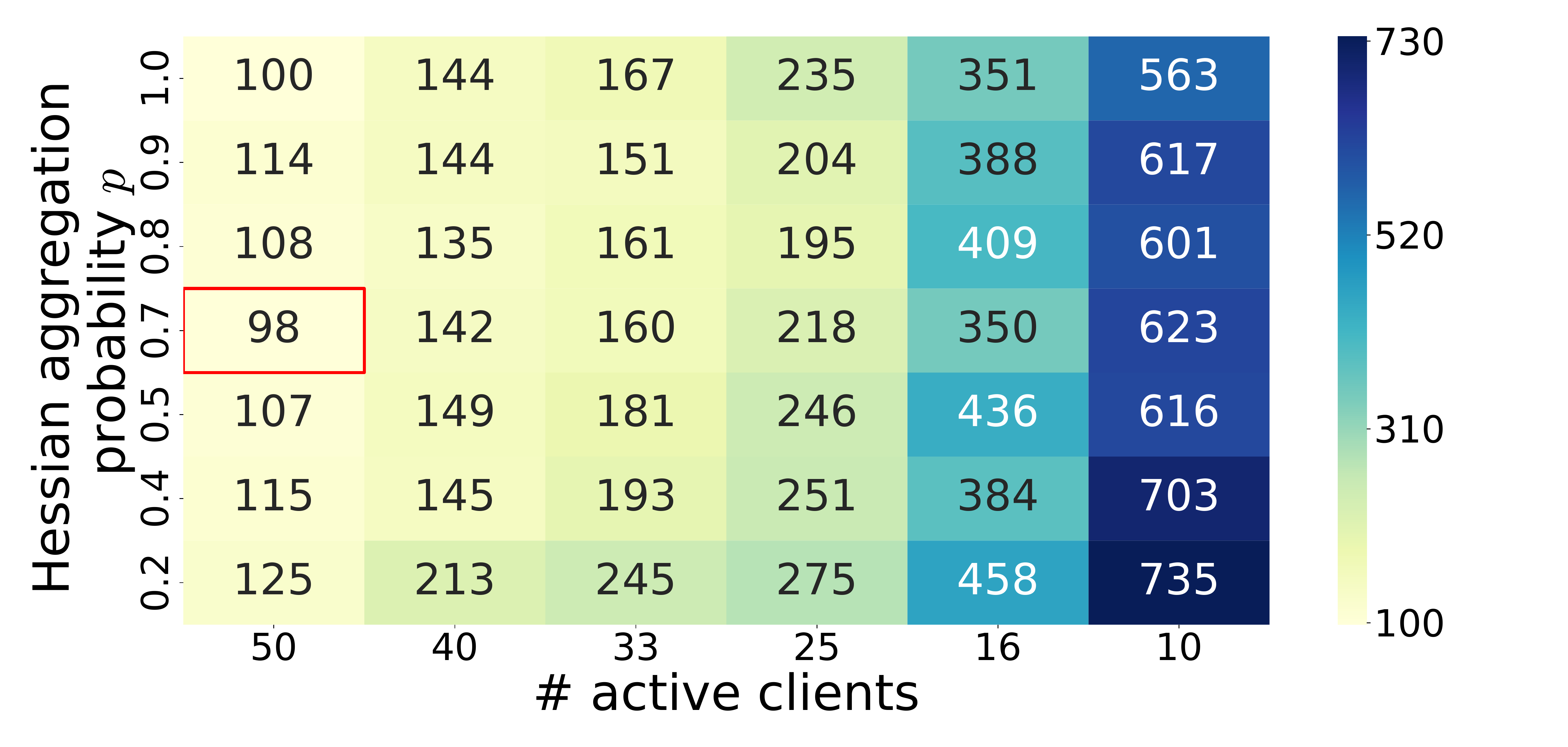} \\
				(a1) \dataname{phishing}, {\scriptsize$ \lambda=10^{-4}$} &
				(b1) \dataname{a1a}, {\scriptsize$ \lambda=10^{-3}$} &
				(c1) \dataname{w2a}, {\scriptsize $\lambda=10^{-4}$} \\
				\includegraphics[width=0.22\linewidth]{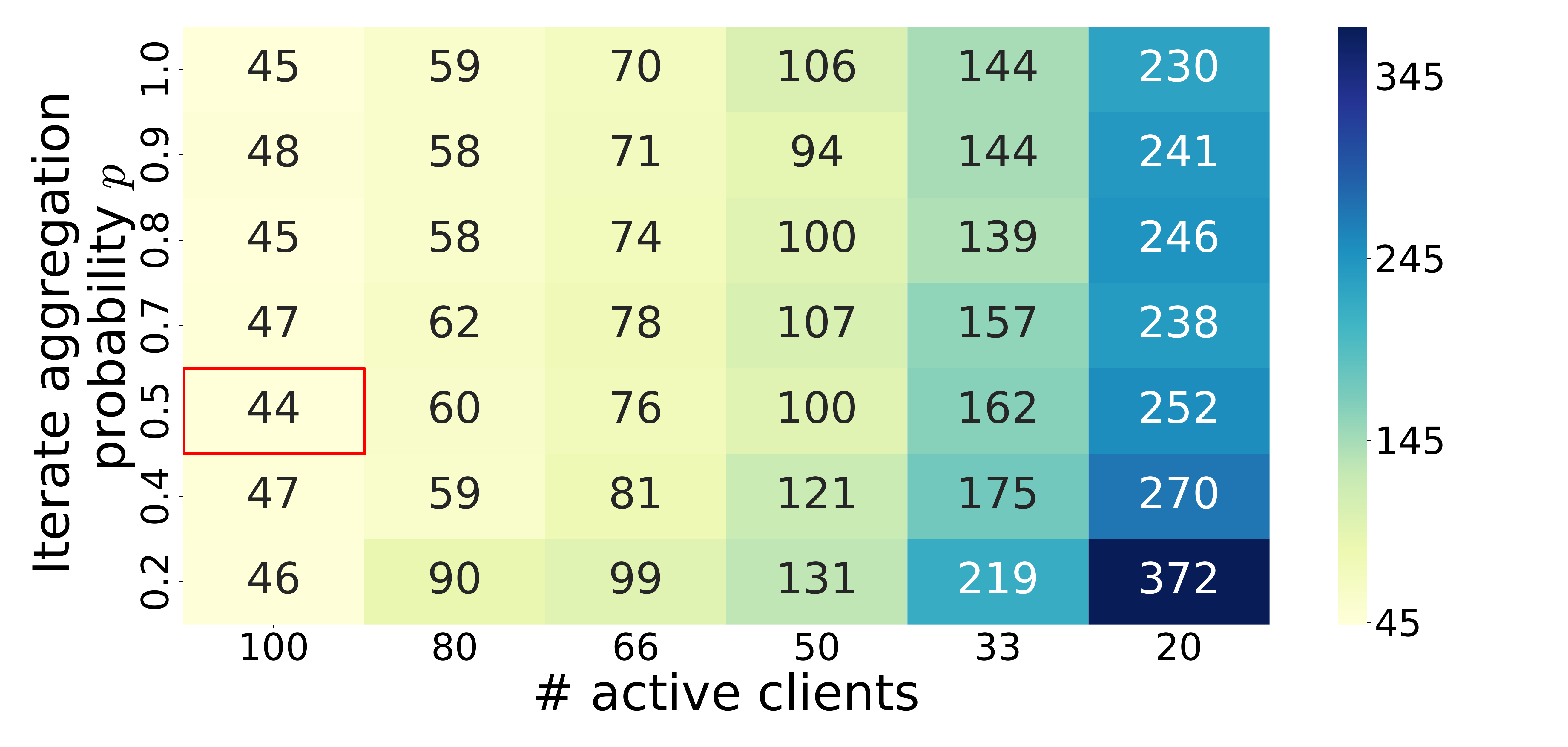} &
				\includegraphics[width=0.22\linewidth]{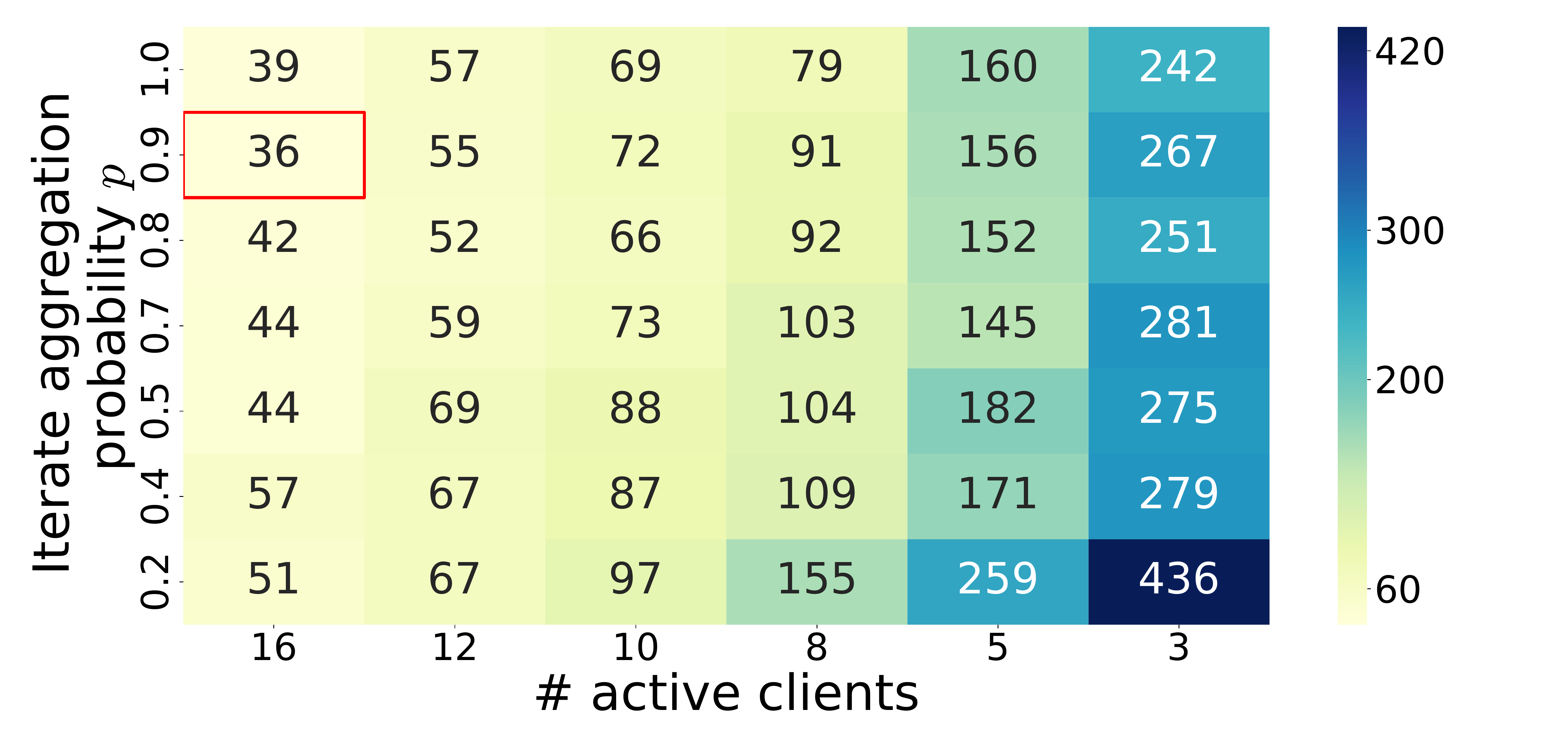} &
				\includegraphics[width=0.22\linewidth]{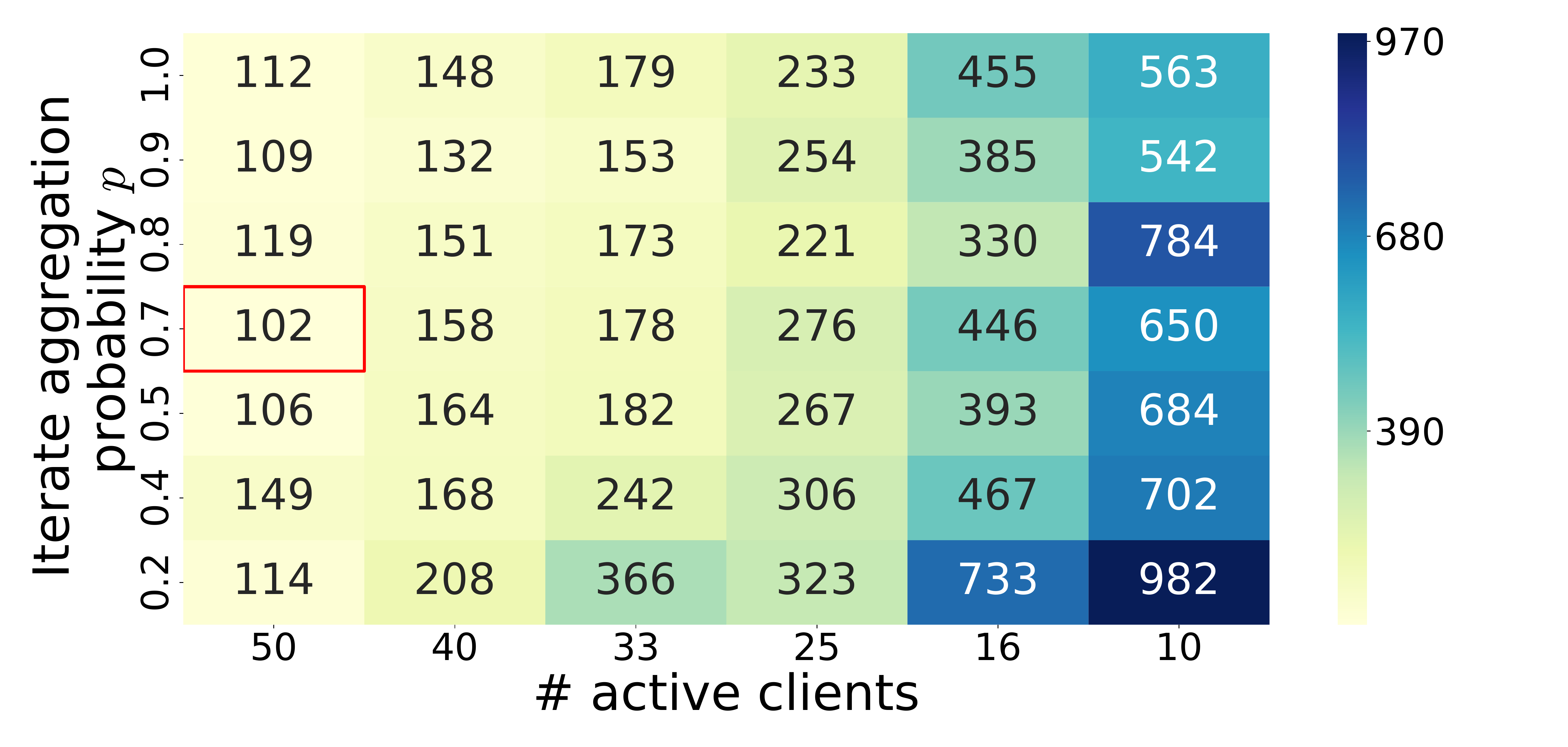} \\
				(a2) \dataname{phishing}, {\scriptsize$ \lambda=10^{-4}$} &
				(b2) \dataname{a1a}, {\scriptsize$ \lambda=10^{-3}$} &
				(c2) \dataname{w2a}, {\scriptsize $\lambda=10^{-4}$} \\
				\includegraphics[width=0.22\linewidth]{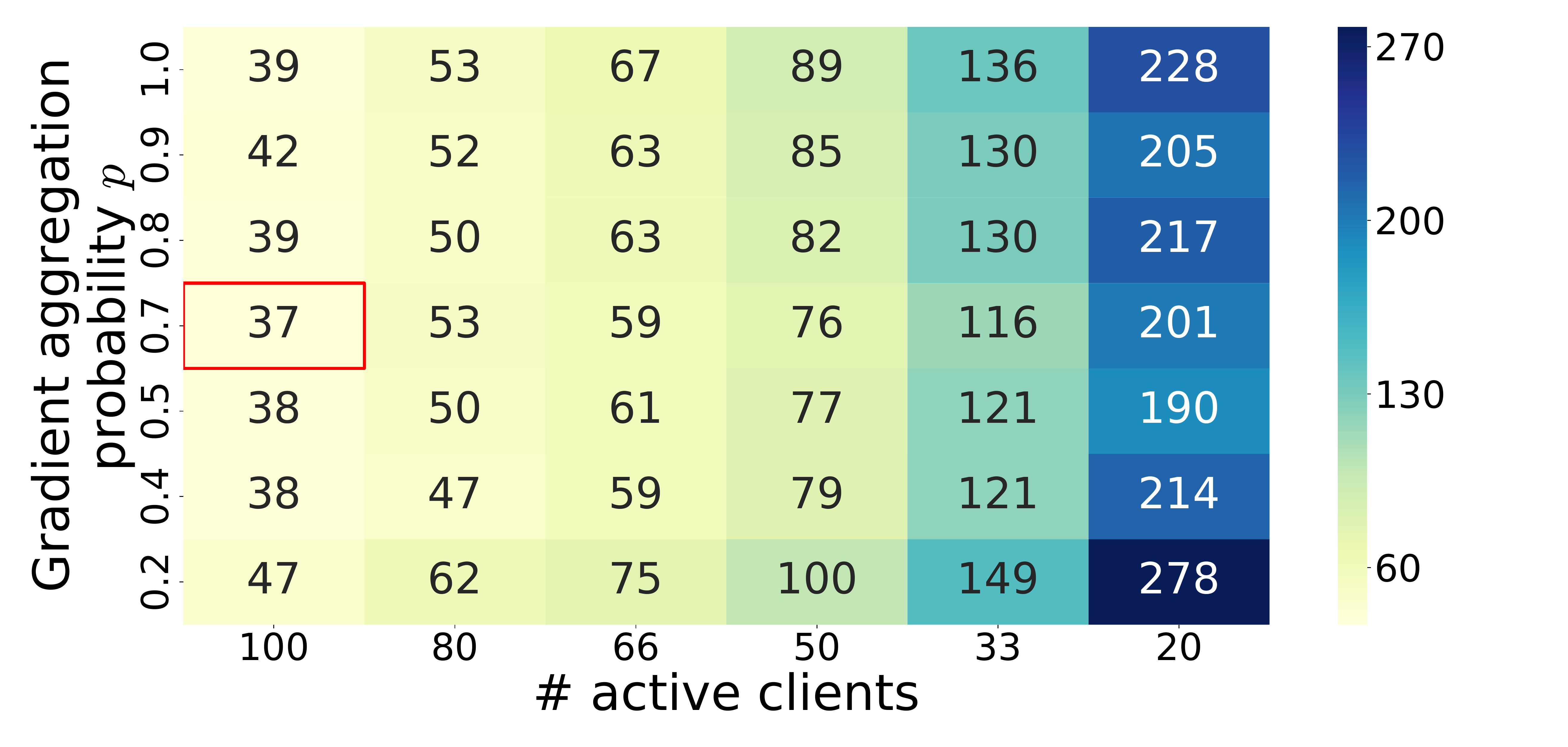} &
				\includegraphics[width=0.22\linewidth]{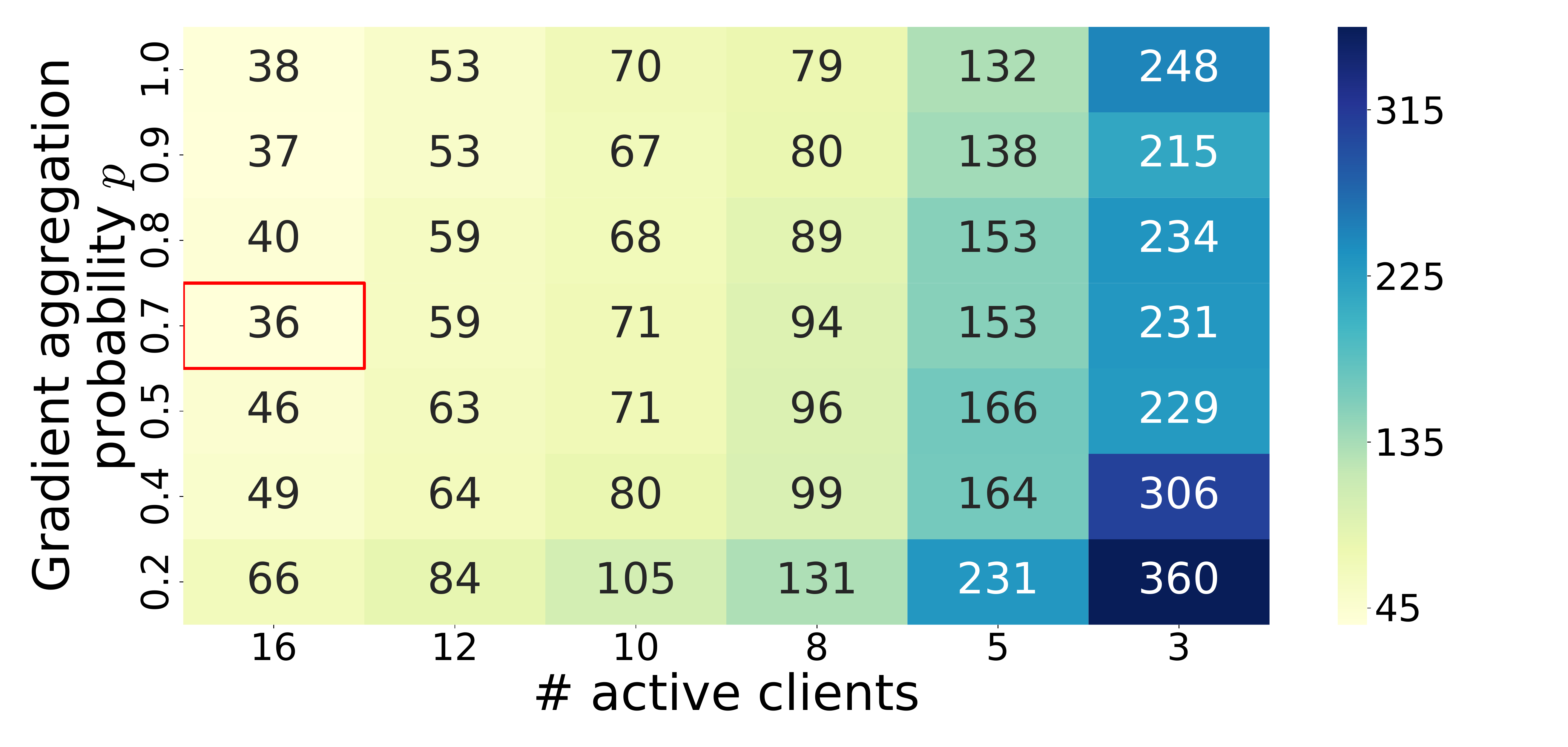} &
				\includegraphics[width=0.22\linewidth]{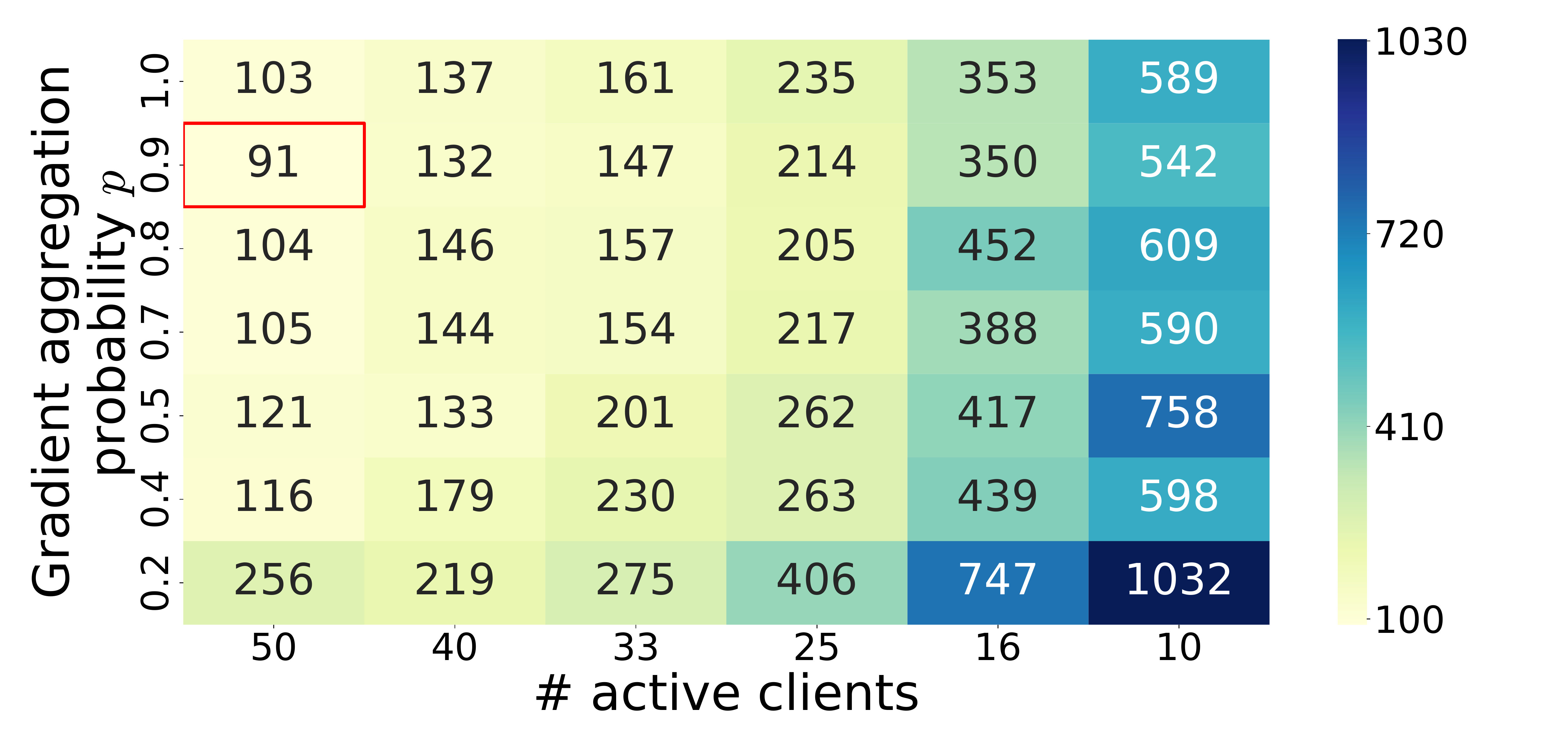} \\
				(a3) \dataname{phishing}, {\scriptsize$ \lambda=10^{-4}$} &
				(b3) \dataname{a1a}, {\scriptsize$ \lambda=10^{-3}$} &
				(c3) \dataname{w2a}, {\scriptsize $\lambda=10^{-4}$} \\
			\end{tabular}       
		\end{center}
		\caption{The performance of \algname{Newton-3PC-BC-PP} with various update strategies in terms of communication complexity (in Mbytes).}
		\label{fig:Newton-3PC-BC-PP}
	\end{figure}

	\subsubsection{Comparison of different 3PC update rules}
	
	Now we test different combinations of 3PC compression mechanisms applied on Hessians and iterates. First, we fix probability parameter of BAG update rule applied on gradients to $p=0.7$. The number of active clients in all cases $\tau=\nicefrac{n}{2}$. We analyze various combinations of 3PC compressors: CBAG (Top-$d$ and $p=0.7$) and 3PCv5 (Top-$\nicefrac{d}{2}$ and $p=0.7$); EF21 (Top-$d$) and EF21 (Top-$\nicefrac{d}{2}$); CBAG (Top-$d$ and $p=0.7$) and EF21 (Top-$\nicefrac{d}{2}$);  EF21 (Top-$d$) and 3PCv5 (Top-$\nicefrac{d}{2}$ and $p=0.7$) applied on Hessians and iterates respectively. Numerical results might be found in Figure~\ref{fig:Newton-3PC-BC-PP_comparison}. We can see that in all cases \algname{Newton-3PC-BC-PP} performs the best with a combination of 3PC compressors that differ from EF21+EF21. This allows to conclude that EF21 update rule is not always the most effective since other 3PC compression mechanisms lead to better performance in terms of communication complexity. Nonetheless one can notice that it is useless to apply CBAG or LAG compression mechanisms on iterates. Indeed, in the case when we skip communication the iterates remain intact, and the next step is equivalent to previous one. Thus, there is no need to carry out the step again.

	\begin{figure}[t]
		\begin{center}
			\begin{tabular}{cccc}
				\includegraphics[width=0.22\linewidth]{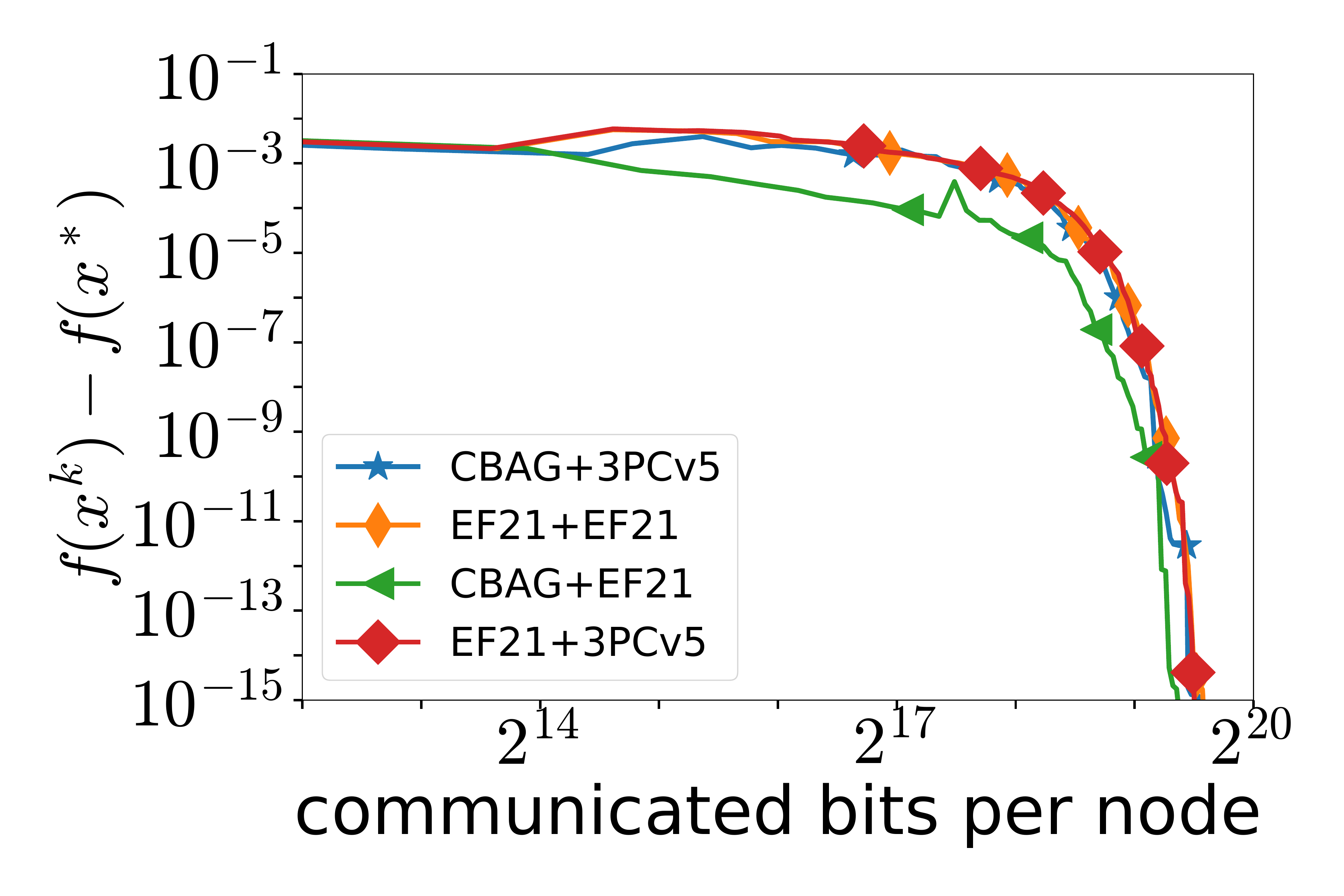} &
				\includegraphics[width=0.22\linewidth]{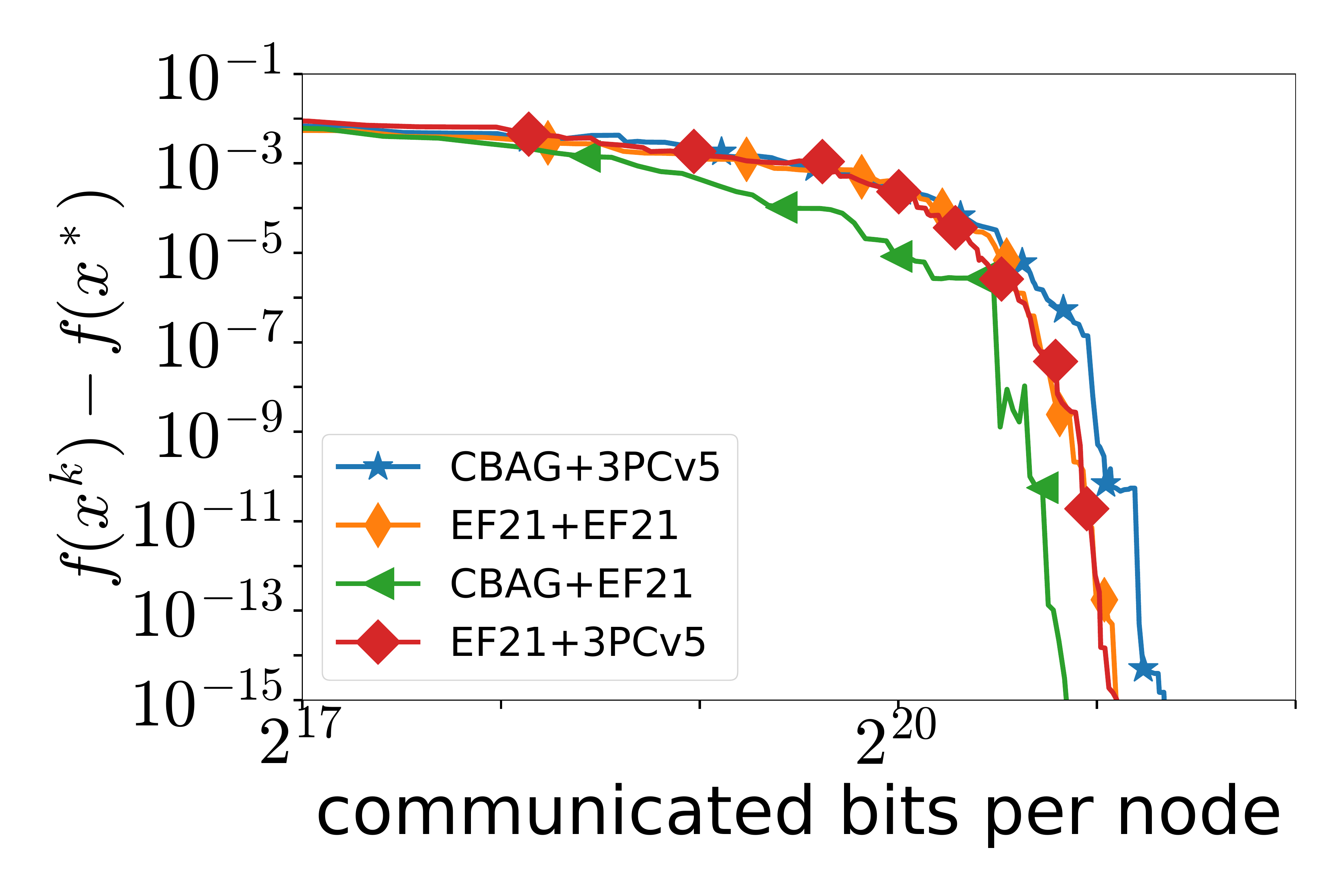} &
				\includegraphics[width=0.22\linewidth]{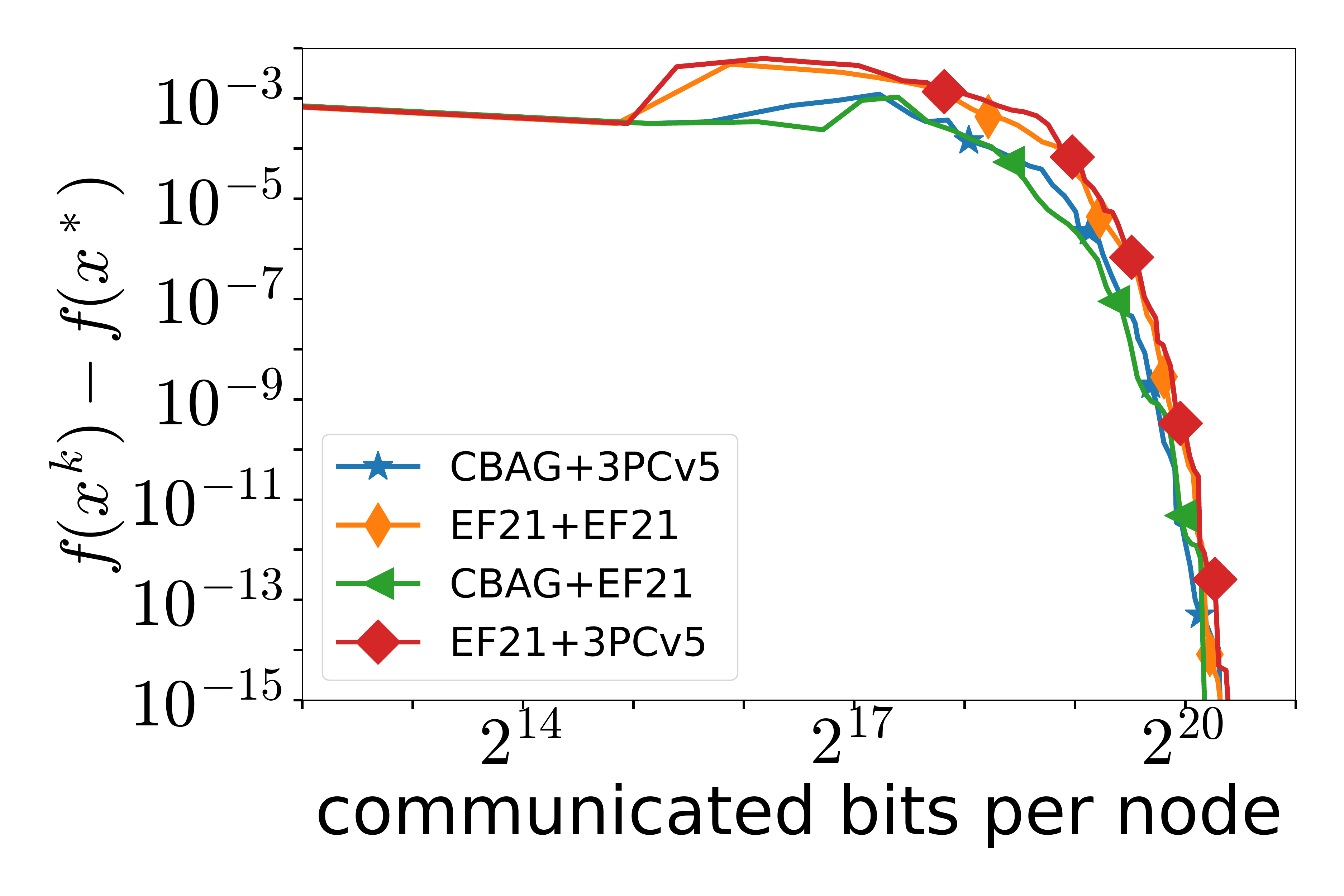} &
				\includegraphics[width=0.22\linewidth]{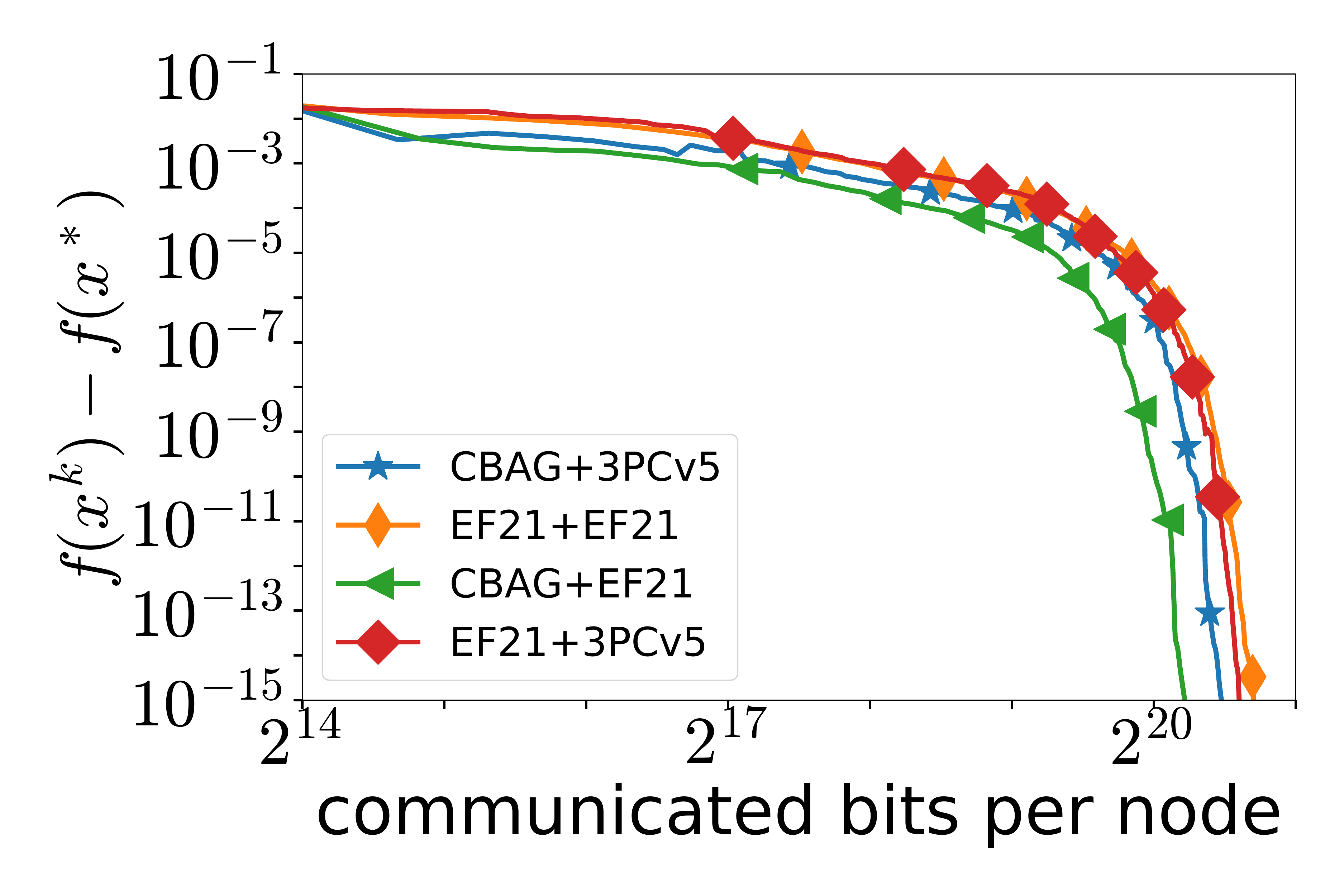} \\
				(a1) \dataname{a1a}, {\scriptsize$ \lambda=10^{-3}$} &
				(b1) \dataname{w2a}, {\scriptsize$ \lambda=10^{-4}$} &
				(c1) \dataname{a9a}, {\scriptsize $\lambda=10^{-3}$} &
				(c1) \dataname{w8a}, {\scriptsize $\lambda=10^{-4}$} \\
			\end{tabular}       
		\end{center}
		\caption{The performance of \algname{Newton-3PC-BC-PP} with different combinations of 3PC compressors applied on Hessians and iterates respectively.}
		\label{fig:Newton-3PC-BC-PP_comparison}
	\end{figure}

	\subsection{Global convergence of \algname{Newton-3PC}}

	Now we investigate the performance of globallly convergent \algname{Newton-3PC-LS}~--- an extension of \algname{Newton-3PC}~--- based on line search as it performs significantly better than \algname{Newton-3PC-CR} based on cubic regularization. The experiments are done on synthetically generated datasets with heterogeneity control. Detailed description of how the datasets are created is given in  section B.12 of \citep{FedNL2021}. Roughly speaking, the generation function has $2$ parameters $\alpha$ and $\beta$ that control a heterogeneity of local data. We denote datasets created in a such way with parameters $\alpha$ and $\beta$ as \dataname{Synt($\alpha$, $\beta$)}. All datasets are generated with dimension $d=100$, split between $n=20$ clients each of which has $m=1000$ local data points. In all cases the regularization parameter is chosen~$\lambda=10^{-4}$.
	
	We compare $5$ versions of \algname{Newton-3PC-LS} combined with EF21 (based on Rank-$1$ compressor), CBAG (based on Rank-$1$ compressor with probability $0.8$), CLAG (based on Rank-$1$ compressor and communication trigger $\zeta=2$), 3PCv2 (based on Top-$\nicefrac{3d}{4}$ and Rand-$\nicefrac{d}{4}$ compressors), and 3PCv4 (based on Top-$\nicefrac{d}{2}$ and Top-$\nicefrac{d}{2}$ compressors). In this series of experiments the initialization of $\mH^0_i$ is equal to zero matrix. The comparison is performed against \algname{ADIANA} \citep{ADIANA} with random dithering ($s=\sqrt{d}$), \algname{Fib-IOS} \citep{IOSFabbro2022}, and \algname{GIANT} \citep{GIANT2018}. 
	
	The numerical results are shown in Figure~\ref{fig:Newton-3PC-LS}. According to them, we observe that \algname{Newton-3PC-LS} is more resistant to heterogeneity than other methods since they outperform others {\it by several orders in magnitude}. Besides, we see that \algname{Newton-CBAG-LS} and \algname{Newton-EF21-LS} are the most efficient among all \algname{Newton-3PC-LS} methods; in some cases the difference is considerable.

	\begin{figure}[ht]
		\begin{center}
			\begin{tabular}{cccc}
				\includegraphics[width=0.22\linewidth]{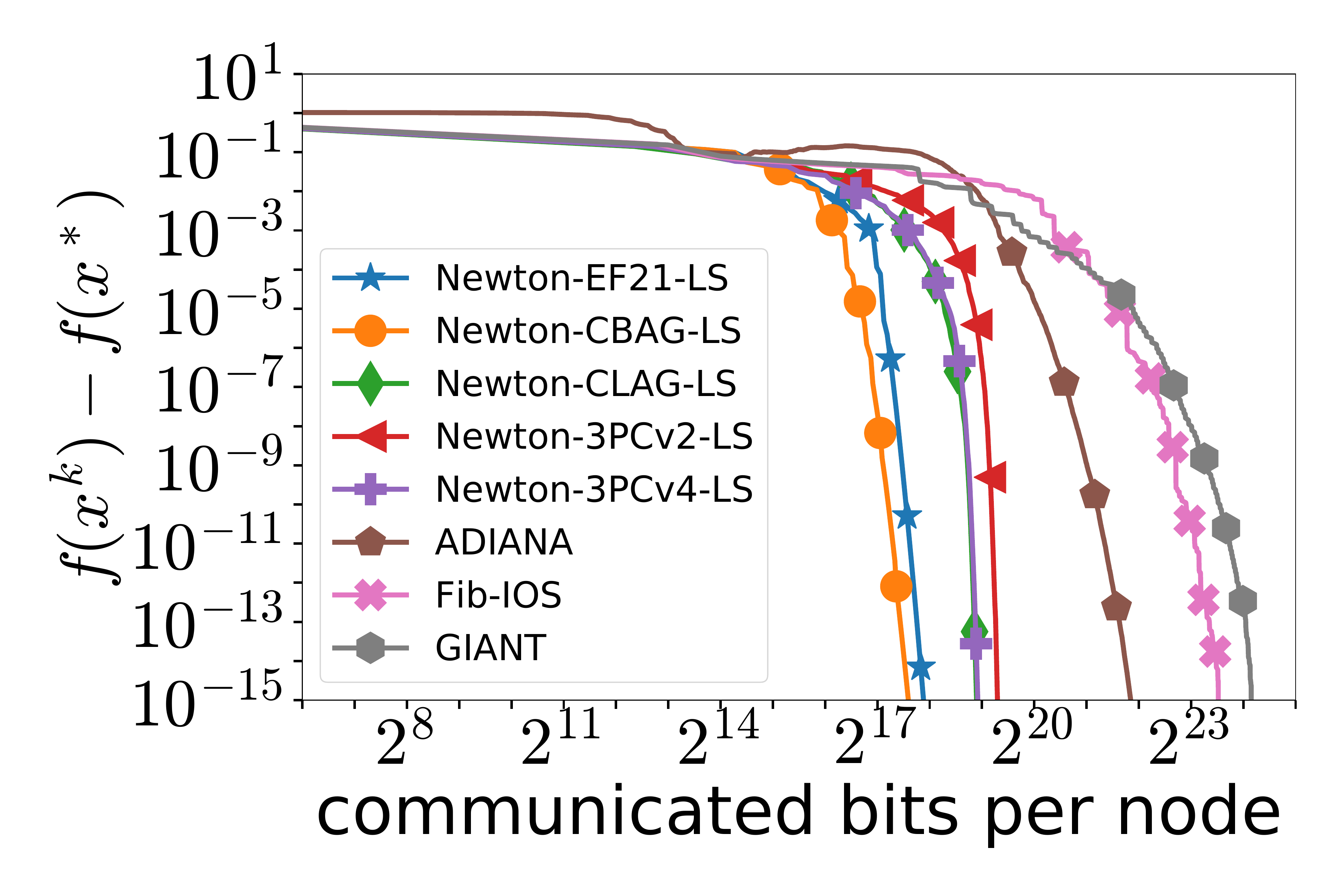} &
				\includegraphics[width=0.22\linewidth]{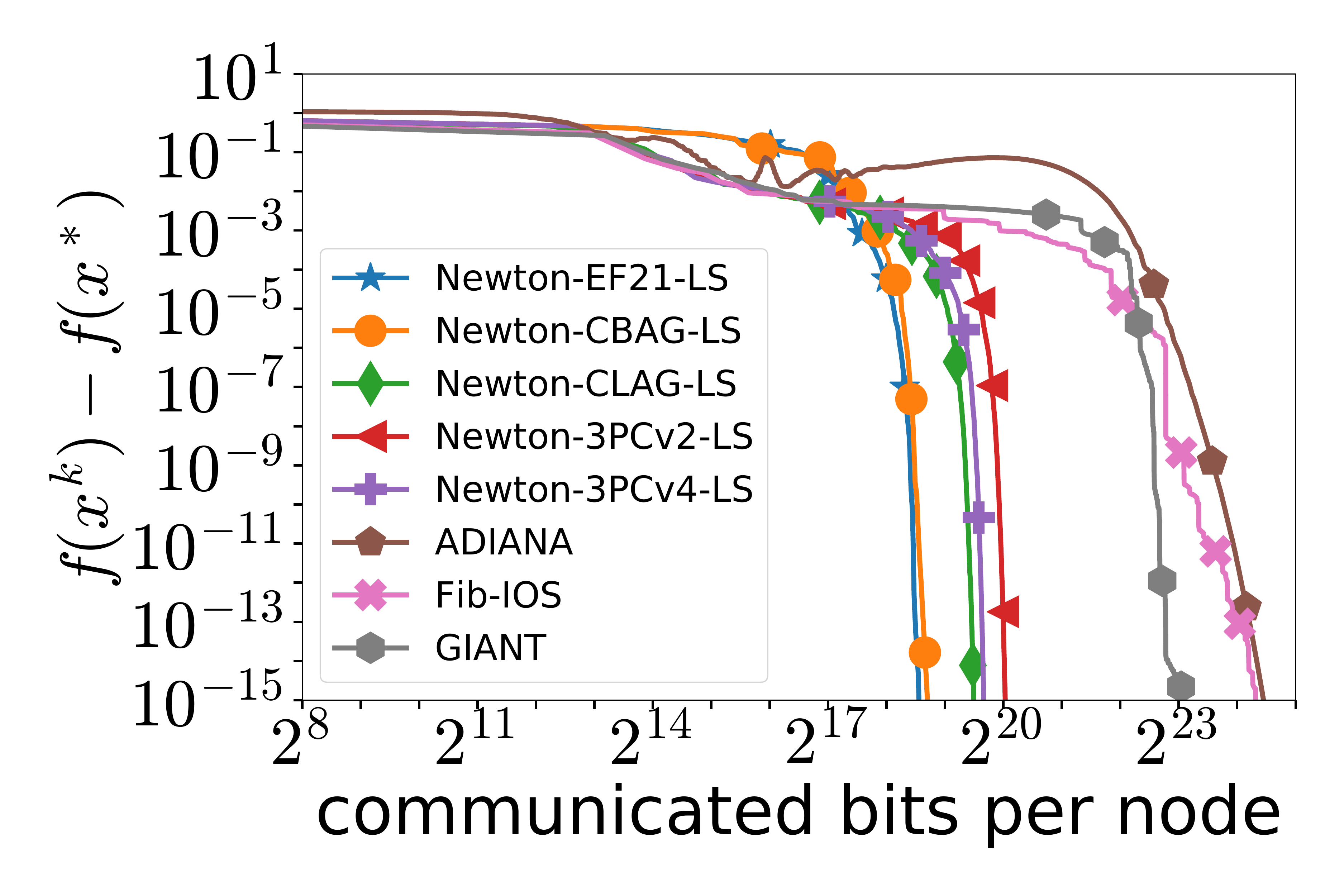} &
				\includegraphics[width=0.22\linewidth]{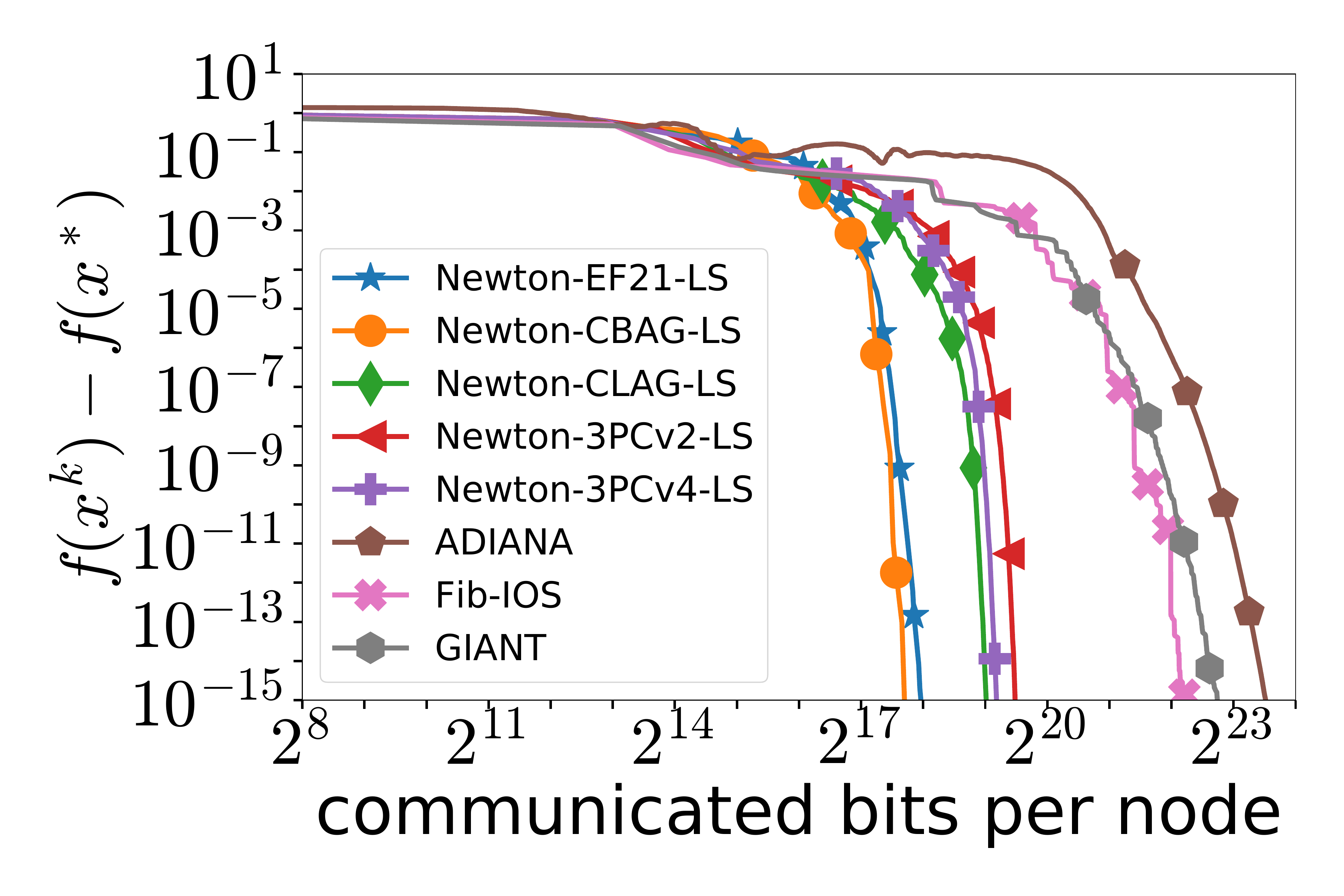} &
				\includegraphics[width=0.22\linewidth]{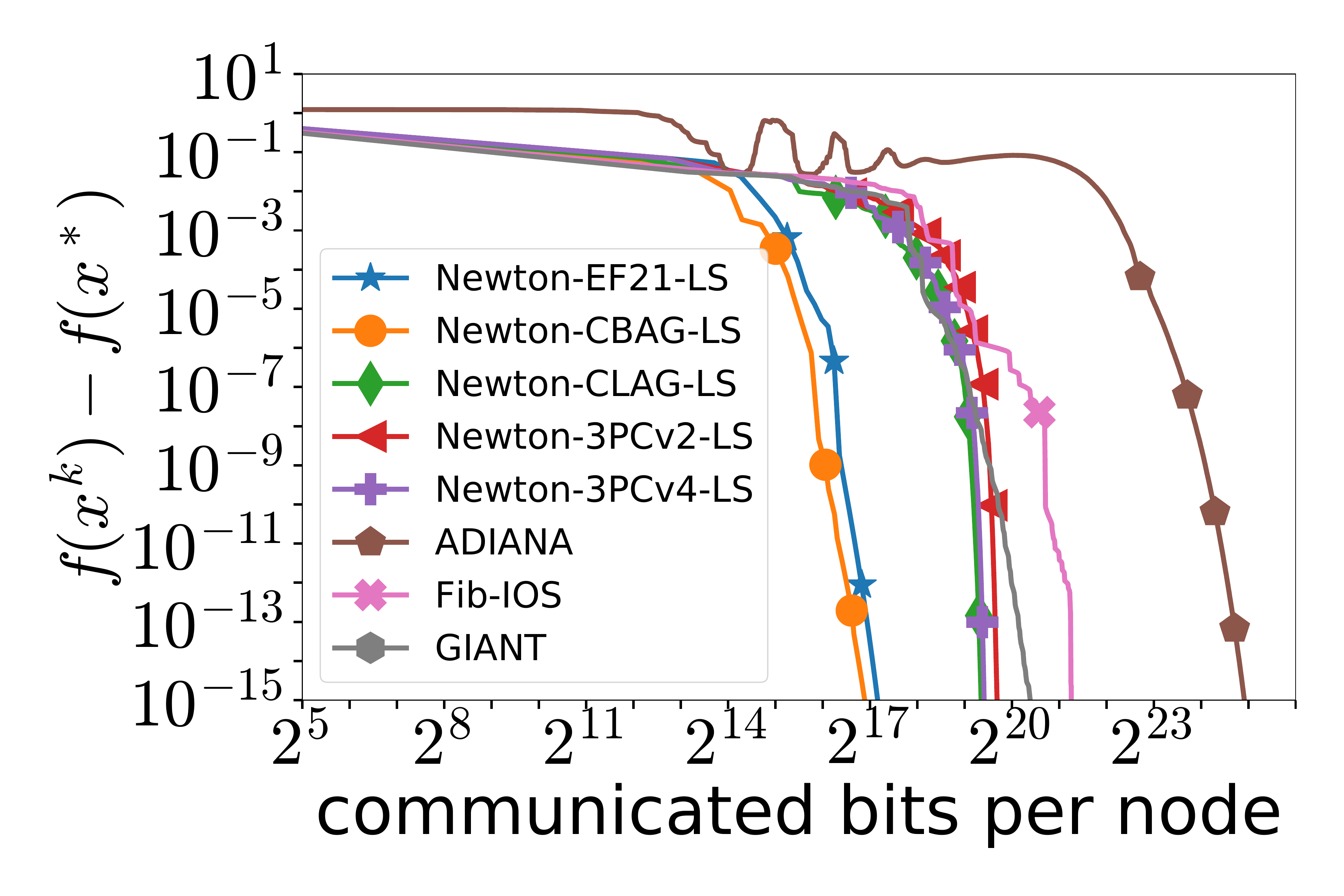} \\
				(a) \dataname{Synt(0.5,0.5)} &
				(b)  \dataname{Synt(1,1)}&
				(c)  \dataname{Synt(1.5,1.5)}&
				(d)  \dataname{Synt(2,2)}\\
			\end{tabular}       
		\end{center}
		\caption{The performance of \algname{Newton-3PC-LS} with different combinations of 3PC compressors applied on Hessians against \algname{ADIANA}, \algname{Fib-IOS}, and \algname{GIANT}.}
		\label{fig:Newton-3PC-LS}
	\end{figure}

\end{document}